\documentclass{article}

\usepackage{microtype}
\usepackage{graphicx}
\usepackage{subfigure}
\usepackage{booktabs} 
\usepackage{soul}

\usepackage[accepted]{icml2024}

\icmltitlerunning{Accelerated Policy Gradient: On the Convergence Rates of the Nesterov Momentum for Reinforcement Learning}

\usepackage{amsmath,amsfonts,bm,dsfont}

\def\eqref#1{equation~\ref{#1}}

\def\1{\bm{1}}

\def\ra{{\textnormal{a}}}

\def\rx{{\textnormal{x}}}

\def\rva{{\mathbf{a}}}

\def\erva{{\textnormal{a}}}

\def\ervx{{\textnormal{x}}}

\def\rmA{{\mathbf{A}}}

\def\vmu{{\bm{\mu}}}
\def\vtheta{{\bm{\theta}}}
\def\va{{\bm{a}}}

\def\ve{{\bm{e}}}

\def\vx{{\bm{x}}}

\def\eva{{a}}

\def\mA{{\bm{A}}}

\def\mH{{\bm{H}}}
\def\mI{{\bm{I}}}
\def\mJ{{\bm{J}}}

\def\mX{{\bm{X}}}

\def\mSigma{{\bm{\Sigma}}}

\DeclareMathAlphabet{\mathsfit}{\encodingdefault}{\sfdefault}{m}{sl}
\SetMathAlphabet{\mathsfit}{bold}{\encodingdefault}{\sfdefault}{bx}{n}
\newcommand{\tens}[1]{\bm{\mathsfit{#1}}}
\def\tA{{\tens{A}}}

\def\tX{{\tens{X}}}

\def\gG{{\mathcal{G}}}

\def\sA{{\mathbb{A}}}
\def\sB{{\mathbb{B}}}

\def\sS{{\mathbb{S}}}

\def\emA{{A}}

\newcommand{\etens}[1]{\mathsfit{#1}}

\def\etA{{\etens{A}}}

\newcommand{\E}{\mathbb{E}}

\newcommand{\R}{\mathbb{R}}

\newcommand{\KL}{D_{\mathrm{KL}}}
\newcommand{\Var}{\mathrm{Var}}

\newcommand{\Cov}{\mathrm{Cov}}

\newcommand{\normltwo}{L^2}
\newcommand{\normlp}{L^p}

\newcommand{\parents}{Pa} 

\DeclareMathOperator*{\argmax}{arg\,max}

\usepackage{setspace}
\usepackage{footmisc}
\usepackage{tablefootnote}
\usepackage{url}
\usepackage[normalem]{ulem}
\usepackage{hyperref}
\usepackage{makecell}
\usepackage{setspace}
\usepackage{amsmath}
\usepackage{mathtools}
\usepackage{amssymb}
\usepackage{amsthm}
\usepackage{algorithm}
\usepackage{algorithmic}
\usepackage{cleveref}
\usepackage{comment}
\usepackage{subfigure}
\usepackage{wrapfig}
\usepackage{lipsum}
\usepackage{ulem}
\usepackage{bm}
\usepackage{thmtools,thm-restate}
\usepackage{nicematrix}
\usepackage{cancel}
\DeclareMathOperator{\cA}{\mathcal{A}}
\DeclareMathOperator{\cV}{\mathcal{V}}
\DeclareMathOperator{\cS}{\mathcal{S}}
\DeclareMathOperator{\cN}{\mathcal{N}}
\DeclareMathOperator{\cF}{\mathcal{F}}
\DeclareMathOperator{\cI}{\mathcal{I}}
\DeclareMathOperator{\cD}{\mathcal{D}}
\DeclareMathOperator{\cP}{\mathcal{P}}
\DeclareMathOperator{\cR}{\mathcal{R}}
\DeclareMathOperator{\cB}{\mathcal{B}}
\DeclareMathOperator{\cH}{\mathcal{H}}
\DeclareMathOperator{\cJ}{\mathcal{J}}
\DeclareMathOperator{\cX}{\mathcal{X}}
\DeclareMathOperator{\cU}{\mathcal{U}}
\DeclareMathOperator{\bR}{\mathbb{R}}
\DeclareMathOperator{\bV}{\mathbb{V}}
\DeclareMathOperator{\bbN}{\mathbb{N}}
\let\emptyset\varnothing
\Crefname{equation}{}{}
\crefname{equation}{}{}
\Crefname{figure}{Figure}{Figures}
\crefname{figure}{figure}{figures}
\Crefname{figure*}{Figure}{Figures}
\crefname{figure*}{figure}{figures}
\crefname{table}{table}{tables}
\newtheorem{definition}{Definition}
\newtheorem*{definition*}{Definition}
\newtheorem{lemma}{Lemma}
\newtheorem*{lemma*}{Lemma}
\newtheorem{remark}{Remark}
\newtheorem*{remark*}{Remark}
\newtheorem{corollary}{Corollary}
\newtheorem*{corollary*}{Corollary}
\newtheorem{theorem}{Theorem}
\newtheorem*{theorem*}{Theorem}

\newtheorem*{prop*}{Proposition}

\newtheorem{assumption}{Assumption}
\newtheorem{claim}{Claim}
\newtheorem*{claim*}{Claim}
\newcommand{\mru}[1]{\textcolor{blue}{#1}}
\newcommand{\naich}[1]{\textcolor{teal}{#1}}
\newcommand{\pch}[1]{\textcolor{black}{#1}}

\newcommand{\cp}[1]{\textcolor{black}{#1}}

\newcommand{\cpcomment}[1]{\textbf{\textcolor{red}{CP: #1}}}
\newcommand{\cpfn}[1]{\footnote{\cpcomment{#1}}}
\newcommand{\updatecomment}[1]{\textcolor{teal}{UPDATE: #1}}
\newcommand{\soutt}[1]{\ifmmode\text{\sout{\ensuremath{#1}}}\else\sout{#1}\fi}
\usepackage{enumitem}
\setlist{leftmargin=5.5mm}

\begin{document}

\twocolumn[
\icmltitle{Accelerated Policy Gradient: On the Convergence Rates of \\ the Nesterov Momentum for Reinforcement Learning}


\icmlsetsymbol{equal}{*}

\begin{icmlauthorlist}
\icmlauthor{Yen-Ju Chen}{equal,yyy}
\icmlauthor{Nai-Chieh Huang}{equal,yyy}
\icmlauthor{Ching-pei Lee}{ism}
\icmlauthor{Ping-Chun Hsieh}{yyy}
\end{icmlauthorlist}

\icmlaffiliation{yyy}{Department of Computer Science, National Yang Ming Chiao Tung University, Hsinchu, Taiwan}
\icmlaffiliation{ism}{Department of Advanced Data Science,  Institute
of Statistical Mathematics, Tokyo, Japan}
\icmlcorrespondingauthor{Yen-Ju Chen}{mru.11@nycu.edu.tw}
\icmlcorrespondingauthor{Nai-Chieh Huang}{naich.cs09@nycu.edu.tw}
\icmlcorrespondingauthor{Ping-Chun Hsieh}{pinghsieh@nycu.edu.tw}

\icmlkeywords{Machine Learning, ICML}
\vskip 0.3in
]



\printAffiliationsAndNotice{\icmlEqualContribution} 

\begin{abstract}
{
Various acceleration approaches for Policy Gradient (PG) have been analyzed within the realm of Reinforcement Learning (RL). However, the theoretical understanding of the widely used momentum-based acceleration method on PG remains largely open. In response to this gap, we adapt the celebrated Nesterov's accelerated gradient (NAG) method to policy optimization in RL, termed \textit{Accelerated Policy Gradient} (APG). To demonstrate the potential of APG in achieving fast convergence, we formally prove that with the true gradient and under the softmax policy parametrization, APG converges to an optimal policy at rates: (i) $\tilde{O}(1/t^2)$ with nearly constant step sizes; (ii) $O(e^{-ct})$ with time-varying step sizes. To the best of our knowledge, this is the first characterization of the convergence rates of NAG in the context of RL. Notably, our analysis relies on one interesting finding: Regardless of the parameter initialization, APG ends up entering a locally nearly-concave regime, where APG can significantly benefit from the momentum, within finite iterations.
Through numerical validation and experiments on the Atari 2600 benchmarks, we confirm that APG exhibits a $\tilde{O}(1/t^2)$ rate with nearly constant step sizes and a linear convergence rate with time-varying step sizes, significantly improving convergence over the standard PG.
}
\end{abstract}

\section{Introduction}
\label{sec:intro}



Policy gradient (PG) is a fundamental technique utilized in the field of reinforcement learning (RL) for policy optimization. It operates by directly optimizing the RL objectives to determine the optimal policy, employing first-order derivatives similar to the gradient descent algorithm in conventional optimization problems.
Notably, PG has demonstrated empirical success \citep{mnih2016asynchronous,wang2016sample,silver2014deterministic,lillicrap2015continuous,schulman2017proximal,espeholt2018impala} and is supported by strong theoretical guarantees \citep{agarwal2021theory, fazel2018global, liu2020improved, bhandari2019global, mei2020global, wang2021global,mei2021understanding,mei2022role,xiao2022convergence,chen2022sample,lan2023policy}.
In a recent study, \citet{mei2020global} characterize the $O(1/t)$ convergence rate of PG in the non-regularized tabular softmax setting.
This convergence behavior aligns with that of the gradient descent algorithm for {convex minimization} problems, despite that the RL objectives lack concave characteristics {for maximization.}
{Additionally, advancements in RL literature have introduced theories on improving the $O(1/t)$ convergence rate through various acceleration techniques. \citet{khodadadian2021linear,xiao2022convergence} offer sharp analyses of Natural Policy Gradient (NPG) with time-varying step sizes, demonstrating a notable linear convergence rate. On the other hand, \citet{mei2021leveraging} utilize the normalized gradient to accelerate the PG, also achieving a linear convergence rate.}

{In addition to the above-mentioned acceleration techniques, \textit{momentum}, which is another popular acceleration method in the optimization literature, emerges as a widely adopted strategy for accelerating policy optimization in RL in practice due to its simplicity \citep{vieillard2020momentum, huang2020momentum}.}
Classic momentum methods, such as Nesterov momentum \citep{nesterov1983method} {and} the heavy-ball method \citep{polyak1964some},
have been seamlessly integrated into optimization solvers and deep learning frameworks like PyTorch \citep{paszke2019pytorch}. These approaches consistently outperform standard PG methods under A2C \citep{konda1999actor} and PPO \citep{schulman2017proximal}, as evidenced by empirical studies \citep{henderson2018did,andrychowicz2021matters}.
Nesterov's Accelerated Gradient (NAG) method, introduced by \citet{nesterov1983method}, is a first-order approach originally designed for convex {minimization} to improve the convergence rate to $O(1/t^2)$ {from the $O(1/t)$ rate of gradient descent}. Despite its wide applicability in various types of optimization problems over the past decades, to the best of our knowledge, NAG has never been theoretically analyzed in the context of RL for its {convergence rate}, mainly due to the non-concavity of RL objectives. Therefore, there exists one important and yet unexplored research question: \textit{Could the Nesterov momentum achieve {convergence rates} better than or comparable to the existing PG and other acceleration methods in the context of RL?}


To answer this question, this paper introduces Accelerated Policy Gradient (APG), which utilizes Nesterov acceleration to address the policy optimization problem of RL. 
Despite the existing knowledge about the NAG methods from previous research \citep{beck2009fast,beck2009fastb,ghadimi2016accelerated,li2015accelerated,su2014differential,carmon2018accelerated}, there remain several fundamental challenges in establishing  convergence rates within the realm of RL: (i) \textit{NAG convergence results under nonconvex problems}: Although there is a plethora of theoretical works studying the convergence of NAG under general nonconvex problems, these results only establish convergence {of first-order optimality measures}. {It is infeasible to characterize the convergence rate based solely on these results.}
(ii) \textit{Inherent characteristics of the momentum term}: From an analytical perspective, the momentum term demonstrates intricate interactions with previous updates. As a result, accurately quantifying the specific impact of momentum during the execution of APG poses a considerable challenge.
(iii) \textit{The nature of the unbounded optimal parameter under softmax parameterization}: A crucial factor in characterizing the sub-optimality gap in the theory of optimization is the norm of the distance between the initial parameter and the optimal parameter \citep{beck2009fast,beck2009fastb,ghadimi2016accelerated}. However, in the case of softmax parameterization in RL, the parameter of the optimal action of each state tends to approach infinity. As a result, the norm involved in the sub-optimality gap becomes infinite, thereby hindering characterization of convergence rates using existing results.

{\textbf{Our Contributions.} Despite the above challenges, we present an affirmative answer to the research question described above and provide the first characterization of the {convergence rate} of NAG in the context of RL.
{Specifically, we present useful insights and novel techniques to tackle the above technical challenges:
{Regarding (i), we discover a fundamental property of the RL objective called \textit{local near-concavity}, which indicates that the RL objective is nearly concave in the proximity of the optimal policy, despite its non-concave global landscape.}
For (ii), we show that the {locally nearly-concave} region is \textit{absorbing} in the sense that even with the effect of the momentum term, the policy parameter could stay in the {nearly-concave} region indefinitely once {entered}. This result is obtained by carefully quantifying the cumulative effect of each momentum term. 
Towards (iii), we introduce surrogate optimal parameters, which have bounded norms and induce nearly-optimal policies, and thereby characterize the convergence rate of APG.
}
}
We summarize the contributions of this paper as follows:
    \vspace{-2mm}
	\begin{itemize}[leftmargin=*]
    \item We introduce APG, which leverages the Nesterov momentum scheme to accelerate the convergence performance of PG for RL. To the best of our knowledge, this is the first formal attempt at understanding the convergence of the Nesterov momentum in the context of RL.
    \item {To demonstrate the potential of APG in achieving fast convergence, we formally establish that APG enjoys a $\tilde{O}(1/t^2)$ convergence rate under softmax policy parameterization {with nearly constant step {sizes}}.\footnote{Note that this result does not contradict the $\Omega(1/t)$ lower bound of the sub-optimality gap of PG in \citep{mei2020global}. Please refer to \Cref{remark: PG lower bound} for a detailed discussion.}
    Furthermore, we offer that APG also achieves a linear convergence rate of $O(e^{-ct})$ if time-varying step {sizes are} allowed. 
    {These deliver better or comparable results in both constant step size and time-varying step size regimes.}
    For ease of reference, a summarized comparison can be found in Table \ref{table:rate}.
    To achieve this, we present several novel insights into RL and APG, including the local near-concavity property as well as the absorbing behavior of APG. Moreover, we show that these properties can also be applied to establish a $\tilde{O}(1/t)$ convergence rate of PG, which is of independent interest. }
    \item {Through numerical validation on MDP problems, we confirm that APG exhibits a $\tilde{O}(1/t^2)$ rate with nearly constant step sizes and a linear convergence rate with time-varying step sizes, thereby substantially improving the convergence behavior over the standard PG. Furthermore, we validate our APG algorithm on the benchmark Atari games and demonstrate the significant benefit provided by the Nesterov momentum.}
\end{itemize}


\section{Related Work}
\label{sec:related}

\begin{table*}[!ht]
\caption{A summary of the convergence rates under tabular softmax policies under PG and various acceleration methods. Here NPG refers to Natural Policy Gradient \citep{kakade2001natural}, and GNPG refers to Geometry-aware Normalized Policy Gradient \citep{mei2021leveraging}. For the sake of clarity, we use $\eta_i$'s to denote the pre-constants in the step sizes of each algorithm.}
\label{table:rate}
\vskip 0.11in
\begin{center}
    \scalebox{0.86}{
        \begin{tabular}{ccccc}
        \toprule
        \multicolumn{1}{l}{Algorithm} & Step Sizes & Gradient Update & Conv. Rate & Reference \\
        \midrule
        \hyperref[algorithm:APG]{APG}
        & Nearly Constant\footnotemark[2] & $\omega^{(t)} + \eta_1 \cdot \frac{t}{t+1} \cdot \nabla_{\theta}{V^{\pi_{\omega}^{(t)}}(\mu)}$ & \boldsymbol{$\tilde{O}({1} / {t^2})$} & \Cref{theorem: MDP convergence rate} \\
        PG & Constant & $\theta^{(t)} + \eta_2 \cdot \nabla_{\theta}{V^{\pi_{\theta}}(\mu)}$ & $O({1} / {t})$ & \citep{mei2020global} \\
        \midrule
        \hyperref[algorithm:APG]{APG}  & Time-Varying & $\omega^{(t)} + \big[ \eta_3 \cdot \beta^{t} \cdot \nabla_{\theta}{V^{\pi_{\omega}^{(t)}}(\mu)} \big]^{+K}$  & \boldsymbol{$O(e^{-ct})$} & \Cref{theorem: adaptive apg convergence rate} \\
        {NPG} & {Time-Varying} & {$\theta^{(t)} + \eta_5 \cdot Q^{\pi_{\theta}^{(t)}}$} & {$O(e^{-ct})$} & {\citep{khodadadian2021linear, mei2021understanding}} \\
        NPG & Time-Varying & $\theta^{(t)} + (\eta_6 + \log({1} \big/ {\pi_{\theta}^{(t)}}) ) \cdot Q^{\pi_{\theta}^{(t)}}$ & $O(e^{-ct})$ & \citep{khodadadian2021linear} \\
        NPG \footnotemark[3] & Time-Varying & $\theta^{(t)} + \eta_7 \cdot \beta^t \cdot Q^{\pi_{\theta}^{(t)}}$ & $O(e^{-ct})$ & \citep{xiao2022convergence} \\
        GNPG & Time-Varying & $\theta^{(t)} + \eta_8 \cdot {\Big|\Big| \nabla_{\theta}{V^{\pi_{\theta}^{(t)}}(\mu)}  \Big|\Big|_2^{-1}} \cdot {\nabla_{\theta}{V^{\pi_{\theta}^{(t)}}(\mu)} }$ & $O(e^{-ct})$ & \citep{mei2021leveraging}\\
        \bottomrule
        \end{tabular}
    }
\end{center}
\vskip -0.1in
\end{table*}

\textbf{Policy Gradient.}
Policy gradient \citep{sutton1999policy} is a popular RL technique that directly optimizes the objective function by computing and using the gradient of the expected return with respect to the policy parameters. It has several popular variants, such as the REINFORCE algorithm \citep{williams1992simple}, actor-critic methods \citep{konda1999actor}, trust region policy optimization (TRPO) \citep{schulman2015trust}, and proximal policy optimization (PPO) \citep{schulman2017proximal}.
Recently, policy gradient methods have been shown to {converge to the global optimal policy}.
{Convergence properties} of standard policy gradient methods under various settings {have} been proven by
\citet{agarwal2021theory}.
Furthermore, \citet{mei2020global} {characterize} a $O(1/t)$ convergence rate of policy gradient under non-regularized tabular softmax parameterization.

{
\textbf{Acceleration Methods for RL.}
Numerous approaches aim to accelerate {PG} methods in RL. One notable approach is {the} Natural Policy Gradient (NPG), inspired by the natural gradient concept. NPG has been demonstrated to have: (i) a $O(1/t)$ rate with constant step sizes \citep{agarwal2021theory, xiao2022convergence}; (ii) linear convergence with adaptively increasing step sizes \citep{khodadadian2021linear} or non-adaptive exponentially growing step sizes \citep{xiao2022convergence}.
Gradient normalization is another effective method for accelerating PG. This approach involves normalizing the true policy gradient by its {Euclidean} norm. \citet{mei2021leveraging} propose Geometry-aware Normalized PG (GNPG) and demonstrate its linear convergence under softmax policies.
For more details, please refer to \Cref{app:add-related}.
To the best of our knowledge, the effects of momentum in the context of RL remain unclear. In this work, we utilize the Nesterov momentum to achieve substantial acceleration for PG. Please refer to \Cref{table:rate} for a detailed comparison.
}

\textbf{Nesterov's Accelerated Gradient.}
Accelerated gradient methods \citep{nesterov1983method,beck2009fastb,su2014differential,attouch2016rate,carmon2017convex,jin2018accelerated} play a pivotal role in the optimization literature due to their ability to achieve convergence rates {faster than that of} the conventional gradient descent algorithm.
Notably, in the convex {minimization} regime, accelerated gradient methods enjoy a convergence rate as fast as $O(1/t^2)$, surpassing the limited $O(1/t)$ {rate offered by gradient descent}.
In order to {further} enhance the performance of accelerated gradient methods, several variants have been proposed. For instance, \citet{beck2009fast} propose a variant of the accelerated gradient method {with restart to achieve monotonicity in the objective value}.
\citet{ghadimi2016accelerated} present a unified analytical framework for a family of accelerated gradient methods that can be applied to solve convex, non-convex, and stochastic optimization problems.
Moreover, {\citet{li2015accelerated} extend the restart approach of \citet{beck2009fast} with monotonic objective values to further ensure sufficient objective improvement, in order to provide convergence guarantees on stationarity measures for} non-convex problems.
Other restart mechanisms have also been applied in multiple recent accelerated gradient methods \citep{o2015adaptive,li2022restarted}. 
The above list of works is by no means exhaustive and is only meant to provide a brief overview of accelerated gradient methods.
Our paper introduces APG, which combines Nesterov's accelerated gradient and PG methods for RL.
This integration enables a substantial acceleration of the convergence rate compared to the standard PG method.
It is essential to note that in RL problems, we are \emph{maximizing} the objective function, which in general is neither convex nor concave. To connect with the literature about convex \emph{minimization}, searching for properties akin to concavity becomes pivotal.

\footnotetext[2]{As $t/(t+1)$ is nearly constant for large $t$, we regard this design as constant step sizes, with a slight abuse of terminology.}
\footnotetext[3]{While \cite{xiao2022convergence} examine the policy mirror descent (PMD), a generalized version of PG, under direct parameterization instead of softmax policies, it is known that by opting for KL-divergence as the Bregman divergence, the update rule of PMD can be reformulated in a manner akin to Natural Policy Gradient (NPG) expressed in the policy space.}
\section{Preliminaries}
\label{sec:prelim}


\textbf{Markov Decision Processes.} For a finite set $\mathcal{X}$, we use $\Delta(\mathcal{X})$ to denote the probability simplex over $\mathcal{X}$. We consider that a finite Markov decision process (MDP) $\mathcal{M} = (\mathcal{S}, \mathcal{A}, \mathcal{P}, r, \gamma, \rho)$ is determined by: (i) a finite state space $\mathcal{S}$, (ii) a finite action space $\mathcal{A}$\footnote[4]{{We consider the non-trivial case where $|\mathcal{A}| \ge 2$.}}, (iii) a transition kernel $\mathcal{P} : \mathcal{S} \times \mathcal{A} \to \Delta(\mathcal{S})$, determining the transition probability $\mathcal{P}(s'\rvert s,a)$ from each state-action pair $(s,a)$ to the next state $s'$, (iv) a reward function $r : \mathcal{S} \times \mathcal{A} \to \mathbb{R}$, (v) a discount factor $\gamma \in [0, 1)$, and (vi) an initial state distribution $\rho \in \Delta(\mathcal{S})$. Without loss of generality, we presume the one-step reward to lie     in the $[0,1]$ interval. Given a policy $\pi : \mathcal{S} \to \Delta(\mathcal{A})$, the value of state $s$ under $\pi$ is defined as
\begin{align}
\label{eq:state_value_function}
    V^\pi(s) \coloneqq \mathbb{E}\bigg[\sum_{t=0}^{\infty} \gamma^{t} r(s_{t}, a_{t})\bigg\vert \pi, s_{0}=s\bigg].
\end{align}
We use $\bm{V}^{\pi}$ to denote the vector of $V^{\pi}(s)$ of all the states $s\in \cS$.
The goal of the learner (or agent) is to search for a policy that maximizes the following objective function as $V^\pi(\rho) \coloneqq \mathbb{E}_{s \sim \rho}{ \left[ V^\pi(s) \right]}$. The $Q$-value (or action-value) and the advantage function of $\pi$ at $(s,a) \in \mathcal{S} \times \mathcal{A}$ are defined as
\begin{equation}
\begin{aligned}
    Q^\pi(s, a) &\coloneqq r(s, a) + \gamma \sum_{s^\prime}{ \mathcal{P}( s^\prime | s, a) V^\pi(s^\prime) }, \\
    A^\pi(s,a) &\coloneqq Q^\pi(s, a) - V^\pi(s),
\end{aligned}
\label{eq:state_action_value_function} 
\end{equation}
where the advantage function reflects the relative benefit of taking the action $a$ at state $s$ under policy $\pi$. The (discounted) state visitation distribution of $\pi$ is defined as
\begin{align*}
     d_{s_0}^{\pi}(s) \coloneqq (1 - \gamma) \sum_{t=0}^{\infty}{ \gamma^t Pr(s_t = s | s_0, \pi, \mathcal{P}) },
\end{align*}
which reflects how frequently the learner would
visit the state $s$ under policy $\pi$, and we let $d_{\rho}^{\pi}(s) \coloneqq \mathbb{E}_{s_0 \sim \rho}{\left[ d_{s_0}^{\pi}(s) \right]}$ be the expected state visitation distribution under the initial state distribution $\rho$.
Given $\rho$, there exists an optimal policy $\pi^*$ such
that\footnote[5]{See, for example, Theorem 1.7 in
	\citep{agarwal2019reinforcement} for the existence of such an optimal policy.}
\begin{align}
    V^{\pi^*}(\rho) = \max_{\pi : \mathcal{S} \to \Delta(\mathcal{A})}{ V^{\pi}(\rho) }. 
    \label{optimal_objective}
\end{align}
We denote $V^*(\rho) \coloneqq V^{\pi^*}(\rho)$ and $Q^*(s, a) \coloneqq Q^{\pi^*}(s, a)$. 
{Additionally, we make the following assumption:}
\begin{assumption}[\textbf{Unique Optimal Action}]
\label{assump:unique_optimal}

{There is a unique optimal action $a^{*}(s)$ for each state $s \in \cS$.}
\end{assumption}

\textbf{Surrogate Initial State Distribution.} While obtaining the true initial state distribution $\rho$ can be challenging in practice, it is fortunate that this issue can be addressed by considering other \textit{surrogate} initial state distribution $\mu$. Notably, {we will show} in the following theoretical proofs that even in the absence of knowledge about $\rho$, convergence guarantees of APG for $V^{*}(\rho)$ can still be obtained for almost every $\mu\in \Delta(\cS)$. This type of result bears high-level resemblance to those of gradient-based methods for general non-convex optimization \citep{lee2019first,o2019behavior}.
In this paper, we consider that $\mu$ is drawn uniformly at random from the $\lvert \cS\rvert$-dimensional probability simplex.

{\textbf{Softmax Parameterization.} For unconstrained $\theta \in \mathbb{R}^{|\mathcal{S}||\mathcal{A}|}$, the tabular softmax policy parameterization is defined as }
\begin{align}
\pi_{\theta}(\cdot | s) \coloneqq
\frac{\exp(\theta_{s,\cdot})}{\sum_{a^{'} \in \mathcal{A}} \exp(\theta_{s,a^{'}})}.
\label{eq:softmax}
\end{align}
{The term ``tabular'' means that we store a parameter table
	$\theta \in \R^{|\cS||\cA|}$ for each state-action pairs, and
	sometimes we would omit it and call \cref{eq:softmax} softmax parameterization.}

\textbf{Smoothness.}
As shown in prior works \citep{agarwal2021theory, mei2020global}, the RL objective function $V^{\pi_\theta}(\rho)$ enjoys smoothness under softmax parameterization. Specifically, the objective function $\theta \to V^{\pi_\theta}(\rho)$ is $\frac{8}{(1-\gamma)^3}$-smooth.


\textbf{Policy Gradient.} 
In policy gradient methods, the parameters are updated by {the gradient of $V^{\pi_{\theta}}(\mu)$ with respect to $\theta$ under a surrogate initial state distribution $\mu \in \Delta(\cS)$.} {\Cref{algorithm:PG}} presents the pseudo code of PG provided by \citep{mei2020global}.

\begin{algorithm} [!h]
	\caption{Policy Gradient (PG) in \citep{mei2020global}}
	\label{algorithm:PG}
	\begin{algorithmic}

\STATE \textbf{Input}:
Step size $\eta = \frac{1}{L}$, where $L$ is the Lipschitz constant of
the gradient of the objective function $V^{\pi_\theta}(\mu)$.
\STATE \textbf{Initialize}:
$\theta^{(1)}(s,a)$ for all $(s,a)$.

\FOR{$t = 1$ to $T$} 
    \STATE
	\vspace{-.2in}
    \begin{align} 
    \theta^{(t+1)} \leftarrow \theta^{(t)} + \eta \nabla_{\theta}{V^{\pi_{\theta}}(\mu)} \Big\rvert_{\theta = \theta^{(t)}}
    \label{eq:PG}
    \end{align}
	\vspace{-.15in}
\ENDFOR
	\end{algorithmic}
\end{algorithm}

\textbf{Nesterov’s Accelerated Gradient (NAG).}
Nesterov's Accelerated Gradient (NAG) \citep{nesterov1983method} is an optimization algorithm that utilizes a variant of momentum known as Nesterov's momentum to expedite the convergence rate.
Specifically, it computes an intermediate "lookahead" estimate of the gradient by evaluating the objective function at a point slightly ahead of the current estimate. We provide the pseudo code of NAG method as \Cref{algorithm:NAG} in \Cref{app:algo}.

\begin{remark}[\textbf{Step-Size Regimes}]
    \normalfont {Please note that the step-size regimes are categorized in \Cref{table:rate} based on the gradient update presented in \cref{eq:PG}. Under the softmax parameterization, the update $\theta^{(t)} + \eta_5 \cdot Q^{\pi_{\theta}^{(t)}}$ is equivalent to $\theta^{(t)} + \eta_5 \cdot A^{\pi_{\theta}^{(t)}}$, as $A^{\pi_{\theta}^{(t)}}$ is derived by adding an action-independent value $V^{\pi_{\theta}^{(t)}}$ to $Q^{\pi_{\theta}^{(t)}}$.
    According to \Cref{lemma:softmax_pg}, it is evident that the step size is implicitly related to $1 / \pi_{\theta}^{(t)}$ when being compared with \cref{eq:PG}. Therefore, we consider the NPG with the update $\theta^{(t)} + \eta_5 \cdot Q^{\pi_{\theta}^{(t)}}$ as an algorithm with time-varying step sizes.}
\end{remark}

\textbf{Notations.} Throughout the paper, we use $\lVert x\rVert$ to denote the $L_2$ norm of a real vector $x$.

\section{Methodology}
\label{sec:alg}

In this section, we present our proposed algorithm, Accelerated Policy Gradient (APG), which integrates Nesterov acceleration with gradient-based reinforcement learning algorithms. In \Cref{sec:alg:apg}, we introduce our central algorithm, APG. Subsequently, in \Cref{sec:alg:behavior}, we provide a motivating example in the bandit setting to illustrate the convergence behavior of APG. 

\begin{algorithm} [!ht]
	\caption{Accelerated Policy Gradient (APG)}
	\label{algorithm:APG}
	\begin{algorithmic}

\STATE \textbf{Input}:
Step size $\eta^{(t)} > 0$.
\STATE \textbf{Initialize}:
$\theta^{(0)}\in \mathbb{R}^{|\mathcal{S}||\mathcal{A}|}$, $\omega^{(0)} = \theta^{(0)}$.

\FOR{$t = 1$ to $T$} 
    \STATE
	\vspace{-.2in}
    \begin{align} 
    \theta^{(t)} &\leftarrow \omega^{(t-1)} + \eta^{(t)} \nabla_{\theta}{V^{\pi_{\theta}}(\mu)} \Big\rvert_{\theta = \omega^{(t-1)}} \label{algorithm:eq1}\\
    \varphi^{(t)} &\leftarrow \theta^{(t)} + \frac{t-1}{t+2}(\theta^{(t)}-\theta^{(t-1)}) \label{algorithm:eq2}\\
    \omega^{(t)} &\leftarrow
    \begin{cases}
        \varphi^{(t)}, & \text{if } V^{\pi_{\varphi}^{(t)}}(\mu) \ge V^{\pi_{\theta}^{(t)}}(\mu), \\
        \theta^{(t)}, & \text{otherwise.}
    \end{cases}
    \label{algorithm:eq3}
    \end{align}
\ENDFOR
	\end{algorithmic}
\end{algorithm}

\subsection{Accelerated Policy Gradient}
\label{sec:alg:apg}

We propose Accelerated Policy Gradient (APG) and present the pseudo
code of our algorithm in \Cref{algorithm:APG}.
{Throughout this paper, we adopt the notation $\pi_{\theta}^{(t)} \coloneqq \pi_{\theta^{(t)}}$, consistently applying this convention to denote various parameters including $\omega$ and $\varphi$.}
{Our algorithm design draws inspiration from the monotone version of NAG, which was originally proposed in the seminal work \citep{beck2009fast} with a restart mechanism and has been widely adopted by variants of NAG, e.g., \citep{li2015accelerated,o2015adaptive,li2022restarted}.
We adapt these updates to the RL objective, specifically $V^{\pi_{\theta}}(\mu)$. It is important to note that we will specify the step sizes $\eta^{(t)}$ in \Cref{lemma: enter local concavity}, as presented in \Cref{sec:theory}.}

{In \Cref{algorithm:APG}, the gradient update is performed in (\ref{algorithm:eq1}). Following this, (\ref{algorithm:eq2}) calculates the momentum for our parameters, which represents a fundamental technique employed in accelerated gradient methods. It is worth noting that in (\ref{algorithm:eq1}), the gradient is computed with respect to $\omega^{(t-1)}$, which is the parameter that the momentum brings us to, rather than $\theta^{(t)}$ itself. This distinction sets APG apart from the standard policy gradient updates (\Cref{algorithm:PG}). Moreover, (\ref{algorithm:eq3}) serves as the restart mechanism, playing the role of affirming our updates to improve the objective value.}

\begin{remark}
\normalfont {{
{One could also directly apply Nesterov's accelerated gradient (cf. \Cref{algorithm:NAG}) to the PG setting. We provide the pseudo code as \Cref{algorithm:nonrestart_APG} in \Cref{app:non monotone APG}.}
However, in the absence of monotonicity, we must address the intertwined effects between the gradient and the momentum to achieve convergence results. Specifically, we empirically demonstrate that the non-monotonicity could occur during the updates of \Cref{algorithm:nonrestart_APG}.
For more detailed information, please refer to \Cref{app:non monotone APG}.}}
\end{remark}

\subsection{A Motivating Example of APG}
\label{sec:alg:behavior}

Prior to the exposition of convergence analysis, we aim to provide some insights into why APG has the potential to attain a fast convergence rate, especially under the intricate non-concave objectives in reinforcement learning.

\begin{wrapfigure}{r}{0.23\textwidth}
\centering
  \vspace{-\intextsep}
  \includegraphics[width=0.23\textwidth]{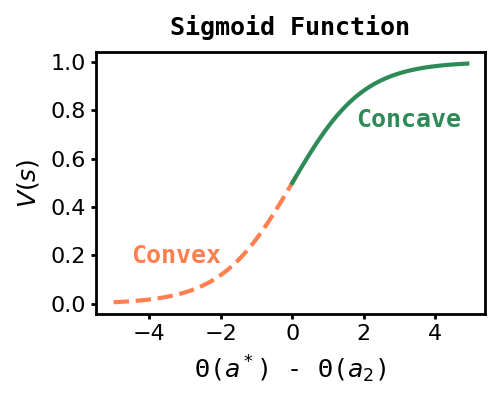}
  \hspace{-3mm}
  \vspace{-\intextsep}
  \caption{{The value function $V(s)$ versus the policy parameter $\theta_{a^*}$ and $\theta_{a_2}$ under a 2-armed bandit problem.}}
  \vspace{-\intextsep}
  \label{exp:sigmoid}
\end{wrapfigure}

Consider a simple two-action bandit with actions $a^*, a_2$ and the reward function $r(a^*) = 1, r(a_2) = 0$. Accordingly, the objective function we aim to optimize is {$\mathbb{E}_{a \sim \pi_\theta} [r(a)] = \pi_{\theta}(a^*) = 1 / (1 + e^{-(\theta_{a^*}-\theta_{a_2}}))$} {under softmax parameterization. This can be expressed as a sigmoid function with the parameter $(\theta_{a^*}-\theta_{a_2})$, as illustrated in \Cref{exp:sigmoid}.}
By deriving the Hessian matrix with respect to our policy parameters $\theta_{a^*}$ and $\theta_{a_2}$, we could characterize the curvature of the objective function around the current policy parameters, which provides useful insights into its local concavity.
Upon analyzing the Hessian matrix, we observe that it exhibits concavity when $\pi_{\theta}(a^*) > 0.5$. The detailed derivation is provided in \Cref{app:experiment}. The aforementioned observation implies that the objective function demonstrates \textit{local concavity} when $\pi_{\theta}(a^*) > 0.5$. Since $\pi^*(a^*) = 1$, it follows that the objective function exhibits local concavity in the proximity of the optimal policy $\pi^*$.
As a result, if one initializes the policy with a high probability assigned to the optimal action $a^*$, then the policy would directly fall in the locally concave part of the objective function.
This allows us to apply the theoretical findings from the existing convergence rate of NAG in \citep{nesterov1983method}, which has demonstrated convergence rates of $O(1/t^2)$ for convex problems.
Based on this insight, we establish the {convergence rate} of APG in the more general MDP setting in Section \ref{sec:theory}.

\section{Convergence Analysis}
\label{sec:theory}


In this section, we take an important first step towards understanding the convergence behavior of APG and discuss the theoretical results of APG in the general MDP setting.
Due to the space limit, we defer the proofs of the following theorems to Appendices \ref{app:asym_conv} and \ref{app:MDP}.



\subsection{Asymptotic Convergence of APG}
\label{sec:theory:convergence}
In this subsection, we will formally present the asymptotic convergence result of APG. This necessitates addressing several key challenges outlined in the introduction. We highlight the challenges tackled in our analysis as follows: (C1) \textit{The existing convergence results for non-convex problems under NAG are not directly applicable}: Note that the asymptotic convergence of standard PG is built on the standard convergence result of gradient descent for non-convex problems (i.e., {convergence of stationarity measures}), as shown in \citep{agarwal2021theory}. While it appears natural to follow the same approach for APG, two fundamental challenges are that the existing results of NAG for non-convex problems typically show best-iterate convergence (e.g., \citep{ghadimi2016accelerated}) rather than last-iterate convergence, and moreover these results hold under the assumption of a bounded domain (e.g., see Theorem 2 of \citep{ghadimi2016accelerated}), which does not hold under the softmax parameterization in RL as the domain of the policy parameters and the optimal $\theta$ could be unbounded.
These are yet additional salient differences between APG and PG. 
(C2) \textit{Characterization of the cumulative effect of each momentum term}: Based on (C1), even if the limiting value functions exist, another crucial obstacle is to precisely quantify the memory effect of the momentum term on the policy's overall evolution.
To address this challenge, we thoroughly examine the accumulation of the gradient and momentum terms, as well as the APG updates, to offer an accurate characterization of the momentum's memory effect on the policy.

Despite the above, we are still able to tackle these challenges and
establish the asymptotic convergence to the global optimal policy of
APG as follows. Recall the definition of the optimal objective in \cref{optimal_objective}.

\begin{restatable}[\textbf{{Asymptotic} Convergence Under Softmax Parameterization}]{theorem}{convergeoptimalthm}
\label{theorem:convergeoptimal}
Consider a tabular softmax parameterized policy $\pi_{\theta}$. For
\hyperref[algorithm:APG]{APG} {with $\eta^{(t)} = \frac{t}{t+1} \frac{(1 - \gamma)^3}{16}$}
{and $\mu$ initialized uniformly at random}, the following holds almost surely:
\begin{equation*}
   \lim_{t\rightarrow \infty}V^{\pi_{\theta}^{(t)}}(s) = V^{*}(s), \forall s\in\cS.
\end{equation*}
\end{restatable}
The complete proof is provided in Appendix \ref{app:asym_conv}.

\begin{remark}
\normalfont Note that Theorem \ref{theorem:convergeoptimal} suggests the use of step size $\eta^{(t)}$ that depends on $t/(t+1)$. This choice is related to one inherent issue of NAG: the choices of step sizes are typically different for the convex and the non-convex problems (e.g., \citep{ghadimi2016accelerated}). 
Recall from Section \ref{sec:alg:behavior} that the RL objective function could be locally concave around the optimal policy despite its non-concavity of the global landscape.
To enable the use of the same step size scheme throughout the whole training process, we find that incorporating the ratio $t/(t+1)$ could achieve the best of both concave and nonconcave cases.
\end{remark}
\begin{remark}
\normalfont In \Cref{theorem:convergeoptimal}, the almost-sure statement comes from the random initialization of the surrogate initial state distribution $\mu$. This condition resembles those of the convergence results of gradient-based methods for non-convex problems (e.g., \citep{lee2019first,o2019behavior}).
\end{remark}
\subsection{Convergence Rate of APG}
\label{sec:theory:bandit}
In this subsection, we leverage the asymptotic convergence of APG and proceed to characterize the convergence rate of APG under softmax parameterization. {We assume that \Cref{assump:unique_optimal} holds for the subsequent results.} {For ease of notation, we denote the actions for each state $s\in\cS$ as $a^*(s), a_2(s), \dots, a_{|\cA|}(s)$, ordered such that $Q^{*}(s, a^*(s)) > Q^{*}(s, a_2(s)) \geq \dots \geq Q^{*}(s, a_{|\cA|}(s))$. We will begin by introducing the two most fundamental concepts throughout this paper: the $C$-nearly concave property and the feasible update domain, as follows:}
{\begin{definition}[\textbf{$C$-Near Concavity}] \label{def:nearly_concavity_general}
	For {$C > 1$},
    a function $f$ is said to be $C$-nearly concave at $\theta$
	relative to a convex set $\mathcal{X} \ni \theta$ 
	if for any $\theta' \in \mathcal{X}$, we have
\begin{align*}
    f(\theta') \le f(\theta) + C \cdot \Big\langle \nabla f(\theta), \theta'-\theta \Big\rangle.
\end{align*}
\end{definition}}



\begin{definition}
\label{def: feasible update domain}
    We define the \textit{feasible update domain} $\mathcal{U}$ as
    \begin{align*}
        \mathcal{U} \coloneqq \{\bm{d} \in \mathbb{R}^{|\cS| |\cA|} : d_{s, a^*(s)} > \max_{a \neq a^*(s)} d_{s, a}, \forall s \in\cS\},
    \end{align*}
    where $d_{s, a} \in \mathbb{R}$ is the update with respect to the parameter $\theta_{s, a}$.
\end{definition}
In other words, a vector $\bm{d}$ in the feasible update domain
$\mathcal{U}$ possesses the property that the update magnitude with respect to $\theta_{s, a^*(s)}$ is the greatest among the actions in each state $s \in \mathcal{S}$.
{We use $\mathcal{U}$ to help us characterize the locally nearly-concave regime as shown below, and subsequently show that the updates of APG would all fall into $\mathcal{U}$ after some finite time.}


\begin{lemma}[\textbf{Locally $C$-Near Concavity; Informal}]
\label{lemma:local nearly concavity}
{Given any $C > 1$ and any $\bm{d} \in \cU$. Let $\theta$ be a
policy parameter satisfying the following conditions for all $s \in
\cS$: (i)
$V^{\pi_{\theta}}(s) > Q^*(s, a_2(s))$; (ii) For
some $M_{C, \bm{d}} > 0$ (depends on $C$ and $\bm{d}$), $\theta_{s,
a^*(s)} - \theta_{s, a} > M_{C, \bm{d}}$, for all $a \neq a^*(s)$.
Then, the objective function $\theta \mapsto V^{\pi_{\theta}}(\mu)$ is
$C$-nearly concave at $\theta$ along the direction $\bm{d}$.}
\end{lemma}

{The formal information regarding the constant $M_{C, \bm{d}}$ mentioned in \Cref{lemma:local nearly concavity} is provided in \Cref{app:MDP}.}

\begin{remark}
\normalfont {Notably, the optimal parameter under the softmax parameterization is unbounded. Consequently, it is challenging to determine whether a ``neighborhood'' exists near the optimal parameter such that the structure of the objective function exhibits desirable properties, such as concavity or near concavity. \Cref{lemma:local nearly concavity} delineates the sufficient conditions under which the objective function $V^{\pi_{\theta}}(\mu)$ achieves $C$-near concavity.}
\end{remark}

\begin{remark}
{\normalfont{{Note that the notion of \textit{weak-quasi-convexity}, as defined in \citep{guminov2017accelerated, hardt2018gradient, bu2020note}, shares similarities to our own definition in \Cref{def:nearly_concavity_general}. While their definition encompasses the direction from any to the optimal parameter, ours in \Cref{lemma:local nearly concavity} is more relaxed.
{Despite that the optimal parameter $\theta^*$ in our problem does not exist due to its unbounded nature, according to the definition of the feasible update domain $\mathcal{U}$, the direction toward the optimal policy $\pi^*$ in the parameter space is always contained in $\mathcal{U}$.}
}}}
\end{remark}
{With \Cref{lemma:local nearly concavity}, our goal is to investigate whether APG updates can move along the local $C$-nearly concavity direction after some finite number of time steps.} To address this, we establish \Cref{lemma: enter local concavity}, which guarantees the existence of a finite time $T$ such that our policy will indeed achieve the condition stated in \Cref{lemma:local nearly concavity} through the APG updates and remain in this region.


\begin{restatable}{lemma}{stationarylemma}
\label{lemma: enter local concavity}
Consider a tabular softmax parameterized policy.
	Under \Cref{assump:unique_optimal} and the setting of \Cref{theorem:convergeoptimal}, 
 the following holds almost surely: Given any $M > 0$, there exists a
 finite time T such that for all $t \ge T$, $s\in\cS$, and $a \neq
 a^*(s)$, we have (i) $\theta_{s, a^*(s)}^{(t)} - \theta_{s, a}^{(t)} > M$, (ii) $V^{\pi_{\theta}^{(t)}}(s) > Q^*(s, a_2(s))$, (iii) $\left.\frac{\partial V^{\pi_{\theta}}(\mu)}{\partial\theta_{s, a^*(s)}}\right\rvert_{\theta = \omega^{(t)}} > 0 > \left.\frac{\partial V^{\pi_{\theta}}(\mu)}{\partial\theta_{s, a}}\right\rvert_{\theta = \omega^{(t)}}$, (iv) $\omega_{s, a^*(s)}^{(t)} - \theta_{s, a^*(s)}^{(t)} \ge \omega_{s, a}^{(t)} - \theta_{s, a}^{(t)}$.
\end{restatable}

\begin{remark}
    \normalfont {Conditions (i) and (ii) are formulated to establish the local $C$-nearly concave conditions within a finite number of time steps. On the other hand, conditions (iii) and (iv) describe two essential properties for verifying that, after a finite time, all the update directions of APG fall within the feasible update domain $\mathcal{U}$. Consequently, the updates executed by APG are aligned with the nearly-concave structure. For a detailed proof regarding the examination of the feasible directions, please refer to \Cref{app:MDP}.}
\end{remark}




{With the results of Lemmas \ref{lemma:local nearly concavity} and \ref{lemma: enter local concavity}, we are able to establish the main result for APG under softmax parameterization, which is a $\tilde{O}(1/t^2)$ convergence rate.}

\begin{theorem}[\textbf{Convergence Rate of APG; Informal}]
\label{theorem: MDP convergence rate}
    Consider a tabular softmax parameterized policy $\pi_{\theta}$. Under \hyperref[algorithm:APG]{APG} with {$\eta^{(t)} = \frac{t}{t+1} \cdot \frac{(1 - \gamma)^3}{16}$} {and $\mu$ is initialized uniformly at random}, the following holds almost surely: There exists a finite time $T$ such that for all $t \ge T$, we have
    \begin{align*}
        &V^{*}(\rho) - V^{{\pi_\theta^{(t)}}}(\rho) \\ 
        &\le \frac{1}{(1 - \gamma)^3} \left \| \frac{{d^{\pi^*}_{\rho}}}{\mu} \right \|_{\infty}  \Bigg(\frac{2 |\cS|(|\cA| - 1)}{t^2 + |\cA| - 1} + \\
        &\frac{512|\cS| \ln^2(t) + 32\left \| 2\theta^{(T)} - (2+T)\big(\omega^{(T)} - \theta^{(T)}\big)\right \|^2}{(1-\gamma)(t+1)t}\Bigg).
    \end{align*}
\end{theorem}


\begin{remark}
    \normalfont {Notably, the intriguing local $C$-near concavity property is not specific to APG. Specifically, we can also establish that PG enters the local $C$-nearly concave regime within finite time steps. Moreover, its updates align with the directions in $\mathcal{U}$, allowing us to view it as an optimization problem under the $C$-nearly concave objective. Accordingly, we can demonstrate that PG under softmax parameterization also enjoys $C$-near concavity and thereby achieves a convergence rate of $\tilde{O}(1/t)$. For more details, please refer to \Cref{app:pg}.}
\end{remark}
\begin{remark}
\normalfont {It is important to note that the logarithmic factor in the sub-optimality gap is a consequence of the unbounded nature of the optimal parameter in softmax parameterization.}
\end{remark}

\begin{remark}
\label{remark: PG lower bound}
\normalfont{{
Regarding the fundamental capability of PG, \citet{mei2020global} has presented a lower bound on the sub-optimality gap for PG in their Theorem 10. This theorem asserts that the $O(1/t)$ convergence rate achievable by PG cannot be further improved. Despite the lower bound presented by \citet{mei2020global}, our results for APG do not contradict theirs. Specifically, while both APG and PG are first-order methods relying solely on first-order derivatives for updates, it is crucial to highlight that APG encompasses a broader class of policy updates with the help of the momentum term in Nesterov acceleration. This allows APG to utilize the gradients with respect to parameters that PG cannot attain, thereby improving the convergence rate and overall performance of APG.
}}
\end{remark}

In addition to \Cref{theorem: MDP convergence rate}, we also
demonstrate that APG exhibits a linear convergence when the
time-varying step sizes are allowed. Specifically, we
establish results similar to \Cref{theorem: MDP convergence rate} and
\Cref{lemma: enter local concavity} with exponentially-growing step
sizes under element-wise clipping updates, leading to the following theorem. For additional details, please refer to \Cref{app:add-adaptive-apg}.

\begin{figure*}[ht]
    \centering
    \hspace{-5mm}
    \subfigure[]{
    \label{exp:2-1}
    \includegraphics[width=0.255\textwidth]{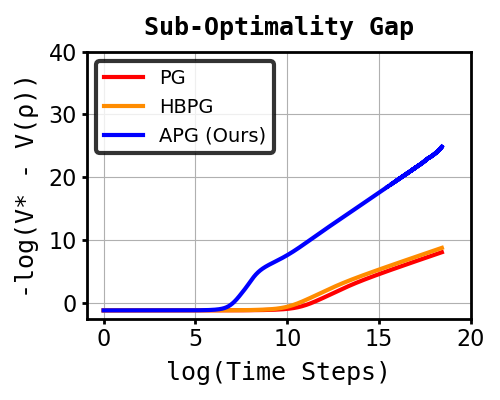}
    }
    \hspace{-5mm}
    \subfigure[]{
    \label{exp:2-2}
    \includegraphics[width=0.255\textwidth]{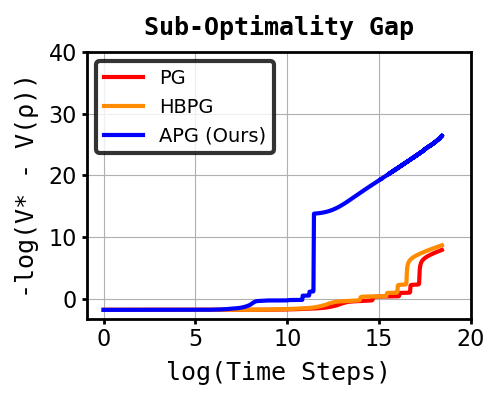}
    }
    \hspace{-5mm}
    \subfigure[]{
    \label{exp:2-3}
    \includegraphics[width=0.255\textwidth]{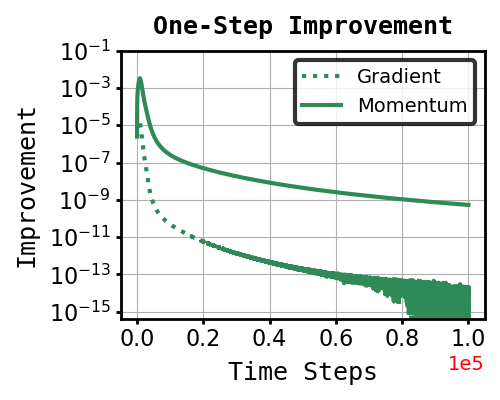}
    }
    \hspace{-5mm}
    \subfigure[]{
    \label{exp:2-4}
    \includegraphics[width=0.255\textwidth]{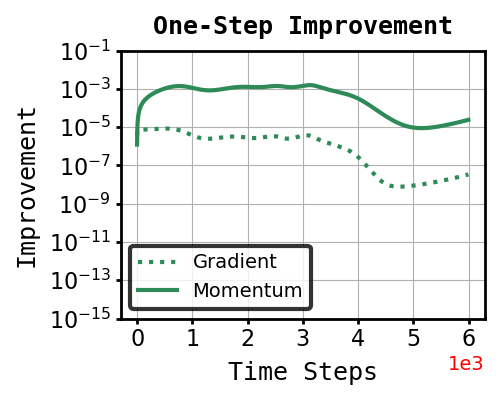}
    }
    \caption{A comparison between the performance of APG, PG, and HBPG under an MDP with 5 states, 5 actions, with the uniform and hard policy initialization: (a)-(b) show the sub-optimality gaps under the uniform and the hard initialization, respectively; (c)-(d) show the one-step improvements of APG from the momentum (i.e., $V^{\pi_{\omega}^{(t)}}(\rho)-V^{\pi_{\theta}^{(t)}}(\rho)$) and the gradient (i.e., $V^{\pi_{\theta}^{(t+1)}}(\rho)-V^{\pi_{\omega}^{(t)}}(\rho)$), under the uniform and the hard initialization, respectively.}
    \label{exp:2}
\end{figure*}

\begin{figure*}[!ht]
    \centering
    \hspace{-5mm}
    \subfigure[]{
    \label{exp:atari-2}
    \includegraphics[width=0.255\textwidth]{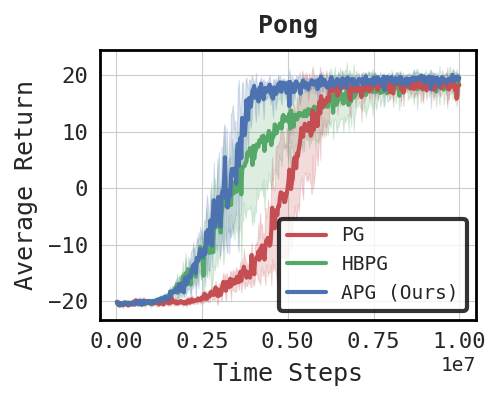}
    }
    \hspace{-5mm}
    \subfigure[]{
    \label{exp:atari-1}
    \includegraphics[width=0.255\textwidth]{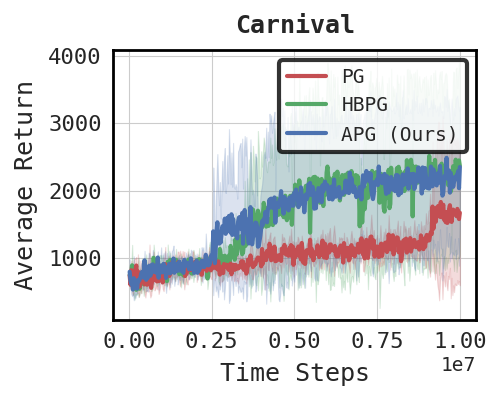}
    }
    \hspace{-5mm}
    \subfigure[]{
    \label{exp:atari-3}
    \includegraphics[width=0.255\textwidth]{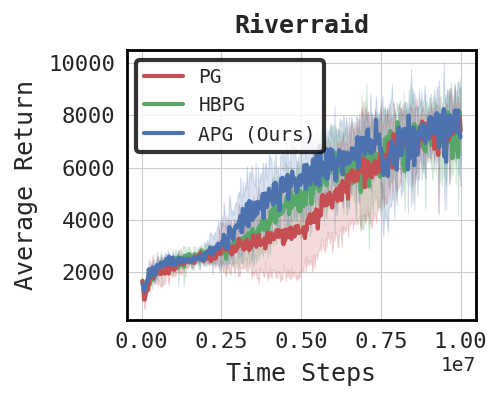}
    }
    \hspace{-5mm}
    \subfigure[]{
    \label{exp:atari-4}
    \includegraphics[width=0.255\textwidth]{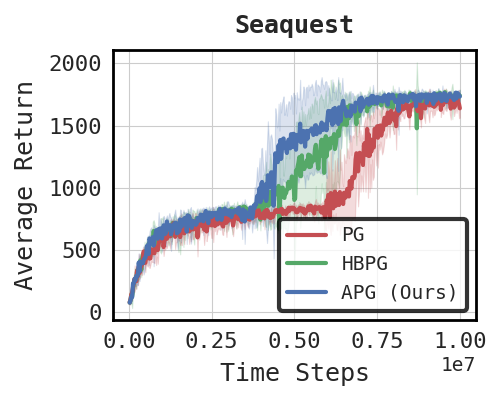}
    }
    \caption{A comparison of the performance of APG and the benchmark algorithms in four Atari 2600 games. All the results are averaged over 5 random seeds (with the shaded area showing the range of $\text{mean} \pm \text{std}$).}
    \label{exp:atari}
\end{figure*}

\begin{definition}
\label{def:clip}
    Given a vector $\bm{x} = [x_1, x_2, \cdots]$, we define the element-wise clipping operator by:
    \begin{align*}
        [\bm{x}]^{+y} = [\text{clip}(x_1, -y, y), \text{clip}(x_2, -y, y), \cdots],
    \end{align*}
    where
    \begin{align}
        \text{clip}(x, a, b) \coloneqq 
        \begin{cases}
            a, & \text{if } x < a, \\
            x, & \text{if } a \leq x \leq b, \\
            b, & \text{if } x > b.
        \end{cases}
    \end{align}
\end{definition}
\begin{restatable}[\textbf{{Convergence Rate of APG with Time-Varying Step Sizes}}]{theorem}{adaptiveapg}
\label{theorem: adaptive apg convergence rate}
    Consider a tabular softmax parameterized policy $\pi_{\theta}$. Under \hyperref[algorithm:APG]{APG} with \cref{algorithm:eq1} modified as 
    \begin{align*}
        \theta^{(t)} &\leftarrow \omega^{(t-1)} + \Big[ \eta^{(t)} \nabla_{\theta}{V^{\pi_{\theta}}(\mu)} \Big\rvert_{\theta = \omega^{(t-1)}} \Big]^{+K},
    \end{align*}
    where $\eta^{(t)} = \beta^t \cdot \frac{(1-\gamma)^3}{8}$ for any $\exp(\frac{1}{8\sqrt{|\cA|}|\cS|} \cdot \frac{1-\gamma}{4 C_\infty - (1-\gamma)}) > \beta > 1$, $\mu$ initialized uniformly at random, and $K = \frac{1}{4\sqrt{|\cA|}|\cS|} \cdot \frac{1-\gamma}{4 C_\infty - (1-\gamma)} + 2\ln \beta$ where $C_\infty \coloneqq \max_{\pi} \lVert \frac{d^{\pi}_\mu}{\mu} \rVert_\infty $, the following holds almost surely: There exists a finite time $T$ such that for all $t \ge T$, we have
    \begin{align*}
        V^{*}(\rho) - V^{{\pi_\theta^{(t)}}}(\rho) = O(e^{-ct}),
    \end{align*}
    for some constant $c > 0$.
\end{restatable}



\begin{remark}
    \normalfont {It is worth noticing that the finite constant time $T$ from \Cref{theorem: MDP convergence rate} and \Cref{theorem: adaptive apg convergence rate} is problem-dependent. Specifically, this constant $T$ shall depend on multiple factors, such as the initial state distribution $\rho$, the initial policy parameter $\theta_0$, the cardinalities of the state space $|\mathcal{S}|$ and the action space $|\mathcal{A}|$, and the discount factor $\gamma$, among other factors. Despite this problem-dependent constant, the rates of convergence in \Cref{theorem: MDP convergence rate} and \Cref{theorem: adaptive apg convergence rate} are $\tilde{O}(1/t^2)$ and ${O}(e^{-ct})$, under any problem instance.}
\end{remark}

\section{Experiments}
\label{sec:disc}

\subsection{Numerical Validation of the Convergence Rates}
\label{sec:disc:exp}

In this subsection, we empirically validate the convergence rate of APG by conducting experiments on an MDP with 5 states and 5 actions. We validate the convergence behavior of APG and two popular methods, namely standard PG and
the algorithm that directly applies heavy-ball momentum \citep{polyak1964some} to PG, which we refer to as Heavy-Ball PG (HBPG).
Detailed configurations and the pseudo code of HBPG are provided in
	\Cref{app:experiment}.
In the subsequent numerical validations, we focus on APG with nearly constant step sizes. The numerical validation for APG with time-varying step sizes is provided in \Cref{app:add-adaptive-apg} due to space limitations.
{Our code for the experiments is available at \url{https://github.com/NYCU-RL-Bandits-Lab/APG}.}


We validate the convergence rate of APG, HBPG, and PG on an MDP with 5 states and 5 actions in two settings: (i) \textit{Uniform initialization} (i.e., all actions have equal initial probability): The training curves of sub-optimality gap for {PG, HBPG, and APG} are depicted in \Cref{exp:2}. {Upon plotting the sub-optimality gaps of PG, HBPG, and APG under uniform initialization on a log-log graph in \Cref{exp:2-1}, we observe that both PG and HBPG exhibit a slope of approximately 1, while APG demonstrates a slope of 2. These slopes match the respective convergence rates of $O(1/t)$ for PG and HBPG and $\tilde{O}(1/t^2)$ convergence rate for APG, as shown in \Cref{theorem: MDP convergence rate}.}
\Cref{exp:2-3} further confirms that the momentum term in APG contributes substantially in terms of policy improvement in the MDP case. 
(ii) \textit{Hard initialization} {(i.e., the optimal action has the smallest initial probability)}: \Cref{exp:2-2} and \Cref{exp:2-4} show that APG could escape from sub-optimality much faster than {PG and HBPG} in the MDP case. This further showcases APG's superiority over {PG and HBPG}. 


{\subsection{APG on Atari Games}\label{sec:disc:atari}
In this subsection, we empirically evaluate the performance of APG on four Atari 2600 games from the Arcade Learning Environment (ALE) \cite{bellemare2013arcade}, including \textit{Pong, Carnival, Riverraid, and Seaquest}. Each environment is associated with a $210 \times 160 \times 3$ binary state representation, which corresponds to the $210 \times 160$ grid and $3$ channels. The detailed configuration is provided in \Cref{app:experiment}.
In \Cref{exp:atari}, we observe that APG exhibits faster or comparable convergence compared to both HBPG and PG across all four environments. This further demonstrates the potential of APG in practice.
}

\section{Concluding Remarks}
\label{sec:con}
Nesterov's Accelerated Gradient method, proposed almost four decades ago, provides a powerful first-order scheme for fast convergence under a broad class of optimization problems.
However, since its introduction, NAG has never been formally analyzed or evaluated within the realm of RL, mainly due to the non-concavity of the RL objective.
In this paper, we propose APG and take an important first step towards understanding NAG in RL.
{We rigorously present that APG can converge to a globally optimal policy at a $\tilde{O}(1/t^2)$ rate in the general MDP setting under softmax policies with nearly constant step sizes. Moreover, we showcase the linear convergence rate of $O(e^{-ct})$ with time-varying step sizes. This demonstrates the potential of APG in attaining fast convergence in RL.}

On the other hand, our work also leaves open interesting research questions about NAG in RL. For example, as our paper mainly focuses on the exact gradient setting, another promising research direction is to extend our results of APG to the stochastic gradient setting, where the advantage function as well as the gradient are estimated from sampled transitions. 

\section*{Impact Statement}
This paper introduces research aimed at pushing the boundaries of the Machine Learning field. There are many potential societal consequences of our work, none which we feel must be specifically highlighted here.

\section*{Acknowledgements}
{This material is based upon work partially supported by the National Science and Technology Council (NSTC), Taiwan under Contract No. 112-2628-E-A49-023 and Contract No. 112-2634-F-A49-001-MBK, also partially supported by the NVIDIA Taiwan R\&D Center, and also supported by the Higher Education Sprout Project of the National Yang Ming Chiao Tung University and Ministry of Education (MOE), Taiwan.}
%
CP is supported in part by the JSPS Grant-in-Aid for Research
Activity Start-up 23K19981 and Grant-in-Aid for Early-Career
Scientists 24K20845.
We also thank the National Center for High-performance Computing (NCHC) for providing computational and storage resources.
\normalem
{
\bibliography{reference}
\bibliographystyle{icml2024}
}

\newpage

\onecolumn
\section*{Appendix}
\appendix

\tableofcontents
\newpage

\label{section:Appendix}
\vspace{-2mm}
\section{Supporting Algorithms}
\label{app:algo}

\begingroup
\allowdisplaybreaks
For ease of exposition, we state the accelerated gradient algorithm
of \citet{ghadimi2016accelerated} in \Cref{algorithm:Ghadimi_AG}. We
have made the following notational changes from the version in
\citep{ghadimi2016accelerated} so that one could easily compare it
with \Cref{algorithm:APG}: (i) The positions of the superscript and
subscript are exchanged. (ii) The original gradient symbol is replaced
with the gradient of our objective (i.e.,
$\nabla_{\theta}{V^{\pi_{\theta}}(\mu)}$). (iii) The iteration counter
is changed from $k$ to $t$. (iv) The original descent algorithm is now
made an ascent algorithm (i.e., the sign in
\cref{eq:Ghadimi_AG_2,eq:Ghadimi_AG_3} is plus instead of minus).
{(v) We introduced a momentum variable $\bar{m}_{\text{init}}$ in \cref{eq:Ghadimi_AG_1} in order to consider a non-zero initial momentum when the algorithm enters the concave regime, which will be clarified in the subsequent discussion}.
\begin{algorithm} [H]
	\caption{The Accelerated Policy Gradient (APG) Algorithm Revised From \cite{ghadimi2016accelerated}}
	\label{algorithm:Ghadimi_AG}
	\begin{algorithmic}

\STATE Input:
$\hat{\theta}^{(0)}, \bar{m}_{\text{init}} \in \mathbb{R}^n$, $\{\alpha^{(t)}\}$ with  $\alpha^{(t)} \in (0,1]$ for all $t \ge 1$,
$\{\beta^{(t)} > 0\}$, and $\{\lambda^{(t)} > 0\}$.
\STATE 0. Set the initial points $\theta^{(0)}_{ag} = \hat{\theta}^{(0)} + \bar{m}_{\text{init}}$ and $t=1$.
\STATE 1. Set
\begin{align}
    \theta^{(t)}_{md} = (1 - \alpha^{(t)}) \theta^{(t-1)}_{ag} + \alpha^{(t)} \hat{\theta}^{(t-1)} \label{eq:Ghadimi_AG_1}
\end{align}

\STATE 2. Compute $\nabla {V^{\pi_{\theta}}(\mu)}$ and set
\begin{align}
    \hat{\theta}^{(t)} &= \hat{\theta}^{(t-1)}+\lambda^{(t)} \nabla_{\theta}{V^{\pi_{\theta}}(\mu)} \Big\rvert_{\theta = \theta^{(t)}_{md}}, \label{eq:Ghadimi_AG_2}\\
    \theta^{(t)}_{ag} &= \theta^{(t)}_{md} + \beta^{(t)} \nabla_{\theta}{V^{\pi_{\theta}}(\mu)} \Big\rvert_{\theta = \theta^{(t)}_{md}}. \label{eq:Ghadimi_AG_3}
\end{align}
\STATE 3. Set $t \leftarrow t+1$ and go to step 1.
	\end{algorithmic}
\end{algorithm}


\begin{lemma}[Equivalence between \Cref{algorithm:APG} and \Cref{algorithm:Ghadimi_AG}]
\label{lemma:equivalent_algorithm}
Suppose that from the $T_{\text{shift}}$-th iteration on, the restart
mechanism of \Cref{algorithm:APG} remains inactive,
\textit{i.e.}, $\omega^{(t)} = \varphi^{(t)}$ for all $t \geq
T_{\text{shift}}$.
Then, starting at the $T_{\text{shift}}$-th iteration,
\Cref{algorithm:APG} is equivalent to initializing \Cref{algorithm:Ghadimi_AG}
at $\theta^{(T_{\text{shift})}}$, with $\alpha^{(t)} =
\frac{2}{t+T_{\text{shift}}+1}$, $\beta^{(t)} =
\eta^{(t+T_{\text{shift}})}$, $\lambda^{(t)} =
\beta^{(t)}/\alpha^{(t)}$, $\bar{m}_{\text{init}} =
-\frac{\omega^{(T_{\text{shift}})} -
\theta^{(T_{\text{shift}})}}{\alpha^{(1)}}$,
and $\hat{\theta}^{(0)} = \theta^{(T_{\text{shift}})} -
\bar{m}_{\text{init}}$,
%
in the sense that $\omega^{(t)} = \theta_{md}^{(t -
T_{\text{shift}}+1)},
\theta^{(t)} = \theta_{ag}^{(t - T_{\text{shift}})}$.
\end{lemma}


\begin{remark}
\normalfont    \Cref{lemma:equivalent_algorithm} shows that our \Cref{algorithm:APG} is equivalent to \Cref{algorithm:Ghadimi_AG} under some proper conditions, so we could leverage the results of \citet{ghadimi2016accelerated}.
\end{remark}

\begin{proof}[Proof of \Cref{lemma:equivalent_algorithm}]  
\phantom{}
Since $\alpha^{(t)} \lambda^{(t)} = \beta^{(t)}$, by subtracting
\cref{eq:Ghadimi_AG_3} from \cref{eq:Ghadimi_AG_2} times
$\alpha^{(t)}$, for all $t \ge 1$, we have
\begin{align}
    \alpha^{(t)} \hat{\theta}^{(t)} - \theta^{(t)}_{ag} = \alpha^{(t)} \hat{\theta}^{(t-1)} - \theta^{(t)}_{md}.
    \label{eq:equivalent_algorithm_1}
\end{align}
Then, substituting $\theta^{(t)}_{md}$ in \cref{eq:equivalent_algorithm_1} by \cref{eq:Ghadimi_AG_1} leads to
\begin{align}
    \hat{\theta}^{(t)} = \frac{\theta^{(t)}_{ag} - (1-\alpha^{(t)})
	\theta^{(t-1)}_{ag} 
}{\alpha^{(t)}}.
    \label{eq:equivalent_algorithm_2}
\end{align}
Plugging \cref{eq:equivalent_algorithm_2} back into \cref{eq:Ghadimi_AG_1}, for $t \ge 2$, we get
\begin{align}
    \theta^{(t)}_{md}
    &= (1 - \alpha^{(t)}) \theta^{(t-1)}_{ag} + \alpha^{(t)}
	\hat{\theta}^{(t-1)} 
	\nonumber \\
    &= (1 - \alpha^{(t)}) \theta^{(t-1)}_{ag} + \alpha^{(t)} \frac{\theta^{(t-1)}_{ag} - (1-\alpha^{(t-1)}) \theta^{(t-2)}_{ag}
}{\alpha^{(t-1)}} 
\nonumber \\
    &= \theta^{(t-1)}_{ag} + \frac{\alpha^{(t)} (1-\alpha^{(t-1)})}{\alpha^{(t-1)}}(\theta^{(t-1)}_{ag} - \theta^{(t-2)}_{ag}). \label{eq:equivalent_algorithm_3}
\end{align}
For $t = 1$, {since $\theta_{ag}^{(0)} = \hat{\theta}^{(0)} +
\bar{m}_{\text{init}}$, }we get
\begin{align}
    \theta^{(1)}_{md}
    &= (1 - \alpha^{(1)}) \theta^{(0)}_{ag} + \alpha^{(1)}
	\hat{\theta}^{(0)}
	\nonumber \\
    &= \theta_{ag}^{(0)} - \alpha^{(1)} \bar{m}_{\text{init}}. 
\label{eq:equivalent_algorithm_6}
\end{align}
By
\cref{eq:Ghadimi_AG_3,eq:equivalent_algorithm_6,eq:equivalent_algorithm_3},
we could simplify \Cref{algorithm:Ghadimi_AG} into a two-variable
update form:
\begin{align*}
    \theta^{(t)}_{md} &= \left\{\begin{matrix}
		\theta_{ag}^{(0)} { - \alpha^{(1)}\bar{m}_{\text{init}} 
}   &&\text{if } t = 1,\\
    \theta^{(t-1)}_{ag} + \frac{\alpha^{(t)}
    	(1-\alpha^{(t-1)})}{\alpha^{(t-1)}} (\theta^{(t-1)}_{ag} -
    	\theta^{(t-2)}_{ag})  &&\text{if } t > 1.
    \end{matrix}\right. \\
    \theta^{(t)}_{ag} &= \theta^{(t)}_{md} + \beta^{(t)}
	\nabla_{\theta}{V^{\pi_{\theta}}(\mu)} \Big\rvert_{\theta =
		\theta^{(t)}_{md}} \quad\quad\quad\quad\quad\quad\quad\enspace\text{for all } t \ge 1.
\end{align*}
Finally, given $T_{\text{shift}} \in \mathbb{N}$, by applying $\alpha^{(t)} = \frac{2}{t+1+T_{\text{shift}}}$, $\beta^{(t)} =
\eta^{(t)}$, and $\bar{m}_{\text{init}} = -\frac{\omega^{(T_{\text{shift}})} - \theta^{(T_{\text{shift}})}}{ \alpha^{(1)}}$ in the update form above, we reach our desired result:
\begin{enumerate}
    \item First, initialize $\theta_{md}^{(1)} = \theta_{ag}^{(0)} - \alpha^{(1)} \bar{m}_{\text{init}} = \theta_{ag}^{(0)} + (\omega^{(T_{\text{shift}})} - \theta^{(T_{\text{shift}})})$.
    \item Then, run the following updates recursively:
    \begin{align}
        \theta^{(t)}_{ag} &= \theta^{(t)}_{md} + \eta^{(t)} \nabla_{\theta}{V^{\pi_{\theta}}(\mu)} \Big\rvert_{\theta = \theta^{(t)}_{md}} \label{eq:equivalent_algorithm_4} \\
        \theta^{(t+1)}_{md} &= \theta^{(t)}_{ag} + \frac{t-1+T_{\text{shift}}}{t+2+T_{\text{shift}}} (\theta^{(t)}_{ag} - \theta^{(t-1)}_{ag}).  \label{eq:equivalent_algorithm_5}
    \end{align}
\end{enumerate}
Note that $\omega^{(T_{\text{shift}})} - \theta^{(T_{\text{shift}})}$ is the momentum at time $T_{\text{shift}} - 1$, so we select $\bar{m}_{\text{init}}$ as $-\frac{\omega^{(T_{\text{shift}})} - \theta^{(T_{\text{shift}})}}{ \alpha^{(1)}}$. Furthermore, if $T_{\text{shift}} = 0$, it is exactly equivalent to the form shown in \Cref{algorithm:APG}.
In summary, $\theta_{md}^{(t+1)}, \theta_{ag}^{(t)}$ in \cref{eq:equivalent_algorithm_4} and \cref{eq:equivalent_algorithm_5} corresponds to $\omega^{(t)}, \theta^{(t)}$ in \Cref{algorithm:APG}, respectively. 
\end{proof}

Below, we present the pseudo code of Nesterov's Accelerated Gradient (NAG) algorithm discussed in \Cref{sec:prelim}.


\begin{algorithm} [H]
	\caption{Nesterov's Accelerated Gradient (NAG) algorithm in \citep{su2014differential}}
	\label{algorithm:NAG}
	\begin{algorithmic}
\STATE \textbf{Input}:
Step size $\eta = \frac{1}{L}$, where $L$ is the Lipschitz constant of
the gradient of the objective function $f$.
\STATE \textbf{Initialize}:
$\theta^{(0)}$ and $\omega^{(0)} = \theta^{(0)}$.

\FOR{$t = 1$ to $T$} 
    \STATE
    \begin{align*} 
    \theta^{(t)} &= \omega^{(t-1)} - \eta \nabla f(\omega^{(t-1)})\\
    \omega^{(t)} &= \theta^{(t)} + \frac{t-1}{t+2} (\theta^{(t)} - \theta^{(t-1)})
    \end{align*}
\ENDFOR
	\end{algorithmic}
\end{algorithm}





\endgroup

\newpage
\section{Supporting Lemmas}
\label{app:sup}

{In this section, we provide useful properties of policy optimization, softmax parameterization, and Accelerated Policy Gradient (APG).
Additionally, we also demonstrate the convergence rate of accelerated
gradient (not necessarily restricted to reinforcement learning) for
nearly concave objective functions.}

\begingroup
\allowdisplaybreaks

\subsection{Useful Properties of Policy Optimization}
We first present properties related to general policy optimization
that are not restricted to softmax parameterization.

\begin{lemma}[\textbf{Performance Difference Lemma in \citep{Kakade2002approx, agarwal2021theory}}]
\label{lemma:perf_diff} 
For any initial state distribution $\mu \in \Delta(\cS)$, the difference in the value under $\mu$ between two policies $\pi$ and $\pi^{\prime}$ can be characterized as
\begin{equation*}
V^{\pi}\left(\mu\right)-V^{\pi^{\prime}}\left(\mu\right)=\frac{1}{1-\gamma} \mathbb{E}_{s \sim d_{\mu}^{\pi}} \mathbb{E}_{a \sim \pi(\cdot \rvert s)}\left[A^{\pi^{\prime}}(s, a)\right].
\end{equation*}
In particular, by specifying $\mu$ as a one-hot distribution (i.e.,
$\mu(s_0)=1$ for some $s_0\in\cS$ and $\mu(s')=0$ for all $s'\neq
s_0$), we have
\begin{equation*}
V^{\pi}\left({s_0}\right)-V^{\pi^{\prime}}\left({s_0}\right)=\frac{1}{1-\gamma} \mathbb{E}_{s \sim d_{{s_0}}^{\pi}} \mathbb{E}_{a \sim \pi(\cdot \rvert s)}\left[A^{\pi^{\prime}}(s, a)\right].
\end{equation*}
\end{lemma}

\begin{lemma} $d^{\pi}_{\mu}(s) \ge (1-\gamma) \cdot \mu(s),
	\ \forall \pi, \forall s \in \mathcal{S}$, where $\mu(s)$ is any starting state distribution of the MDP.
\label{lemma:lower_bound_of_state_visitation_distribution}
\end{lemma}

\begin{proof}[Proof of \Cref{lemma:lower_bound_of_state_visitation_distribution}]  

Since $\gamma \in [0,1)$, we have
\begin{align*}
d^{\pi}_{\mu}(s)
&= \underset{s_0 \sim \mu}{\mathbb{E}} \left [ d^{\pi}_{\mu}(s) \right]\\
&= \underset{s_0 \sim \mu}{\mathbb{E}} \left [ (1-\gamma) \cdot \sum_{t=0}^{\infty}\gamma^t \cdot  \mathbb{P}(s_t=s \: | \: s_0, \pi) \right]\\
&\ge \underset{s_0 \sim \mu}{\mathbb{E}} \left [ (1-\gamma) \cdot \mathbb{P}(s_0=s \: | \: s_0, \pi) \right]\\
&= (1-\gamma) \cdot \mu(s).
\qedhere
\end{align*}
\end{proof}

\begin{lemma}
$V^{*}(\rho) - V^{{\pi}}(\rho) \le \frac{1}{1-\gamma} \cdot \Big \| \frac{{d^{\pi^*}_{\rho}}}{\mu} \Big \|_{\infty} \cdot \left ( V^{*}(\mu) - V^{{\pi}}(\mu) \right )$ for any probability distributions $\rho, \mu$ and any policy $\pi$.
\label{lemma:performance_difference_in_rho}
\end{lemma}
\begin{proof}[Proof of \Cref{lemma:performance_difference_in_rho}]  
\begin{align}
V^{*}(\rho) - V^{{\pi}}(\rho)
&= \frac{1}{1-\gamma} \cdot \sum_{s \in \mathcal{S}} d^{\pi^*}_{\rho} (s) \sum_{a \in \mathcal{A}} \pi^*(a|s) \cdot A^{{\pi}}(s,a) \label{eq1: lemma:performance_difference_in_rho} \\
&= \frac{1}{1-\gamma} \cdot \sum_{s \in \mathcal{S}} d^{\pi^*}_{\mu}(s) \cdot \frac{d^{\pi^*}_{\rho}(s)}{d^{\pi^*}_{\mu}(s)}  \sum_{a' \in \mathcal{A}} \pi^*(a|s) \cdot A^{{\pi}}(s,a) \nonumber \\
&\le \frac{1}{1-\gamma} \cdot \left \| \frac{{d^{\pi^*}_{\rho}}}{d^{\pi^*}_{\mu}} \right \|_{\infty} \cdot \sum_{s \in \mathcal{S}} d^{\pi^*}_{\mu}(s)  \sum_{a' \in \mathcal{A}} \pi^*(a|s) \cdot A^{{\pi}}(s,a) \nonumber \\
&\le \frac{1}{(1-\gamma)^2} \cdot \left \| \frac{{d^{\pi^*}_{\rho}}}{\mu} \right \|_{\infty} \cdot \sum_{s \in \mathcal{S}} d^{\pi^*}_{\mu}(s)  \sum_{a' \in \mathcal{A}} \pi^*(a|s) \cdot A^{{\pi}}(s,a) \label{eq2: lemma:performance_difference_in_rho} \\
&= \frac{1}{1-\gamma} \cdot \left \| \frac{{d^{\pi^*}_{\rho}}}{\mu} \right \|_{\infty} \cdot \left ( V^{*}(\mu) - V^{{\pi}}(\mu) \right ), \label{eq3: lemma:performance_difference_in_rho}
\end{align}
where (\ref{eq1: lemma:performance_difference_in_rho}) and (\ref{eq3: lemma:performance_difference_in_rho}) hold by \Cref{lemma:perf_diff} and (\ref{eq2: lemma:performance_difference_in_rho}) holds due to \Cref{lemma:lower_bound_of_state_visitation_distribution}.
\end{proof}

\begin{lemma}
For any policy $\pi$, we have $\sum_{a \in \cA} \pi(a|s) A^{\pi}(s,a)
= 0, \ \forall s \in \cS$.
\label{lemma:sum_pi_A}
\end{lemma}

\begin{proof}[Proof of \Cref{lemma:sum_pi_A}]  

\begin{align}
\sum_{a \in \cA} \pi(a|s) A^{\pi}(s,a) 
&= \sum_{a \in \cA} \pi(a|s) \Big( Q^{\pi}(s,a) - V^{\pi}(s) \Big) \nonumber \\
&= \sum_{a \in \cA} \pi(a|s) Q^{\pi}(s,a) - \sum_{a \in \cA} \pi(a|s) V^{\pi}(s) \nonumber \\
&= V^{\pi}(s) - V^{\pi}(s), \label{eq:sum_pi_A 1}\\
&=0, \nonumber
\end{align}
where \cref{eq:sum_pi_A 1} leverages the Bellman consistency equation in Lemma 1.4 of \citep{agarwal2019reinforcement}.
\end{proof}

\begin{lemma}
    \label{lemma: Q^* a_2}
	Assume \cref{assump:unique_optimal} holds.
Given any $s \in \mathcal{S}$,
    let $a^*(s)$ be the optimal action at the state $s$, and $a_i(s)$,
	where $2 \le i \le |\cA|$, be such that $Q^*(s, a^*(s)) \ge Q^*(s,
	a_2(s)) \ge \dots \ge Q^*(s, a_{|\cA|}(s))$. Then, for any policy
	$\pi$ and any $a_i(s)$ with $2 \le i \le |\cA|$, we have $Q^*(s, a_2(s)) \ge Q^{\pi}(s, a_i(s))$.
    Furthermore, we also have $V^*(s) = Q^*(s, a^*(s))$ for all $s\in\cS$.
 \end{lemma}
    \begin{proof}[Proof of \Cref{lemma: Q^* a_2}]

			Theorem 1.7 of \citep{agarwal2019reinforcement} gives that
			\begin{equation}
				Q^*(s, a) = \sup_{\bar \pi} Q^{\bar \pi}(s, a),\quad \forall (s, a) \in \cS \times \cA.
				\label{eq:Q}
			\end{equation}
			Therefore,
			for any policy $\pi$ and $a_i(s)$ for $2 \le i \le |\cA|$, we have
        \begin{align}
			\nonumber
            Q^{\pi}(s, a_i(s)) &\le \sup_{\pi'} Q^{\pi'}(s, a_i(s))\\
            &= Q^*(s, a_i(s)) \label{Q^* a_2 eq2} \\
            &\le Q^*(s, a_2(s)), \label{Q^* a_2 eq3}
        \end{align}
		where \cref{Q^* a_2 eq2} holds by \cref{eq:Q}, and \cref{Q^*
		a_2 eq3} is from the definition of $a_2(s)$.
        Furthermore, by page 14 of \citep{szepesvari2022algorithms} and the definition of $a^*(s)$, we have $V^*(s) = \argmax_{a\in\cA} Q^*(s, a) = Q^*(s, a^*(s))$ for all $s\in\cS$.
    \end{proof}
    
\begin{lemma}
\label{lemma:pi diff infinite norm}
For any fixed state $s\in\cS$, any fixed action $a\in\cA$, and any policies $\pi_1, \pi_2$, we have
\begin{equation*}
    \Big\lVert  \pi_1(\cdot | s) - \pi_2(\cdot | s)
	\Big\rVert_{\infty} \le 1 - \min\{\pi_1(a | s), \pi_2(a | s)\}.
\end{equation*}
\end{lemma}
\begin{proof}[Proof of \Cref{lemma:pi diff infinite norm}]
    For any $a' \in \cA$, we have
    \begin{equation}
        \Big\lvert \pi_1(a' | s) - \pi_2(a' | s) \Big\rvert \le \max\{ \pi_1(a' | s), \pi_2(a' | s) \}, \label{eq: pi diff max}
    \end{equation}
	where \cref{eq: pi diff max} holds because both terms are nonnegative. In addition,
    \begin{align}
        \Big\lvert \pi_1(a' | s) - \pi_2(a' | s) \Big\rvert &\le \max\{ \pi_1(a' | s), \pi_2(a' | s) \} - \min\{\pi_1(a' | s), \pi_2(a' | s)\} \nonumber \\
        &\le 1 - \min\{\pi_1(a' | s), \pi_2(a' | s)\}. \label{eq: pi diff min}
    \end{align}
    Now, we fix $a \in \cA$ and show that for any $a' \in \cA$ we have
    \begin{equation}
        \Big\lvert \pi_1(a' | s) - \pi_2(a' | s) \Big\rvert \le 1 - \min\{\pi_1(a | s), \pi_2(a | s)\}. \label{eq: pi diff result}
    \end{equation}
    When $a' = a$, \cref{eq: pi diff result} holds by \cref{eq: pi
	diff min}. When $a' \neq a$, we note that the maximum of $\pi_1(a' | s)$ and $\pi_2(a' | s)$ can be at most $1 - \pi_1(a | s)$ and $1 - \pi_2(a | s)$, respectively.
	Thus, by \cref{eq: pi diff max}, we have
    \begin{align*}
        \Big\lvert \pi_1(a' | s) - \pi_2(a' | s) \Big\rvert &\le \max\{ \pi_1(a' | s), \pi_2(a' | s) \} \\
        &\le \max\{1 - \pi_1(a | s), 1 - \pi_2(a | s)\} \\
        &= 1 - \min\{\pi_1(a | s), \pi_2(a | s)\}.
    \end{align*}
	The desired result is then obtained by taking the maximum over all
	$a'\in\cA$.
\end{proof}

\begin{lemma}
\label{lemma: reward range to value range}
    Given an infinite horizon discount MDP $\mathcal{M}$ as defined in \Cref{sec:prelim}. Suppose that for all $(s, a) \in \mathcal{S}\times\mathcal{A}$, the reward function $r(s, a)$ of the MDP $\mathcal{M}$ lies in $[0, R_{\text{max}}]$ for some $R_{\text{max}} \ge 0$. Then, we have $V^{\pi}(s) \in [0, R_{\text{max}} / (1-\gamma)]$ for any $s\in\mathcal{S}$ and any policy $\pi$. Furthermore, $Q^{\pi}(s, a) \in [0, R_{\text{max}} / (1-\gamma)]$ and $\lvert A^{\pi}(s, a)\rvert \in [0, R_{\text{max}} / (1-\gamma)]$ for any $(s, a)\in\mathcal{S}\times \mathcal{A}$ and any policy $\pi$.
\end{lemma}
\begin{proof}[Proof of \Cref{lemma: reward range to value range}]
    By the definition in \cref{eq:state_value_function} and that $r(s,
	a) \in [0, R_{\text{max}}]$, for any policy $\pi$ and any state $s\in\mathcal{S}$, we have
    \begin{equation*}
        V^\pi(s) = \mathbb{E}\bigg[\sum_{t=0}^{\infty} \gamma^{t} r(s_{t}, a_{t})\bigg\vert \pi, s_{0}=s\bigg] \le \mathbb{E}\bigg[\sum_{t=0}^{\infty} \gamma^{t} R_{\text{max}}\bigg\vert \pi, s_{0}=s\bigg] = \sum_{t=0}^{\infty} \gamma^{t} R_{\text{max}} = \frac{R_{\text{max}}}{1-\gamma}.
    \end{equation*}
    For the $Q$-value, by \cref{eq:state_action_value_function}, given any $(s, a)\in\mathcal{S}\times \mathcal{A}$ and any policy $\pi$,
    \begin{equation*}
        Q^\pi(s, a) = r(s, a) + \gamma \sum_{s^\prime}{ \mathcal{P}( s^\prime | s, a) V^\pi(s^\prime) } \le R_{\text{max}} + \gamma  \cdot\frac{R_{\text{max}}}{1-\gamma} = \frac{R_{\text{max}}}{1-\gamma}.
    \end{equation*}
    For the advantage function, by
	\cref{eq:state_action_value_function}, we combine the above
	results of $V^\pi(s)$ and $Q^\pi(s, a)$ and obtain that $\lvert A^{\pi}(s, a)\rvert \in [0, R_{\text{max}} / (1-\gamma)]$ for any $(s, a)\in\mathcal{S}\times \mathcal{A}$ and any policy $\pi$.
\end{proof}

\subsection{Useful Properties of Softmax Parameterization}
This section provides results specific for softmax parameterization.



\begin{lemma}[\textbf{Lemma 1. in \citep{mei2020global}}]
\label{lemma:softmax_pg}
The softmax policy gradient with respect to $\theta$ is
\begin{align*}
    \frac{\partial V^{\pi_{\theta}}(\mu)}{\partial \theta_{s,a}} = \frac{1}{1-\gamma} \cdot d^{\pi_\theta}_{\mu}(s) \cdot \pi_\theta(a|s) \cdot A^{\pi_\theta}(s,a), \text{ for each } (s, a) \in \cS \times \cA.
\end{align*}
\end{lemma}
\begin{lemma}[\textbf{Equation in \citep{agarwal2021theory}}, page 18]
\label{lemma:softmax_pg_wrt_pi}
    The softmax policy gradient with respect to $\pi_{\theta}$ is
    \begin{align*}
        \frac{\partial V^{\pi_{\theta}}(\mu)}{\partial \pi_{\theta}(a|s)} = \frac{1}{1-\gamma} \cdot d^{\pi_\theta}_{\mu}(s) \cdot A^{\pi_\theta}(s,a), \text{ for each } (s, a) \in \cS \times \cA. 
    \end{align*}
\end{lemma}

\begin{lemma}[\textbf{Lemma 2. in \citep{mei2020global}}]
\label{lemma:bandit_smoothness}
Under softmax policy parameterization,
$\theta \to \pi_\theta^\top r$ is $5/2$-smooth for any $r \in [0,
1]^{|\cA|}$.
\end{lemma}

\begin{lemma}[\textbf{Lemma 7. in \citep{mei2020global}}]
\label{lemma:mdp_smoothness}
Under softmax policy parameterization,
$\theta \to V^{\pi_\theta}(\rho)$ is $\frac{8}{(1-\gamma)^3}$-smooth
for any $\gamma \in [0,1)$.
\end{lemma}

\begin{lemma}[\textbf{Lemma 8. in \citep{mei2020global}}]
\label{lemma:non-uniform L}
For all $\theta \in \mathbb{R}^{|\cS|\times|\cA|}$, we have,
\begin{align*}
    \Big\lVert \nabla_\theta V^{\pi_\theta} \rvert_{\theta = \theta^{(t)}} \Big\rVert_2
    \ge \frac{\min_s \pi_\theta(a^*(s)|S))}{\sqrt{|\cS|}\cdot \lVert d^{\pi^*}_{\rho} / d^{\pi^*}_{\mu} \rVert_\infty} \cdot (V^*(\rho) - V^{\pi_\theta}(\rho)).
\end{align*}
    
\end{lemma}

\begin{definition}
[\textbf{Definition 3. in \citep{mei2021leveraging}}]
\label{definition: non-uniform smoothness in mei}
The function $f: \Theta \rightarrow \mathbb{R}$ satisfies $\beta(\omega)$ non-uniform smoothness if $f$ is differentiable and for all $\omega, \theta \in \Theta$, 
\begin{align}
\label{eq: non-uniform smoothness in mei}
    \Big| f(\theta) - f(\omega) - \langle \frac{d f(\omega)}{d \omega}, \theta - \omega \rangle \Big| \le \frac{\beta(\omega)}{2} \cdot \lVert \theta - \omega \rVert_{2}^{2},
\end{align}
where $\beta$ is a positive valued function: $\beta: \Theta \rightarrow (0, \infty)$.
\end{definition}

\begin{lemma}
\label{lemma: smoothness equivalence}
If a function satisfies \cref{eq: non-uniform smoothness in mei} in \Cref{definition: non-uniform smoothness in mei}, then $f$ satisfies
\begin{equation*}
    \left\lVert \frac{d f(\theta)}{d \theta} - \frac{d f(\omega)}{d \omega}\right\rVert_2 \le \beta(\omega) \cdot \lVert \theta - \omega\rVert_{2}.
\end{equation*}
\end{lemma}
\begin{proof}[Proof of \Cref{lemma: smoothness equivalence}]
    Let $\omega, \theta \in \Theta$ and let $p\in \Theta$ be chosen later.  By \cref{eq: non-uniform smoothness in mei}, we have the upper bound,
    \begin{align*}
        D &\coloneqq f(\theta + p) - f(\omega) + f(\omega - p) - f(\theta) \\
        &\le \big\langle \frac{d f(\omega)}{d \omega}, \theta + p - \omega \big\rangle + \frac{\beta(\omega)}{2} \lVert \theta + p - \omega \rVert_{2}^{2} + \big\langle \frac{d f(\theta)}{d \theta}, \omega - p - \theta \big\rangle + \frac{\beta(\omega)}{2} \lVert \omega - p - \theta \rVert_{2}^{2} \\
        &= - \big\langle \frac{d f(\omega)}{d \omega} - \frac{d f(\theta)}{d \theta}, \omega - \theta - p \big\rangle + \beta(\omega) \lVert \omega - \theta - p \rVert_{2}^{2}
    \end{align*}
    and the lower bound
    \begin{align*}
        D &= f(\theta + p) - f(\theta) + f(\omega - p) - f(\omega) \\
        &\ge \big\langle \frac{d f(\theta)}{d \theta}, p \big\rangle - \frac{\beta(\omega)}{2} \lVert p \rVert_{2}^{2} + \big\langle \frac{d f(\omega)}{d \omega}, -p \big\rangle - \frac{\beta(\omega)}{2} \lVert p\rVert_{2}^{2} \\
        &= - \big\langle \frac{d f(\omega)}{d \omega} - \frac{d f(\theta)}{d \theta}, p \big\rangle - \beta(\omega) \lVert p \rVert_{2}^{2}.
    \end{align*}
    By combining the above upper and lower bounds, we obtain
    \[
    - \big\langle \frac{d f(\omega)}{d \omega} - \frac{d f(\theta)}{d \theta}, \omega - \theta - 2p \big\rangle \le  \beta(\omega) \lVert \omega - \theta - p \rVert_{2}^{2} + \beta(\omega) \lVert p \rVert_{2}^{2}.
    \]
    Taking $p = \frac{1}{2} [\omega - \theta - \frac{1}{\beta(\omega)}(\frac{d f(\omega)}{d \omega} - \frac{d f(\theta)}{d \theta})]$ implies that
    \begin{align*}
        \frac{1}{\beta(\omega)} \left\lVert \frac{d f(\omega)}{d \omega} - \frac{d f(\theta)}{d \theta} \right\rVert_{2}^{2} &\le \frac{\beta(\omega)}{4} \left\lVert \omega - \theta + \frac{1}{\beta(\omega)} (\frac{d f(\omega)}{d \omega} - \frac{d f(\theta)}{d \theta})\right\rVert_{2}^{2} + \frac{\beta(\omega)}{4} \left\lVert \omega - \theta - \frac{1}{\beta(\omega)} (\frac{d f(\omega)}{d \omega} - \frac{d f(\theta)}{d \theta})\right\rVert_{2}^{2} \\
        &= \frac{\beta(\omega)}{2} \left\lVert \omega - \theta\right\rVert_{2}^{2} + \frac{1}{2\beta(\omega)} \left\lVert \frac{d f(\omega)}{d \omega} - \frac{d f(\theta)}{d \theta}\right\rVert_{2}^{2}.
    \end{align*}
    By rearraging the terms, we have
    \[
        \left\lVert \frac{d f(\theta)}{d \theta} - \frac{d f(\omega)}{d \omega}\right\rVert_2^{2} \le \beta(\omega) \cdot \lVert \theta - \omega\rVert_{2}^{2}.
    \]
    By taking square root on each side, we obtain the desired result.
\end{proof}

\begin{lemma}[\textbf{Lemma 6. in \citep{mei2021leveraging}}]
\label{lemma:NS1}
Denote $\omega_\zeta \coloneqq \omega + \zeta \cdot (\theta - \omega)$ with some $\zeta \in [0,1]$. $\theta \to V^{\pi_\omega}(\mu)$ satisfies $\beta(\theta_{\zeta})$ non-uniform smoothness with
\begin{align*}
    \beta(\theta_{\zeta}) = \Big[ 3 + \frac{4 \cdot (C_\infty - (1-\gamma))}{1-\gamma}\Big] \cdot \sqrt{|\cS|} \cdot \lVert \nabla_\theta V^{\pi_\theta} \rvert_{\theta = \omega_\zeta} \rVert_2,
\end{align*}
where $C_\infty \coloneqq \max_{\pi} \lVert \frac{d^{\pi}_\mu}{\mu} \rVert_\infty \le \frac{1}{\min_s \mu(s)} < \infty$.
\end{lemma}

\begin{lemma}
[\textbf{Lemma 7. in \citep{mei2021leveraging}}]
    \label{lemma:NS2}
    Let $\eta = \frac{1-\gamma}{6\cdot(1-\gamma)+8\cdot(C_\infty-(1-\gamma))}\cdot\frac{1}{\sqrt{|\cS|}}$ and 
    \begin{align*}
        \theta \leftarrow \omega + \eta \cdot \nabla_\theta V^{\pi_\theta} \rvert_{\theta = \omega} / \lVert \nabla_\theta V^{\pi_\theta} \rvert_{\theta = \omega} \rVert_2.
    \end{align*}
    Denote $\omega_\zeta \coloneqq \omega + \zeta \cdot (\theta - \omega)$ with some $\zeta \in [0,1]$. We have,
    \begin{align*}
        \lVert \nabla_\theta V^{\pi_\theta} \rvert_{\theta = \omega_\zeta} \rVert_2 \le 2 \lVert \nabla_\theta V^{\pi_\theta} \rvert_{\theta = \omega} \rVert_2.
    \end{align*}
\end{lemma}

\begin{lemma}
\label{lemma: state visitation distribution fix}
Consider softmax parameterization.
Given {$C > 1$} and a probability distribution $\mu$ over the
state space $\cS$, we let
\begin{equation*}
	M_C \coloneqq \ln \Big[ \frac{|\cS||\cA|^2}{{(C-1)}(1-\gamma)^2
\min_{s \in S} \mu(s)} \Big].
\end{equation*}
For any two parameters $\theta, \theta' \in \mathbb{R}^{|\cS||\cA|}$,
if
\begin{equation}
	\min \{ \theta'_{s, a^*(s)} - \theta'_{s, a}, \theta_{s, a^*(s)} -
	\theta_{s, a} \} > M_C,\quad \forall s\in\cS, \, \forall a \neq a^*(s),
	\label{eq:Mbound}
\end{equation}
then we have
\begin{equation*}
	\left\lVert \frac{d_{\mu}^{\pi_{\theta'}}}{d_{\mu}^{\pi_{\theta}}} 
	\right\rVert_{\infty} < C.
\end{equation*}
\end{lemma}
\begin{remark}
    \normalfont{
	{
	\Cref{lemma: state visitation distribution fix} is a general
	property of RL under softmax policies, regardless of the
	algorithm applied.
	In our subsequent proof of \Cref{theorem: MDP convergence rate},
	we will let $\theta$ and $\theta'$ respectively represent the
	parameters before and after an APG update.
	Moreover, we also establish in \Cref{lemma:local nearly concavity}
	that the condition \cref{eq:Mbound} is indeed achievavable for
APG.}}
%
\end{remark}



\begin{proof}[Proof of \Cref{lemma: state visitation distribution fix}]
		Using \Cref{lemma:pi diff infinite norm} and \cref{eq:Mbound}, we have
    \begin{align}
        \Big\lVert \pi_{\theta'} - \pi_{\theta} \Big\lVert_{\infty} 
        &\le \max_{s \in \cS} \big\{ 1 - \min\{ \pi_\theta(a^*(s)|s),
			\pi_{\theta'}(a^*(s)|s) \} \big\} \nonumber\\
        &\le \max_{s \in \cS} \big\{ 1 - \min\{ \frac{e^{\theta_{s, a^*(s)}}}{e^{\theta_{s, a^*(s)}} + (|\cA| - 1)e^{\theta_{s, a^*(s)}-M_C}}, \frac{e^{\theta'_{s, a^*(s)}}}{e^{\theta'_{s, a^*(s)}} + (|\cA| - 1)e^{\theta'_{s, a^*(s)}-M_C}} \} \big\} \nonumber \\
        &< \frac{|\cA|e^{-M_C}}{1 + |\cA|e^{-M_C}} \nonumber \\
        &< \frac{{(C-1)}(1 - \gamma)^2 \min_{s\in\cS}
	\mu(s)}{|\cS| |\cA|}. \label{eq1: state visitation distribution fix}
    \end{align}
    
    Next, we characterize the ratio $d_{\mu}^{\pi_{\theta'}}(s) /
	d_{\mu}^{\pi_{\theta}}(s)$ for a fixed $s$.
 {We can regard
	$d_{\mu}^{\pi}(s)$ as the value function of a specific
	MDP.}
Specifically, we construct this MDP, denoted as $\mathcal{M}(s)$, as follows.
We let $\mathcal{M}(s)$ share identical parameters with $\mathcal{M}$ defined in \Cref{sec:prelim}.
The only distinction lies in its reward function $r_{\mathcal{M}(s)}$,
which we define as $r_{\mathcal{M}(s)}(s', a') = (1 -
\gamma)\mathbb{I}\{s' = s\}$, where $\mathbb{I}\{s' = s\}$ is the
indicator function whose value is $1$ if $s' = s$ and $0$ otherwise.
We can observe that
\begin{equation}
\label{eq:d to V}
    {\mathbb{E}\Big[\sum_{t}^{\infty} \gamma^t (1 - \gamma)\mathbb{I}\{s_t = s\} \Big| \pi, s_0 \sim \mu\Big] = (1-\gamma) \sum_{t}^{\infty} \gamma^t Pr(s_t = s|\pi, s_0 \sim \mu).}
\end{equation}

We denote $V_{\mathcal{M}(s)}^{\pi}(\mu)$ as the value function of
$\pi$ calculated under the MDP $\mathcal{M}(s)$.
Moreover, we define $A^{\pi}_{\mathcal{M}(s)}$ as the advantage
function and $d_{\mathcal{M}(s), \mu}^{\pi}$ as the state visitation
distribution of $\pi$ in $\mathcal{M}(s)$, respectively. By the
definition of the state visitation distribution and the value
function, together with \cref{eq:d to V}, we obtain $d_{\mu}^{\pi}(s) = V_{\mathcal{M}(s)}^{\pi}(\mu)$ for any $\pi$.
    Consequently, we have 
    \begin{align}
        \Big\lvert d_{\mu}^{\pi_{\theta'}}(s) - d_{\mu}^{\pi_{\theta}}(s) \Big\rvert
        &= \Big\lvert V_{\mathcal{M}(s)}^{\pi_{\theta'}}(\mu) - V_{\mathcal{M}(s)}^{\pi_{\theta}}(\mu) \Big\rvert \nonumber\\
        &= \Big\lvert \sum_{s', a'} d_{\mathcal{M}(s), \mu}^{\pi_{\theta'}}(s') \pi_{\theta'}(a'|s') A^{\pi_{\theta}}_{\mathcal{M}(s)}(s', a') \Big\rvert \label{eq3: state visitation distribution fix}\\
        &= \Big\lvert \sum_{s', a'} d_{\mathcal{M}(s), \mu}^{\pi_{\theta'}}(s') (\pi_{\theta'}(a'|s') - \pi_{\theta}(a'|s')) A^{\pi_{\theta}}_{\mathcal{M}(s)}(s', a') \Big\rvert \label{eq4: state visitation distribution fix}\\
        &\le \sum_{s', a'} \Big\lvert \underbrace{d_{\mathcal{M}(s),
	\mu}^{\pi_{\theta'}}(s')}_{\le 1} (\pi_{\theta'}(a'|s') -
	\pi_{\theta}(a'|s'))
	\underbrace{A^{\pi_{\theta}}_{\mathcal{M}(s)}(s', a')}_{\le 1 \text{ {by \Cref{lemma: reward range to value range}}}}
	\Big\rvert \nonumber \\
        &\le \frac{1}{1-\gamma} \sum_{s', a'} \Big\lvert \pi_{\theta'}(a'|s') - \pi_{\theta}(a'|s') \Big\rvert \nonumber \\
        &\le \frac{|\cS||\cA|}{1-\gamma} \cdot \Big\lVert \pi_{\theta'} - \pi_{\theta} \Big\lVert_{\infty} \nonumber \\
        &< {(C-1)} (1 - \gamma) \min_{s' \in \cS} \mu(s')
		\label{eq:statevisit6} \\
        &\le {(C-1)} (1 - \gamma)\mu(s), \label{eq5: state visitation distribution fix}
    \end{align}
    where \cref{eq3: state visitation distribution fix} holds by
	leveraging \Cref{lemma:perf_diff} in the MDP $\mathcal{M}(s)$,
	\cref{eq4: state visitation distribution fix} is from \Cref{lemma:sum_pi_A},
	and we used \cref{eq1: state visitation distribution fix} to obtain
	\cref{eq:statevisit6}.
	Finally, \cref{eq5: state visitation distribution fix} leads to
    \begin{align}
        \frac{d_{\mu}^{\pi_{\theta'}}(s)}{d_{\mu}^{\pi_{\theta}}(s)}
        &< \frac{d_{\mu}^{\pi_{\theta}}(s) + {(C-1)} (1 -
	\gamma)\mu(s)}{d_{\mu}^{\pi_{\theta}}(s)} \nonumber\\
        &= 1 + \frac{{(C-1)} (1 - \gamma)\mu(s)}{d_{\mu}^{\pi_{\theta}}(s)} \nonumber \\
        &\le 1 + \frac{{(C-1)}(1 - \gamma) \mu(s)}{(1 - \gamma)\mu(s)} \label{eq7: state visitation distribution fix} \\
        &= {C}, \nonumber 
    \end{align}
    where
	\cref{eq7: state visitation distribution fix} used the fact that
	$d_{\mu}^{\pi_{\theta}}(s) \ge (1 - \gamma)\mu(s)$ for all $s \in
\cS$. Since $s$ is arbitrary, the desired result is obtained by taking
the maximum over all $s \in \cS$.
\qedhere

\end{proof}

According to \Cref{def: feasible update domain}, $\mathcal{U}$ is a
set consisting of vectors in the high-dimensional space of $\mathbb{R}^{|\cS| \times |\cA|}$.
Fortunately, by focusing at one state $s$ at a time,
we can greatly simplify the analysis.
We first define the following notations and then prove a critical
lemma for our analysis.

\begin{definition}
\label{def: subvector}
    Consider any $\theta\in \mathbb{R}^{|\cS||\cA|}$ and $\bm{d} \in \mathbb{R}^{|\cS||\cA|}$.
	For all $s \in \cS$, we define $\theta_{s_i, \cdot} \in
	\mathbb{R}^{|\cA|}$ and $\bm{d_{s_i, \cdot}} \in
	\mathbb{R}^{|\cA|}$ to be respectively the subvectors of $\theta$
	and  $\bm{d}$ along the state $s_i$ as follows:
    \begin{align*}
		\theta_{s_i,\cdot} \coloneqq [\theta_{s_i, a^*(s_i)},\,  \theta_{s_i,
		a_2(s_i)},\,  \cdots,\,  \theta_{s_i, a_{|\cA|}(s_i)}],\quad
		\bm{d_{s_i,\cdot}} \coloneqq [d_{s_i, a^*(s_i)} ,\, d_{s_i,
		a_2(s_i)} ,\, \cdots ,\, d_{s_i, a_{|\cA|}(s_i)}].
    \end{align*}
\end{definition}

\begin{lemma}
\label{lemma: theta to pi}
{Assume that $|\cA| \geq 2$. Consider any vector $\bm{d}$ in the feasible update
domain $\mathcal{U}$ such that $\bm{d_{s, \cdot}}$ is a unit vector
for all $s \in \cS$. Let
\begin{equation}
	M_d \coloneqq 2\max_{s \in \cS, i > 1}
	\Big\{\ln\Big[2(|\mathcal{A}| - 1)\Big] - \ln[d_{s, a^*(s)} - d_{s,
	a_i(s)}]\Big\}.
	\label{eq:Md}
\end{equation}
Let $\theta \in \mathbb{R}^{|\cS||\cA|}$ be a vector such that $\theta_{s, a^*(s)} - \theta_{s, a} > M_d$ for all $s \in \cS$ and $a \neq a^*(s)$.
Then, under softmax parameterization, we have the following
\begin{itemize}[leftmargin=*]
    \item The function $\theta \mapsto \pi_{\theta}(a | s)$ is convex {along the direction $\bm{d}$} for all $s\in\cS$ and $a \neq a^*(s)$.
    \item The function $\theta \mapsto \pi_{\theta}(a^*(s) | s)$ is concave {along the direction $\bm{d}$} for all $s\in\cS$.
\end{itemize}
}
\end{lemma}
\begin{proof}[Proof of \Cref{lemma: theta to pi}]
    {Given that the function $\theta \to \pi_\theta(a|s)$ is
	independent of any $s' \neq s$,} for the first item it suffices to establish the convexity of the function {$\theta_{s, \cdot} \to
	\pi_{\theta}(a|s)$ alone $\bm{d_{s, \cdot}}$} for all $a \neq a^{*}(s)$ for
	any given $s \in \cS$.
	Following that, since $\pi_{\theta}(a^*(s) | s) = 1 - \sum_{a \neq
	a^*} \pi_{\theta}(a | s)$ can be viewed as a finite summation of
	concave functions along $\bm{d}$ with respect to $\theta$, $\theta
	\to \pi_{\theta}(a^*(s) | s)$ is a concave along {$\bm{d_{s,
	\cdot}}$}, and the second item ensues.

	We first note that $\ln[2 (|\cA| - 1)] > 0$ because $|\cA| \geq 2$,
and $\ln[d_{s,a^*(s)} - d_{s,a_i(s)}] \le 0$ because $d_{s, a^*(s)} - d_{s, a_i(s)} \in (0,1]$ for all $i$ by the definition of $\cU$ and
the fact that $\|\bm{d_{s, \cdot}}\| \leq 1$.
		Therefore, we see that $M_d > 0$.
    Given an action $a_i(s) \neq a^*(s)$, since the convexity of a
	function is determined by its behavior along arbitrary lines
	in its domain (the domain $\cU$ is clearly convex), it suffices to show that the following
	twice-differentiable function is concave (\textit{i.e.}, the second derivative is non-positive) along the unit vector $\bm{d_{s, \cdot}}$ when $k \to 0$:
    \begin{align}
    f(k)
    &\coloneqq \frac{e^{\theta_{s, a_i(s)} + k \cdot d_{s, a_i(s)}}}{e^{\theta_{s, a^*(s)} + k \cdot d_{s, a^*(s)}} + e^{\theta_{s, a_2(s)} + k \cdot d_{s, a_2(s)}} + \cdots + e^{\theta_{s, a_{|\mathcal{A}|}(s)} + k \cdot d_{s, a_{|\cA|}(s)}}} \nonumber \\
    &= \frac{1}{e^{(\theta_{s, a^*(s)} - \theta_{s, a_i(s)}) + k \cdot
    (d_{s, a^*(s)} - d_{s, a_i(s)})} + 
    \sum_{j=2}^{|\mathcal{A}|} e^{(\theta_{s, a_j(s)} - \theta_{s, a_i(s)}) + k \cdot
    (d_{s, a_j(s)} - d_{s, a_i(s)})}}. \nonumber
    \end{align}

By defining $g(k) \coloneqq 1 / f(k)$ and taking the second derivative of $f(k)$, we have
\begin{align}
    f^{''}(k) 
    &= \frac{2 (g^{'}(k))^2}{g(k)^3} - \frac{g^{''}(k)}{g(k)^2}. \nonumber 
\end{align}
Let $k \to 0$, we then get
\begin{align*}
    f^{''}(0)
    &= \frac{1}{g(0)^2} \left( \frac{2 (g^{'}(0))^2}{g(0)} - g^{''}(0) \right).
\end{align*}

Since $g(k) \ge 0$ for any $k \in \bR$, $f^{''}(0) \geq 0$ if and only
if
\begin{align}
    2 (g^{'}(0))^2 - g^{''} (0) \cdot g(0) \geq 0. \label{eq: local_cocavity_2}
\end{align}
By direct calculation, we have
$m'(0) = \sum_{j \neq i} (d_j - d_i) \exp(\theta_{s, a_j(s)} - \theta_{s, a_i(s)})$ and $m''(0) = \sum_{j \neq i} (d_j - d_i)^2 \exp(\theta_{s, a_j(s)} - \theta_{s, a_i(s)})$.
Therefore, \cref{eq: local_cocavity_2} leads to
\begin{align}
    &~2(g'(0))^2 - g^{''} (0) \cdot g(0) \nonumber \\
    =&~ 2 \sum_{j \neq i} (d_{s, a_j(s)} - d_{s, a_i(s)})^2 \exp(2\theta_{s, a_j(s)} - 2\theta_{s, a_i(s)}) \\
    & \quad + 2 \sum_{u, v \neq i, u < v} 2(d_{s, a_u(s)} - d_{s, a_i(s)}) (d_{s, a_v(s)} - d_{s, a_i(s)}) \exp(\theta_{s, a_u(s)} + \theta_{s, a_v(s)} - 2\theta_{s, a_i(s)}) \nonumber \\
    & \quad - \sum_{j \neq i} (d_{s, a_j(s)} - d_{s, a_i(s)})^2 \exp(2\theta_{s, a_j(s)} - 2\theta_{s, a_i(s)})
    - \sum_{j \neq i} (d_{s, a_j(s)} - d_{s, a_i(s)})^2 \exp(\theta_{s, a_j(s)} - \theta_{s, a_i(s)})\nonumber  \\
    & \quad - \sum_{u, v \neq i, u < v} \left( (d_{s, a_u(s)} - d_{s, a_i(s)})^2 + (d_{s, a_v(s)} - d_{s, a_i(s)})^2 \right) \exp(\theta_{s, a_u(s)} + \theta_{s, a_v(s)} - 2\theta_{s. a_i(s)}) \nonumber \\
    =&~ \sum_{j \neq i} (d_{s, a_j(s)} - d_{s, a_i(s)})^2 \left( \exp(2\theta_{s, a_j(s)} - 2\theta_{s, a_i(s)}) - \exp(\theta_{s, a_j(s)} - \theta_{s, a_i(s)}) \right)\nonumber \\
    & \quad + \sum_{u, v \neq i, u < v} 2 (d_{s, a_u(s)} - d_{s, a_i(s)}) (d_{s, a_v(s)} - d_{s, a_i(s)}) \exp(\theta_{s, a_u(s)} + \theta_{s, a_v(s)} - 2\theta_{s, a_i(s)}) \nonumber \\
    & \quad - \sum_{u, v \neq i, u < v} {(d_{s, a_u(s)} - d_{s, a_v(s)})^2} \exp(\theta_{s, a_u(s)} + \theta_{s, a_v(s)} - 2\theta_{s, a_i(s)}) \nonumber \\
    \ge&~ (d_{s, a^*(s)} - d_{s, a_i(s)})^2 \left( \exp(2\theta_{s, a^*(s)} - 2\theta_{s, a_i(s)}) - \exp(\theta_{s, a^*(s)} - \theta_{s, a_i(s)}) \right)\nonumber \\
    & \quad - \sum_{j \neq i, a_j \neq a^*} (d_{s, a_j(s)} - d_{s, a_i(s)})^2 \exp(\theta_{s, a_j(s)} - \theta_{s, a_i(s)}) \nonumber \\
    & \quad + \sum_{u, v \neq i, u < v} 2 (d_{s, a_u(s)} - d_{s, a_i(s)}) (d_{s, a_v(s)} - d_{s, a_i(s)}) \exp(\theta_{s, a_u(s)} + \theta_{s, a_v(s)} - 2\theta_{s, a_i(s)}) \nonumber \\
    & \quad - \sum_{u, v \neq i, u < v} {(d_{s, a_u(s)} - d_{s, a_v(s)})^2} \exp(\theta_{s, a_u(s)} + \theta_{s, a_v(s)} - 2\theta_{s, a_i(s)}) \nonumber \\
    \ge&~ (d_{s, a^*(s)} - d_{s, a_i(s)})^2 \left( \exp(2\theta_{s, a^*(s)} - 2\theta_{s, a_i(s)}) - \exp(\theta_{s, a^*(s)} - \theta_{s, a_i(s)}) \right)\nonumber \\
    & \quad - \sum_{j \neq i, a_j \neq a^*} 2 \exp(\theta_{s, a_j(s)} - \theta_{s, a_i(s)}) \nonumber \\
    & \quad - \sum_{u, v \neq i, u < v} 2 \exp(\theta_{s, a_u(s)} + \theta_{s, a_v(s)} - 2\theta_{s, a_i(s)})
    - \sum_{u, v \neq i, u < v} {2} \exp(\theta_{s, a_u(s)} + \theta_{s, a_v(s)} - 2\theta_{s, a_i(s)}) \label{eq: theta to pi 3} \\
    \ge&~ (d_{s, a^*(s)} - d_{s, a_i(s)})^2 \exp(\theta_{s, a^*(s)} - \theta_{s, a_i(s)})
	\left( \exp(\theta_{s, a^*(s)} - \theta_{s, a_i(s)}) - 1 \right)
	\nonumber\\
	& \quad - 2(|\cA| - 1) \max_{j \neq i, a_j \neq a^*} \exp(\theta_{s, a_j(s)} - \theta_{s, a_i(s)})
    - {4} \binom{|\cA|-1}{2} \max_{u, v \neq i, u < v} \exp(\theta_{s, a_u(s)} + \theta_{s, a_v(s)} - 2\theta_{s, a_i(s)})\nonumber \\
    \ge&~ (d_{s, a^*(s)} - d_{s, a_i(s)})^2 \exp(\theta_{s, a^*(s)} - \theta_{s, a_i(s)}) \left( \exp(\theta_{s, a^*(s)} - \theta_{s, a_i(s)}) - 1 \right) \nonumber \\
	& \quad - 2(|\cA| - 1) \exp(\theta_{s, a^*(s)} - \theta_{s, a_i(s)} - M_d)
    - {4} \binom{|\cA|-1}{2} \exp(2\theta_{s, a^*(s)} - 2\theta_{s, a_i(s)} - M_d), \label{eq: theta to pi 5} \\
    \ge&~ \exp(\theta_{s, a^*(s)} - \theta_{s, a_i(s)}) \Big[(d_{s, a^*(s)} - d_{s, a_i(s)})^2 (\exp(\theta_{s, a^*(s)} - \theta_{s, a_i(s)}) - 1) \nonumber \\
    & \quad - 2(|\cA| - 1)^2 \exp(\theta_{s, a^*(s)} - \theta_{s, a_i(s)} - M_d)\Big], \label{eq: theta to pi 6}
\end{align}
where \cref{eq: theta to pi 3} holds by using the bounds $(d_{s, a_u(s)} - d_{s, a_i(s)}) (d_{s, a_v(s)} -
d_{s, a_i(s)}) \ge -1$ and {$(d_{s, a_u(s)} - d_{s, a_v(s)})^2 \le 2$}
(both from the Cauchy-Schwarz inequality and that $\|\bm{d_{s,\cdot}}\|
= 1$),
and \cref{eq: theta to pi 5,eq: theta to pi 6} is from our hypothesis of $\theta_{s, a^*(s)} - \theta_{s, a} > M_d > 0$ for all $a \neq a^*(s)$.

Since $\exp(\theta_{s, a^*(s)} - \theta_{s, a_i(s)}) > 0$, to make
\cref{eq: theta to pi 6} non-negative, we simply need to verify that
\begin{align*}
    (d_{s, a^*(s)} - d_{s, a_i(s)})^2 (\exp(\theta_{s, a^*(s)} - \theta_{s, a_i(s)}) - 1)
	\geq 2 (|\mathcal{A}| - 1)^2 \exp(\theta_{s, a^*(s)} - \theta_{s, a_i(s)} - M_d),
\end{align*}
which is equivalent to
\begin{align*}
    \exp(-M_d) &\leq \frac{(d_{s, a^*(s)} - d_{s, a_i(s)})^2 (\exp(\theta_{s, a^*(s)} - \theta_{s, a_i(s)}) - 1)}{ 2(|\mathcal{A}| - 1)^2 \exp(\theta_{s, a^*(s)} - \theta_{s, a_i(s)})} \\
    &= \frac{(d_{s, a^*(s)} - d_{s, a_i(s)})^2 }{ 2(|\mathcal{A}| - 1)^2 } \cdot \Big(1 - \frac{1}{\exp(\theta_{s, a^*(s)} - \theta_{s, a_i(s)})}\Big).
\end{align*}
\color{black}
From \cref{eq:Md},
since $2(|\cA|-1)  > \sqrt{2}$
and $-\ln[d_{s, a^*(s)} - d_{s, a_i(s)}] \ge 0$ for all $i >
1$ (as argued before), we get
$M_d \geq \ln(2)$.
Therefore, if $\theta_{s, a^*(s)} - \theta_{s, a_i(s)} > M_d$ for all
$s\in\cS$ and all $i > 1$,
we indeed obtain
\begin{align}
    \exp(-M_d) &\leq \frac{(d_{s, a^*(s)} - d_{s, a_i(s)})^2 }{ 4(|\mathcal{A}| - 1)^2 }
	\nonumber\\
	&= \frac{(d_{s, a^*(s)} - d_{s, a_i(s)})^2 }{ 2(|\mathcal{A}| - 1)^2 } \Big( 1 - \frac{1}{2} \Big) \nonumber \\
    &\leq \frac{(d_{s, a^*(s)} - d_{s, a_i(s)})^2 }{ 2(|\mathcal{A}| - 1)^2 } \Big( 1 - \frac{1}{\exp(M_d)} \Big) \label{eq: picking M 1} \\
    &\leq \frac{(d_{s, a^*(s)} - d_{s, a_i(s)})^2 }{2(|\mathcal{A}| - 1)^2} \Big( 1 - \frac{1}{\exp(\theta_{s, a^*(s)} - \theta_{s, a_i(s)})} \Big), \label{eq: picking M 2}
\end{align}
where \cref{eq: picking M 1,eq: picking M 2} hold by $\theta_{s,
a^*(s)} - \theta_{s, a_i(s)} > M_d \geq \ln(2)$ for all $i > 1$.
\end{proof}

\begin{lemma}
\label{lemma:feasible update domain improvement lemma}
{Under softmax parameterization, let $\theta \in
	\R^{|\mathcal{S}||\mathcal{A}|}$ be a parameter such that
	$V^{\pi_{\theta}}(s) > Q^*(s, a_2(s))$ for all $s\in\cS$.
Then, for any $s\in\cS$ and any $\theta' \coloneqq \theta + \bm{d}$ with
$\bm{d} \in \mathcal{U}$,
we have
\begin{equation*}
    \sum_{a\in\cA} \big(\pi_{\theta'}(a|s) - \pi_{\theta}(a|s)\big) A^{\pi_{\theta}}(s, a) \ge 0.
\end{equation*}
}
\end{lemma}
\begin{proof}[Proof of \Cref{lemma:feasible update domain improvement lemma}]
By \Cref{lemma: Q^* a_2} and our assumption on
$V^{\pi_\theta}(s)$, we see that $V^{\pi_{\theta}}(s) > Q^*(s, a_2(s)) \ge Q^{\pi_{\theta}}(s,
a_i(s))$ for all $i \ge 2$.
From \cref{eq:state_action_value_function}, this means
\begin{equation}
A^{\pi_{\theta}}(s, a) \le 0,\quad \forall a \neq a^*(s).
\label{eq:neg}
\end{equation}
Combining \cref{eq:neg} with \Cref{lemma:sum_pi_A} then gives
\begin{equation}
	A^{\pi_{\theta}}(s ,a^*(s)) \ge 0.
	\label{eq:pos}
\end{equation}

{
Since $\bm{d} \in \mathcal{U}$, we have $d_{s, a^*(s)} > d_{s, a}$ for
all $s \in \cS$ and all $a \neq a^*(s)$. Therefore,
\begin{align}
    &~\sum_{a \in \cA}(\pi_{\theta'}(a|s) - \pi_{\theta}(a|s)) A^{\pi_{\theta}}(s, a) \nonumber \\
    =&~ \sum_{a \in \cA} \pi_{\theta'}(a|s) A^{\pi_{\theta}}(s, a) \label{eq: sum pi A equals to 0 improvement} \\
    =&~ \sum_{a \in \cA} \frac{\exp{(\theta'_{s, a})}}{\sum_{a'} \exp{(\theta'_{s, a'})}}  A^{\pi_{\theta}}(s, a) \nonumber \\
    =&~ \frac{\sum_{a'} \exp{(\theta_{s, a'})}}{\sum_{a'} \exp{(\theta'_{s, a'})}} \sum_{a \in \cA} \frac{\exp{(\theta_{s, a} + d_{s, a})}}{\sum_{a'} \exp{(\theta_{s, a'})}}  A^{\pi_{\theta}}(s, a) \nonumber \\
    =&~ \frac{\sum_{a'} \exp{(\theta_{s, a'})}}{\sum_{a'} \exp{(\theta'_{s, a'} - \max_{a \neq a^*(s)} d_{s, a})}} \sum_{a \in \cA} \frac{\exp{(\theta_{s, a} + (d_{s, a} - \max_{a \neq a^*(s)} d_{s, a}))}}{\sum_{a'} \exp{(\theta_{s, a'})}}  A^{\pi_{\theta}}(s, a) \nonumber \\
    =&~ \frac{\sum_{a'} \exp{(\theta_{s, a'})}}{\sum_{a'} \exp{(\theta'_{s, a'} - \max_{a \neq a^*(s)} d_{s, a})}} \Big( \frac{\exp{(\theta_{s, a^*(s)} + (d_{s, a^*(s)} - \max_{a \neq a^*(s)} d_{s, a}))}}{\sum_{a'} \exp{(\theta_{s, a'})}}  A^{\pi_{\theta}}(s, a^*(s)) \nonumber \\
    &\quad + \sum_{a \neq a^*(s)} \frac{\exp{(\theta_{s, a} + (d_{s, a} - \max_{a \neq a^*(s)} d_{s, a}))}}{\sum_{a'} \exp{(\theta_{s, a'})}}  A^{\pi_{\theta}}(s, a) \Big) \nonumber \\
    \ge&~ \frac{\sum_{a'} \exp{(\theta_{s, a'})}}{\sum_{a'} \exp{(\theta'_{s, a'} - \max_{a \neq a^*(s)} d_{s, a})}} \Big( \frac{\exp{(\theta_{s, a^*(s)} )}}{\sum_{a'} \exp{(\theta_{s, a'})}}  A^{\pi_{\theta}}(s, a^*(s)) + \sum_{a \neq a^*(s)} \frac{\exp{(\theta_{s, a})}}{\sum_{a'} \exp{(\theta_{s, a'})}}  A^{\pi_{\theta}}(s, a) \Big) \label{eq: pi inscreasing 1-1} \\
    =&~ \frac{\sum_{a'} \exp{(\theta_{s, a'})}}{\sum_{a'} \exp{(\theta'_{s, a'} - \max_{a \neq a^*(s)} d_{s, a})}} \sum_{a \in \cA} \pi_{\theta}(a|s) A^{\pi_{\theta}}(s, a) = 0, \label{eq: pi inscreasing 1-2} 
\end{align}
where \cref{eq: sum pi A equals to 0 improvement} and \cref{eq: pi
inscreasing 1-2} hold due to \Cref{lemma:sum_pi_A},
and \cref{eq: pi inscreasing 1-1} holds by the definition of $\cU$
that $d_{s, a^*(s)} > \max_{a \neq a^*(s)} d_{s, a} \ge d_{s, a(s)}$
for all $s \in \cS$ and all $a \neq a^*$ and \cref{eq:neg,eq:pos}.
}    
\end{proof}

\begin{lemma}
\label{lemma:grad update improvement}
{Under softmax parameterization, for \hyperref[algorithm:APG]{APG},
we have $V^{\pi_{\theta}^{(t+1)}}(s) \ge V^{\pi_{\omega}^{(t)}}(s)$
for any $s\in\cS$ and any $t \ge 0$ as long as $\eta^{(t} > 0$.}
\end{lemma}
\begin{proof}[Proof of \Cref{lemma:grad update improvement}]
{By \Cref{lemma:perf_diff}, it suffices to show that
$\sum_{a\in\cA} \pi_{\theta}^{(t+1)}(a|s) A^{\pi_{\omega}^{(t)}}(s, a)
\ge 0$ for all $s\in\cS$.
Indeed,
by defining
\[
Z_{\theta}^{(t)}(s) \coloneqq
\sum_{a\in\cA}\exp(\theta^{(t)}_{s, a}), \quad Z_{\omega}^{(t)}(s)
\coloneqq \sum_{a\in\cA}\exp(\omega^{(t)}_{s, a}),\quad \partial_{s,
a}^{(t)} \coloneqq \left.\frac{\partial V^{\pi_{\theta}}(\mu)}{\partial
	\theta_{s, a}}\right|_{\theta = \omega^{(t)}},
\]
we have that for any $s\in\cS$,
\begin{align}
    \sum_{a\in\cA} \pi_{\theta}^{(t+1)}(a|s) A^{\pi_{\omega}^{(t)}}(s, a) &= \sum_{a\in\cA} \frac{\exp(\theta_{s, a}^{(t+1)})}{Z^{(t+1)}_{\theta}(s)} A^{\pi_{\omega}^{(t)}}(s, a) \nonumber \\
    &= \frac{Z^{(t)}_{\omega}(s)}{Z^{(t+1)}_{\theta}(s)}\sum_{a\in\cA} \frac{\exp(\theta_{s, a}^{(t+1)})}{Z^{(t)}_{\omega}(s)} A^{\pi_{\omega}^{(t)}}(s, a) \nonumber \\
    &= \frac{Z^{(t)}_{\omega}(s)}{Z^{(t+1)}_{\theta}(s)}\sum_{a\in\cA} \frac{\exp(\omega_{s, a}^{(t)} + \eta^{(t+1)} \partial_{s, a}^{(t)})}{Z^{(t)}_{\omega}(s)} A^{\pi_{\omega}^{(t)}}(s, a) \nonumber \\
    &\ge \frac{Z^{(t)}_{\omega}(s)}{Z^{(t+1)}_{\theta}(s)}\sum_{a\in\cA} \frac{\exp(\omega_{s, a}^{(t)})}{Z^{(t)}_{\omega}(s)} A^{\pi_{\omega}^{(t)}}(s, a) \label{eq: grad update eq1} \\
    &= \frac{Z^{(t)}_{\omega}(s)}{Z^{(t+1)}_{\theta}(s)}\sum_{a\in\cA} \pi_{\omega}^{(t)}(a|s) A^{\pi_{\omega}^{(t)}}(s, a) = 0, \label{eq: grad update eq2}
\end{align}
where \cref{eq: grad update eq1} holds by the fact that the sign of
$\eta^{(t+1)} \partial_{s, a}^{(t)}$ is the same as
$A^{\pi_{\omega}^{(t)}}(s, a)$ by \Cref{lemma:softmax_pg}, and
\cref{eq: grad update eq2} is true due to \Cref{lemma:sum_pi_A}.}
\end{proof}


\subsection{Useful Properties of Accelerated Policy Gradient}

Throughout the appendices, We use $\nabla_{s,a}^{(t)}$ as the shorthand for $\frac{\partial V^{\pi_{\theta}}(\mu)}{\partial \theta_{s,a}}\big\rvert_{\theta=\omega^{(t)}}$.
By $L$-smooth, we mean a function is differentiable with its gradient
$L$-Lipschitz continuous.
For any pair of positive integers $(j,t)$, we define
\begin{equation}
\label{eq:G}
    G(j,t)\coloneqq\begin{cases}
        1&,  \text{if } t=j {\text{ or } j=0},\\
        1+\mathbb{I} \left\{V^{\pi_{\varphi}^{(j+h)}}(\mu) \ge V^{\pi_{\theta}^{(j+h)}}(\mu) \enspace \text{ for } \enspace h=1\right \}\frac{j}{j+3}&,  \text{if } t=j+1,\\
        1+\mathbb{I} \left\{V^{\pi_{\varphi}^{(j+h)}}(\mu) \ge V^{\pi_{\theta}^{(j+h)}}(\mu) \enspace \text{ for } \enspace h=1 \right \}\frac{j}{j+3} \\
        \quad + \mathbb{I} \left \{V^{\pi_{\varphi}^{(j+h)}}(\mu) \ge V^{\pi_{\theta}^{(j+h)}}(\mu) \enspace \text{ for all } \enspace h=1,2 \right \}\frac{(j+1)j}{(j+4)(j+3)}&, \text{if } t=j+2,\\
        1+\mathbb{I}\left\{V^{\pi_{\varphi}^{(j+h)}}(\mu) \ge V^{\pi_{\theta}^{(j+h)}}(\mu) \enspace \text{ for } \enspace h=1\right \}\frac{j}{j+3}\\
        \quad +\mathbb{I}\left\{V^{\pi_{\varphi}^{(j+h)}}(\mu) \ge V^{\pi_{\theta}^{(j+h)}}(\mu) \enspace \text{ for all } \enspace h=1,2\right \}\frac{(j+1)j}{(j+4)(j+3)} \\
        \quad+\mathbb{I}\left\{V^{\pi_{\varphi}^{(j+h)}}(\mu) \ge V^{\pi_{\theta}^{(j+h)}}(\mu) \enspace \text{ for all } \enspace h=1,2,3\right \}\frac{(j+2)(j+1)j}{(j+5)(j+4)(j+3)}&, \text{if } t=j+3,\\
        1+\mathbb{I}\left\{V^{\pi_{\varphi}^{(j+h)}}(\mu) \ge V^{\pi_{\theta}^{(j+h)}}(\mu) \enspace \text{ for } \enspace h=1\right \}\frac{j}{j+3}\\
        \quad+\mathbb{I}\left\{V^{\pi_{\varphi}^{(j+h)}}(\mu) \ge V^{\pi_{\theta}^{(j+h)}}(\mu) \enspace \text{ for all } \enspace h=1,2\right \}\frac{(j+1)j}{(j+4)(j+3)}\\
        \quad+\mathbb{I}\left\{V^{\pi_{\varphi}^{(j+h)}}(\mu) \ge V^{\pi_{\theta}^{(j+h)}}(\mu) \enspace \text{ for all } \enspace h=1,2,3 \right \}\frac{(j+2)(j+1)j}{(j+5)(j+4)(j+3)}\\
        \quad+\sum_{k=4}^{t-j} \mathbb{I}\left\{V^{\pi_{\varphi}^{(j+h)}}(\mu) \ge V^{\pi_{\theta}^{(j+h)}}(\mu) \enspace \text{ for all } \enspace h=1,2,\dots,k\right \}\frac{(j+2)(j+1)j}{(j+k+2)(j+k+1)(j+k)}&, \text{if } t\geq j+4\\
        0&, \text{otherwise}.
    \end{cases}
\end{equation}


\begin{lemma}
    \label{lemma:theta as sum of gradients}
    Under APG, we could express the policy parameter as follows:

    \textbf{a)} For $t\in \{1,2,3,4\}$, we have
    \begin{align}
  \theta_{s,a}^{(1)}=&\eta^{(1)}\nabla_{s,a}^{(0)}+\theta_{s,a}^{(0)},\label{eq:theta as sum of gradients 1}\\
  \theta_{s,a}^{(2)}=&\eta^{(2)}\nabla_{s,a}^{(1)}+\eta^{(1)}\nabla_{s,a}^{(0)}+\theta_{s,a}^{(0)}\label{eq:theta as sum of gradients 2}\\
\theta_{s,a}^{(3)}=&\eta^{(3)}\nabla_{s,a}^{(2)}+\eta^{(2)}(1+\mathbb{I}\left\{V^{\pi_{\varphi}^{(2)}}(\mu) \ge V^{\pi_{\theta}^{(2)}}(\mu)\right \} \cdot \frac{1}{4})\nabla_{s,a}^{(1)}+\eta^{(1)}\nabla_{s,a}^{(0)}+\theta_{s,a}^{(0)}\label{eq:theta as sum of gradients 3}\\  
\theta_{s,a}^{(4)}=&\eta^{(4)}\nabla_{s,a}^{(3)} \nonumber \\
+&\eta^{(3)}\Big(1+\mathbb{I}\left\{V^{\pi_{\varphi}^{(3)}}(\mu) \ge V^{\pi_{\theta}^{(3)}}(\mu)\right \} \cdot\frac{2}{5} \Big) \nabla_{s,a}^{(2)} \nonumber \\
+&\eta^{(2)} \Big( 1+\mathbb{I}\left\{V^{\pi_{\varphi}^{(2)}}(\mu) \ge V^{\pi_{\theta}^{(2)}}(\mu)\right \} 
\Big( \frac{1}{4} + \mathbb{I}\left\{V^{\pi_{\varphi}^{(3)}}(\mu) \ge V^{\pi_{\theta}^{(3)}}(\mu)\right \} \cdot\frac{2\cdot 1}{5\cdot 4} \Big) \Big) \nabla_{s,a}^{(1)} \nonumber \\
+&\eta^{(1)}\nabla_{s,a}^{(0)} \nonumber \\
+&\theta_{s,a}^{(0)}\label{eq:theta as sum of gradients 4}
    \end{align}
    
    \textbf{b)} For $t\geq 4$, we have
    \begin{align}
  \theta_{s,a}^{(t+1)}
  =&\eta^{(t+1)}\nabla_{s,a}^{(t)} \nonumber \\
  +& \eta^{(t)}\Big(1+\mathbb{I}\left\{V^{\pi_{\varphi}^{(t)}}(\mu) \ge V^{\pi_{\theta}^{(t)}}(\mu)\right \} \cdot\frac{t-1}{t+2}\Big)\nabla^{(t-1)}_{s,a} \nonumber \\
  +& \eta^{(t-1)}\Big(1+\mathbb{I}\left\{V^{\pi_{\varphi}^{(t-1)}}(\mu) \ge V^{\pi_{\theta}^{(t-1)}}(\mu)\right \} \cdot \nonumber \\
  &\quad\quad\quad\enspace\Big( \frac{t-2}{t+1}+\mathbb{I}\left\{V^{\pi_{\varphi}^{(t)}}(\mu) \ge V^{\pi_{\theta}^{(t)}}(\mu)\right \} \cdot\frac{(t-1)(t-2)}{(t+2)(t+1)}\Big)\Big)\nabla^{(t-2)}_{s,a} \nonumber \\
  +&\sum_{j=1}^{t-3}\eta^{(j+1)}\Big(1+\mathbb{I}\left\{V^{\pi_{\varphi}^{(j+1)}}(\mu) \ge V^{\pi_{\theta}^{(j+1)}}(\mu)\right \} \cdot \nonumber \\
  &\quad\quad\enspace \Big( \frac{j}{j+3} +\mathbb{I}\left\{V^{\pi_{\varphi}^{(j+2)}}(\mu) \ge V^{\pi_{\theta}^{(j+2)}}(\mu)\right \} \cdot \nonumber \\
  &\quad\quad\enspace \Big( \frac{(j+1)j}{(j+4)(j+3)} +\mathbb{I}\left\{V^{\pi_{\varphi}^{(j+3)}}(\mu) \ge V^{\pi_{\theta}^{(j+3)}}(\mu)\right \} \cdot \nonumber \\
  & \quad\quad\enspace \Big( \frac{(j+2)(j+1)j}{(j+5)(j+4)(j+3)} + \sum_{k=4}^{t-j} \mathbb{I}\left\{V^{\pi_{\varphi}^{(j+k)}}(\mu) \ge V^{\pi_{\theta}^{(j+k)}}(\mu)\right \} \cdot \nonumber \\
  & \quad\quad\enspace\frac{(j+2)(j+1)j}{(j+k+2)(j+k+1)(j+k)} \Big)\Big)\Big)\Big)\nabla_{s,a}^{(j)} \nonumber \\
+&\eta^{(1)}\nabla_{s,a}^{(0)} \nonumber \\
+&\theta_{s,a}^{(0)}. \label{eq:theta as sum of gradients 5}
    \end{align} 

{Therefore, in summary, we have
\begin{equation}
	\theta_{s,a}^{(t+1)} ={\sum_{j=0}^{t}G(j,t)\cdot
	\eta^{(j+1)}\nabla_{s,a}^{(j)}+\theta_{s,a}^{(0)}}, \quad \forall
	t \geq 0.
	\label{eq:theta as sum of gradients 7}
\end{equation}}
\end{lemma}
\begin{proof}[Proof of Lemma \ref{lemma:theta as sum of gradients}]
    Regarding \text{a)}, one could verify \cref{eq:theta as sum of gradients 1}-\cref{eq:theta as sum of gradients 4} by directly using the APG update in Algorithm \ref{algorithm:APG}.

	For \text{b)}, we prove this by induction. Specifically, suppose \cref{eq:theta as sum of gradients 5} holds for all iterations up to $t$. By the APG update, we know
    \begin{equation}
        \theta_{s,a}^{(t+1)}= \theta_{s,a}^{(t)}+\eta^{(t+1)}\nabla_{s,a}^{(t)}+\mathbb{I}\left\{V^{\pi_{\varphi}^{(t)}}(\mu) \ge V^{\pi_{\theta}^{(t)}}(\mu)\right \} \frac{t-1}{t+2}(\theta_{s,a}^{(t)}-\theta_{s,a}^{(t-1)}).\label{eq:theta as sum of gradients 8}
    \end{equation}
    By plugging into \cref{eq:theta as sum of gradients 8} the
	expressions of $\theta_{s,a}^{(t)}$ and $\theta_{s,a}^{(t-1)}$
	from \cref{eq:theta as sum of gradients 5},
	we could see that \cref{eq:theta as sum of
	gradients 5} continues to hold at the $(t+1)$-th iteration.
\end{proof}

\begin{lemma}
\label{lemma:sum conservation}
Under APG for softmax parameterization, for any iteration $k$ and any state-action pair $(s,a)$, we have 
\begin{equation*}
\sum_{a\in \cA}\theta_{s,a}^{(k)}=\sum_{a\in \cA}\theta_{s,a}^{(0)}.    
\end{equation*}
\end{lemma}
\begin{proof}[Proof of Lemma \ref{lemma:sum conservation}]
    We prove this by induction through the following two claims.

\textbf{Claim (a)}. $\sum_{a\in \cA}\theta^{(1)}_{s,a}=\sum_{a\in \cA}\theta^{(0)}_{s,a}$ and $\sum_{a\in \cA}\theta^{(2)}_{s,a}=\sum_{a\in \cA}\theta^{(0)}_{s,a}$.

Note that under APG, we have
\begin{equation}
    \theta^{(1)}_{s,a}=\omega^{(0)}_{s,a}+\eta^{(1)} \frac{\partial V^{\pi_\theta}(\mu)}{\partial \theta_{s,a}}\Big\rvert_{\theta=\omega^{(0)}}=\theta^{(0)}_{s,a}+\eta^{(1)} \cdot\frac{1}{1-\gamma}d^{\pi_\theta^{(0)}}(s)\pi_{\theta}^{(0)}(a\rvert s)A^{\pi_\theta^{(0)}}(s,a),\label{eq:theta conservation eq1}
\end{equation}
where the second equality holds by the initial condition of APG (i.e., $\omega^{(0)}=\theta^{(0)}$) as well as the softmax policy gradient in Lemma \ref{lemma:softmax_pg}.
By summing \cref{eq:theta conservation eq1} over all the actions, we
have $\sum_{a\in \cA}\theta^{(1)}_{s,a}=\sum_{a\in
\cA}\theta^{(0)}_{s,a}$ from \Cref{lemma:sum_pi_A}.
Similarly, we have
\begin{align}
    \theta^{(2)}_{s,a}&=\omega^{(1)}_{s,a}+\eta^{(2)}\cdot \frac{\partial V^{\pi_\theta}(\mu)}{\partial \theta_{s,a}}\Big\rvert_{\theta=\omega^{(1)}} \nonumber \\
    &=\theta^{(1)}_{s,a}+\mathbb{I} \left\{V^{\pi_{\varphi}^{(1)}}(\mu) \ge V^{\pi_{\theta}^{(1)}}(\mu)\right \} \cdot \frac{0}{3}\cdot(\theta^{(1)}_{s,a}-\theta^{(0)}_{s,a})
    +\eta^{(2)} \cdot\frac{1}{1-\gamma}d^{\pi_\omega^{(1)}}(s)\pi_{\omega}^{(1)}(a\rvert s)A^{\pi_\omega^{(1)}}(s,a).\label{eq:theta conservation eq3}
\end{align}
By taking the sum of \cref{eq:theta conservation eq3} over all the
actions, we have $\sum_{a\in \cA}\theta^{(2)}_{s,a}=\sum_{a\in
\cA}\theta^{(0)}_{s,a}$ by $\sum_{a\in
\cA}\theta^{(1)}_{s,a}=\sum_{a\in \cA}\theta^{(0)}_{s,a}$ and
\Cref{lemma:sum_pi_A}.

\textbf{Claim (b).} If $\sum_{a\in \cA}\theta^{(k)}_{s,a}=\sum_{a\in \cA}\theta^{(0)}_{s,a}$ for all $k\in \{1,\cdots, M\}$, then $\sum_{a\in \cA}\theta^{(M+1)}_{s,a}=\sum_{a\in \cA}\theta^{(0)}_{s,a}$.

We use an argument similar to \cref{eq:theta conservation eq3} as follows.
\begin{align}
    \theta^{(M+1)}_{s,a}&=\omega^{(M)}_{s,a}+\eta^{(M+1)}\cdot \frac{\partial V^{\pi_\theta}(\mu)}{\partial \theta_{s,a}}\Big\rvert_{\theta=\omega^{(M)}} \nonumber \\
	&\begin{aligned}
    =&\theta^{(M)}_{s,a}+\mathbb{I} \left\{V^{\pi_{\varphi}^{(M)}}(\mu) \ge V^{\pi_{\theta}^{(M)}}(\mu)\right \} \cdot\frac{M-1}{M+2}\cdot(\theta^{(M)}_{s,a}-\theta^{(M-1)}_{s,a}) \\
    &\quad +\eta^{(M+1)} \cdot\frac{1}{1-\gamma}d^{\pi_\omega^{(M)}}(s)\pi_{\omega}^{(M)}(a\rvert s)A^{\pi_\omega^{(M)}}(s,a).
\end{aligned}
	\label{eq:theta conservation eq5}
\end{align}
By taking the sum of \cref{eq:theta conservation eq5} over $a \in
\mathcal{A}$, we see that $\sum_{a\in
\cA}\theta^{(M+1)}_{s,a}=\sum_{a\in \cA}\theta^{(0)}_{s,a}$.
\end{proof}

\subsection{\texorpdfstring{$O(1 / t^{2})$}{} Convergence Rate Under Nearly Concave Objectives}
\label{app:sup:stationary}


We present several theoretical results adapted from
\citep{ghadimi2016accelerated} to our nearly concave regime in this
subsection.

\begin{theorem}[\textbf{Theorem 1(b) in \citep{ghadimi2016accelerated} with a slight modification}]
\label{theorem:ghadimi_thm1}
Let $\left \{ \theta^{(t)}_{md}, \theta^{(t)}_{ag} \right \}_{t \ge
1}$ be computed by \Cref{algorithm:Ghadimi_AG} for softmax
parameterization and $\Gamma_t$ be defined as:
\begin{align}
    \Gamma^{(t)} \coloneqq  \left\{\begin{matrix}
    1,  & \text{if } \ t = 1, \\
    (1-\alpha^{(t)})\Gamma^{(t-1)},  & \text{if } \ t \ge 2.
    \end{matrix}\right.
    \label{eq:ghadimi_thm1_0}
\end{align}
Assume that the objective function $V^{\pi_\theta}(\mu)$ is $C$-nearly
concave at $\theta_{md}^{(t)}$ along the direction $\theta_{ag}^{(t)} - \theta_{md}^{(t)}$
with a constant $C \in (1, \frac{3}{2}]$ for every $t \ge 1$.
Suppose $\theta^{(t)}_{ag} - \theta_{md}^{(t)} \in \cU$ for every $t \ge 1$.
If ${\alpha^{(t)}}, {\beta^{(t)}}, {\lambda^{(t)}}$ are chosen such that
\begin{align}
    0 < \alpha^{(t)} \lambda^{(t)} \le \beta^{(t)} <
	\frac{2-C}{L},\quad
    \frac{\alpha^{(1)}}{\lambda^{(1)}\Gamma^{(1)}} \ge \frac{\alpha^{(2)}}{\lambda^{(2)}\Gamma^{(2)}} \ge \cdots,
    \label{eq:ghadimi_thm1_1}
\end{align}
where $L$ is the Lipschitz constant of the gradient of the objective.
Then given any $t \ge 1$, and for any $\theta^{**} = \theta^{(t)}_{md} + \bm{d}$ with some $\bm{d}\in\cU$, we have
\begin{align}
    V^{\pi_{\theta^{**}}}(\mu) - V^{\pi_{\theta_{ag}}^{(t)}}(\mu) \le \Gamma^{(t)}  \frac{\alpha^{(1)}\left \| \theta^{**} - \theta_{ag}^{(0)} + \bar{m}_{\text{init}}  \right \|^2}{\lambda^{(1)}} \label{eq:ghadimi_thm1_3} 
\end{align}
where $\bar{m}_{\text{init}},\theta_{ag}^{(0)} \in \mathbb{R}^n$ are input variables of \Cref{algorithm:Ghadimi_AG}.
\end{theorem}

\begin{proof}[Proof of \Cref{theorem:ghadimi_thm1}]  
First, by \Cref{lemma:mdp_smoothness} and \cref{eq:Ghadimi_AG_3}, we have
\begin{align}
    -V^{\pi_{\theta_{ag}}^{(t)}}(\mu)
    &\le -V^{\pi_{\theta_{md}}^{(t)}}(\mu) - \left \langle \nabla_{\theta}{V^{\pi_{\theta}}(\mu)} \Big\rvert_{\theta = \theta_{md}^{(t)}}, \theta_{ag}^{(t)} - \theta_{md}^{(t)} \right \rangle + \frac{L}{2}\left \| \theta_{ag}^{(t)} - \theta_{md}^{(t)} \right \|^2
	\nonumber\\
    &= -V^{\pi_{\theta_{md}}^{(t)}}(\mu) - \beta^{(t)} \left \|
	\nabla_{\theta}{V^{\pi_{\theta}}(\mu)} \Big\rvert_{\theta =
		\theta_{md}^{(t)}} \right \|^2 + \frac{L{(\beta^{(t)}})^2}{2} \left \| \nabla_{\theta}{V^{\pi_{\theta}}(\mu)} \Big\rvert_{\theta = \theta_{md}^{(t)}} \right \|^2. \label{eq:ghadimi_thm1_4}
\end{align}

Therefore, by the near-concavity of the objective (from our assumption) and \cref{eq:Ghadimi_AG_1}, we have
\begin{align}
	\nonumber
    &~-V^{\pi_{\theta_{md}}^{(t)}}(\mu) + \left[ (1-\alpha^{(t)}) V^{\pi_{\theta_{ag}}^{(t-1)}}(\mu) + \alpha^{(t)} V^{\pi_{\theta^{**}}}(\mu) \right] \\
	\nonumber
    =&~ \alpha^{(t)} \left[ V^{\pi_{\theta^{**}}}(\mu) - V^{\pi_{\theta_{md}}^{(t)}}(\mu) \right] + (1-\alpha^{(t)}) \left[ V^{\pi_{\theta_{ag}}^{(t-1)}}(\mu) - V^{\pi_{\theta_{md}}^{(t)}}(\mu) \right] \\
	\nonumber
     \le&~ C \cdot \alpha^{(t)} \left \langle \nabla_{\theta}{V^{\pi_{\theta}}(\mu)} \Big\rvert_{\theta = \theta_{md}^{(t)}}, \theta^{**} - \theta_{md}^{(t)} \right \rangle
    + C \cdot (1-\alpha^{(t)}) \left \langle \nabla_{\theta}{V^{\pi_{\theta}}(\mu)} \Big\rvert_{\theta = \theta_{md}^{(t)}}, \theta_{ag}^{(t-1)} - \theta_{md}^{(t)} \right \rangle \\
	\nonumber
    =&~ C \cdot \left \langle \nabla_{\theta}{V^{\pi_{\theta}}(\mu)} \Big\rvert_{\theta = \theta_{md}^{(t)}}, \alpha^{(t)}(\theta^{**} - \theta_{md}^{(t)}) + (1-\alpha^{(t)})(\theta_{ag}^{(t-1)} - \theta_{md}^{(t)}) \right \rangle \\
    =&~ C \cdot \alpha^{(t)}\left \langle  \nabla_{\theta}{V^{\pi_{\theta}}(\mu)} \Big\rvert_{\theta = \theta_{md}^{(t)}}, \theta^{**} - \hat{\theta}^{(t-1)} \right \rangle 
	\label{eq:ghadimi_thm1_45}
\end{align}

By \cref{eq:Ghadimi_AG_2}, we have
\begin{align*}
    &~\left \| \theta^{**} - \hat{\theta}^{(t-1)} \right \|^2 - 2
	\lambda^{(t)} \left \langle \nabla_{\theta}{V^{\pi_{\theta}}(\mu)}
	\Big\rvert_{\theta = \theta_{md}^{(t)}}, \theta^{**} -
	\hat{\theta}^{(t-1)} \right \rangle + \left({\lambda^{(t)}}\right)^2 \left \| \nabla_{\theta}{V^{\pi_{\theta}}(\mu)} \Big\rvert_{\theta = \theta_{md}^{(t)}} \right \|^2 \\
    =&~ \left \| \theta^{**} - \hat{\theta}^{(t-1)} -
	\lambda^{(t)}\nabla_{\theta}{V^{\pi_{\theta}}(\mu)}
	\Big\rvert_{\theta = \theta_{md}^{(t)}} \right \|^2 \\
	=&~ \left \| \theta^{**} - \hat{\theta}^{(t)} \right \|^2,
\end{align*}
which directly leads to
\begin{align}
	\nonumber
    &~\alpha^{(t)} \left \langle
	\nabla_{\theta}{V^{\pi_{\theta}}(\mu)} \Big\rvert_{\theta =
		\theta_{md}^{(t)}}, \theta^{**} - \hat{\theta}^{(t-1)} \right
		\rangle \\
    =&~ \frac{\alpha^{(t)}}{2 \lambda^{(t)}} \left[ \left \|
	\theta^{**} - \hat{\theta}^{(t-1)}  \right \|^2 - \left \| \theta^{**} -
\hat{\theta}^{(t)} \right \|^2 \right]+
    \frac{\alpha^{(t)}\lambda^{(t)}}{2} \left \| \nabla_{\theta}{V^{\pi_{\theta}}(\mu)} \Big\rvert_{\theta = \theta_{md}^{(t)}} \right \|^2. \label{eq:ghadimi_thm1_5}
\end{align}

Combining \cref{eq:ghadimi_thm1_4}-\cref{eq:ghadimi_thm1_5}, we have
\begin{align*}
    -V^{\pi_{\theta_{ag}}^{(t)}}(\mu)
    &\le -\left[ (1-\alpha^{(t)}) V^{\pi_{\theta_{ag}}^{(t-1)}}(\mu) + \alpha^{(t)} V^{\pi_{\theta^{**}}}(\mu) \right] \\
    &\quad + C \cdot \frac{\alpha^{(t)}}{2 \lambda^{(t)}} \left[ \left \| \theta^{**} - \hat{\theta}^{(t-1)}  \right \|^2 - \left \| \theta^{**} - \hat{\theta}^{(t)} \right \|^2 \right] \\
    &\quad + C \cdot \frac{\alpha^{(t)}\lambda^{(t)}}{2} \left \| \nabla_{\theta}{V^{\pi_{\theta}}(\mu)} \Big\rvert_{\theta = \theta_{md}^{(t)}} \right \|^2 \\
    &\quad - \beta^{(t)} \left \| \nabla_{\theta}{V^{\pi_{\theta}}(\mu)} \Big\rvert_{\theta = \theta_{md}^{(t)}} \right \|^2 \\
    &\quad + \frac{L{(\beta^{(t)}})^2}{2} \left \| \nabla_{\theta}{V^{\pi_{\theta}}(\mu)} \Big\rvert_{\theta = \theta_{md}^{(t)}} \right \|^2 \\
    &\le -\left[ (1-\alpha^{(t)}) V^{\pi_{\theta_{ag}}^{(t-1)}}(\mu) + \alpha^{(t)} V^{\pi_{\theta^{**}}}(\mu) \right] \\
    &\quad + C \cdot \frac{\alpha^{(t)}}{2 \lambda^{(t)}} \left[ \left \| \theta^{**} - \hat{\theta}^{(t-1)}  \right \|^2 - \left \| \theta^{**} - \hat{\theta}^{(t)} \right \|^2 \right] \\
    &\quad - \frac{\beta^{(t)}}{2} \left(2 - C - L \beta^{(t)} \right) \left \| \nabla_{\theta}{V^{\pi_{\theta}}(\mu)} \Big\rvert_{\theta = \theta_{md}^{(t)}} \right \|^2 \\
\end{align*}

Adding $V^{\pi_{\theta^{**}}}(\mu)$ to both sides of the above inequality and using \cref{eq:ghadimi_thm1_0}, we have
\begin{align}
    \frac{V^{\pi_{\theta^{**}}}(\mu) -V^{\pi_{\theta_{ag}}^{(t)}}(\mu)}{\Gamma^{(t)}} 
    &\le C \cdot \sum_{k = 1}^{t} \frac{\alpha^{(k)}}{2 \lambda^{(k)} \Gamma^{(k)}} \left[ \left \| \theta^{**} - \hat{\theta}^{(k-1)}  \right \|^2 - \left \| \theta^{**} - \hat{\theta}^{(k)} \right \|^2 \right] \nonumber \\
    &\quad - \sum_{k = 1}^{t} \frac{\beta^{(k)}}{2\Gamma^{(k)}}
	\underbrace{\left(2 - C - L \beta^{(k)} \right)}_{> 0 \text{, by
		\cref{eq:ghadimi_thm1_1}}} \left \| \nabla_{\theta}{V^{\pi_{\theta}}(\mu)} \Big\rvert_{\theta = \theta_{md}^{(k)}} \right \|^2 \nonumber \\
    &\le \frac{\alpha^{(1)} C \cdot \left \| \theta^{**} - \hat{\theta}^{(0)}  \right \|^2}{2 \lambda^{(1)}} 
    \label{eq:ghadimi_thm1_6} \\
    &\le \frac{\alpha^{(1)}\left \| \theta^{**} - \hat{\theta}^{(0)}  \right \|^2}{\lambda^{(1)}} 
    \label{eq:ghadimi_thm1_7} \\
    &= \frac{\alpha^{(1)}\left \| \theta^{**} - \theta_{ag}^{(0)} + \bar{m}_{\text{init}}  \right \|^2}{\lambda^{(1)}} \nonumber
\end{align}
where \cref{eq:ghadimi_thm1_6} holds by the fact that
\begin{align*}
    \sum_{k = 1}^{t} \frac{\alpha^{(k)}}{\lambda^{(k)} \Gamma^{(k)}} \left[ \left \| \theta^{**} - \hat{\theta}^{(k-1)}  \right \|^2 - \left \| \theta^{**} - \hat{\theta}^{(k)} \right \|^2 \right] \le \frac{\alpha^{(1)} \left \| \theta^{**} - \hat{\theta}^{(0)}  \right \|^2}{ \lambda^{(1)} \Gamma^{(1)}} = \frac{ \alpha^{(1)} \left \| \theta^{**} - \hat{\theta}^{(0)}  \right \|^2}{ \lambda^{(1)}},
\end{align*}
and \cref{eq:ghadimi_thm1_7} is from our assumption of $C \leq \frac{3}{2} < 2$.
Finally, we obtain the desired result by rearranging \cref{eq:ghadimi_thm1_7}.
\end{proof}


\begin{corollary}[\textbf{Corollary 1 in \citep{ghadimi2016accelerated} with a slight modification}]
\label{cor:ghadimi_cor1}
Given $T_{\text{shift}} \in \mathbb{N}$, suppose that $\left \{ \alpha^{(t)} \right \}$, $\left \{
	\beta^{(t)} \right \}$ and $\left \{ \lambda^{(t)} \right \}$ in
	\Cref{algorithm:Ghadimi_AG} for softmax parameterization are set to
\begin{align}
    \alpha^{(t)} &= \frac{2}{(t+1)+T_{\text{shift}}}, \quad
	\beta^{(t)} = \frac{t+T_{\text{shift}}}{(t+1)+T_{\text{shift}}} \cdot \frac{1}{2L}, \quad 
    \lambda^{(t)} = \frac{(t+1)+T_{\text{shift}}}{2} \cdot \beta^{(t)}, \quad
    \text{where } T_{\text{shift}} > 0,
    \label{eq:ghadimi_cor1_1}
\end{align}
where $L$ is the Lipschitz constant of the gradient of the objective.
Assume that $V^{\pi_\theta}(\mu)$ is $\frac{3}{2}$-nearly
concave at $\theta_{md}^{(t)}$ along the direction $\theta_{ag}^{(t)} - \theta_{md}^{(t)}$ for every $t \ge 1$.
Suppose $\theta^{(t)}_{ag} - \theta_{md}^{(t)} \in \cU$ for every $t \ge 1$.
Then given any $t \ge 1$, and for any $\theta^{**} = \theta^{(t)}_{md} + \bm{d}$ with some $\bm{d}\in\cU$, we have
\begin{align*}
    V^{\pi_{\theta^{**}}}(\mu) - V^{\pi_{\theta}^{(t)}}(\mu)
    \le \frac{8L\left \| \theta^{**} - \theta_{ag}^{(0)} + \bar{m}_{\text{init}}\right \|^2}{(t+T_{\text{shift}}+1)(t+T_{\text{shift}})} 
	= O\left(\frac{1}{t^2}\right),
\end{align*}
where $\bar{m}_{\text{init}},\theta_{ag}^{(0)} \in \mathbb{R}^n$ are
input variables of \Cref{algorithm:Ghadimi_AG}.
\end{corollary}

\begin{remark}
	\normalfont In \Cref{theorem:ghadimi_thm1} and \Cref{cor:ghadimi_cor1},
	addition to extending from concavity to near concavity, we
	have also made the following modifications: (i) We have introduced a
	iteration counter $T_{\text{shift}}$, since our objective is not nearly-concave
	initially and the theoretical result needs to be revised to
	account for the shifted initial step size when the nearly-concave
	region is entered. (ii) We have adjusted $\lambda$ from
	$\frac{t}{2}$ to $\frac{t+1}{2}$ and $\beta$ from $\frac{1}{2L}$
	to $\frac{t+T_{\text{shift}}}{(t+1)+T_{\text{shift}}} \cdot \frac{1}{2L}$ to ensure
	applicability of both \Cref{lemma:equivalent_algorithm} and
	\Cref{theorem:ghadimi_thm1}. 
    
\end{remark}

\begin{remark}
    \normalfont \Cref{theorem:ghadimi_thm1} and
	\Cref{cor:ghadimi_cor1} are built upon \textit{local near
	concavity} of the objective function. {In Appendix \ref{app:MDP},
we will show that such local near concavity indeed holds under APG in the MDP setting.}
\end{remark}
\begin{proof}[Proof of \Cref{cor:ghadimi_cor1}]  
We leverage \Cref{theorem:ghadimi_thm1} to reach our desired result.
It suffices to show that the choice of $\left \{\alpha^{(t)},
\lambda^{(t)}, \beta^{(t)} \right \}$ in \cref{eq:ghadimi_cor1_1}
satisfy \cref{eq:ghadimi_thm1_1}. Because $\alpha^{(t)} \cdot
\lambda^{(t)} = \beta^{(t)} < \frac{1}{2L}$, the first part of \cref{eq:ghadimi_thm1_1} holds.
By the definition of $\Gamma^{(t)}$ in \cref{eq:ghadimi_thm1_0}, we have:
\begin{align}
    \Gamma^{(t)} = \frac{(2+T_{\text{shift}})(1+T_{\text{shift}})}{((t+1)+T_{\text{shift}})(t+T_{\text{shift}})}.
	\label{eq:gamma}
\end{align}
Therefore,
\begin{align*}
    \frac{\alpha^{(t)}}{\lambda^{(t)}\Gamma^{(t)}} = \frac{\frac{2}{(t+1)+T_{\text{shift}}}}{\frac{(t+1)+T_{\text{shift}}}{2} \cdot \frac{t+T_{\text{shift}}}{(t+1)+T_{\text{shift}}} \cdot \frac{1}{2L} \cdot \frac{(2+T_{\text{shift}})(1+T_{\text{shift}})}{((t+1)+T_{\text{shift}})(t+T_{\text{shift}})}} = \frac{8L}{(2+T_{\text{shift}})(1+T_{\text{shift}})},
\end{align*}
which satisfies the second condition in \cref{eq:ghadimi_thm1_1}.
We hence reach the desired result by plugging \cref{eq:gamma} and the
choice of $\lambda^{(1)}$ in \cref{eq:ghadimi_cor1_1} into \cref{eq:ghadimi_thm1_3}.
\end{proof}

\endgroup

\newpage
\section{Asymptotic Convergence}
\label{app:asym_conv}



\subsection{Supporting Lemmas for Asymptotic Convergence of APG}
\label{app:asym_conv:lemmas}



\begin{lemma}
\label{lemma: convergence to stationarity}
Under APG for softmax parameterzation with {uniformly randomly initialized} surrogate initial state
distribution $\mu$,
and suppose there are $\eta' \geq \eta > 0$ such that $\eta^{(t)} \in [\eta, \eta']$ for all $t$.
{Then, almost surely,}
the limits of both $V^{\pi_{\omega}^{(t)}}(\mu)$
and $V^{\pi_{\theta}^{(t)}}(\mu)$ when $t$ approaches infinity exist,
$\lim_{t\rightarrow \infty} V^{\pi_{\omega}^{(t)}}(\mu) =\lim_{t\rightarrow \infty} V^{\pi_{\theta}^{(t)}}(\mu)$,
and
\begin{equation}
   \lim_{t\rightarrow \infty} \big\lVert \nabla_\theta
   V^{\pi_{\theta}}(s)\rvert_{\theta=\omega^{(t)}} \big\rVert =0,\quad
   \lim_{t\rightarrow \infty} \big\lVert \nabla_\theta V^{\pi_{\theta}}(s)\rvert_{\theta=\theta^{(t)}} \big\rVert =0.
   \label{eq:grad0}
\end{equation}
\end{lemma}

\begin{proof}[Proof of \Cref{lemma: convergence to stationarity}]
{For any $t\in \mathbb{N}$, we know that $V^{\pi_{\theta}^{(t+1)}}(s) \geq V^{\pi_{\omega}^{(t)}}(s)$ for all $s\in\cS$ by \Cref{lemma:grad update improvement}. Then, since $V^{\pi}(\mu) = \sum_{s} \mu(s) V^{\pi}(s)$, we have $V^{\pi_{\theta}^{(t+1)}}(\mu) \geq V^{\pi_{\omega}^{(t)}}(\mu)$.
This together with \cref{algorithm:eq3} implies that for any $t\in
\mathbb{N}$, we have}
\begin{equation}
   V^{\pi_{\omega}^{(t+1)}}(\mu) \geq V^{\pi_{\theta}^{(t+1)}}(\mu)\geq    V^{\pi_{\omega}^{(t)}}(\mu) \geq V^{\pi_{\theta}^{(t)}}(\mu).\label{eq: V omega >= Vtheta}
\end{equation}
By \cref{eq: V omega >= Vtheta} and that $V^{\pi_\theta}(\mu)\leq
\frac{1}{1-\gamma}$ for all $\theta$ from \Cref{lemma: reward range to
value range},
 we know from the monotone convergence theorem that the
limits of both $V^{\pi_{\omega}^{(t)}}(\mu)$ and $V^{\pi_{\theta}^{(t)}}(\mu)$ exist and 
$\lim_{t\rightarrow \infty} V^{\pi_{\omega}^{(t)}}(\mu) =\lim_{t\rightarrow \infty} V^{\pi_{\theta}^{(t)}}(\mu)$.
{Therefore,
by the update rule \cref{algorithm:eq1} of APG, the
Lipschitz continuity of $\nabla_\theta V^{\pi_\theta}(\mu)$ from
\Cref{lemma:mdp_smoothness}, and our assumption that $\eta^{(t)}$ is
bounded away from zero, we further see
\begin{equation}
   \lim_{t\rightarrow \infty} \big\lVert \nabla_\theta V^{\pi_{\theta}}(\mu)\rvert_{\theta=\omega^{(t)}} \big\rVert =0. \label{eq: staionarity in omega_t}
\end{equation}}
By \cref{eq: staionarity in omega_t}, {application
	of \cref{algorithm:eq1} again, the Lipschitz continuity of
	$\nabla_\theta V^{\pi_\theta}(\mu)$, and the boundedness
	of $\{\eta^{(t)}\}$},
we also have
\begin{equation}
   \lim_{t\rightarrow \infty} \big\lVert \nabla_\theta V^{\pi_{\theta}}(\mu)\rvert_{\theta=\theta^{(t)}} \big\rVert =0. \label{eq: staionarity in theta_t}
\end{equation}
By the uniformly random initialization of $\mu$, we know that $\min_{s\in\cS}\mu(s) > 0$ almost surely.
Thus, the expression of softmax policy gradient in
\Cref{lemma:softmax_pg}, \cref{eq: staionarity in omega_t}, \cref{eq:
staionarity in theta_t}, and
\Cref{lemma:lower_bound_of_state_visitation_distribution} imply that
with probability one, for all $(s,a)$,
\begin{align}
    &\pi_{\omega}^{(t)}(a|s)A^{\pi_{\omega}^{(t)}}(s,a)\rightarrow 0, \quad \text{as } t\rightarrow \infty,\label{eq:sum of Is+ and Is- eq1} \\
    &{\pi_{\theta}^{(t)}(a|s)A^{\pi_{\theta}^{(t)}}(s,a)\rightarrow 0, \quad \text{as } t\rightarrow \infty,}\label{eq:pi A theta goes to 0} 
\end{align}
proving \cref{eq:grad0}.
\end{proof}

In the proofs of the subsequent lemmas (from \Cref{lemma: monotone limits exist} to \Cref{lemma:compare pi(a+) and pi(a-)}), the upper bound for the step size is needed only for leveraging the result in \Cref{lemma: convergence to stationarity}.
Other than that,
these proofs only require the step sizes $\{\eta^{(t)}\}$ to be lower-bounded away from zero, but do not require an upper bound as in \Cref{lemma: convergence to stationarity}. We will demonstrate in \Cref{app:add-adaptive-apg} that \cref{eq:grad0} still holds under a time-varying step size setting, and therefore relaxing the requirement for an upper bound for the step sizes under such settings.

\begin{lemma}
\label{lemma: monotone limits exist}
Under the setting of \Cref{lemma: convergence to stationarity},
the following statements hold almost surely:
\begin{itemize}
    \item The limits $\lim_{t\rightarrow \infty}V^{\pi_\omega^{(t)}}(s)$, $\lim_{t\rightarrow \infty}Q^{\pi_\omega^{(t)}}(s,a)$, and $\lim_{t\rightarrow \infty}A^{\pi_\omega^{(t)}}(s,a)$ all exist for all $s\in \cS$ and all $a \in \cA$.
    \item The limits $\lim_{t\rightarrow \infty}V^{\pi_\theta^{(t)}}(s)$, $\lim_{t\rightarrow \infty}Q^{\pi_\theta^{(t)}}(s,a)$, and $\lim_{t\rightarrow \infty}A^{\pi_\theta^{(t)}}(s,a)$ all exist for all $s\in \cS$ and all $a \in \cA$.
\end{itemize}
Furthermore, we have
\begin{equation*}
    \lim_{t\rightarrow \infty}V^{\pi_\omega^{(t)}}(s) = \lim_{t\rightarrow \infty}V^{\pi_\theta^{(t)}}(s); \quad \lim_{t\rightarrow \infty}Q^{\pi_\omega^{(t)}}(s,a) = \lim_{t\rightarrow \infty}Q^{\pi_\theta^{(t)}}(s,a); \quad \lim_{t\rightarrow \infty}A^{\pi_\omega^{(t)}}(s,a) = \lim_{t\rightarrow \infty}A^{\pi_\theta^{(t)}}(s,a)
\end{equation*}
for all $s\in\cS$ and all $a\in\cA$ when these limits exist when these
limits exist.

\end{lemma}
\begin{proof}[Proof of \Cref{lemma: monotone limits exist}]
Recall that we use $\bm{V}^{\pi}$ to denote the $\lvert
\cS\rvert$-dimensional vector whose $s$-th component is $V^\pi(s)$ for
all $s\in \cS$. Our first goal is to show that the limit of the value
vector $\bm{V}^{\pi_{\omega}^{(t)}}$ exists, \textit{i.e.},
$\lim_{t\rightarrow\infty} V^{\pi_{\omega}^{(t)}}(s)$ exist for all
$s\in\cS$.
We define
\begin{align}
    \mathcal{V} \coloneqq \Big\{ \bm{V}^{\pi} : \big\lvert \pi(a | s)
	A^{\pi}(s, a)\big\rvert = 0 , \forall (s, a) \in \cS\times\cA \Big\}. \label{def:mathcal{V}}
\end{align}
Before starting our proof, we make the following claim.
\begin{claim}
	\label{claim}
	The cardinality of $\mathcal{V}$ is at most $\lvert \cA\rvert^{\lvert \cS\rvert}$.
\end{claim}
To prove \Cref{claim}, we first observe that there are $\lvert
\cA\rvert^{\lvert \cS\rvert}$ deterministic policies. {In addition, by
\Cref{lemma:sum_pi_A}, we know that the value vector of any deterministic policy belongs to $\mathcal{V}$.
For any stochastic policy $\pi$ such that $\bm{V}^{\pi} \in
\mathcal{V}$, by \cref{def:mathcal{V}}, it follows that for any
$(s, a)$ satisfying the condition $\pi(a|s) > 0$, we have $A^{\pi}(s, a)
= 0$. Consequently, selecting the deterministic policy $\pi'$
associated with $\pi$, wherein for all states $s \in \cS$ we set
$\pi'(a|s) = 1$ for some action $a$ satisfying $\pi(a|s) > 0$, allows
us to confirm that $\bm{V}^{\pi} = \bm{V}^{\pi'}$ using
\Cref{lemma:perf_diff}. Hence, for every value vector $\bm{V}^{\pi}$
in $\mathcal{V}$, there exists at least one corresponding
deterministic policy. Since there are only $\lvert
\cA\rvert^{\lvert \cS\rvert}$ deterministic policies,  the number of
distinct value vectors in $\mathcal{V}$ is at most $\lvert
\cA\rvert^{\lvert \cS\rvert}$, establishing the claim.}

With \Cref{claim} proven, we are now ready to prove \Cref{lemma:
monotone limits exist}.
For any point $x$ and any set $Y$,
we define 
\[
	d(x, Y) \coloneqq \inf_{y \in Y} \lVert y - x \rVert_{1}
\]
as the distance between $x$ and $Y$ measured by the $L_1$-norm.
We first show that 
\begin{equation}
	\label{eq:conv}
	\lim_{t\rightarrow\infty} d(\bm{V}^{\pi_{\omega}^{(t)}}, \cV) =
	0.
\end{equation}

Given any $\varepsilon > 0$, let $\varepsilon' = \frac{(1-\gamma)\varepsilon}{|\cS|^2}$.
By \cref{eq:sum of Is+ and Is- eq1}, we know that there exists
$T \geq 0$ such that
\begin{equation}
	\label{eq:piA}
	\lvert
	\pi_{\omega}^{(t)}(a|s)A^{\pi_{\omega}^{(t)}}(s,a)\rvert <
	\frac{\varepsilon'}{ |\cA|},\quad \forall (s,a) \in \cS \times \cA, \quad
	\forall t \geq T.
\end{equation}
Since $\sum_{a\in\cA} \pi(a|s) = 1$ for any $s \in \cS$ and any policy
$\pi$, there must be at least one action at each state satisfying
$\pi_{\omega}^{(t)}(a|s) \ge 1 / |\cA|$. We can thus observe from
\cref{eq:piA} that for
any $t \ge T$ and any $s\in\cS$, there must be an action $a_{s}^{(t)}$ such that $\lvert A^{\pi_{\omega}^{(t)}}(s,a_{s}^{(t)})\rvert < \varepsilon'$.
Therefore, for every $t \ge T$, we pick $\bm{V}^{(t)}
\in \cV$ as the value vector corresponding to the deterministic policy
$\pi^{(t)}$ with $\pi^{(t)}(a_s^{(t)} | s) = 1$ and $\pi^{(t)}(a|s) = 0$ for all $a\neq a_s^{(t)}$ for all $s\in\cS$.
(Note that the value vector of any deterministic policy belongs to
$\mathcal{V}$.)
Then, we can see that for every $t \ge T$,
\begin{align}
    d(\bm{V}^{\pi_{\omega}^{(t)}}, \cV) &\le
    \lVert \bm{V}^{(t)} - \bm{V}^{\pi_{\omega}^{(t)}} \rVert_{1} \nonumber \\
    &= \sum_{s} \Big\lvert \frac{1}{1-\gamma} \sum_{s'} d_{s}^{\pi^{(t)}}(s') \sum_{a} \pi^{(t)}(a | s') A^{\pi_{\omega}^{(t)}}(s', a)\Big\rvert \label{eq: norm convergence eq1} \\
    &= \sum_{s} \Big\lvert \frac{1}{1-\gamma} \sum_{s'}
	\underbrace{d_{s}^{\pi^{(t)}}(s')}_{\in [0, 1]}
	A^{\pi_{\omega}^{(t)}}(s', a_{{s'}}^{(t)})\Big\rvert \label{eq: norm convergence eq2} \\
    &\le \sum_{s} \frac{1}{1-\gamma} \sum_{s'} \Big\lvert
	A^{\pi_{\omega}^{(t)}}(s', a_{{s'}}^{(t)})\Big\rvert \nonumber \\
    &=\frac{|\cS|}{1-\gamma} \sum_{s'} \Big\lvert
	A^{\pi_{\omega}^{(t)}}(s', a_{{s'}}^{(t)})\Big\rvert \nonumber \\
    &< \frac{|\cS|^2}{1-\gamma} \varepsilon' = \varepsilon, \nonumber
\end{align}
where \cref{eq: norm convergence eq1} holds by \Cref{lemma:perf_diff}
and \cref{eq: norm convergence eq2} is due to the fact that
$\pi^{(t)}(a_s^{(t)} | s) = 1$. Since $\varepsilon$ is arbitrary, we have
proven \cref{eq:conv}.

Accordingly, since $\cV$ is a finite and bounded (and thus compact) set according to \Cref{claim},
there are two possible scenarios:
(i) There exists $\bm{V} \in \cV$ such that
$\lim_{t \rightarrow \infty} 
\bm{V}^{\pi_{\omega}^{(t)}}=\bm{V}$; (ii)
$\big\{\bm{V}^{\pi_{\omega}^{(t)}}\big\}_{t=1}^{\infty}$ has at least
two limit points in $\mathcal{V}$.

For (i), we obtain the desired results. Regarding (ii), we prove that
with probability one, this scenario will not occur.
{For any two distinct value vectors $\bm{V}^{\pi_1}\neq \bm{V}^{\pi_2}$
in $\mathcal{V}$, we know that the set of $\mu$ such that
$V^{\pi_1}(\mu) = V^{\pi_2}(\mu)$ is a subset in $\Delta(\cS)$ with dimension strictly less than the dimension of $\Delta(\cS)$ due
to the given hypothesis $\bm{V}^{\pi_1}\neq
\bm{V}^{\pi_2}$.
Thus, the set of $\mu$ such that $V^{\pi_1}(\mu) = V^{\pi_2}(\mu)$ is
of measure zero in $\Delta(\cS)$. Additionally, since there are only
finitely many distinct value vectors in $\mathcal{V}$ by \Cref{claim}, the union of those sets of $\mu$ is still of measure zero in $\Delta(\cS)$. Hence, by the uniformly random initialization of the surrogate initial state distribution $\mu$, we have that with probability one, any two distinct value vectors $\bm{V}^{\pi_1}\neq \bm{V}^{\pi_2}$ in $\mathcal{V}$ satisfy
$V^{\pi_1}(\mu) \neq V^{\pi_2}(\mu)$.}

Let $\bm{V}_1$ and $\bm{V}_2$ be two distinct limit points in
in scenario (ii), and $\{i_t\}$ and $\{j_t\}$ be two infinite index
sequences such that
\begin{equation*}
	\lim_{t \rightarrow \infty} \bm{V}^{\pi_{\omega}^{(i_t)}} = \bm{V}_1, \quad
	\lim_{t \rightarrow \infty} \bm{V}^{\pi_{\omega}^{(j_t)}} = \bm{V}_2,
\end{equation*}
which implies
\begin{equation}
	\label{eq:limit}
	\lim_{t \rightarrow \infty} V^{\pi_{\omega}^{(i_t)}}(\mu) =
	V_1(\mu), \quad
	\lim_{t \rightarrow \infty} V^{\pi_{\omega}^{(j_t)}}(\mu) =
	V_2(\mu).
\end{equation}

We have argued that with probability one, $V_1(\mu) \neq V_2(\mu)$,
and let us assume without loss of generality that $V_1(\mu) >
V_2(\mu)$ and denote $\epsilon \coloneqq V_1(\mu) - V_2(\mu)$.
From \cref{eq:limit}, almost surely there is $T_1 \geq 0$ such that
\begin{equation}
V^{\pi_{\omega}^{(i_t)}}(\mu) > V_1(\mu) - \frac{\epsilon}{2},\quad \forall
t \geq T_1.
\label{eq:Vvalue}
\end{equation}
By \cref{eq: V omega >= Vtheta}, we see that
$V^{\pi_{\omega}^{(t)}}(\mu)$ is monotonically increasing, which
together with \cref{eq:Vvalue} implies that
\[
V^{\pi_{\omega}^{(t)}}(\mu) \geq V_2 (\mu) + \frac{\epsilon}{2},
\quad \forall t \geq i_{T_1},
\]
contradicting the second equality of \cref{eq:limit} and that $\{j_t\}$ should
be an infinite sequence.
Therefore,
$\{\bm{V}^{\pi_{\omega}^{(t)}}\}$ converges to a point.
Since
$\lim_{t\rightarrow\infty} V^{\pi_{\omega}^{(t)}}(s)$ exists for all
$s\in\cS$,
we see that $\lim_{t\rightarrow \infty}Q^{\pi_\omega^{(t)}}(s,a)$ and
therefore $\lim_{t\rightarrow \infty}A^{\pi_\omega^{(t)}}(s,a)$ also
exist for all $s\in \cS$ and all $a \in \cA$ according to \cref{eq:state_value_function}.
Existence of the limit of $\{\bm{V}^{\pi_{\theta}^{(t)}}\}$,
$\{Q^{\pi_\theta^{(t)}}(s,a)\}$, and $\{A^{\pi_\theta^{(t)}}(s,a)\}$
follows the same argument above with \cref{eq:sum of Is+ and Is- eq1}
replaced by \cref{eq:pi A theta goes to 0}.

Finally, we demonstrate that the limits of the value function, the
$Q$-value function, and the advantage function with respect to
$\omega^{(t)}$ and $\theta^{(t)}$ are equal.
It suffices to show that $\lim_{t\rightarrow
\infty}V^{\pi_\omega^{(t)}}(s) = \lim_{t\rightarrow
\infty}V^{\pi_\theta^{(t)}}(s)$ for all $s\in\cS$, \textit{i.e.}, $\lim_{t\rightarrow \infty} \bm{V}^{\pi_{\omega}^{(t)}} = \lim_{t\rightarrow \infty} \bm{V}^{\pi_{\theta}^{(t)}}$.
We will prove it by contradiction.
Let 
\[
	\bm{V}^{\pi_{\omega}^{(\infty)}} \coloneqq \lim_{t\rightarrow
	\infty} \bm{V}^{\pi_{\omega}^{(t)}},\quad  \bm{V}^{\pi_{\theta}^{(\infty)}} \coloneqq \lim_{t\rightarrow \infty} \bm{V}^{\pi_{\theta}^{(t)}}
\]
and suppose that $\bm{V}^{\pi_{\omega}^{(\infty)}} \neq
\bm{V}^{\pi_{\theta}^{(\infty)}}$.
Recall that from our uniformly random initialization of $\mu$, if
$\bm{V}^{\pi_1}\neq \bm{V}^{\pi_2}$ in $\mathcal{V}$, then
$V^{\pi_1}(\mu) \neq V^{\pi_2}(\mu)$ almost surely.
We know that both $\bm{V}^{\pi_{\omega}^{(\infty)}}$ and $\bm{V}^{\pi_{\theta}^{(\infty)}}$ belong to $\mathcal{V}$.
By the hypothesis, with probability one we have
$V^{\pi_{\omega}^{(\infty)}}(\mu) \neq
V^{\pi_{\theta}^{(\infty)}}(\mu)$, leading to the contradiction due to
the fact that $\lim_{t\rightarrow \infty} V^{\pi_{\omega}^{(t)}}(\mu)
=\lim_{t\rightarrow \infty} V^{\pi_{\theta}^{(t)}}(\mu)$ from \cref{eq: V omega >= Vtheta}.
Hence, we obtain that with probability one, $\lim_{t\rightarrow \infty}V^{\pi_\omega^{(t)}}(s) = \lim_{t\rightarrow \infty}V^{\pi_\theta^{(t)}}(s)$ for all $s\in\cS$, and by \cref{eq:state_action_value_function} we also obtain the desired results for both the $Q$-value function and the advantage function.
\qedhere
\end{proof}
In the sequel, we use the following notations.
\begin{equation}
\begin{cases}
A^{(\infty)}(s,a) &\coloneqq
\lim_{t\rightarrow\infty}A^{\pi_{\omega}^{(t)}}(s,a) = \lim_{t\rightarrow\infty}A^{\pi_{\theta}^{(t)}}(s,a),\\
Q^{(\infty)}(s,a) &\coloneqq \lim_{t\rightarrow\infty}Q^{\pi_{\omega}^{(t)}}(s,a) = \lim_{t\rightarrow\infty}Q^{\pi_{\theta}^{(t)}}(s,a),\\
V^{(\infty)}(s) &\coloneqq
\lim_{t\rightarrow\infty}V^{\pi_{\omega}^{(t)}}(s) \quad = \lim_{t\rightarrow\infty}V^{\pi_{\theta}^{(t)}}(s).
\end{cases}
\label{eq:def}
\end{equation}
We divide the action space into the following subsets based on the advantage function:
\begin{align*}
    I_s^{+}&:=\{a\in \cA: A^{(\infty)}(s,a)>0\},\\
    I_s^{-}&:=\{a\in \cA: A^{(\infty)}(s,a)<0\},\\
    I_s^{0}&:=\{a\in \cA: A^{(\infty)}(s,a)=0\}.
\end{align*}
The above action sets are well-defined as the limiting value functions
exist by \Cref{lemma: monotone limits exist}.

For each state $s$, we define 
\begin{equation*}
    \Delta_s:=\min_{a\in I_s^+\cup I_s^-}\lvert A^{(\infty)}(s,a)\rvert.
\end{equation*}
Accordingly, we know that for each state $s\in \cS$, there must exist some $\bar{T}_s$ such that the following hold:
\begin{itemize}[leftmargin=*]
\item (i) For all $a\in I_s^+$,
    \begin{equation}
        A^{\pi_{\omega}^{(t)}}(s,a)\geq +\frac{\Delta_s}{4}, \quad \text{for all } t\geq \bar{T}_s,\label{eq:adv I_s+}
    \end{equation}
\item  (ii) For all $a\in I_s^-$,
\begin{equation}
    A^{\pi_{\omega}^{(t)}}(s,a)\leq -\frac{\Delta_s}{4}, \quad\text{for all }t\geq \bar{T}_s.\label{eq:adv I_s-}
\end{equation}
\item (iii) For all $a\in I_s^0$,
\begin{equation}
    \lvert A^{\pi_{\omega}^{(t)}}(s,a)\lvert \leq \frac{\Delta_s}{4}, \quad\text{for all }t\geq \bar{T}_s.\label{eq:adv I_s0}
\end{equation}
\end{itemize}

\begin{lemma}
\label{lemma:sum of Is+ and Is- goes to zero}
{Under the setting of \Cref{lemma: convergence to stationarity},}
for any state $s\in \cS$,
we have 
\begin{equation*}
	\lim_{t \rightarrow \infty}\,
	\sum_{a\in I_s^{+}\cup I_s^{-}}\pi_\omega^{(t)}(a\rvert
	s)= 0,\quad
	\lim_{t \rightarrow \infty}\,
	\sum_{a\in I_s^{+}\cup I_s^{-}}\pi_\theta^{(t)}(a\rvert
	s)= 0.
\end{equation*}
As a result, we also have
\begin{equation*}
	\lim_{t \rightarrow \infty}\,
	\sum_{a\in I_s^{0}}\pi_\omega^{(t)}(a\rvert
	s)= 1,\quad
	\lim_{t \rightarrow \infty}\,
	\sum_{a\in I_s^{0}}\pi_\theta^{(t)}(a\rvert
	s)= 1.
\end{equation*}
\end{lemma}

\begin{proof}[Proof of \Cref{lemma:sum of Is+ and Is- goes to zero}]
	For any $a\in I_s^{+}$, we have
	$\lim_{t\rightarrow\infty}A^{\pi_{\omega}^{(t)}}(s,a) =
	\lim_{t\rightarrow\infty}A^{\pi_{\theta}^{(t)}}(s,a) > 0$. By
	\cref{eq:sum of Is+ and Is- eq1} and \cref{eq:pi A theta goes to
	0}, this implies that $\pi_\omega^{(t)}(a\rvert s)\rightarrow 0$
	and $\pi_\theta^{(t)}(a\rvert s)\rightarrow 0$ as $t\rightarrow
	\infty$ for all $a\in I_s^{+}$.
    Similarly, for any $a\in I_s^{-}$, we also have $\lim_{t\rightarrow\infty}A^{\pi_{\omega}^{(t)}}(s,a) = \lim_{t\rightarrow\infty}A^{\pi_{\theta}^{(t)}}(s,a) < 0$. Again, by \cref{eq:sum of Is+ and Is- eq1} and \cref{eq:pi A theta goes to 0}, this implies that $\pi_\omega^{(t)}(a\rvert s)\rightarrow 0$ and $\pi_\theta^{(t)}(a\rvert s)\rightarrow 0$ as $t\rightarrow \infty$ for all $a\in I_s^{-}$.
    Hence, we conclude that $\sum_{a\in I_s^{+}\cup I_s^{-}}\pi_\omega^{(t)}(a\rvert s)\rightarrow 0$ and $\sum_{a\in I_s^{+}\cup I_s^{-}}\pi_\theta^{(t)}(a\rvert s)\rightarrow 0$ as $t\rightarrow \infty$ due to the finiteness of the cardinality of $\mathcal{A}$.
\end{proof}

{\begin{lemma}
\label{lemma:Is+ bounded}
    Consider any state $s\in \cS$. Let $a$ be an action in
	$I_{s}^{+}$.
{Under the setting of \Cref{lemma: convergence to stationarity},}
	$\{\theta_{s,a}^{(t)}\}_{t=1}^{\infty}$ and $\{\omega_{s,a}^{(t)}\}_{t=1}^{\infty}$ are bounded from below.
\end{lemma}}
\begin{proof}[Proof of \Cref{lemma:Is+ bounded}]
	Recall the part of our definition in \cref{eq:adv I_s+} of
	$\bar{T}_s$.
	From \Cref{lemma:softmax_pg}, we see that
\begin{equation}
	\eta^{(t)} \nabla_{\theta_{s,a}}{V^{\pi_{\theta}}(\mu)}
	\Big\rvert_{\theta = \omega^{(t-1)}} \geq 0,\quad \forall t \geq
	\bar{T}_s+1.
	\label{eq:posgrad}
\end{equation}
	We let
	$\delta_{\bar{T}_s}:=\theta_{s,a}^{(\bar{T}_s)}-\theta_{s,a}^{(\bar{T}_s-1)}$.
	Regarding the case $\delta_{\bar{T}_s} \geq 0$,
	by \cref{eq:posgrad} and the update rules in
	\cref{algorithm:eq1,algorithm:eq2},
	the sequences $\{\theta_{s,a}^{(t)}\}_{t\geq \bar T_s+1},
	\{\omega_{s,a}^{(t)}\}_{t\geq \bar T_s+1}$ are monotonically
	increasing and therefore lower-bounded.

	On the other hand, if $\delta_{\bar{T}_s} < 0$, for any $M\in \bbN$, we have
    \begin{align}
        \theta^{(\bar{T}_s+M)}_{s,a}
        &=\omega^{(\bar{T}_s+M-1)}_{s,a}+\eta^{(\bar{T}_s+M)} \frac{\partial V^{\pi_\theta}(\mu)}{\partial \theta_{s,a}}\Big\rvert_{\theta=\omega^{(\bar{T}_s+M-1)}} \nonumber \\
        &\geq \theta^{(\bar{T}_s+M-1)}_{s,a}
         +\mathbb{I} \left\{V^{\pi_{\varphi}^{(\bar{T}_s + M -1)}}(\mu) \ge V^{\pi_{\theta}^{(\bar{T}_s + M -1)}}(\mu)\right \} \cdot \frac{\bar{T}_s+M-2}{\bar{T}_s+M+1}(\theta_{s,a}^{(\bar{T}_s+M-1)}-\theta_{s,a}^{(\bar{T}_s+M-2)}) \nonumber \\
        &\geq \theta_{s,a}^{(\bar{T}_s)}+\frac{\bar{T}_s-1}{\bar{T}_s+2}\delta_{\bar{T}_s}+\frac{\bar{T}_s(\bar{T}_s-1)}{(\bar{T}_s+3)(\bar{T}_s+2)}\delta_{\bar{T}_s}+\cdots+\frac{(\bar{T}_s+M-2)\cdots (\bar{T}_s-1)}{(\bar{T}_s+M+1)\cdots (\bar{T}_s+2)}\delta_{\bar{T}_s}\label{eq:Is+ bounded eq3}\\
        &= \theta_{s,a}^{(\bar{T}_s)}+\Big[\frac{\bar{T}_s-1}{\bar{T}_s+2}+\frac{\bar{T}_s(\bar{T}_s-1)}{(\bar{T}_s+3)(\bar{T}_s+2)}+\sum_{\tau=2}^{M}\frac{(\bar{T}_s+1)\bar{T}_s(\bar{T}_s-1)}{(\bar{T}_s+\tau+2)(\bar{T}_s+\tau+1)(\bar{T}_s+\tau)}\Big] \delta_{\bar{T}_s},\nonumber
    \end{align}
    where \cref{eq:Is+ bounded eq3} holds by disregarding the updates
	contributed by the gradients taken at all $t > \bar{T}_s$, which
	are always nonnegative according to \cref{eq:posgrad}.
    
    We note that for any $M\in \bbN$,
    \begin{align}
        &~\sum_{\tau=2}^{M}\frac{(\bar{T}_s+1)\bar{T}_s(\bar{T}_s-1)}{(\bar{T}_s+\tau+2)(\bar{T}_s+\tau+1)(\bar{T}_s+\tau)}\nonumber\\
        =&~(\bar{T}_s+1)\bar{T}_s(\bar{T}_s-1)\sum_{\tau=2}^{M}\frac{1}{2}\Big(\frac{1}{(\bar{T}_s+\tau)(\bar{T}_s+\tau+1)}-\frac{1}{(\bar{T}_s+\tau+1)(\bar{T}_s+\tau+2)}
		\Big)\nonumber\\
        =&~(\bar{T}_s+1)\bar{T}_s(\bar{T}_s-1)\cdot
		\frac{1}{2}\Big(\frac{1}{(\bar{T}_s+2)(\bar{T}_s+3)}-\frac{1}{(\bar{T}_s+M+1)(\bar{T}_s+M+2)}\Big)\nonumber\\
        \leq&~ \frac{\bar{T}_s}{2}.\label{eq:Is+ bounded eq8}
    \end{align}
    Therefore, it follows that for any $M\in \bbN$,
    \begin{align*}
        \theta^{(\bar{T}_s+M)}_{s,a}\geq
		\theta^{(\bar{T}_s)}_{s,a}-\left(2+\frac{\bar{T}_s}{2}\right)\lvert \delta_{\bar{T}_s}\rvert.
    \end{align*}
    Hence, $\theta^{(t)}_{s,a}\geq \theta^{(\bar{T}_s)}_{s,a}-(2+\frac{\bar{T}_s}{2})\lvert \delta_{\bar{T}_s}\rvert $ for all $t\geq \bar{T}_s$.
    {By \cref{eq:posgrad}, if there is some $\bar{T}_s' \ge \bar{T}_s$ such that $\theta^{(t)}_{s, a} - \theta^{(t-1)}_{s, a} \ge 0$ for any $a\in I_{s}^{+}$, we obtain that $\{\omega^{(t)}_{s, a}\}_{t=1}^{\infty}$ is also bounded from below since $a\in I_{s}^{+}$. Otherwise, if there is no such $\bar{T}_s'$, then $\{\theta\}_{t=1}^{\infty}$ is monotonically decreasing after time $\bar{T}_s'$. Thus, we can bound $\{\omega^{(t)}_{s, a}\}_{t=1}^{\infty}$ from below due to the momentum update \cref{algorithm:eq2}, where $\omega^{(t)}_{s, a} \ge 2(\theta^{(\bar{T}_s)}_{s,a}-(2+\frac{\bar{T}_s}{2})\lvert \delta_{\bar{T}_s}\rvert ) - \max_{t \le \bar{T}_s} \theta_{s, a}^{(t)}$ for all $t \ge \bar{T}_s$.
    Hence, we obtain that $\{\omega^{(t)}_{s,a}\}_{t=1}^{\infty}$ is alsobounded from below for all $s\in\cS$ and $a\in I_{s}^{+}$.}
\end{proof}

{\begin{lemma}
\label{lemma:Is- bounded}
    Consider any state $s\in \cS$. Let $a$ be an action in
	$I_{s}^{-}$.
{Under the setting of \Cref{lemma: convergence to stationarity},}
	$\{\theta_{s,a}^{(t)}\}_{t=1}^{\infty}$ and $\{\omega_{s,a}^{(t)}\}_{t=1}^{\infty}$ are bounded from above.
\end{lemma}}
\begin{proof}[Proof of \Cref{lemma:Is- bounded}]
We follow the same procedure as that in
\Cref{lemma:Is+ bounded}.  Again, let $\delta_{\bar{T}_s}:=\theta_{s,a}^{(\bar{T}_s)}-\theta_{s,a}^{(\bar{T}_s-1)}$.
	From \cref{eq:adv I_s+} and \Cref{lemma:softmax_pg}, we see that
\begin{equation}
	\eta^{(t)} \nabla_{\theta_{s,a}}{V^{\pi_{\theta}}(\mu)}
	\Big\rvert_{\theta = \omega^{(t-1)}} \leq 0,\quad \forall t \geq
	\bar{T}_s+1.
	\label{eq:neggrad}
\end{equation}
Therefore, by \cref{algorithm:eq1,algorithm:eq2,eq:neggrad}, we have
\begin{align*}
    \theta^{(\bar{T}_s+1)}_{s,a}&=\omega^{(\bar{T}_s)}_{s,a}+\eta^{(\bar{T}_s+1)}\cdot \frac{\partial V^{\pi_\theta}(\mu)}{\partial \theta_{s,a}}\Big\rvert_{\theta=\omega^{(\bar{T}_s)}}\leq \omega^{(\bar{T}_s)}_{s,a}.
\end{align*}
Regarding the case of $\delta_{\bar{T}_s} \leq 0$, the result directly
holds by \cref{eq:neggrad} as argued in the proof of \Cref{lemma:Is+
bounded}.
Considering $\delta_{\bar{T}_s} > 0$, for any $M\in \bbN$, we have
\begin{align*}
    \theta^{(\bar{T}_s+M)}_{s,a}
    &=\omega^{(\bar{T}_s+M-1)}_{s,a}+\eta^{(\bar{T}_s+M)}\cdot \frac{\partial V^{\pi_\theta}(\mu)}{\partial \theta_{s,a}}\Big\rvert_{\theta=\omega^{(\bar{T}_s+M-1)}}\\
    &\leq \theta^{(\bar{T}_s+M-1)}_{s,a}
    + \mathbb{I} \left\{V^{\pi_{\varphi}^{(T_0 + M -1)}}(\mu) \ge V^{\pi_{\theta}^{(T_0 + M -1)}}(\mu)\right \} \cdot\frac{\bar{T}_s+M-2}{\bar{T}_s+M+1}(\theta_{s,a}^{(\bar{T}_s+M-1)}-\theta_{s,a}^{(\bar{T}_s+M-2)})\\
    &\leq \theta_{s,a}^{(\bar{T}_s)}+\frac{\bar{T}_s-1}{\bar{T}_s+2}\delta_{\bar{T}_s}+\frac{\bar{T}_s(\bar{T}_s-1)}{(\bar{T}_s+3)(\bar{T}_s+2)}\delta_{\bar{T}_s}+\cdots+\frac{(\bar{T}_s+M-2)\cdots (\bar{T}_s-1)}{(\bar{T}_s+M+1)\cdots (\bar{T}_s+2)}\delta_{\bar{T}_s}\\
    &=\theta_{s,a}^{(\bar{T}_s)}+\Big[\frac{\bar{T}_s-1}{\bar{T}_s+2}+\frac{\bar{T}_s(\bar{T}_s-1)}{(\bar{T}_s+3)(\bar{T}_s+2)}+\sum_{\tau=2}^{M}\frac{(\bar{T}_s+1)\bar{T}_s(\bar{T}_s-1)}{(\bar{T}_s+\tau+2)(\bar{T}_s+\tau+1)(\bar{T}_s+\tau)}\Big] \delta_{\bar{T}_s}.
\end{align*}
By \cref{eq:Is+ bounded eq8}, we know $\sum_{\tau=2}^{M}\frac{(\bar{T}_s+1)\bar{T}_s(\bar{T}_s-1)}{(\bar{T}_s+\tau+2)(\bar{T}_s+\tau+1)(\bar{T}_s+\tau)}\leq \frac{\bar{T}_s}{2}$. 
As a result, for any $M\in \bbN$,
\begin{align*}
    \theta^{(\bar{T}_s+M)}_{s,a}\leq \theta^{(\bar{T}_s)}_{s,a}+(2+\frac{\bar{T}_s}{2})\lvert \delta_{\bar{T}_s}\rvert.
\end{align*}
Hence, $\theta^{(t)}_{s,a}\leq \theta^{(\bar{T}_s)}_{s,a}+(2+\frac{\bar{T}_s}{2})\lvert \delta_{\bar{T}_s}\rvert $ for all $t\geq \bar{T}_s$.
{By using the same argument as in \Cref{lemma:Is+ bounded}, we have $\omega_{s, a}^{(t)} \le 2(\theta^{(\bar{T}_s)}_{s,a}+(2+\frac{\bar{T}_s}{2})\lvert \delta_{\bar{T}_s}\rvert ) - \min_{t\le\bar{T}_s}$ for all $s\in I_{s}^{-}$.
Hence, we obtain that $\{\omega^{(t)}_{s,a}\}_{t=1}^{\infty}$ is bounded from above for all $s\in\cS$ and $a\in I_{s}^{+}$.}
\end{proof}

\begin{lemma}
\label{lemma:max Is0}
Consider any state $s\in \cS$.
{Under the setting of \Cref{lemma: convergence to stationarity},}
if $I_s^{+}$ is non-empty, then
	\[
		\lim_{t \rightarrow \infty}\; \max_{a\in I_s^{0}}\,
		\omega_{s,a}^{(t)}= \infty,\quad
		\lim_{t \rightarrow \infty}\;
		\max_{a\in I_s^{0}} \theta_{s,a}^{(t)}= \infty.
	\]
\end{lemma}
\begin{proof}
    By \Cref{lemma:sum of Is+ and Is- goes to zero},
	we know $\sum_{a\in I_s^{0}}\pi_{\omega}^{(t)}(a\rvert
	s)\rightarrow 1$ and $\sum_{a\in I_s^{0}}\pi_{\theta}^{(t)}(a\rvert s)\rightarrow 1$ as $t\rightarrow \infty$.
    Moreover, by \Cref{lemma:Is+ bounded}, we know
	$\{\omega_{s,a}^{(t)}\}$ and $\{\theta_{s,a}^{(t)}\}$ are lower-bounded for all $a\in I_s^{+}$.
    Therefore, under the softmax policy parameterization, we obtain
	the desired results.
\end{proof}

Recall from \cref{eq:adv I_s-} that for all $a\in I_s^-$, we have $A^{\pi_{\omega}^{(t)}}(s,a)\leq -\frac{\Delta_s}{4}$ for all $t\geq \bar{T}_s$.
{\begin{lemma}
\label{lemma:Is- negative infinity}
Consider any state $s\in \cS$.
{Under the setting of \Cref{lemma: convergence to stationarity},}
if $I_s^{+}$ is non-empty, then for any $a\in I_s^{-}$, we have $\omega_{s,a}^{(t)}\rightarrow -\infty$ as $t\rightarrow \infty$.
\end{lemma}}
\begin{proof}[Proof of \Cref{lemma:Is- negative infinity}]
We prove this by contradiction. 
Motivated by the proof of Lemma C11 in \citep{agarwal2021theory}, our
proof here extends their argument to the case with momentum by
considering the cumulative effect of all the gradient terms on the
policy parameter $\theta^{(t)}$.

{We first show that $\lim_{t\rightarrow\infty}\theta^{(t)}_{s, a} = -\infty$.
We prove 
\begin{equation}
	\liminf_{t\rightarrow\infty}\theta^{(t)}_{s, a} = -\infty
	\label{eq:liminf}
\end{equation}
first by contradiction.} Given an action $a\in I_s^{-}$, suppose that there exists $\vartheta$ such that $\theta_{s,a}^{(t)}>\vartheta$ for all $t\geq\bar{T}_s$.
Then, by \Cref{lemma:sum conservation,lemma:max Is0}, we know there must exist an action $a'\in \cA$ such that $\liminf_{t\rightarrow \infty}\theta_{s,a'}^{(t)}=-\infty$.
Let $\delta>0$ be some positive scalar such that $\theta_{s,a}^{(\bar{T}_s)}\geq \vartheta -\delta$. For each $t\geq \bar{T}_s$, define 
\begin{equation}
    \nu(t):=\sup\{\tau:\theta_{s,a'}^{(\tau)}\geq \vartheta-\delta,\quad \bar{T}_s\leq \tau\leq t \},\label{eq:Is- negative infinity 1}
\end{equation}
which is the latest iteration that $\theta_{s,a'}^{(\tau)}$ remains no
smaller than $\vartheta-\delta$.
We consider two cases: (i) $\{\nu(t)\}_{t=1}^{\infty}$ is bounded, and
(ii) $\lim_{t \rightarrow \infty} \nu(t) = \infty$.

For case (i), we define the following index set
\begin{equation}
    \cJ^{(t)}:=\Big\{\tau: \frac{\partial V^{\pi_\theta}(\mu)}{\partial \theta_{s,a'}}\Big\rvert_{\theta=\omega^{(\tau)}}<0,\quad \nu(t)<\tau < t\Big\}.\label{eq:Is- negative infinity 2}
\end{equation}
and denote the cumulative effect (up to iteration $t$) of the gradient terms from those iterations in $\cJ^{(t)}$ as
\begin{equation}
    Z^{(t)}:=\sum_{t'\in \cJ^{(t)}} \eta^{(t'+1)}\cdot \frac{\partial V^{\pi_\theta}(\mu)}{\partial \theta_{s,a'}}\Big\rvert_{\theta=\omega^{(t')}}\cdot G(t',t),\label{eq:Is- negative infinity 3}
\end{equation}
where $G(t',t)$ is defined in \cref{eq:G}.
If $\cJ^{(t)}=\emptyset$, we define $Z^{(t)}=0$.

By setting $R_{\text{max}} = 1$ in \Cref{lemma: reward range to value range}, we have the upper bounds $V^{\pi}(s) \le 1/(1-\gamma)$ for any $s\in\cS$ and any $\pi$, $Q^{\pi}(s, a) \le 1/(1-\gamma)$ for any $(s, a) \in\cS\times\cA$ and any $\pi$, and also $\lvert A(s, a)\rvert \le 1/(1-\gamma)$ for any $(s, a) \in\cS\times\cA$ and any $\pi$.
Leveraging these upper bounds and \Cref{lemma:softmax_pg}, we have
\begin{equation}
	\lvert\partial V^{\pi_{\theta}}(\mu)/\partial\theta_{s,a}\rvert \leq
	1/(1-\gamma)^2,\quad \forall (s, a) \in\cS\times\cA.
	\label{eq:boundV}
\end{equation}
Accordingly, for any $t>\bar{T}_s$, we have
\begin{align}
	\nonumber
    Z^{(t)}
    &\leq \sum_{t'\in \cJ^{(t)}} \eta^{(t'+1)}\cdot \frac{\partial V^{\pi_\theta}(\mu)}{\partial \theta_{s,a'}}\Big\rvert_{\theta=\omega^{(t')}}\cdot
		G(t',t)+\sum_{t':t'\notin \cJ^{(t)}, \nu(t)<t'<t}
		\underbrace{\eta^{(t'+1)}\cdot \frac{\partial V^{\pi_\theta}(\mu)}{\partial \theta_{s,a'}}\Big\rvert_{\theta=\omega^{(t')}}\cdot G(t',t)}_{\geq
				0, \text{by the definition of }\cJ^{(t)}}\\
    &\quad\quad +\sum_{t'\leq\nu(t)}\underbrace{\eta^{(t'+1)}\cdot
	\Big(\frac{\partial V^{\pi_\theta}(\mu)}{\partial \theta_{s,a'}}\Big\rvert_{\theta=\omega^{(t')}}+\frac{1}{(1-\gamma)^2}\Big)\cdot
G(t',t)}_{\geq 0, \text{ by \cref{eq:boundV}
}}\nonumber\\
    &= \sum_{t' \le t} \eta^{(t'+1)}\cdot \frac{\partial V^{\pi_\theta}(\mu)}{\partial \theta_{s,a'}}\Big\rvert_{\theta=\omega^{(t')}}\cdot G(t',t)
		+\sum_{t'\leq\nu(t)}\eta^{(t'+1)}\frac{1}{(1-\gamma)^2}G(t',t)
		\nonumber\\
    &= (\theta_{s,a'}^{(t)}-\theta_{s,a'}^{(1)}) +\sum_{t'\leq\nu(t)}\eta^{(t'+1)}\frac{1}{(1-\gamma)^2}G(t',t),\label{eq:Is- negative infinity 6}
\end{align}
where \cref{eq:Is- negative infinity 6} is from the update scheme of
APG as in \Cref{algorithm:APG}.
Since $\{\nu(t)\}_{t=1}^{\infty}$ is bounded, there is a finite constant $M$ such that
$\sum_{t'\leq\nu(t)}\eta^{(t'+1)}\frac{1}{(1-\gamma)^2}G(t',t) \leq M$
for all $t$.
Therefore, by taking the limit infimum on \cref{eq:Is- negative infinity 6}, we know 
\begin{equation}
    \liminf_{t\rightarrow \infty} Z^{(t)} = -\infty.\label{eq:Is- negative infinity 7}
\end{equation}
Now we are ready to quantify $\theta_{s,a}^{(t)}$ for the action $a\in I_s^-$. 
For all $t'\in\cJ^{(t)}$, we have
\begin{align}
    \frac{\lvert (\partial V^{\pi_{\theta}}(\mu)/\partial \theta_{s,a})|_{\theta=\omega^{(t')}}\rvert}{\lvert (\partial V^{\pi_{\theta}}(\mu)/\partial \theta_{s,a'})|_{\theta=\omega^{(t')}}\rvert}&=\bigg\lvert \frac{\pi_{\omega}^{(t')}(a\rvert s)A^{\pi_{\omega}^{(t')}}(s,a)}{\pi_{\omega}^{(t')}(a'\rvert s)A^{\pi_{\omega}^{(t')}}(s,a')} \bigg\rvert\geq \exp(\vartheta-\theta_{s,a'}^{(t')})\cdot\frac{(1-\gamma)\Delta_s}{4}\geq \exp(\delta)\cdot\frac{(1-\gamma)\Delta_s}{4},\label{eq:Is- negative infinity 8}
\end{align}
where the first inequality follows from the softmax parameterization,
the assumed upper bound of $\theta_{s,a}^{(t)}$, that $\lvert
A^{\pi_{\omega}^{(t')}}(s,a')\rvert \leq 1/(1-\gamma)$, and that $A^{\pi_{\omega}^{(t')}}(s,a) \leq
-\Delta_s/4$, and the second one holds by the definition of $\nu(t)$
in \cref{eq:Is- negative infinity 1}.
For any $\cJ^{(t)}\neq \emptyset$, we have
\begin{align}
	\nonumber
    \theta_{s,a}^{(t)}-\theta^{(1)}_{s,a}&=\sum_{t':1\leq
		t'<\bar{T}_s}\eta^{(t'+1)}\cdot \frac{\partial V^{\pi_\theta}(\mu)}{\partial \theta_{s,a}}\Big\rvert_{\theta=\omega^{(t')}}\cdot
			G(t',t)+\sum_{t': t'\geq\bar{T}_s}\eta^{(t'+1)}\cdot
			\frac{\partial V^{\pi_\theta}(\mu)}{\partial \theta_{s,a}}\Big\rvert_{\theta=\omega^{(t')}}\cdot
			G(t',t)\\
    &\leq \sum_{t':1\leq t'<\bar{T}_s}\eta^{(t'+1)}\cdot \frac{\partial V^{\pi_\theta}(\mu)}{\partial \theta_{s,a}}\Big\rvert_{\theta=\omega^{(t')}}\cdot G(t',t)+\sum_{t': t'\in \cJ^{(t)}}\eta^{(t'+1)}\cdot \frac{\partial V^{\pi_\theta}(\mu)}{\partial \theta_{s,a}}\Big\rvert_{\theta=\omega^{(t')}}\cdot G(t',t)\label{eq:Is- negative infinity 10}\\
    &\leq \underbrace{\sum_{t':1\leq t'<\bar{T}_s}\eta^{(t'+1)} \frac{\partial V^{\pi_\theta}(\mu)}{\partial \theta_{s,a}}\Big\rvert_{\theta=\omega^{(t')}} G(t',t)}_{<\infty \text{ and does not depend on }t }+\exp(\delta) \frac{(1-\gamma)\Delta_s}{4}\underbrace{\sum_{t': t'\in \cJ^{(t)}}\eta^{(t'+1)} \frac{\partial V^{\pi_\theta}(\mu)}{\partial \theta_{s,a'}}\Big\rvert_{\theta=\omega^{(t')}} G(t',t)}_{\equiv Z^{(t)}},\label{eq:Is- negative infinity 11}
\end{align}
where \cref{eq:Is- negative infinity 10} holds by the fact that
$A^{\pi_{\omega}^{(t')}}(s,a)<0$ for all $t\geq \bar{T}_s$, and \cref{eq:Is- negative infinity 11} is a direct result of \cref{eq:Is- negative infinity 8}.
Therefore, by taking the limit infimum on \cref{eq:Is- negative infinity 11}, we have $\liminf_{t\rightarrow \infty}\theta_{s,a}^{(t)} =-\infty$, which leads to a contradiction.

For the case (ii), given that $\nu(t)$ becomes unbounded as $t\rightarrow\infty$, there are infinitely many $\tau \ge \bar{T}_{s}$ such that $\theta_{s,a'}^{(\tau)}\geq \vartheta-\delta$. In addition, recall that $a'$ is an action that satisfies $\liminf_{t\rightarrow \infty}\theta_{s,a'}^{(t)}=-\infty$. Therefore, we can find a subsequence $\{t_i\}_{i=1}^{\infty}$ of time indices satisfying the following properties:
\begin{itemize}
    \item $t_1 \ge \bar{T}_s$
    \item $\lim_{i\rightarrow\infty} \theta_{s, a'}^{(t_i)} = -\infty$.
    \item $\theta_{s, a'}^{(t_i)} < \vartheta - \delta$ for all $i \ge 1$.
    \item There is always a $\tau_{i} \ge \bar{T}_s$ satisfying that
		$\theta_{s,a'}^{(\tau_i)}\geq \vartheta-\delta$, $\tau_i \in
		(t_i, t_{i+1})$, and we specifically let $\tau_i \coloneqq \inf\{\tau \in (t_i, t_{i+1}): \theta_{s,a'}^{(\tau_i)}\geq \vartheta-\delta\}$ for each $i \ge 1$.
\end{itemize}
{Since $\lim_{i\rightarrow\infty} \theta_{s, a'}^{(t_i)} = -\infty$,
we can select $\{t_i\}$ in a way such that $\theta^{(t_i)}_{s, a'} -
\theta^{(t_i-1)}_{s, a'} < 0$ for all $i$.
Therefore,}
\begin{align}
    \theta^{(\tau_i)}_{s, a'} - \theta^{(t_i)}_{s, a'}
	&= {(\theta^{(t_i)}_{s, a'} - \theta^{(t_i-1)}_{s,
	a'})\cdot (G({t_i-1}, \tau_i) - 1) + \sum_{t_i \le t' <
	\tau_i} \eta^{(t'+1)}\cdot \frac{\partial V^{\pi_\theta}(\mu)}{\partial \theta_{s,a'}}\Big\rvert_{\theta=\omega^{(t')}}\cdot G(t',\tau_i)} \nonumber \\
    &\le \sum_{t_i \le t' < \tau_i} \eta^{(t'+1)}\cdot \frac{\partial V^{\pi_\theta}(\mu)}{\partial \theta_{s,a'}}\Big\rvert_{\theta=\omega^{(t')}}\cdot G(t',\tau_i) \nonumber \\
    &\le \sum_{t_i \le t' < \tau_i} \Big\lvert \eta^{(t'+1)}\cdot \frac{\partial V^{\pi_\theta}(\mu)}{\partial \theta_{s,a'}}\Big\rvert_{\theta=\omega^{(t')}}\cdot G(t',\tau_i) \Big\rvert \nonumber \\
    &\le \frac{4}{(1-\gamma)\Delta_s} \cdot \sum_{t_i \le t' < \tau_i}
	\Big\lvert \eta^{(t'+1)}\cdot \frac{\partial
		V^{\pi_\theta}(\mu)}{\partial \theta_{s,{a}}}\Big\rvert_{\theta=\omega^{(t')}} \cdot G(t',\tau_i)\Big\rvert, \label{eq:Is- negative infinity 12}
\end{align}
where \cref{eq:Is- negative infinity 12} is from that
\begin{equation}
    \frac{\lvert (\partial V^{\pi_{\theta}}(\mu)/\partial \theta_{s,a'})|_{\theta=\omega^{(t')}}\rvert}{\lvert (\partial V^{\pi_{\theta}}(\mu)/\partial \theta_{s,a})|_{\theta=\omega^{(t')}}\rvert}=\bigg\lvert \frac{\pi_{\omega}^{(t')}(a'\rvert s)A^{\pi_{\omega}^{(t')}}(s,a')}{\pi_{\omega}^{(t')}(a\rvert s)A^{\pi_{\omega}^{(t')}}(s,a)} \bigg\rvert \leq \bigg\lvert \frac{A^{\pi_{\omega}^{(t')}}(s,a')}{A^{\pi_{\omega}^{(t')}}(s,a)} \bigg\rvert \leq \frac{4}{(1-\gamma)\Delta_s},
\end{equation}
where the first inequality holds because $\theta^{(t')}_{s, a'} \le
\theta^{(t')}_{s, a}$, implying $\pi_{\omega}^{(t')}(a'\rvert s) \le \pi_{\omega}^{(t')}(a\rvert s)$. Since $\theta^{(\tau_i)}_{s, a'} - \theta^{(t_i)}_{s, a'} \rightarrow \infty$ as $i\rightarrow\infty$, we have
\begin{equation}
\label{eq: grad Is- infinity}
    \lim_{i \rightarrow\infty}\sum_{t_i \le t' < \tau_i} \Big\lvert \eta^{(t'+1)}\cdot \frac{\partial V^{\pi_\theta}(\mu)}{\partial \theta_{s,a}}\Big\rvert_{\theta=\omega^{(t')}}\cdot G(t',\tau_i) \Big\rvert = \lim_{i \rightarrow\infty}\Big\lvert\sum_{t_i \le t' < \tau_i} \eta^{(t'+1)}\cdot \frac{\partial V^{\pi_\theta}(\mu)}{\partial \theta_{s,a}}\Big\rvert_{\theta=\omega^{(t')}}\cdot G(t',\tau_i) \Big\rvert = \infty,
\end{equation}
where the first equality holds because $a\in I^{-}_s$ together with
$t' \ge \bar{T}_s$ implies that $\frac{\partial
	V^{\pi_\theta}(\mu)}{\partial
		\theta_{s,a}}\big\rvert_{\theta=\omega^{(t')}}< 0$ for all $t'$
	included in the summation. {
		Now consider
		\[
    \theta^{(t)}_{s, a} - \theta^{(\bar T_s)}_{s, a}
	= {(\theta^{(\bar T_s)}_{s, a} - \theta^{(\bar T_s-1)}_{s,
	a})\cdot (G({\bar T_s -1}, t) - 1) + \sum_{t' = \bar T_s
	}^{t-1} \eta^{(t'+1)}\cdot \frac{\partial
		V^{\pi_\theta}(\mu)}{\partial
			\theta_{s,a}}\Big\rvert_{\theta=\omega^{(t')}}\cdot
		G(t',\tau_i)}.
		\]
		By an argument similar to \cref{eq:Is+ bounded eq8},
		the first term above is at most $(\theta^{(\bar{T}_s)}_{s, a}
		- \theta^{(\bar{T}_s-1)}_{s, a}) \cdot O(\bar{T}_s)$
		regardless of $t$.
Combining this with \cref{eq: grad Is- infinity,eq:neggrad},
we obtain $\liminf_{t \rightarrow\infty} \theta^{(t)}_{s, a} = -\infty$,
which also leads to a contradiction.}
Therefore, \cref{eq:liminf} is proven.

Then, we show that there exists a time $\bar{T}_s'
\ge \bar{T}_s$ such that $\theta^{(\bar T_s'+1)}_{s, a} -
\theta^{(\bar T_s')}_{s, a}
\le 0$.
If such a $\bar{T}_s'$ does not exist, which means that $\{\theta_{s,
a}^{(t)}\}$ is monotonically increasing, we see a contradiction with
\cref{eq:liminf}.
Since $a\in I^{-}_s$, from \cref{eq:neggrad}, we know that if
$\theta^{(\bar{T}_s'+1)}_{s, a} - \theta^{(\bar{T}_s')}_{s, a} \leq
0$, then $\theta^{(t+1)}_{s, a} - \theta^{(t)}_{s, a} \leq 0$ for all
$t \geq \bar T_s'$, meaning that $\theta^{(t)}_{s, a}$ is monotonically
decreasing after $\bar T_s'$.
Therefore, combining this monotinicity after $\bar T_s'$ and
\cref{eq:liminf},
we obtain
$\lim_{t\rightarrow\infty} \theta^{(t)}_{s, a} = -\infty$.

Finally, we show that $\lim_{t\rightarrow\infty} \omega^{(t)}_{s, a} = -\infty$. Since $\theta^{(t+1)}_{s, a} - \theta^{(t)}_{s, a} \leq 0$ for all
$t \geq \bar{T}_s'$, by \cref{algorithm:eq2} and that $a\in I_{s}^{-}$, we know that $\omega_{s, a}^{(t)} \le \theta_{s, a}^{(t)}$ and thus $\omega^{(t)}_{s, a}$ is decreasing for all $t \ge \bar T_s'$. Combining the above properties with $\lim_{t\rightarrow\infty} \theta^{(t)}_{s, a} = -\infty$, we arrive at $\lim_{t\rightarrow\infty} \omega^{(t)}_{s, a} = -\infty$.
\end{proof}

In the following lemma, we use the notation
{$\Delta\theta_{s,a}^{(t)}\coloneqq\theta_{s,a}^{(t+1)}-\theta_{s,a}^{(t)}$},
for each state-action pair $(s,a)$ and each $t\in \bbN$.

\begin{lemma}
    \label{lemma: ordering unchanged}
    Consider any state $s$ with non-empty $I_s^+$ and some $a_+\in
	I_s^+$ and $a\in I_s^0$.
{Under the setting of \Cref{lemma: convergence to stationarity},}
	if
	{$\omega_{s,a_+}^{(\tau)}>\omega_{s,a}^{(\tau)}$} and
	$\Delta\theta_{s,a_+}^{(\tau)}>\Delta\theta_{s,a}^{(\tau)}$ for
	some $\tau > \bar{T}_s$, then
	{$\omega_{s,a_+}^{(t)}>\omega_{s,a}^{(t)}$} and
	$\Delta\theta_{s,a_+}^{(t)}>\Delta\theta_{s,a}^{(t)}$ for all $t>
	\tau$.
\end{lemma}

\begin{proof}[Proof of \Cref{lemma: ordering unchanged}]
    We prove this by induction. Suppose at some time $\tau > \bar{T}_s$, we have $\omega_{s,a_+}^{(\tau)}>\omega_{s,a}^{(\tau)}$ and $\Delta\theta_{s,a_+}^{(\tau)}>\Delta\theta_{s,a}^{(\tau)}$.
    Recalling our definition in \cref{eq:adv I_s+} and \cref{eq:adv I_s0}, we have from \Cref{lemma:softmax_pg} that
    \begin{align}
        \left.\frac{\partial V^{\pi_{\theta}}(\mu)}{\partial\theta_{s,a_{+}}}\right|_{\theta = \omega^{(\tau)}}&=\frac{1}{1-\gamma} \cdot d^{\pi^{(\tau)}_{\omega}}_{\mu}(s) \cdot \pi^{(\tau)}_\omega(a_+\rvert s) \cdot A^{\pi_{\omega}^{(\tau)}}(s,a_+)\nonumber\\
        &>\frac{1}{1-\gamma} \cdot d^{\pi^{(\tau)}_{\omega}}_{\mu}(s) \cdot \pi_\omega^{(\tau)}(a\rvert s) \cdot A^{\pi_{\omega}^{(\tau)}}(s,a)\label{eq: ordering unchanged 3}\\
        &=\left.\frac{\partial V^{\pi_{\theta}}(\mu)}{\partial\theta_{s,a}}\right|_{\theta = \omega^{(\tau)}},\label{eq: ordering unchanged 4}
    \end{align}
    where \cref{eq: ordering unchanged 3} holds by
	$\omega_{s,a_+}^{(\tau)}>\omega_{s,a}^{(\tau)}$ and the fact that
	$\tau >\bar{T}_s$ implies $A^{\pi_{\omega}^{(\tau)}}(s,a_+)\geq
	A^{\pi_{\omega}^{(\tau)}}(s,a)$ from \cref{eq:adv I_s+}.
	By the update rule \cref{algorithm:eq1}, \cref{eq: ordering
	unchanged 4} and that
	$\omega_{s,a_+}^{(\tau)}>\omega_{s,a}^{(\tau)}$ imply
		\begin{equation}
			\label{eq:thetaorder}
			\theta_{s,a_+}^{(\tau + 1)}>\theta_{s,a}^{(\tau + 1)}.
		\end{equation}
	Thus,
	from $\Delta\theta_{s,a_+}^{(\tau)}>\Delta\theta_{s,a}^{(\tau)}$,
	\cref{algorithm:eq2,algorithm:eq3,eq:thetaorder}, we have
    \begin{align}
        \omega_{s,a_+}^{(\tau+1)}
        &=\theta_{s,a_+}^{(\tau+1)}+\mathbb{I} \left\{V^{\pi_{\varphi}^{(\tau+1)}}(\mu) \ge V^{\pi_{\theta}^{(\tau+1)}}(\mu)\right \} \cdot\frac{\tau}{\tau+3}(\theta_{s,a_+}^{(\tau+1)}-\theta_{s,a_+}^{(\tau)}) \nonumber \\
        &>\theta_{s,a}^{(\tau+1)}+\mathbb{I}
		\left\{V^{\pi_{\varphi}^{(\tau+1)}}(\mu) \ge
		V^{\pi_{\theta}^{(\tau+1)}}(\mu)\right \}
		\cdot\frac{\tau}{\tau+3}(\theta_{s,a}^{(\tau+1)}-\theta_{s,a}^{(\tau)})=\omega_{s,a}^{(\tau+1)}.\label{eq: ordering unchanged 1}
    \end{align}
    By \cref{eq: ordering
	unchanged 1}, \cref{eq:softmax}, and that $A^{\pi_{\omega}^{(\tau+1)}}(s,a_+)\geq A^{\pi_{\omega}^{(\tau+1)}}(s,a)$, we have 
    \begin{equation}
        \left.\frac{\partial V^{\pi_{\theta}}(\mu)}{\partial\theta_{s,a_{+}}}\right|_{\theta = \omega^{(\tau+1)}} > \left.\frac{\partial V^{\pi_{\theta}}(\mu)}{\partial\theta_{s,a}}\right|_{\theta = \omega^{(\tau + 1)}}.\label{eq: ordering unchanged tau + 1}
    \end{equation}
    Furthermore, by \cref{eq: ordering unchanged tau + 1} and that
	$\Delta\theta_{s,a_+}^{(\tau)}>\Delta\theta_{s,a}^{(\tau)}$, we have
    \begin{align*}
        \Delta\theta_{s,a_+}^{(\tau+1)}&=\mathbb{I} \left\{V^{\pi_{\varphi}^{(\tau+1)}}(\mu) \ge V^{\pi_{\theta}^{(\tau+1)}}(\mu)\right \} \cdot\frac{\tau}{\tau+3}\Delta\theta_{s,a_+}^{(\tau)}+\eta^{(\tau+2)}\left.\frac{\partial V^{\pi_{\theta}}(\mu)}{\partial\theta_{s,a_{+}}}\right|_{\theta = \omega^{(\tau+1)}}\\
        &>\mathbb{I} \left\{V^{\pi_{\varphi}^{(\tau+1)}}(\mu) \ge V^{\pi_{\theta}^{(\tau+1)}}(\mu)\right \} \cdot\frac{\tau}{\tau+3}\Delta\theta_{s,a}^{(\tau)}+\eta^{(\tau+2)}\left.\frac{\partial V^{\pi_{\theta}}(\mu)}{\partial\theta_{s,a}}\right|_{\theta = \omega^{(\tau+1)}}=\Delta\theta_{s,a}^{(\tau+1)}.
    \end{align*}
    By repeating the above argument, we know $\omega_{s,a_+}^{(t)}>\omega_{s,a}^{(t)}$ and $\Delta\theta_{s,a_+}^{(t)}>\Delta\theta_{s,a}^{(t)}$ for all $t>\tau$.
\end{proof}

Next, we take a closer look at the actions in $I_s^0$. We further
decompose $I_s^0$ into two subsets as follows: Consider any state $s$ with
a non-empty $I_s^+$, for any $a_+\in I_s^+$, we define
\begin{align*}
    B_s^0(a_+)&:=\Big\{a\in I_s^0: \text{For any } t\geq \bar{T}_s, \text{either } {\omega^{(t)}_{s,a}\geq \omega^{(t)}_{s,a_+}} \text{or } \Delta\theta^{(t)}_{s,a}\geq \Delta\theta^{(t)}_{s,a_+}\Big\}
\end{align*}
and use $\bar{B}_s^0(a_+)$ to denote its complement in $I_s^0$. We could thus also write $\bar{B}_s^0(a_+)$ as
\begin{equation*}
    \bar{B}_s^0(a_+):=\Big\{a\in I_s^0: {\omega^{(t)}_{s,a}< \omega^{(t)}_{s,a_+}} \text{and } \Delta\theta^{(t)}_{s,a}<\Delta\theta^{(t)}_{s,a_+} \text{for some } t\geq \bar{T}_s \Big\}.
\end{equation*}

\begin{lemma}
    \label{lemma:sum of Bs0 goes to 1}
	Consider any state $s$.
{Under the setting of \Cref{lemma: convergence to stationarity},}
    if $I_s^+$ is nonempty, then:
    
    \textbf{(a)} For all $a_+\in I_s^+$, we have
    \begin{equation*}
	\lim_{t\rightarrow \infty}\,
	\sum_{a\in B_s^{0}(a_+)}\pi_\omega^{(t)}(a\rvert s) = 1.
    \end{equation*}

    \textbf{(b)} For all $a_+\in I_s^+$, we have
    \begin{equation*}
	\lim_{t\rightarrow \infty}\,
	\max_{a\in B_s^{0}(a_+)}\omega_{s,a}^{(t)} = \infty.
    \end{equation*}

    \textbf{(c)} For all $a_+\in I_s^+$, we have
    \begin{equation*}
	\lim_{t\rightarrow \infty}\,
        \sum_{a\in B_s^{0}(a_+)}\omega_{s,a}^{(t)} = \infty.
    \end{equation*}    
\end{lemma}
\begin{proof}[Proof of \Cref{lemma:sum of Bs0 goes to 1}]
Regarding \textbf{(a)}, by the definition of $\bar{B}_{s}^{0}(a_+)$, for each $a\in \bar{B}_{s}^{0}(a_+)$, there must exist some $T'\geq \bar{T}_s$ such that {$\omega_{s,a_+}^{(T')}>\omega_{s,a}^{(T')}$} and $\Delta\theta_{s,a_+}^{(T')}>\Delta\theta_{s,a}^{(T')}$.
Then, by \Cref{lemma: ordering unchanged}, we know 
\begin{equation}
    {\omega_{s,a_+}^{(t)}>\omega_{s,a}^{(t)}} \quad \text{ and }\quad \Delta\theta_{s,a_+}^{(t)}>\Delta\theta_{s,a}^{(t)},\quad \text{for all }t\geq T'.\label{eq:sum of Bs0 goes to 1 eq4}
\end{equation}
Moreover, by \Cref{lemma:sum of Is+ and Is- goes to zero} and that
$a_+\in I_s^+$, we have {$\pi_{\omega}^{(t)}(a_+\rvert s)\rightarrow 0$} as
$t\rightarrow \infty$. Based on \cref{eq:sum of Bs0 goes to 1 eq4},
this further implies that {$\pi_{\omega}^{(t)}(a\rvert s)\rightarrow 0$} as
$t\rightarrow \infty$, for all $a\in \bar{B}_{s}^{0}(a_+)$.
Hence, together with \cref{lemma:sum of Is+ and Is- goes to zero} we
obtain the desired result.

Regarding \textbf{(b)}, by the result in \textbf{(a)} and
the argument in the proof of \Cref{lemma:max Is0}, we directly see
that the statement holds true.

Regarding \textbf{(c)}, let us consider any action $a\in {B}_{s}^{0}(a_+)$.
By the definition of $B_s^{0}(a_+)$, at each iteration $t\geq \bar{T}_s$, either {$\omega^{(t)}_{s,a}\geq \omega^{(t)}_{s,a_+}$} or $\Delta\theta^{(t)}_{s,a}\geq \Delta\theta^{(t)}_{s,a_+}$ holds.
Given the result in \textbf{(b)}, it suffices to show that $\{\omega_{s,a}^{(t)}\}_{t=1}^{\infty}$ is bounded from below for all $a \in {B}_{s}^{0}(a_+)$. We prove it by contradiction.
Suppose that there exists an action
$\bar{a} \in {B}_{s}^{0}(a_+)$ such that $\liminf_{t\rightarrow\infty}
\omega^{(t)}_{s, \bar{a}} = -\infty$.
By \Cref{lemma:Is+ bounded}, we know that
$\{\omega_{s,a_+}^{(t)}\}_{t=1}^{\infty}$ is bounded from below.
Therefore, once $\omega^{(t)}_{s,\bar{a}}< \omega^{(t)}_{s,a_+}$ for
some $t$, we will have $\Delta\theta^{(t)}_{s,\bar{a}}\geq
\Delta\theta^{(t)}_{s,a_+}$.
Following an argument similar to that in
the
proof of \Cref{lemma:Is- negative infinity} (for case (ii)) to choose a subsequence of
$\omega^{(t)}_{s,\bar{a}}$ converging to negative infinity and since $\omega^{(t)}$ is $\theta^{(t)}$ add momentum with restart, we know that if momentum goes to negative infinity, by $\Delta\theta^{(t)}_{s,\bar{a}}\geq \Delta\theta^{(t)}_{s,a_+}$, we will have $\liminf_{t\rightarrow\infty}
\omega^{(t)}_{s, a_+} = -\infty$. Otherwise, if $\theta^{(t)}$ goes to negative infinity, we can observe that the total momentum effect goes to negative infinity, which also leads to $\liminf_{t\rightarrow\infty}
\omega^{(t)}_{s, a_+} = -\infty$ by $\Delta\theta^{(t)}_{s,\bar{a}}\geq \Delta\theta^{(t)}_{s,a_+}$.
This leads to the desired contradiction and concludes our proof.
\end{proof}

\begin{lemma}
\label{lemma:B_s0 bar}
	Consider any state $s$.
{Under the setting of \Cref{lemma: convergence to stationarity},}
    for any $a_+\in I_s^+$, the following hold.
    
    \textbf{(a)} There exists $T_{a_+}$ such that for all $a\in \bar{B}_s^{0}(a_+)$,
    \begin{equation*}
    {\pi_{\omega}^{(t)}(a_+\rvert s )>\pi_{\omega}^{(t)}(a\rvert s ),\quad \text{for all } t>T_{a_+}.}
    \end{equation*}
    \textbf{(b)} There exists $T_{a_+}^{\dagger}$ such that for all $a\in \bar{B}_s^{0}(a_+)$, 
    \begin{equation}
        {\lvert A^{\pi_\omega^{(t)}}(s,a)\rvert < \frac{\pi_{\omega}^{(t)}(a_+\rvert s)}{\pi_{\omega}^{(t)}(a\rvert s)} \cdot \frac{\Delta_s}{16 \lvert \cA\rvert}, \quad \text{for all } t>T_{a_+}^\dagger.}\label{eq:B_s0 bar eq2}
    \end{equation}
    This also implies that 
    \begin{equation}
        {\sum_{a\in \bar{B}_s^{0}(a_+)}\pi_{\omega}^{(t)}(a\rvert s) A^{\pi_\omega^{(t)}}(s,a) > -{\pi_{\omega}^{(t)}(a_+\rvert s)}\cdot \frac{\Delta_s}{16 }, \quad \text{for all } t>T_{a_+}^\dagger.}\label{eq:B_s0 bar eq3}
    \end{equation}
\end{lemma}
\begin{proof}[Proof of \Cref{lemma:B_s0 bar}]
    Regarding \textbf{(a)}, for each $a\in \bar{B}_s^{0}(a_+)$, we define 
    \begin{equation*}
        u_a(a_+):=\inf\{\tau\geq \bar{T}_s: {\omega_{s,a_+}^{(\tau)}>\omega_{s,a}^{(\tau)}} \text{ and } \Delta\theta_{s,a_+}^{(\tau)}>\Delta\theta_{s,a}^{(\tau)}\}.
    \end{equation*}
    By the definition of $\bar{B}_s^{0}(a_+)$, we know the following
	two facts: (i) $u_a(a_+)$ is finite for all $a\in
	\bar{B}_s^{0}(a_+)$. (ii) By \Cref{lemma: ordering unchanged}, for all $t \geq u_a(a_+)$, we have {$\omega_{s,a_+}^{(t)}>\omega_{s,a}^{(t)}$} and $\Delta\theta_{s,a_+}^{(t)}>\Delta\theta_{s,a}^{(t)}$.
    Therefore, by choosing $T_{a_+}:=\max_{a\in \bar{B}_s^{0}(a_+)} u_a(a_+)$, we must have {$\pi_{\omega}^{(t)}(a_+\rvert s )>\pi_{\omega}^{(t)}(a\rvert s )$} for all $t>T_{a_+}$.

	Regarding \textbf{(b)}, we have from \textbf{(a)} 
	\begin{equation*}
		\frac{\pi_{\omega}^{(t)}(a_+\rvert
	s)}{\pi_{\omega}^{(t)}(a\rvert s)}>1, \quad \forall t> T_{a_+}.
\end{equation*}
 Thus, we know that for each $a\in \bar{B}_s^{0}(a_+)\subseteq I_s^0$, there exists some finite $t'_{a}>T_{a_+}$ such that 
     \begin{equation*}
         {\lvert A^{\pi_\omega^{(t)}}(s,a)\rvert < \frac{\pi_\omega^{(t)}(a_+\rvert s)}{\pi_\omega^{(t)}(a\rvert s)} \cdot \frac{\Delta_s}{16 \lvert \cA\rvert}, \quad \text{for all } t\geq t'_a.}
     \end{equation*}
     As a result, by choosing $T_{a_+}^{\dagger}:=\max_{a\in\bar{B}_{s}^{0}(a_+)}t'_a$, we conclude that \cref{eq:B_s0 bar eq2}-\cref{eq:B_s0 bar eq3} indeed hold. 
\end{proof}

\begin{lemma}
    \label{lemma:compare pi(a+) and pi(a-)}
{Under the setting of \Cref{lemma: convergence to stationarity},}
    if $I_s^{+}$ is non-empty, then for any $a_+\in I_s^+$, there exists some finite $\tilde{T}_{a_+}$ such that
    \begin{equation}
        {\sum_{a\in I_s^-}\pi_{\omega}^{(t)}(a\rvert
		s)A^{\pi_\omega^{(t)}}(s,a)> -\pi_{\omega}^{(t)}(a_+\rvert
		s)\frac{\Delta_s}{16}, \quad \forall t\geq \tilde{T}_{a_+}.}\label{eq:compare pi(a+) and pi(a-) eq1}
    \end{equation}
\end{lemma}
\begin{proof}
    {Consider any $a_-\in I_s^-$. By \Cref{lemma:Is+
	bounded,lemma:Is- negative infinity}, we know
	$\{\omega_{s,a_+}^{(t)}\}$ is bounded from below and $\omega_{s,a_-}^{(t)}\rightarrow -\infty$ as $t\rightarrow\infty$.
    This implies that $\pi_{\omega}^{(t)}(a_-\rvert s)/\pi_{\omega}^{(t)}(a_+\rvert
	s)\rightarrow 0$ as $t \rightarrow \infty$.
    Therefore, there exists some finite $t'_{a_-} \ge \bar{T}_s$ such that
    \begin{equation}
        \frac{\pi_{\omega}^{(t)}(a_-\rvert
	s)}{\pi_{\omega}^{(t)}(a_+\rvert s)} < \frac{\Delta_s
	(1-\gamma)}{16\lvert \cA\rvert }, \quad \forall t \geq t'_{a_-}.
	\label{eq:piratio}
    \end{equation}
    By the definition of $\bar{T}_s$ in \cref{eq:adv I_s-} and
	\Cref{lemma: reward range to value range}, we know that
	\begin{equation}
		-\frac{1}{1-\gamma} \le A^{\pi_\omega^{(t)}}(s, a_{-}) < 0,\quad
		\forall t \geq t'_{a_-}.
		\label{eq:Abound}
	\end{equation}
	Combining \cref{eq:piratio,eq:Abound} leads to
    \begin{equation}
        \pi_{\omega}^{(t)}(a_-\rvert s) A^{\pi_\omega^{(t)}}(s, a_{-}) > -\pi_{\omega}^{(t)}(a_+\rvert s)\frac{\Delta_s}{16\lvert\mathcal{A}\rvert}.
		\label{eq:singleitem}
    \end{equation}
    By choosing $\tilde{T}_{a_+}:=\max_{a_-\in I_s^-}t'_{a_-}$ and
	summing \cref{eq:singleitem} over $a_i \in I_s^-$, we obtain
	\cref{eq:compare pi(a+) and pi(a-) eq1}.}
\end{proof}
\newpage
\subsection{Putting Everything Together: Asymptotic Convergence of APG}
\label{app:asym_conv:all}
Now we are ready to put everything together and prove \Cref{theorem:convergeoptimal},
which we restate below.

\convergeoptimalthm*

\begin{proof}[Proof of \Cref{theorem:convergeoptimal}]

By \Cref{lemma:perf_diff}, it suffices to prove that $I_s^+ = \emptyset$ for all $s \in \cS$.
We prove this by contradiction.
Suppose there exists at least one
state $s\in \cS$ with a non-empty $I_s^+$.
Consider an action $a_+\in I_s^+$. 
Recall the definitions of $\bar{T}_s$, $T_{a_+},T_{a_+}^\dagger$, and
$\tilde{T}_{a_+}$ from \cref{eq:adv I_s+}-\cref{eq:adv I_s0},
\Cref{lemma:B_s0 bar}, and \Cref{lemma:compare pi(a+) and pi(a-)}.
We define $T_{\max}:=\max\{\bar{T}_s,T_{a_+},T_{a_+}^{\dagger}, \tilde{T}_{a_+}\}$,
then we have for all {$t>T_{\max}$},
\begin{align}
    0&= \sum_{a\in B_s^0(a_+)} \pi_{\omega}^{(t)}(a\rvert
	s)A^{\pi_{\omega}^{(t)}}(s,a)+\underbrace{\sum_{a\in \bar{B}_s^0(a_+)}
	\pi_{\omega}^{(t)}(a\rvert s)A^{\pi_{\omega}^{(t)}}(s,a)}_{>-\pi_{\omega}^{(t)}(a_+\rvert
	s)\frac{\Delta_s}{16} \text{ by \Cref{lemma:B_s0 bar}}} \nonumber \\
    &\quad +\underbrace{\sum_{a\in I_s^+} \pi_{\omega}^{(t)}(a\rvert
	s)A^{\pi_{\omega}^{(t)}}(s,a)}_{\geq \pi_{\omega}^{(t)}(a_+\rvert
	s)\frac{\Delta_s}{4}}+\underbrace{\sum_{a\in I_s^-}
	\pi_{\omega}^{(t)}(a\rvert s)A^{\pi_{\omega}^{(t)}}(s,a)}_{>-\pi_{\omega}^{(t)}(a_+\rvert
	s)\frac{\Delta_s}{16} \text{by \Cref{lemma:compare pi(a+) and pi(a-)}} }\label{eq:convergeoptimal 1}\\
    &>\sum_{a\in B_s^0(a_+)} \pi_{\omega}^{(t)}(a\rvert s)A^{\pi_{\omega}^{(t)}}(s,a)+\frac{1}{8}\cdot\pi_{\omega}^{(t)}(a_+\rvert s)\Delta_s\nonumber \\
    &>\sum_{a\in B_s^0(a_+)} \pi_{\omega}^{(t)}(a\rvert s)A^{\pi_{\omega}^{(t)}}(s,a),\label{eq:convergeoptimal 3}
\end{align}
where \cref{eq:convergeoptimal 1} uses \Cref{lemma:sum_pi_A}.
By \Cref{lemma:softmax_pg}, \cref{eq:convergeoptimal 3} further implies
\begin{equation}
	\label{eq:negative}
	\sum_{a\in B_s^0(a_+)} \left.\frac{\partial V^{\pi_{\theta}}(\mu)}{\partial
		\theta_{s,a}}\right|_{{\theta = \omega^{(t)}}}<0, \quad \forall t> T_{\max}.
\end{equation}
Moreover, {for any $t>T_{\max} + 4$}, we have
\begin{align}
    &~\sum_{a\in
	B_s^0(a_+)}\theta_{s,a}^{(t)}-\theta_{s,a}^{(0)}\nonumber\\
	=&~\sum_{a\in B_s^0(a_+)}\sum_{t'=0}^{t-1}\Big(
	\eta^{(t'+1)}\cdot \left.\frac{\partial V^{\pi_{\theta}}(\mu)}{\partial
		\theta_{s,a}}\right|_{{\theta = \omega^{(t')}}}\cdot G(t',t)\Big)\label{eq:update expansion}\\
		=&~\sum_{t'=0}^{t-1} \eta^{(t'+1)}G(t',t)\cdot
		\Big(\sum_{a\in B_s^0(a_+)} \left.\frac{\partial V^{\pi_{\theta}}(\mu)}{\partial
		\theta_{s,a}}\right|_{{\theta = \omega^{(t')}}}\Big)\nonumber\\
    =&~\underbrace{\sum_{t'=0}^{T_{\max}} \eta^{(t'+1)}G(t',t)\cdot
	\Big(\sum_{a\in B_s^0(a_+)} \left.\frac{\partial V^{\pi_{\theta}}(\mu)}{\partial
		\theta_{s,a}}\right|_{{\theta = \omega^{(t')}}}\Big)}_{\text{Independent of $t$ and thus finite}}
	+ \sum_{t'=T_{\max}+1}^{t-1} \eta^{(t'+1)}G(t',t)\cdot \underbrace{\Big(\sum_{a\in B_s^0(a_+)} \left.\frac{\partial V^{\pi_{\theta}}(\mu)}{\partial
		\theta_{s,a}}\right|_{{\theta = \omega^{(t')}}}\Big)}_{<0 \text{ by \cref{eq:negative}}}\label{eq:convergeoptimal 7-2},
\end{align}
where \cref{eq:update expansion} holds due to \Cref{lemma:theta as sum of gradients} and $t > T_{\text{max}} + 4 > 4$. By taking the limit of $t$ to infinity at both sides of \cref{eq:convergeoptimal 7-2}, the left-hand side
goes to
positive infinity by \Cref{lemma:sum of Bs0 goes to 1} but the
right-hand side is bounded from above. This leads to a contradiction and hence completes the proof.
\end{proof}

\begin{corollary}
\label{cor:policy convergence}
Consider a tabular softmax parameterized policy and assume that
\Cref{assump:unique_optimal} holds.
Under the setting of \Cref{theorem:convergeoptimal},
the following holds almost surely:
\begin{equation*}
   \lim_{t\rightarrow \infty}V^{\pi_{\omega}^{(t)}}(s) = V^{*}(s),\quad
    \lim_{t\rightarrow \infty} \pi_{\theta}(a^*(s) | s) = 1, \quad
	\lim_{t\rightarrow \infty} \pi_{\omega}(a^*(s) | s) = 1, \quad \forall s\in\cS.
\end{equation*}
\end{corollary}
\begin{proof}[Proof of \Cref{cor:policy convergence}]
By \Cref{lemma: monotone limits exist} and
\Cref{theorem:convergeoptimal}, we know that almost surely
\begin{equation}
	\lim_{t\rightarrow \infty}V^{\pi_{\omega}^{(t)}}(s) =
	\lim_{t\rightarrow \infty}V^{\pi_{\theta}^{(t)}}(s) =
	V^*(s),\quad \forall s\in\cS.
	\label{eq:sameV}
\end{equation}
From \Cref{assump:unique_optimal} and page 14 of \citep{szepesvari2022algorithms},
\cref{eq:sameV} then implies that the
policy weight of the optimal action $a^*(s)$ converges to $1$, proving
the desired results.
\end{proof}

\newpage
\section{Convergence Rate of APG With Nearly Constant Step Sizes}
\label{app:MDP}

\begingroup
\allowdisplaybreaks

\subsection{Proofs of \texorpdfstring{\Cref{lemma:local nearly concavity,lemma: enter local concavity}}{}}
%
This subsection proves \Cref{lemma:local nearly concavity,lemma: enter
local concavity} using the lemmas we have developed to prepare for
proving \Cref{theorem: MDP convergence rate}.
{Before delving into the proof of \Cref{lemma:local nearly
concavity,lemma: enter local concavity}, we state a lemma useful for
proving them.}

\begin{lemma}
\label{lemma: Q-ordering fixed}
Under the setting of \Cref{cor:policy convergence}, there exists a
finite time  $T_V$ such that $V^{\pi_{\theta}^{(t)}}(s) > Q^*(s,
a_2(s))$ for all $s \in \cS$ and $t \ge T_V$ {almost surely}.
\end{lemma}
\begin{proof}[Proof of \Cref{lemma: Q-ordering fixed}]
    By \Cref{theorem:convergeoptimal}, $V^{\pi_{\theta}^{(t)}}(s) \to
	V^*(s)$ for all $s \in \cS$ almost surely. Hence, given $\varepsilon = V^*(s) -
	Q^*(s, a_2(s)) > 0$ {by \Cref{lemma: Q^* a_2},}
almost surely there is a finite $T_V(s)$ such that $V^{\pi_{\theta}^{(t)}}(s) > Q^*(s, a_2(s))$ for any $t \ge T_V(s)$.
Accordingly, by noting that $\cS$ is a finite set and that
any finite intersection of measure zero events is still measure zero,
we let $T_V = \max_{s\in\cS} T_V(s)$ and the proof is completed.
\end{proof}

\textbf{\Cref{lemma:local nearly concavity} (Locally $C$-Near Concavity; Formal).}
Given any $C > 1$ and a direction $\bm{d}$ satisfying the setting of \Cref{lemma: theta to pi}. Let $\theta$ be a policy parameter satisfying the following conditions: (i) $V^{\pi_{\theta}}(s) > Q^*(s, a_2(s))$ for all $s \in \cS$; (ii) $\theta_{s, a^*(s)} - \theta_{s, a} > M_{C, \bm{d}}$, for all $s\in\cS$, and $a \neq a^*(s)$, where
\begin{equation}
    M_{C, \bm{d}} \coloneqq \max\left\{
 {2 \max_{s \in \cS, i > 1} \Big\{\ln\Big[2(|\mathcal{A}| - 1)\Big] - \ln[d_{s, a^*(s)} - d_{s, a_i(s)}]\Big\}}, \ln \Big[ \frac{|\cS||\cA|^2}{{(C-1)}(1-\gamma)^2 \min_{s \in S} \mu(s)} \Big] \right\}.
 \label{eq:Mc}
\end{equation}
Then, we have that the objective function $\theta \mapsto V^{\pi_{\theta}}(\mu)$ is $C$-nearly concave along the direction $\bm{d}$.



\begin{proof}[Proof of \Cref{lemma:local nearly concavity}]

Given any $s \in \cS$, according to \Cref{lemma: theta to pi},
provided that $\theta_{s, a^{*}} - \theta_{s, a} > {M_d}$ for all $a \neq
a^{*}$, the function $\theta \to \pi_{\theta}(a^{*}(s) | s)$ is
concave {{along the direction $\bm{d}$}},
and the functions
$\theta \to \pi_{\theta}(a | s)$ for all $a \neq a^*(s)$ are convex
along the direction $\bm{d}$.
Suppose $\theta$ is a parameter satisfies $\theta_{s, a^*(s)} -
\theta_{s, a} > M_{C, \bm{d}}$
for all $s\in\cS$ and all $a \neq a^*(s)$, and consider $\theta'$ as
a parameter obtained from updating $\theta$ along the direction $\bm{d}$. To show that $\theta \to
V^{\pi_{\theta}}(\mu)$ is $C$-nearly concave along the desired
direction, we have to show that
\begin{equation*}
    V^{\pi_{\theta'}}(\mu) - V^{\pi_{\theta}}(\mu) \le C \cdot \Big\langle \nabla_{\theta} V^{\pi_{\theta}}(\mu)\Big|_{\theta=\theta}, \theta' - \theta\Big\rangle.
\end{equation*}

By \Cref{lemma:perf_diff,lemma:sum_pi_A}, 
\begin{align}
    V^{\pi_{\theta'}}(\mu) - V^{\pi_{\theta}}(\mu) &= \frac{1}{1 - \gamma} \sum_{s, a} d_{\mu}^{\pi_{\theta'}}(s) \pi_{\theta'}(a|s) A^{\pi_{\theta}}(s, a) \nonumber \\
    &= \frac{1}{1 - \gamma} \sum_{s, a} d_{\mu}^{\pi_{\theta'}}(s) (\pi_{\theta'}(a|s) - \pi_{\theta}(a|s)) A^{\pi_{\theta}}(s, a) \nonumber \\
    &= \sum_{s, a} \frac{d_{\mu}^{\pi_{\theta'}}(s)}{d_{\mu}^{\pi_{\theta}}(s)} (\pi_{\theta'}(a|s) - \pi_{\theta}(a|s)) \frac{1}{1 - \gamma} d_{\mu}^{\pi_{\theta}}(s) A^{\pi_{\theta}}(s, a) \nonumber \\
    &\le \left\lVert \frac{d_{\mu}^{\pi_{\theta'}}}{d_{\mu}^{\pi_{\theta}}} 
    \right\rVert_{\infty} \Big\langle \nabla_{\pi_{\theta}} V^{\pi_{\theta}}(\mu)|_{\pi_{\theta}=\pi_{\theta}}, \pi_{\theta'} - \pi_{\theta} \Big\rangle, \label{eq: local nearly concave d_pi ratio}
\end{align}

{where \cref{eq: local nearly concave d_pi ratio} holds by \Cref{lemma:softmax_pg_wrt_pi,lemma:feasible update domain improvement lemma}.}

Since $\theta'$ is obtained by updating $\theta$ with the direction $\bm{d}$, we have
\begin{equation*}
    \theta_{s, a^{*}(s)}' - \theta_{s, a}' > \theta_{s, a^{*}(s)} -
	\theta_{s, a} > M_{C, \bm{d}} \ge \Big[ \frac{|\cS||\cA|^2}{{(C-1)}(1-\gamma)^2
\min_{s \in S} \mu(s)} \Big].
\end{equation*}
Thus, by \Cref{lemma: state visitation distribution fix}, we have that
$\lVert d_{\mu}^{\pi_{\theta'}}/d_{\mu}^{\pi_{\theta}}\rVert_{\infty}
< C$.
Moreover, by \Cref{lemma:softmax_pg_wrt_pi}, we have
\begin{equation*}
    \Big\langle \nabla_{\pi_{\theta}} V^{\pi_{\theta}}(\mu)|_{\pi_{\theta}=\pi_{\theta}}, \pi_{\theta'} - \pi_{\theta} \Big\rangle = \Big\langle \nabla_{\theta} V^{\pi_{\theta}}(\mu)|_{\theta=\theta}, \frac{\pi_{\theta'} - \pi_{\theta}}{\pi_{\theta}} \Big\rangle.
\end{equation*}
Hence, it remains to show that
\begin{equation}
    \Big\langle \nabla_{\theta} V^{\pi_{\theta}}(\mu)|_{\theta=\theta}, \frac{\pi_{\theta'} - \pi_{\theta}}{\pi_{\theta}} \Big\rangle \le \Big\langle \nabla_{\theta} V^{\pi_{\theta}}(\mu)|_{\theta=\theta}, \theta' - \theta \Big\rangle. \label{eq: nearly concave target}
\end{equation}

According to the concavity of $\theta \to \pi_{\theta}(a^*(s) | s)$ along $\bm{d}$ {for all $s\in\cS$}, we have
\begin{align}
    \frac{\pi_{\theta'}(a^*(s) | s) - \pi_{\theta}(a^*(s) | s)}{\pi_{\theta}(a^*(s) | s)} &\le \frac{\Big\langle \nabla_{\theta} \pi_{\theta}(a^*(s) | s)|_{\theta=\theta}, \theta' - \theta\Big\rangle}{\pi_{\theta}(a^*(s) | s)} \nonumber \\
    &= \frac{\sum_{a} \frac{\partial \pi_{\theta}(a^*(s) | s)}{\partial \theta_{s, a}} \cdot (\theta'_{s, a} - \theta_{s, a})}{\pi_{\theta}(a^*(s) | s)} \nonumber \\
    &= (1 - \pi_{\theta}(a^*(s) | s)) \cdot (\theta'_{s, a^*(s)} - \theta_{s, a^*(s)})
    - \sum_{a \neq a^*(s)} \pi_{\theta}(a|s) \cdot (\theta'_{s, a} - \theta_{s, a}) \nonumber \\
    &= (\theta'_{s, a^*(s)} - \theta_{s, a^*(s)})
    - \sum_{a} \pi_{\theta}(a|s) \cdot (\theta'_{s, a} - \theta_{s, a}).\label{eq: local concave a^*}
\end{align}


On the other hand, since $\theta \to \pi_{\theta}(a_i | s)$ is convex
along $\bm{d}$ for all $a_i \neq a^*(s)$, we have
\begin{align}
    \frac{\pi_{\theta'}(a_i | s) - \pi_{\theta}(a_i | s)}{\pi_{\theta}(a_i | s)} &\ge \frac{\Big\langle \nabla_{\theta} \pi_{\theta}(a_i | s)\Big|_{\theta=\theta}, \theta' - \theta\Big\rangle}{\pi_{\theta}(a_i | s)} \nonumber \\
    &= \frac{\sum_{a} \frac{\partial \pi_{\theta}(a_i | s)}{\partial \theta_{s, a}} \cdot (\theta'_{s, a} - \theta_{s, a})}{\pi_{\theta}(a_i | s)} \nonumber \\
    &= (1 - \pi_{\theta}(a_i | s)) \cdot (\theta'_{s, a_i} - \theta_{s, a_i})
    - \sum_{a \neq a_i} \pi_{\theta}(a|s) \cdot (\theta'_{s, a} - \theta_{s, a}) \nonumber \\
    &= (\theta'_{s, a_i} - \theta_{s, a_i})
    - \sum_{a} \pi_{\theta}(a|s) \cdot (\theta'_{s, a} - \theta_{s, a}). \label{eq: local concave a_i}
\end{align}

{
%
Finally, we put everything together.
By \cref{eq: local concave a^*,eq: local concave a_i,eq:neg,eq:pos},
we have
\begin{align}
    \Big\langle \nabla_{\theta} V^{\pi_{\theta}}(\mu)\Big|_{\theta=\theta}, \frac{\pi_{\theta'} - \pi_{\theta}}{\pi_{\theta}} \Big\rangle 
    &= \sum_{s, a} \left.\frac{\partial V^{\pi_{\theta}}(\mu)}{\partial \theta_{s, a}}\right|_{\theta=\theta_{s, a}} \frac{\pi_{\theta'}(a|s) - \pi_{\theta}(a|s)}{\pi_{\theta}(a|s)}  \nonumber \\
    &\le \sum_{s, a} \left.\frac{\partial V^{\pi_{\theta}}(\mu)}{\partial \theta_{s, a}}\right|_{\theta=\theta_{s, a}} \Big( (\theta'_{s, a} - \theta_{s, a}) -  \sum_{a'} \pi_{\theta}(a'|s) \cdot (\theta'_{s, a'} - \theta_{s, a'}) \Big) \label{eq: nearly concave put everything 1} \\
    &= \Big\langle \nabla_{\theta} V^{\pi_{\theta}}(\mu)\Big|_{\theta=\theta}, \theta' - \theta \Big\rangle\nonumber \\
    &\quad - \sum_{s}\sum_{a'} \pi_{\theta}(a'|s) (\theta'_{s, a'} - \theta_{s, a'}) \cdot \sum_{a} \left.\frac{\partial V^{\pi_{\theta}}(\mu)}{\partial \theta_{s, a}}\right|_{\theta=\theta_{s, a}} \nonumber\\
    &= \Big\langle \nabla_{\theta} V^{\pi_{\theta}}(\mu)\Big|_{\theta=\theta}, \theta' - \theta \Big\rangle, \label{eq: nearly concave put everything 3}
\end{align}
where \cref{eq: nearly concave put everything 1} is from
\cref{eq: local concave a^*,eq: local concave a_i,eq:neg,eq:pos},
	and
	\cref{eq: nearly concave put everything 3} is because }
\begin{align*}
    \sum_{a} \left.\frac{\partial V^{\pi_{\theta}}(\mu)}{\partial \theta_{s, a}}\right|_{\theta = \theta_{s, a}} &= \sum_{a} \frac{1}{1 - \gamma} d_{\mu}^{\pi_{\theta}}(s) \pi_{\theta}(a | s) A^{\pi_{\theta}}(s, a) \\
    &= \frac{1}{1 - \gamma} d_{\mu}^{\pi_{\theta}}(s)\sum_{a}  \pi_{\theta}(a | s) A^{\pi_{\theta}}(s, a) = 0
\end{align*}
from \Cref{lemma:softmax_pg,lemma:sum_pi_A}.
We have therefore obtained (\ref{eq: nearly concave target}) and thus
the proof is completed.
\end{proof}

\stationarylemma*
\begin{proof}[Proof of \Cref{lemma: enter local concavity}]
	We first show the $\theta^{(t)}_{s, a^*(s)} - \theta^{(t)}_{s, a} > M$ part.
	{By \Cref{cor:policy convergence}, almost surely we have}
\begin{equation}
	\lim_{t \to \infty} \, \pi_{\theta}^{(t)}(a^*(s) | s) = 1, \quad
	\forall s \in \cS.
	\label{eq:pilimit}
\end{equation}
Under this event,
assume for contradiction that there is a $M > 0$ such that
	no such a finite $T$ exists.
	Then there is an infinite sequence $\{t_i\}$ such that for each
	$t_i$, there is a state $s_{t_i}$ and an action $a_{t_i} \neq
	a^*(s_{t_i})$ satisfying that $\theta_{s_{t_i},
	a^*(s_{t_i})}^{(t_i)} - \theta_{s_{t_i}, a_{t_i}}^{(t_i)} \le M$.
	Since the action space and the state space are finite, there must
	be a state $s$, an action $\tilde{a}$, and an infinite subsequence
	$\{t_{i_j}\}_{j=1}^{\infty}$ of $\{t_i\}$ such that $\theta_{s,
	a^*(s)}^{(t_{i_j})} - \theta_{s, \tilde{a}}^{(t_{i_j})} \le M$ for
	all $j \in \bbN$.
	For each of such $t_{i_j}$, we have
    \begin{equation*}
		\pi_{\theta}^{(t_{i_j})}(a^*(s) | s) = \frac{\exp(\theta_{s,
		a^*(s)}^{(t_{i_j})})}{\sum_{a} \exp(\theta_{s,
		a}^{(t_{i_j})})} \le \frac{\exp(\theta_{s,
			\tilde{a}}^{(t_{i_j})} + M)}{\exp(\theta_{s,
				\tilde{a}}^{(t_{i_j})} + M) + \exp(\theta_{s,
					\tilde{a}}^{(t_{i_j})})} = \frac{\exp(M)}{\exp(M) + 1} < 1,
    \end{equation*}
	which contradicts with \cref{eq:pilimit}.
	Therefore, with probability one, for any $M > 0$,
	there is a corresponding $T_M < \infty$ such that
	\begin{equation}
		\theta^{(t)}_{s, a^*(s)} - \theta^{(t)}_{s, a} > M,\quad
		\forall s \in \cS, a \neq a^*(s),\quad
		\forall t \geq T_M.
		\label{eq:part1}
	\end{equation}
	
	From \Cref{lemma: Q-ordering fixed}, we directly see that with
	probability one, there is $T_V < \infty$ such that
	\begin{equation}
		\label{eq:part2}
		V^{\pi_{\theta}^{(t)}}(s) > Q^*(s, a_2(s)), \quad \forall s
		\in \cS, \quad \forall t \geq T_V.
	\end{equation}

	Given any $s \in \cS$, by \Cref{cor:policy convergence} and \cref{eq:part2}, we
	know that almost surely there is a finite $T_1(s)$
	such that $V^{\pi_{\omega}^{(t)}}(s) > Q^*(s, a_2(s))$ for all $t
	\geq T_1(s)$.
    By \Cref{lemma: Q^* a_2}, it then implies
	\[
		Q^{\pi_{\omega}^{(t)}}(s, a) \le Q^*(s, a_2(s)) <
		V^{\pi_{\omega}^{(t)}}(s), \quad \forall a \neq a^*(s),\quad \forall
		t \geq T_1(s).
	\]
	Hence, we have from \cref{eq:state_action_value_function} that
	$A^{\pi_{\omega}^{(t)}}(s, a) < 0$ for all $a \neq a^*(s)$ and all
	$t \geq T_1(s)$.
	This together with \Cref{lemma:sum_pi_A} and that
	$\pi_{\omega^{(t)}}(a|s) > 0$ for all $a$ (from the nature of
	softmax parameterization) then implies $A^{\pi_{\omega}^{(t)}}(s,
	a^*(s)) > 0$.
	Therefore, by \Cref{lemma:softmax_pg}, we see that with
	probability one,
	\begin{equation}
		\left.\frac{\partial V^{\pi_{\theta}}(\mu)}{\partial\theta_{s,
		a^*(s)}}\right\rvert_{\theta = \omega^{(t)}} > 0 >
		\left.\frac{\partial V^{\pi_{\theta}}(\mu)}{\partial\theta_{s,
		a}}\right\rvert_{\theta = \omega^{(t)}},\quad \forall a \neq
		a^*(s),\quad \forall t \geq T_1(s).
		\label{eq:part3}
	\end{equation}
	By taking $T_1 \coloneqq \max_{s \in \cS} T_1(s)$, \cref{eq:part3}
	implies that with probability one,
	\begin{equation}
		\left.\frac{\partial V^{\pi_{\theta}}(\mu)}{\partial\theta_{s,
		a^*(s)}}\right\rvert_{\theta = \omega^{(t)}} > 0 >
		\left.\frac{\partial V^{\pi_{\theta}}(\mu)}{\partial\theta_{s,
		a}}\right\rvert_{\theta = \omega^{(t)}},\quad \forall a \neq
		a^*(s),\quad \forall s \in \cS,\quad \forall t \geq T_1.
		\label{eq:part3_1}
	\end{equation}

    We now discuss the part regarding (iv).
	We claim that
	\begin{claim}
		\label{claim1}
		For any $t \ge T_1$ in \cref{eq:part3_1}, if there is $(s,a')
		\in \cS \times \cA$ such that
		\begin{equation}
			\omega_{s,a^*(s)}^{(t)} - \theta_{s,a^*(s)}^{(t)} \ge
			\omega_{s,a'}^{(t)} - \theta_{s,a'}^{(t)},
			\label{eq:part4assum}
		\end{equation}
		then
		\begin{equation*}
			\omega_{s,a^*(s)}^{(t+1)} - \theta_{s,a^*(s)}^{(t+1)} \ge
			\omega_{s,a'}^{(t+1)} - \theta_{s,a'}^{(t+1)}.
		\end{equation*}
	\end{claim}

	If $a' = a^*(s)$, there is nothing to prove.
	When $a' \neq a^*(s)$,
	by considering the update \cref{algorithm:eq1}, we have
    \begin{align}
        \theta^{(t+1)}_{s,a} = \theta^{(t)}_{s,a} +
		(\omega^{(t)}_{s,a} - \theta^{(t)}_{s,a}) +
		\eta^{(t+1)}\left.\frac{\partial
			V^{\pi_{\theta}}(\mu)}{\partial\theta_{s,
			a}}\right\rvert_{\theta = \omega^{(t)}},
			\quad \forall a \in \cA. \label{eq: momentum dominant 1}
    \end{align}
	Let $\mathbb{I}^{(t+1)} \coloneqq \mathbb{I} \left\{V^{\pi_{\varphi}^{(t+1)}}(\mu) \ge V^{\pi_{\theta}^{(t+1)}}(\mu)\right \}$.
	Combining \cref{eq: momentum dominant 1} with
    \cref{algorithm:eq2,algorithm:eq3} then leads to
    \begin{align}
        \omega_{s,a^*(s)}^{(t+1)} - \theta_{s,a^*(s)}^{(t+1)} &= \mathbb{I}^{(t+1)} \frac{t}{t+3}(\theta_{s,a^*(s)}^{(t+1)} - \theta_{s,a^*(s)}^{(t)}) \nonumber \\
        &= \mathbb{I}^{(t+1)} \frac{t}{t+3}\left(\omega^{(t)}_{s,a^*(s)} - \theta^{(t)}_{s,a^*(s)} + \eta^{(t+1)}\left.\frac{\partial V^{\pi_{\theta}}(\mu)}{\partial\theta_{s, a^*(s)}}\right\rvert_{\theta = \omega^{(t)}}\right) \nonumber \\
        &\ge \mathbb{I}^{(t+1)} \frac{t}{t+3}\left(\omega^{(t)}_{s,a'}
		- \theta^{(t)}_{s,a'} + \eta^{(t+1)}\left.\frac{\partial
			V^{\pi_{\theta}}(\mu)}{\partial\theta_{s, a'}}\right\rvert_{\theta = \omega^{(t)}}\right) \label{eq: momentum dominant 2} \\
        &= \mathbb{I}^{(t+1)}\frac{t}{t+3}(\theta_{s,a'}^{(t+1)} -
		\theta_{s,a'}^{(t)})  \nonumber \\ 
        &= \omega_{s,a'}^{(t+1)} - \theta_{s,a'}^{(t+1)},\nonumber 
    \end{align}
	proving \Cref{claim1}, where (\ref{eq: momentum dominant 2}) follows
	from \cref{eq:part4assum,eq:part3_1}.

	With \Cref{claim1}, we are now ready to prove that with
	probability one, there exists a finite $T_{\omega} \ge T_1$ such
	that
	\begin{equation}
		\omega_{s, a^*(s)}^{(t)} - \theta_{s, a^*(s)}^{(t)} \ge
		\omega_{s, a}^{(t)} - \theta_{s, a}^{(t)}, \quad \forall s \in
		\cS, \quad \forall a \in \cA, \quad \forall t \geq T_{\omega}.
		\label{eq:part4}
	\end{equation}
	Suppose for contradiction that there is no such a $T_{\omega}$.
	Then, as $\cS$ and $\cA$ are finite sets, there is an infinite
	sequence $\{t_i\}$ and a state-action pair $(s,\tilde a)$ such that
	\begin{equation}
		\omega_{s,a^*(s)}^{(t_i)} - \theta_{s,a^*(s)}^{(t_i)} <
		\omega_{s,\tilde{a}}^{(t_i)} - \theta_{s,\tilde{a}}^{(t_i)},
		\quad \forall i \geq 0.
		\label{eq:part4_2}
	\end{equation}
    If there is a $\hat t \ge T_1$ such that $\omega_{s,a^*(s)}^{(\hat
	t)} - \theta_{s,a^*(s)}^{(\hat t)} \ge \omega_{s,\tilde{a}}^{(\hat
	t)} - \theta_{s,\tilde{a}}^{(\hat t)}$, \Cref{claim1} implies
	$\omega_{s,a^*(s)}^{(t)} - \theta_{s,a^*(s)}^{(t)} \ge
	\omega_{s,\tilde{a}}^{(t)} - \theta_{s,\tilde{a}}^{(t)}$ for every
	$t \ge \hat t$, which contradicts \cref{eq:part4_2}.
	Therefore, 
	\begin{equation}
		\omega_{s,a^*(s)}^{(t)} - \theta_{s,a^*(s)}^{(t)} <
		\omega_{s,\tilde{a}}^{(t)} - \theta_{s,\tilde{a}}^{(t)}, \quad
		\forall t \ge T_1.
		\label{eq:wrongbound}
	\end{equation}
	The strict inequality in \cref{eq:wrongbound} implies
	that
	\begin{equation}
		\mathbb{I}^{(t)}\neq 0, \quad \forall t \geq T_1,
		\label{eq:norestart}
	\end{equation}
	for it otherwise leads to $\omega^{(t)} - \theta^{(t)} = 0$, and
	thus $\omega_{s,a^*(s)}^{(t)} - \theta_{s,a^*(s)}^{(t)} =
	\omega_{s,\tilde{a}}^{(t)} - \theta_{s,\tilde{a}}^{(t)}$ by
	\cref{algorithm:eq3}.

    For any $N > T_1$, we have
    \begin{align}
        \theta_{s,a^*(s)}^{(N)} &= \theta_{s,a^*(s)}^{(T_1(s))} +
		\sum_{t=T_1}^{N-1} (\theta_{s,a^*}^{(t+1)} -
		\theta_{s,a^*}^{(t)})\nonumber\\
        &= \theta_{s,a^*}^{(T_1)} + \sum_{t=T_1}^{N-1} \frac{t+3}{t}(\omega_{s,a^*}^{(t+1)} - \theta_{s,a^*}^{(t+1)}) \label{eq: momentum dominant 4} \\
        &< \theta_{s,a^*}^{(T_1)} + \sum_{t=T_1}^{N-1} \frac{t+3}{t}(\omega_{s,\tilde{a}}^{(t+1)} - \theta_{s,\tilde{a}}^{(t+1)}) \label{eq: momentum dominant 5} \\
        &=\theta_{s,a^*}^{(T_1)} + \sum_{t=T_1}^{N-1} (\theta_{s,\tilde{a}}^{(t+1)} - \theta_{s,\tilde{a}}^{(t)})  \label{eq: momentum dominant 6} \\
        &= \theta_{s,a^*}^{(T_1)} - \theta_{s,\tilde{a}}^{(T_1)} + \theta_{s,\tilde{a}}^{(N)}, \label{eq: momentum dominant 7}
    \end{align}
    where \cref{eq: momentum dominant 4,eq: momentum dominant 6} used
	\cref{algorithm:eq2,algorithm:eq3,eq:norestart}, and \cref{eq:
	momentum dominant 5} is from \cref{eq:wrongbound}.
	Letting $N \rightarrow \infty$ in \cref{eq: momentum dominant
	7},
	we see a direct contradiction with \cref{eq:pilimit}.
	Hence, there is a $T_{\omega} \ge T_1$ such that \cref{eq:part4}
	holds.

	Finally, the proof is completed by letting $T = \max\{T_M, T_V,
	T_\omega\}$.
\end{proof}

\subsection{Convergence Rate of APG}
We need thew following lemma to faciliate the proof of \Cref{theorem:
MDP convergence rate}.

\begin{lemma}
\label{lemma: di_lower_bdd}
    Given any state $s \in \cS$ and a corresponding vector $\bm{d_{s, \cdot}}
	\coloneqq [d_{s, a^*(s)}, d_{s, a_2(s)}, \cdots, d_{s,
	a_{|\cA|}(s)}]$ satisfying (i) $d_{s, a^*(s)} + \sum_{i=2}^{|\cA|}
	d_{s, a_i(s)} = 0$, and (ii) $\lVert \bm{d_{s, \cdot}} \rVert_2^2 = 1$, we have
    \begin{itemize}
        \item If $d_{s, a^*(s)} > 0 \ge \max_{i \ge 2} d_{s,
			a_i(s)} $, then $\min_{i \ge 2} \{ d_{s, a^*(s)} -  d_{s, a_i(s)} \} > |\cA|^{-\frac{3}{2}}$.
        \item If $d_{s, a^*(s)} > (1+\varepsilon_r) \max_{i \ge 2}
				d_{s, a_i(s)} > 0$, then $\min_{i \ge 2} \{ d_{s, a^*(s)} -  d_{s, a_i(s)} \} > \frac{\varepsilon_r}{1+\varepsilon_r}|\cA|^{-\frac{3}{2}}$.
    \end{itemize}
\end{lemma}

\begin{proof}[Proof of \Cref{lemma: di_lower_bdd}]
We fix some state $s\in\cS$. Since $d_{s, a^*(s)} + \sum_{i=2}^{|\cA|} d_{s, a_i(s)} = 0$, we have
\begin{align}
    |d_{s, a}| \le \sum_{a': d_{a'} > 0} d_{s, a'}, \quad \forall a \in \cA.  \label{eq: di_lower_bdd_1}
\end{align}
By taking square on both side of \cref{eq: di_lower_bdd_1}, we obtain
from $\lVert \bm{d_{s, \cdot}} \rVert_2^2 = 1$ that
\begin{align}
    1
    = \sum_{a \in \cA} d_{s, a}^2
    \le \sum_{a \in \cA} \big( \sum_{a': d_{a'} > 0} d_{s, a'} \big)^2  
    = |\cA| \big( \sum_{a': d_{a'} > 0} d_{s, a'} \big)^2. \label{eq: di_lower_bdd_2}
\end{align}

Using the assumption of $d_{s, a^*(s)} > \max_{i \ge 2} d_{s, a_i(s)}$
and taking the square root of \cref{eq: di_lower_bdd_2}, we get
\begin{align}
    d_{s, a^*(s)} > \frac{1}{|\cA|} \big( \sum_{a': d_{a'} > 0} d_{s, a'} \big) \ge |A|^{-\frac{3}{2}}. \label{eq: di_lower_bdd_3}
\end{align}
Thus, if $d_{s, a^*(s)} > 0 \ge \max_{i \ge 2} d_{s, a_i(s)}$, by \cref{eq: di_lower_bdd_3}, we have
\begin{align*}
    \min_{i \ge 2} \{ d_{s, a^*(s)} -  d_{s, a_i(s)} \} \ge d_{s, a^*(s)} > |\cA|^{-\frac{3}{2}}.
\end{align*}
On the other hand, if $d_{s, a^*(s)} > (1+\varepsilon_r) \max_{i \ge 2} d_{s, a_i(s)} > 0$, again by \cref{eq: di_lower_bdd_3}, we have 
\begin{align*}
    \min_{i \ge 2} \{ d_{s, a^*(s)} -  d_{s, a_i(s)} \} > d_{s, a^*(s)} -\frac{d_{s, a^*(s)}}{(1+\varepsilon_r)} \ge  \frac{\varepsilon_r}{1+\varepsilon_r}|\cA|^{-\frac{3}{2}}.
\end{align*}
\end{proof}

With all the tools prepared, now we restate and then prove
\Cref{theorem: MDP convergence rate}.

\textbf{\Cref{theorem: MDP convergence rate} (Convergence Rate of APG; Formal)}.
Consider reinforcement learning with tabular softmax parameterized policies $\pi_{\theta}$. Under \hyperref[algorithm:APG]{APG} with {$\eta^{(t)} = \frac{t}{t+1} \cdot \frac{(1 - \gamma)^3}{16}$} {and $\mu$ initialized uniformly at random}, the following holds almost surely: There exists a finite time $T$ such that for all $t \ge T$, we have
{
    \begin{align*}
        V^{*}(\rho) - V^{{\pi_\theta^{(t)}}}(\rho) \le \frac{1}{(1-\gamma)^3} \left \| \frac{{d^{\pi^*}_{\rho}}}{\mu} \right \|_{\infty} \Bigg( \frac{2 |\cS|(|\cA| - 1)}{t^2 + |\cA| - 1} +  \frac{512|\cS| \ln^2(t) + 32\left \| 2\theta^{(T)} - (2+T)\big(\omega^{(T)} - \theta^{(T)}\big)\right \|^2}{(1-\gamma)(t+1)t} \Bigg).
    \end{align*}
}
    
\begin{proof}[Proof of \Cref{theorem: MDP convergence rate}]
{We will first show that there exists a finite time $T$ such that
	for all $t \ge T$, the restart mechanism is inactive and the
	objective function $V^{\pi_{\theta}}(\mu)$ will be
	$\frac{3}{2}$-nearly concave at each iterate along the APG update
	directions. We will then leverage \Cref{lemma:equivalent_algorithm} and \Cref{cor:ghadimi_cor1} to characterize the convergence rate.}

{Let 
	\[
		M' \coloneqq \ln \Big[ \frac{2|\cS||\cA|^2}{(1-\gamma)^2 \min_{s \in S} \mu(s)} \Big].
	\]
By \Cref{lemma: enter local concavity}, there exists a finite time $T'$ such that for all $t \ge T'$, $s\in\cS$, and $a\neq a^*(s)$, we have (i) $\theta_{s, a^*(s)} - \theta_{s, a} > M'$, (ii) $V^{\pi_{\theta}^{(t)}}(s) > Q^*(s, a_2(s))$, (iii) $\left.\frac{\partial V^{\pi_{\theta}}(\mu)}{\partial\theta_{s, a^*(s)}}\right\rvert_{\theta = \omega^{(t)}} > 0 > \left.\frac{\partial V^{\pi_{\theta}}(\mu)}{\partial\theta_{s, a}}\right\rvert_{\theta = \omega^{(t)}}$, (iv) $\omega_{s, a^*(s)}^{(t)} - \theta_{s, a^*(s)}^{(t)} \ge \omega_{s, a}^{(t)} - \theta_{s, a}^{(t)}$.}

{We first show that the restart mechanism is inactive for all $t >
	T'$. The momentum term at time $t$ is
\begin{equation}
\label{eq: momentum at time t}
\frac{t-1}{t+2} (\theta^{(t)} - \theta^{(t-1)}) = \frac{t-1}{t+2}\Big(\omega^{(t-1)}  - \theta^{(t-1)} + \eta^{(t)} \nabla_{\theta}{V^{\pi_{\theta}}(\mu)} \Big\rvert_{\theta = \omega^{(t-1)}}\Big).
\end{equation}
We can observe that \cref{eq: momentum at time t} lies in the feasible
update domain $\mathcal{U}$ for any $t > T'$ by the properties (iii)
and (iv) above. Hence, by combining \Cref{lemma:perf_diff,lemma:feasible update domain
improvement lemma,lemma:sum_pi_A}, we know that
$V^{\pi_{\varphi^{(t)}}}(\mu) \ge V^{\pi_{\theta^{(t)}}}(\mu)$,
meaning that there is no restart at $t > T'$.
Additionally, by (iii), we know that the gradient updates also lie in $\mathcal{U}$.
}

{To show that the near concavity of the objective, we separately consider two cases for each state $s\in\cS$: (I)
	$\max_{a \neq a^*(s)} \Big\{ \omega^{(T')}_{s, a} -
	\theta^{(T')}_{s, a} + \eta^{(T'+1)}\left.\frac{\partial
		V^{\pi_{\theta}}(\mu)}{\partial\theta_{s,
		a}}\right\rvert_{\theta = \omega^{(T')}} \Big\} \le 0$; (II)
		$\max_{a \neq a^*(s)} \Big\{ \omega^{(T')}_{s, a} -
		\theta^{(T')}_{s, a} + \eta^{(T'+1)}\left.\frac{\partial
			V^{\pi_{\theta}}(\mu)}{\partial\theta_{s,
			a}}\right\rvert_{\theta = \omega^{(T')}} \Big\} > 0$. For
			case (I), we can see that for any $t \ge T'$, $\max_{a
			\neq a^*(s)} \Big\{ \omega^{(t)}_{s, a} - \theta^{(t)}_{s,
			a} + \eta^{(t+1)}\left.\frac{\partial
				V^{\pi_{\theta}}(\mu)}{\partial\theta_{s,
				a}}\right\rvert_{\theta = \omega^{(t)}} \Big\} \le 0$
				by the APG updates and (iii).
				Let
\begin{equation}
\label{eq: nonpositive d_2}
    M_{(i)} \coloneqq \max\left\{2\ln\Big[2(|\mathcal{A}| - 1)\Big] - 2\ln\Big[|\cA|^{-\frac{3}{2}}\Big], M' \right\}.
\end{equation}
From \Cref{lemma: enter local concavity},
we know that there is a finite time $T_s \ge T'$ such that (i)-(iv) above with
$M'$ replaced by $M_{(i)}$ remain true
for all $t > T_s$.
Let $\bm{d^{(t)}_{s, \cdot}}
\coloneqq \frac{\theta^{(t)}_{s, \cdot} - \theta^{(t-1)}_{s,
\cdot}}{\lVert \theta^{(t)}_{s, \cdot} - \theta^{(t-1)}_{s, \cdot}
\rVert_2}$ be the scaled APG update direction at time $t$ such that
$\lVert \bm{d^{(t)}_{s, \cdot}} \rVert_2 = 1$.
By \Cref{lemma:sum conservation}, we have $\sum_{a\in\cA} d^{(t)}_{s, a} = 0$ for any $t$ and $s\in\cS$.
We also know that in this case, $\max_{a\neq a^*(s)} d^{(t)}_{s, a}
\le 0$ for all $t > T_s$. \Cref{lemma: di_lower_bdd} then leads to $\min_{a\neq a^*(s)} \{ d^{(t)}_{s, a^*(s)} - d^{(t)}_{s, a} \} > |\cA|^{-\frac{3}{2}}$ for any $t > T_s$.
By noting that $M'$ is exactly the second term in the maximum in
\cref{eq:Mc} with $C = 3/2$,
we observe that the setting of $M_{(i)}$ in \Cref{eq: nonpositive d_2}
satisfies all the conditions in \Cref{lemma:local nearly concavity},
so the objective function $\theta \to V^{\pi_{\theta}}(\mu)$ is
$\frac{3}{2}$-nearly concave along our APG update directions for all
$t > T_s$ for the state $s$ in the case (I).}

{Regarding case (II), let
\begin{equation*}
    \varepsilon_{T'} \coloneqq \frac{ \omega^{(T')}_{s, a^*(s)} -
	\theta^{(T')}_{s, a^*(s)} +{\eta^{(t'+1)}} \left.\frac{\partial V^{\pi_{\theta}}(\mu)}{\partial\theta_{s, a^*(s)}}\right\rvert_{\theta = \omega^{(T')}}}{ \max_{a \neq a^*(s)} \Big\{ \omega^{(T')}_{s, a} - \theta^{(T')}_{s, a} + {\eta^{(t'+1)}} \left.\frac{\partial V^{\pi_{\theta}}(\mu)}{\partial\theta_{s, a}}\right\rvert_{\theta = \omega^{(T')}} \Big\}} - 1.
\end{equation*}
Note that $\varepsilon_{T'} > 0$ due to (iii) and (iv) and the hypothesis of this case.
Let
\begin{equation}
\label{eq: positive d_2}
    M_{(ii)} \coloneqq \max\left\{2\ln\Big[2(|\mathcal{A}| - 1)\Big] - 2\ln\Big[\frac{\varepsilon_{T'}}{1+\varepsilon_{T'}} |\cA|^{-\frac{3}{2}}\Big], M' \right\}.
\end{equation}
Again by \Cref{lemma: enter local concavity},
we know that there is a finite time $T'_{s} \ge T'$ such that (i)-(iv) above with
$M'$ replaced by $M_{(ii)}$ remain true
for all $t > T'_{s}$.
Similarly, we demonstrate that the objective function $\theta \to
V^{\pi_{\theta}}(\mu)$ is $\frac{3}{2}$-nearly concave along our APG
update directions for all $t > T'_{s}$ for the state $s$ in this case.
Following the argument for (I), we only need to show that
$\frac{\varepsilon_{T'}}{1+\varepsilon_{T'}} |\cA|^{-\frac{3}{2}}$ is
a lower bound for $\min_{a\neq a^*(s)} \{ d^{(t)}_{s, a^*(s)} -
d^{(t)}_{s, a} \}$ for every $t > T'_s$.
We use the same definition of $\bm{d^{(t)}_{s, \cdot}}$ as above.
Then, for any $k \in \mathbb{N}_+$ and $s \in \cS$ satisfying
$\max_{a\neq a^*(s)} d^{(T'+k)}_{s, a} \le 0$, from the argument for
(I), $|\cA|^{-\frac{3}{2}}$ is a lower bound for $\min_{a\neq a^*(s)} \{ d^{(t)}_{s, a^*(s)} - d^{(t)}_{s, a} \}$.
Moreover, since $\frac{\varepsilon_{T'}}{1+\varepsilon_{T'}} < 1$,
$\frac{\varepsilon_{T'}}{1+\varepsilon_{T'}} |\cA|^{-\frac{3}{2}}$ is
also a valid lower bound for it.
On the other hand,
for any $k \in \mathbb{N}_+$ and $s \in \cS$ satisfying $\max_{a\neq
a^*(s)} d^{(T'+k)}_{s, a} > 0$, by the APG updates, (iii) and (iv)
abvove, and that there is no restart after time $T'$ (as proven above), we have
\begin{align}
    &~\frac{d^{(T'+k)}_{s, a^*(s)}}{\max_{a\neq a^*(s)} \{ d^{(T'+k)}_{s, a} \}} \nonumber\\ 
    =&~\frac{\theta^{(T'+k)}_{s, a^*(s)} - \theta^{(T'+k-1)}_{s, a^*(s)}}{\max_{a \neq a^*(s)} \{ \theta^{(T'+k)}_{s, a} - \theta^{(T'+k-1)}_{s, a} \}} \nonumber\\
    =&~ \frac{ \big(\prod_{j=1}^{k-1} \frac{T'+k-1-j}{T'+k+2-j}\big) (\omega^{(T')}_{s, a^*(s)} - \theta^{(T')}_{s, a^*(s)}) + \sum_{t=1}^{k} \Big( \big( \prod_{j=1}^{t-1} \frac{T'+k-1-j}{T'+k+2-j} \big) \eta^{(T'+k-t+1)}\overbrace{\left.\frac{\partial V^{\pi_{\theta}}(\mu)}{\partial\theta_{s, a^*(s)}}\right\rvert_{\theta = \omega^{(T'+k-t)}}}^{>0} \Big)}{ \max_{a \neq a^*(s)} \Big\{ \big(\prod_{j=1}^{k-1} \frac{T'+k-1-j}{T'+k+2-j}\big) (\omega^{(T')}_{s, a} - \theta^{(T')}_{s, a})+ \sum_{t=1}^{k} \Big( \big( \prod_{j=1}^{t-1} \frac{T'+k-1-j}{T'+k+2-j} \big) \eta^{(T'+k-t+1)} \underbrace{\left.\frac{\partial V^{\pi_{\theta}}(\mu)}{\partial\theta_{s, a}}\right\rvert_{\theta = \omega^{(T'+k-t)}}}_{<0} \Big) \Big\}} \nonumber\\
    >&~ \frac{ \big(\prod_{j=1}^{k-1} \frac{T'+k-1-j}{T'+k+2-j}\big) \Big(\omega^{(T')}_{s, a^*(s)} - \theta^{(T')}_{s, a^*(s)} + \eta^{(T'+1)}\left.\frac{\partial V^{\pi_{\theta}}(\mu)}{\partial\theta_{s, a^*(s)}}\right\rvert_{\theta = \omega^{(T')}}\Big)}{ \max_{a \neq a^*(s)} \Big\{ \big(\prod_{j=1}^{k-1} \frac{T'+k-1-j}{T'+k+2-j}\big) \Big(\omega^{(T')}_{s, a} - \theta^{(T')}_{s, a} + \eta^{(T'+1)}\left.\frac{\partial V^{\pi_{\theta}}(\mu)}{\partial\theta_{s, a}}\right\rvert_{\theta = \omega^{(T')}}\Big) \Big\}} \nonumber\\
    =&~ \frac{ \omega^{(T')}_{s, a^*(s)} - \theta^{(T')}_{s, a^*(s)} + \eta^{(T'+1)}\left.\frac{\partial V^{\pi_{\theta}}(\mu)}{\partial\theta_{s, a^*(s)}}\right\rvert_{\theta = \omega^{(T')}}}{ \max_{a \neq a^*(s)} \Big\{ \omega^{(T')}_{s, a} - \theta^{(T')}_{s, a} + \eta^{(T'+1)}\left.\frac{\partial V^{\pi_{\theta}}(\mu)}{\partial\theta_{s, a}}\right\rvert_{\theta = \omega^{(T')}} \Big\}} \nonumber\\
    =&~ 1 + \varepsilon_{T'}. \label{eq: epsilon_r}
\end{align}
Therefore, by \Cref{lemma: di_lower_bdd} and \cref{eq: epsilon_r}, we
obtain $\min_{a\neq a^*(s)} \{ d^{(t)}_{s, a^*(s)} - d^{(t)}_{s, a} \}
> \frac{\varepsilon_r}{1+\varepsilon_r} |\cA|^{-\frac{3}{2}}$ for any
$t > T'_s$.
Finally, let $T \coloneqq \max_{s\in\cS} \max\{T_s, T'_s\} + 1$, we
obtain a finite time $T$ (since $|\mathcal{S}|$ is finite) such that
the objective function $V^{\pi_{\theta}}(\mu)$ is indeed
$\frac{3}{2}$-nearly concave along the APG update directions and the
restart mechanism is inactive for all $t \ge T$.
}

    As the objective function is $\frac{3}{2}$-nearly concave along the APG update directions and there will be no restart mechanism for all $t \ge T$,
we now apply \Cref{lemma:equivalent_algorithm} to establish the
asymptotic convergence rate.
	It is necessary to account for a shift in the initial step size due to the passage of time $T$. 
    Specifically, if we divide the update process into two phases based on time $T$, the latter phase will commence with a modified step size. 
    {Specifically, the latter phase could be considered as  \Cref{algorithm:APG} with $\theta^{(0)} = \theta^{(T)}$, $\omega^{(0)} = \omega^{(T)}$ and $\eta^{(0)} = \eta^{(T+1)}$. }
    {We define the surrogate solution $\theta^{**(t)}$ at time $t$ such that $\theta^{**(t)}_{s,
	\cdot} \coloneqq [\theta^{**(t)}_{s, a^*(s)}, \theta^{**(t)}_{s,
	a_2(s)}, \cdots, \theta^{**(t)}_{s, a_{|\cA|}(s)}]$
	satisfies $\theta^{**(t)}_{s,\cdot}= [2\ln(t), 0, 0, \cdots, 0]$
	for all $s \in \cS$.}

    By \Cref{lemma:mdp_smoothness,lemma:equivalent_algorithm} and
	\Cref{cor:ghadimi_cor1} with $T_{\text{shift}} = T$ and $\theta^{**} = \theta^{**(t)}$,  we have that 
    {
    \begin{align}          
        V^{\pi^{(t+T)}_{\theta^{**}}}(\mu) - V^{\pi^{(t+T)}_{\theta}}(\mu) 
        &\le \frac{16\left \| 2\theta^{**(t+T)} - 2\theta^{(T)} - (2+T)\big(\omega^{(T)} - \theta^{(T)}\big)\right \|^2}{(1-\gamma)^3(t+T+1)(t+T)} \nonumber \\
		&\le \frac{{128} \left \| \theta^{**(t+T)} \right\|^2
		+ {32}\left \| 2\theta^{(T)} - (2+T)\big(\omega^{(T)} - \theta^{(T)}\big)\right \|^2}{(1-\gamma)^3(t+T+1)(t+T)} \nonumber \\
		&= \frac{{512}|\cS| \ln^2(t+T) + {32}\left \| 2\theta^{(T)} - (2+T)\big(\omega^{(T)} - \theta^{(T)}\big)\right \|^2}{(1-\gamma)^3(t+T+1)(t+T)}. \label{eq: theorem convergence rate 1-1}
    \end{align}
    }

    Additionally, the sub-optimality gap between the optimal
	$V^*$ and the surrogate optimal solution can be bounded by

    \begin{align}
        V^*(\mu) - V^{\pi^{(t)}_{\theta^{**}}}(\mu) &= \frac{1}{1 - \gamma} \sum_{s} d^{\pi^*}_{\mu}(s) \sum_{a} (\pi^*(a | s) - \pi^{(t)}_{\theta^{**}}(a | s)) \cdot A^{\pi^{(t)}_{\theta^{**}}}(s ,a) \label{eq: theorem convergence rate 1} \\
        &\le \frac{1}{1 - \gamma} \sum_{s} \underbrace{d^{\pi^*}_{\mu}(s)}_{\le 1} \sum_{a} \lvert \pi^*(a | s) - \pi^{(t)}_{\theta^{**}}(a | s)\rvert  \cdot \underbrace{\lvert A^{\pi^{(t)}_{\theta^{**}}}(s ,a)\rvert}_{\le \frac{1}{1-\gamma} \text{by \Cref{lemma: reward range to value range}}} \nonumber \\
        &\le \frac{1}{(1 - \gamma)^2} \sum_{s} \sum_{a} \lvert \pi^*(a | s) - \pi^{(t)}_{\theta^{**}}(a | s)\rvert \nonumber \\
        &= \frac{1}{(1 - \gamma)^2} \sum_{s} \Big(\lvert \pi^*(a^*(s) | s) - \pi^{(t)}_{\theta^{**}}(a^*(s) | s)\rvert + \sum_{a \neq a^*(s)} \lvert \pi^*(a | s) - \pi^{(t)}_{\theta^{**}}(a | s)\rvert \Big) \nonumber \\
		&= \frac{1}{(1 - \gamma)^2} \sum_{s} {2}\Big(1 - \pi^{(t)}_{\theta^{**}}(a^*(s) | s)\Big) \label{eq: theorem convergence rate 2-1}\\
        &\le \frac{2 |\cS|}{(1 - \gamma)^2} \left( 1 - \frac{\exp(2\ln(t))}{\exp(2\ln(t)) + \exp(0) \cdot (|\cA| - 1)} \right) \label{eq: theorem convergence rate 2}\\
        &= \frac{2 |\cS|}{(1 - \gamma)^2}\left( \frac{|\cA| - 1}{t^2 + |\cA| - 1} \right), \label{eq: theorem convergence rate 3}
    \end{align}
    
    {where \cref{eq: theorem convergence rate 1} holds by
	\Cref{lemma:perf_diff} and \Cref{lemma:sum_pi_A}, \cref{eq:
	theorem convergence rate 2-1} holds due to $\pi^*(a^*(s) | s) = 1$
and $\pi^*(a | s) = 0$ for all $a \neq a^*(s)$.}


    Finally, by \Cref{lemma:performance_difference_in_rho}, \cref{eq: theorem convergence rate 1-1,eq: theorem convergence rate 3}, we can then change
	the convergence rate bound for $V^{\pi}(\mu)$ to for $V^{\pi}(\rho)$ and get the convergence rate:
 {
    \begin{align*}
        V^{*}(\rho) - V^{{\pi_\theta^{(t)}}}(\rho)
        &\le \frac{1}{1-\gamma} \cdot \left \| \frac{{d^{\pi^*}_{\rho}}}{\mu} \right \|_{\infty} \cdot \left ( V^{*}(\mu) - V^{{\pi_\theta^{(t)}}}(\mu) \right ) \\
        &= \frac{1}{1-\gamma} \cdot \left \| \frac{{d^{\pi^*}_{\rho}}}{\mu} \right \|_{\infty} \cdot \left(\left( V^{{\pi_\theta^{**(t)}}}(\mu) - V^{{\pi_\theta^{(t)}}}(\mu) \right) + \left( V^{*}(\mu) - V^{{\pi_\theta^{**(t)}}}(\mu) \right)\right) \\
        &\le \frac{1}{1-\gamma} \cdot \left \| \frac{{d^{\pi^*}_{\rho}}}{\mu} \right \|_{\infty} \\
        &\quad \cdot \Bigg( \frac{2 |\cS|}{(1 - \gamma)^2} \frac{|\cA| - 1}{t^2 + |\cA| - 1}  \\
        &\quad\quad + \frac{512|\cS| \ln^2(t) + 32\left \| 2\theta^{(T)} - (2+T)\big(\omega^{(T)} - \theta^{(T)}\big)\right \|^2}{(1-\gamma)^3(t+1)t}\Bigg) \\
        &= \frac{1}{(1-\gamma)^3} \left \| \frac{{d^{\pi^*}_{\rho}}}{\mu} \right \|_{\infty} \Bigg( \frac{2 |\cS|(|\cA| - 1)}{t^2 + |\cA| - 1} +  \frac{512|\cS| \ln^2(t) + 32\left \| 2\theta^{(T)} - (2+T)\big(\omega^{(T)} - \theta^{(T)}\big)\right \|^2}{(1-\gamma)(t+1)t} \Bigg) 
    \end{align*}
}
for all $t \ge T$, and the proof is complete.
\end{proof}

\endgroup

\newpage
\section{APG With Time-Varying Step Sizes}
\label{app:add-adaptive-apg}

\begingroup
\allowdisplaybreaks

\subsection{Asymptotic Convergence of APG With Time-Varying Step Sizes}

\begin{lemma}
    \label{lemma: NS2-apg}
    Given any $t \in \mathbb{N}_{+}$, let $\eta^{(t)} = \beta^t \cdot \frac{(1-\gamma)^3}{8}$ for any $\exp(\frac{1}{8\sqrt{|\cA|}|\cS|} \cdot \frac{1-\gamma}{4 C_\infty - (1-\gamma)}) > \beta > 1$ and $K = \frac{1}{4\sqrt{|\cA|}|\cS|} \cdot \frac{1-\gamma}{4 C_\infty - (1-\gamma)} + 2\ln \beta$ where $C_\infty \coloneqq \max_{\pi} \lVert \frac{d^{\pi}_\mu}{\mu} \rVert_\infty$. Consider
    \begin{align*}
        \theta &\leftarrow \omega + \Big[ \eta^{(t)} \nabla_{\theta}{V^{\pi_{\theta}}(\mu)} \Big\rvert_{\theta = \omega} \Big]^{+K} 
    \end{align*}
    and denote $\omega_\zeta \coloneqq \omega + \zeta \cdot (\theta - \omega)$ with some $\zeta \in [0,1]$. Under softmax parameterization, we have,
    \begin{align*}
        \lVert \nabla_\theta V^{\pi_\theta} \rvert_{\theta = \omega_\zeta} \rVert_2 \le 2 \lVert \nabla_\theta V^{\pi_\theta} \rvert_{\theta = \omega} \rVert_2.
    \end{align*}
\end{lemma}
\begin{proof}[Proof of \Cref{lemma: NS2-apg}]
By the proofs of Lemmas 3 and 7 in \citep{mei2021leveraging}, we know that it suffices to replace the upper bound of $\lVert \theta' - \theta \rVert_{2}$ in (234) in \citep{mei2021leveraging} with our updates. Additionally, the upper bound of the norm of our updates is $\sqrt{|\cS||\cA|} K$. Thus, we have the bound as follows:
\begin{equation}
\label{eq: unsimplified}
    \lVert \nabla_\theta V^{\pi_\theta} \rvert_{\theta = \omega_\zeta} \rVert_2 \le \frac{1}{1 - \frac{4 C_{\infty} - (1-\gamma)}{1-\gamma} \sqrt{|\cS|} \sqrt{|\cS||\cA|} K} \lVert \nabla_\theta V^{\pi_\theta} \rvert_{\theta = \omega} \rVert_2.
\end{equation}
We elaborate on the bound of $\frac{1}{1 - \frac{4 C_{\infty} - (1-\gamma)}{1-\gamma} |\cS| \sqrt{|\cA|} K}$,
\begin{align}
    \frac{1}{1 - \frac{4 C_{\infty} - (1-\gamma)}{1-\gamma}  |\cS| \sqrt{|\cA|} K} &= \frac{1}{1 - \frac{4 C_{\infty} - (1-\gamma)}{1-\gamma}  |\cS| \sqrt{|\cA|} (\frac{1}{4\sqrt{|\cA|}|\cS|} \cdot \frac{1-\gamma}{4 C_\infty - (1-\gamma)} + 2\ln \beta)} \nonumber \\
    &= \frac{1}{1 - \frac{1}{4} - \frac{4 C_{\infty} - (1-\gamma)}{1-\gamma}  |\cS| \sqrt{|\cA|} 2 \ln\beta} \nonumber \\
    &\le \frac{1}{1 - \frac{1}{4} - \frac{4 C_{\infty} - (1-\gamma)}{1-\gamma}  |\cS| \sqrt{|\cA|}  \cdot 2 \cdot (\frac{1}{8\sqrt{|\cA|}|\cS|} \cdot \frac{1-\gamma}{4 C_\infty - (1-\gamma)})} \label{eq1: plug in beta} \\
    &= \frac{1}{1 - \frac{1}{2}} = 2, \label{eq2: upper bound of the constant}
\end{align}
where \cref{eq1: plug in beta} holds by plugging in the upper bound of $\beta$.
Hence, we combine \cref{eq: unsimplified,eq2: upper bound of the constant} and obtain the results.
\end{proof}

\begin{lemma}
\label{lemma: convergence to stationarity clipping}
{Consider the setting of \Cref{theorem: adaptive apg convergence rate}.
Then, almost surely,
the limits of both $V^{\pi_{\omega}^{(t)}}(\mu)$
and $V^{\pi_{\theta}^{(t)}}(\mu)$ when $t$ approaches infinity exist,
$\lim_{t\rightarrow \infty} V^{\pi_{\omega}^{(t)}}(\mu) =\lim_{t\rightarrow \infty} V^{\pi_{\theta}^{(t)}}(\mu)$,
and
\begin{equation}
   \lim_{t\rightarrow \infty} \big\lVert \nabla_\theta
   V^{\pi_{\theta}}(s)\rvert_{\theta=\omega^{(t)}} \big\rVert =0,\quad
   \lim_{t\rightarrow \infty} \big\lVert \nabla_\theta V^{\pi_{\theta}}(s)\rvert_{\theta=\theta^{(t)}} \big\rVert =0.
   \label{eq:grad0 clipping}
\end{equation}
}
\end{lemma}

\begin{proof}[Proof of \Cref{lemma: convergence to stationarity clipping}]
{Following the argument in \Cref{lemma: convergence to stationarity}, we have $\lim_{t\rightarrow \infty} V^{\pi_{\omega}^{(t)}}(\mu) =\lim_{t\rightarrow \infty} V^{\pi_{\theta}^{(t)}}(\mu)$.}
We prove $\lim_{t\rightarrow \infty} \big\lVert \nabla_\theta V^{\pi_{\theta}}(s)\rvert_{\theta=\omega^{(t)}} \big\rVert =0$ by contradiction. Suppose that the limit of $\big\lVert \nabla_\theta V^{\pi_{\theta}}(s)\rvert_{\theta=\omega^{(t)}} \big\rVert$ does not exist or is not $0$. Then, we know that there exists a sequence $\{t_i\}_{i=1}^{\infty}$ and $\varepsilon_0 > 0$ such that for every $t_i$, there is a state-action pair $(s_{t_i}, a_{t_i})$ satisfying
\begin{equation}
\label{eq: grad bounded away from 0}
    \left.\frac{\partial V^{\pi_{\theta}}(\mu)}{\partial \theta_{s_{t_i}, a_{t_i}}}\right|_{\theta = \omega^{(t_i)}}> \varepsilon_0.
\end{equation}
{Note that \cref{eq: grad bounded away from 0} is satisfied since the summation of the gradient over all actions at any given state $s$ is zero. We define $\partial_{s,
a}^{(t)} \coloneqq \frac{\partial V^{\pi_{\theta}}(\mu)}{\partial
	\theta_{s, a}}|_{\theta = \omega^{(t)}}$ and $Z^{(t)}_{\phi}(s) \coloneqq \sum_{a\in\cA} \exp(\phi_{s, a}^{(t)})$ for any policy parameter $\phi^{(t)}$. Then, for any $t_i \in \{t_i\}_{i=1}^{\infty}$,
\begin{align}
    V^{\pi_{\theta}^{(t_i + 1)}}(\mu) - V^{\pi_{\omega}^{(t_i)}}(\mu) &= \frac{1}{1-\gamma} \sum_{s} d_{\mu}^{\pi_{\theta}^{(t_i + 1)}}(s) \sum_{a\in\cA} \pi_{\theta}^{(t_i + 1)}(a | s) A^{\pi_{\omega}^{(t_i)}}(s, a) \label{eq1: omega grad norm to 0 PDL} \\
    &\ge \frac{1}{1-\gamma} d_{\mu}^{\pi_{\theta}^{(t_i + 1)}}(s_{t_i}) \sum_{a\in\cA} \pi_{\theta}^{(t_i + 1)}(a | s_{t_i}) A^{\pi_{\omega}^{(t_i)}}(s_{t_i}, a) \label{eq2: omega grad norm to 1 state-wise improvement} \\
    &\ge \mu(s_{t_i}) \sum_{a\in\cA} \pi_{\theta}^{(t_i + 1)}(a | s_{t_i}) A^{\pi_{\omega}^{(t_i)}}(s_{t_i}, a) \label{eq3: omega grad norm to 1 d_mu mu} \\
    &= \mu(s_{t_i}) \frac{Z^{(t_i)}_{\omega}(s_{t_i})}{Z^{(t_i+1)}_{\theta}(s_{t_i})} \sum_{a\in\cA} \frac{\exp(\omega^{(t_i)}_{s_{t_i}, a}) + \text{clip}\{\eta^{(t_i+1)} \partial^{(t_i)}_{s_{t_i}, a}, -K, K\}}{Z^{(t_i)}_{\omega}(s_{t_i})} A^{\pi_{\omega}^{(t_i)}}(s_{t_i}, a) \nonumber \\
    &\ge \mu(s_{t_i}) \frac{Z^{(t_i)}_{\omega}(s_{t_i})}{Z^{(t_i+1)}_{\theta}(s_{t_i})} \left[\Big(\frac{\exp(\omega^{(t_i)}_{s_{t_i}, a_{t_i}}) + \text{clip}\{\eta^{(t_i+1)} \partial^{(t_i)}_{s_{t_i}, a_{t_i}}, -K, K\}}{Z^{(t_i)}_{\omega}(s_{t_i})}\Big) A^{\pi_{\omega}^{(t_i)}}(s_{t_i}, a_{t_i})\right. \nonumber \\
    &\qquad\qquad\qquad\qquad\quad + \left. \sum_{a \neq a_{t_i}} \frac{\exp(\omega^{(t_i)}_{s_{t_i}, a})}{Z^{(t_i)}_{\omega}(s_{t_i})} A^{\pi_{\omega}^{(t_i)}}(s_{t_i}, a)\right] \label{eq4: omega grad norm to 1 update sign} \\
    &= \mu(s_{t_i}) \frac{Z^{(t_i)}_{\omega}(s_{t_i})}{Z^{(t_i+1)}_{\theta}(s_{t_i})} (\exp(\min\{\eta^{(t_i+1)} \partial^{(t_i)}_{s_{t_i}, a_{t_i}}, K\}) - 1) \pi_{\omega}^{(t_i)}(a_{t_i} | s_{t_i}) A^{\pi_{\omega}^{(t_i)}}(s_{t_i}, a_{t_i}) \label{eq5: omega grad norm to 1 sum 0} \\
    &\ge \mu(s_{t_i})\frac{\sum_{a\in\cA} \exp(\omega_{s_{t_i}, a}^{(t_i)})}{\sum_{a\in\cA} \exp(\omega_{s_{t_i}, a}^{(t_i)} + K)} (\exp(\min\{\eta^{(t_i+1)} \partial^{(t_i)}_{s_{t_i}, a_{t_i}}, K\}) - 1) \frac{(1-\gamma)\varepsilon_0}{d_{\mu}^{\pi_{\theta}}(s_{t_i})}  \label{eq6: omega grad norm to 1} \\
    &\ge \mu(s_{t_i}) \exp(-K) (\exp(\min\{\eta^{(t_i+1)} \varepsilon_0, K\}) - 1) \frac{(1-\gamma)\varepsilon_0}{d_{\mu}^{\pi_{\theta}}(s_{t_i})} > 0, \label{eq7: omega grad norm to 1 bounded away from 0}
\end{align}
where \cref{eq1: omega grad norm to 0 PDL} holds by \Cref{lemma:perf_diff}, \cref{eq2: omega grad norm to 1 state-wise improvement} holds by \Cref{lemma:grad update improvement}, \cref{eq3: omega grad norm to 1 d_mu mu} holds by \Cref{lemma:lower_bound_of_state_visitation_distribution}, \cref{eq4: omega grad norm to 1 update sign} is due to the fact that the signs of the updates are the same as the advantage function values with respect to the corresponding actions (see \Cref{lemma:perf_diff,lemma:softmax_pg}), \cref{eq5: omega grad norm to 1 sum 0} is true because of $\partial^{(t_i)}_{s_{t_i}, a_{t_i}} > 0$ and \Cref{lemma:sum_pi_A}, \cref{eq6: omega grad norm to 1} holds due to that the maximum magnitude of the updates is $K$ by clipping, \cref{eq: grad bounded away from 0}, and \Cref{lemma:softmax_pg}. We can observe that \cref{eq7: omega grad norm to 1 bounded away from 0} is a value bounded away from $0$, which leads to a constant improvement at time $t_i$. This contradicts with $\lim_{t\rightarrow \infty} V^{\pi_{\omega}^{(t)}}(\mu) =\lim_{t\rightarrow \infty} V^{\pi_{\theta}^{(t)}}(\mu)$. Hence, we have
\begin{equation}
\label{eq: limit grad omega to 0}
    \lim_{t\rightarrow \infty} \big\lVert \nabla_\theta V^{\pi_{\theta}}(s)\rvert_{\theta=\omega^{(t)}} \big\rVert =0.
\end{equation}}

{Finally, we show that $\lim_{t\rightarrow \infty} \big\lVert \nabla_\theta V^{\pi_{\theta}}(s)\rvert_{\theta=\theta^{(t)}} \big\rVert =0$. By the triangle inequality, we have
\begin{equation}
\label{eq: triangle}
    \big\lVert \nabla_\theta V^{\pi_{\theta}}(s)\rvert_{\theta=\theta^{(t)}} \big\rVert  \le \big\lVert \nabla_\theta V^{\pi_{\theta}}(s)\rvert_{\theta=\theta^{(t)}} - \nabla_\theta V^{\pi_{\theta}}(s)\rvert_{\theta=\omega^{(t-1)}} \big\rVert + \big\lVert \nabla_\theta V^{\pi_{\theta}}(s)\rvert_{\theta=\omega^{(t-1)}} \big\rVert.
\end{equation}
By \Cref{definition: non-uniform smoothness in mei,lemma: smoothness equivalence,lemma:NS1,lemma: NS2-apg}, we have
\begin{equation}
\label{eq: nabla diff}
    \big\lVert \nabla_\theta V^{\pi_{\theta}}(s)\rvert_{\theta=\theta^{(t)}} - \nabla_\theta V^{\pi_{\theta}}(s)\rvert_{\theta=\omega^{(t-1)}} \big\rVert \le \Big[ 3 + \frac{4 \cdot (C_\infty - (1-\gamma))}{1-\gamma} \sqrt{|\cS|}  \Big]\cdot 2 \cdot \big\lVert \nabla_\theta V^{\pi_{\theta}}(s)\rvert_{\theta=\omega^{(t-1)}} \big\rVert \cdot \big\lVert \theta^{(t)} - \omega^{(t-1)}\big\rVert.
\end{equation}
By the clipped update, we know that $\big\lVert \theta^{(t)} - \omega^{(t-1)}\big\rVert$ is always upper bounded by $\sqrt{|\cS||\cA|} K$. Thus, \cref{eq: nabla diff} approaches $0$ as $t\rightarrow\infty$ by \cref{eq: limit grad omega to 0}, leading to that the RHS of \cref{eq: triangle} goes to $0$ as $t\rightarrow\infty$. Hence, we have proven \cref{eq:grad0 clipping}.}
\end{proof}

\begin{theorem}
   [\textbf{Asymptotic Convergence Under Softmax Parameterization}]
   \label{theorem:global_convergence_adaptive_apg}
   Consider the setting of \Cref{theorem: adaptive apg convergence rate},
   the following holds almost surely: $\lim_{t\rightarrow \infty} V^{\pi_{\theta}^{(t)}}(s) = V^{*}(s)$, for all $s\in\cS$.
\end{theorem}

\begin{proof}[Proof of \Cref{theorem:global_convergence_adaptive_apg}]
{We note that the proofs of \Cref{theorem:convergeoptimal} and \Cref{lemma: monotone limits exist}-\ref{lemma:compare pi(a+) and pi(a-)} are all based on the results of \Cref{lemma: convergence to stationarity}. Now we have provided the counterpart of \Cref{lemma: convergence to stationarity} as \Cref{lemma: convergence to stationarity clipping}, the requirement for the step size upper bounded in these lemmas can be removed.
By using arguments similar to those in the proof of \Cref{theorem:convergeoptimal} and \Cref{lemma: monotone limits exist}-\ref{lemma:compare pi(a+) and pi(a-)}, we have
the results. 
}
\end{proof}

\subsection{Convergence Rate of APG With Time-Varying Step Sizes}

\begin{restatable}{lemma}{stationarylemma2}
\label{lemma: adaptive apg enter local concavity}
   Consider a tabular softmax parameterized policy.
	Under \Cref{assump:unique_optimal} and the setting of \Cref{theorem:global_convergence_adaptive_apg}, 
 the following holds almost surely: Given any $M > 0$, there exists a
 finite time T such that for all $t \ge T$, $s\in\cS$, and $a \neq
 a^*(s)$, we have (i) $\theta_{s, a^*(s)}^{(t)} - \theta_{s, a}^{(t)} > M$, (ii) $V^{\pi_{\theta}^{(t)}}(s) > Q^*(s, a_2(s))$, (iii) $\left.\frac{\partial V^{\pi_{\theta}}(\mu)}{\partial\theta_{s, a^*(s)}}\right\rvert_{\theta = \omega^{(t)}} > 0 > \left.\frac{\partial V^{\pi_{\theta}}(\mu)}{\partial\theta_{s, a}}\right\rvert_{\theta = \omega^{(t)}}$, (iv) $\omega_{s, a^*(s)}^{(t)} - \theta_{s, a^*(s)}^{(t)} \ge \omega_{s, a}^{(t)} - \theta_{s, a}^{(t)}$.
\end{restatable}
\begin{proof}[Proof of \Cref{lemma: adaptive apg enter local concavity}]
Applying \Cref{theorem:global_convergence_adaptive_apg} and the logic in its proof to the argument of the proof of \Cref{lemma: enter local concavity}, we can obtain the stated results.
\end{proof}

\adaptiveapg*
\begin{proof}[Proof of \Cref{theorem: adaptive apg convergence rate}]

By \Cref{lemma:perf_diff}, we have
\begin{align*}
    V^{*}(\mu) - V^{\pi_{\theta}^{(t)}}(\mu) &= \frac{1}{1 - \gamma} \sum_{s, a} d_{\mu}^{\pi^*}(s) \pi^*(a|s) A^{\pi_{\theta}^{(t)}}(s, a) \\
    &= \frac{1}{1 - \gamma} \sum_{s, a} d_{\mu}^{\pi^*}(s) (\pi^*(a|s) - \pi_{\theta}^{(t)}(a|s)) A^{\pi_{\theta}^{(t)}}(s, a) \\
    &\le \frac{1}{1 - \gamma} \sum_{s, a} \Bigg | d_{\mu}^{\pi^*}(s) (\pi^*(a|s) - \pi_{\theta}^{(t)}(a|s)) A^{\pi_{\theta}^{(t)}}(s, a) \Bigg | \\
    &\le \frac{2}{(1 - \gamma)^2} \left\lVert {d_{\mu}^{\pi^*}} \right\rVert_{\infty} \sum_{s, a} \Bigg | \pi^*(a|s) - \pi_{\theta}^{(t)}(a|s) \Bigg | \\
    &= \frac{2}{(1 - \gamma)^2} \left\lVert {d_{\mu}^{\pi^*}} \right\rVert_{\infty} \sum_{s} \Bigg( \Bigg | \pi^*(a^*|s) - \pi_{\theta}^{(t)}(a^*|s) \Bigg | + \sum_{a \neq a^*} \Bigg | \pi^*(a|s) - \pi_{\theta}^{(t)}(a|s) \Bigg | \Bigg) \\ 
    &= \frac{2}{(1 - \gamma)^2} \left\lVert {d_{\mu}^{\pi^*}} \right\rVert_{\infty} \sum_{s} 2(1-\pi_\theta^{(t)}(a^*|s)).
\end{align*}

To obtain the desired outcome, it is sufficient to show that there
exist $\bar{c}, 
\tilde{c} > 0$ such that
\begin{equation}
   1 - \pi_{\theta}^{(t)}(a^*(s) | s) \le \tilde{c} e^{-\bar{c}t}
  \label{eq:lineardecay}
   \end{equation}
   for
all $s \in \cS$ and all $t \ge T$.
Given
\begin{align*}
    M \coloneqq \ln \frac{|\cA|-1}{\exp\Big( \frac{1}{4\sqrt{|\cA|}|\cS|} \cdot \frac{1-\gamma}{4 C_\infty - (1-\gamma)} \Big)-1} > 0,
\end{align*}
where $C_\infty \coloneqq \max_{\pi} \lVert \frac{d^{\pi}_\mu}{\mu} \rVert_\infty $.
Selecting $T_M$ from \Cref{lemma: adaptive apg enter local concavity} such that $\theta_{s, a^*(s)}^{(t)} - \theta_{s, a}^{(t)} > M$ for all $t \ge T_M, s \in \cS$, and $a \in \cA$, for contradiction we assume that after time-step $T_M$, for any $\bar{c}, \tilde{c} > 0$ there exists some state $s'$ and an infinite time sequence $\{t_i\}_{i=1}^{\infty}$ such that
\begin{equation}
\label{eq:violate}
1 - \pi_{\theta}^{(t)}(a^*(s') | s') > \tilde{c} e^{-\bar{c}t},\quad \forall t \in \{ t_i\}_{i=1}^{\infty}. 
\end{equation}
Define
\begin{align*}
    C \coloneqq \inf_{s' \in \cS, t \ge T} \Big\{ \frac{(1-\gamma)^2}{8} \cdot d^{\pi_\omega^{(t)}}_{\mu}(s') \cdot \pi_\omega^{(t)}(a^*(s')|s') \cdot (Q^{\pi_{\theta}^{(t)}}(s', a^*(s')) - \max_{a \neq {a^*(s')}} Q^{\pi_{\theta}^{(t)}}(s', a^*(s'))) \Big\},
\end{align*}
then we have $C > 0$ since $d^{\pi_\omega^{(t)}}_{\mu}(s') \ge (1-\gamma)\mu(s')$ by \Cref{lemma:lower_bound_of_state_visitation_distribution}, $\pi_\omega^{(t)}(a^*(s')|s') > \frac{1}{|\cA|}$ since $\theta_{s', a^*(s')} > \theta_{s', a}$ for all $a \neq a^*(s')$ by property (i) of \Cref{lemma: adaptive apg enter local concavity}, and $Q^{\pi_{\theta}^{(t)}}(s', a^*(s')) \ge V^{\pi_\theta^{(t)}}(s') > Q^{*}(s', a^*(s'))) > \max_{a \neq {a^*(s')}} Q^{\pi_{\theta}^{(t)}}(s', a^*(s'))$ by properties (ii) of \Cref{lemma: adaptive apg enter local concavity} and \Cref{lemma: Q^* a_2}.

By choosing $\tilde{c} = \frac{K}{C}$ and $\bar{c} = \beta$ and using the corresponding $s'$ and $\{t_i\}$, we have:
\begin{align}
    A^{\pi_{\theta}^{(t)}}(s', a^*(s'))
    &= Q^{\pi_{\theta}^{(t)}}(s', a^*(s')) - V^{\pi_{\theta}^{(t)}}(s') \nonumber \\
    &= Q^{\pi_{\theta}^{(t)}}(s', a^*(s')) - \sum_{a \in \cA} \pi_{\theta}^{(t)}(a(s')|s') Q^{\pi_{\theta}^{(t)}}(s', a) \nonumber \\
    &= (1-\pi_{\theta}^{(t)}(a^*(s')|s')) Q^{\pi_{\theta}^{(t)}}(s', a^*(s')) - \sum_{a \neq a^*(s')} \pi_{\theta}^{(t)}(a|s') {Q^{\pi_{\theta}^{(t)}}(s', a)} \nonumber \\
    &\ge (1-\pi_{\theta}^{(t)}(a^*(s')|s')) Q^{\pi_{\theta}^{(t)}}(s', a^*(s')) - \sum_{a \neq a^*(s')} \pi_{\theta}^{(t)}(a|s') \cdot \max_{a \neq {a^*}} Q^{\pi_{\theta}^{(t)}}(s', a) \nonumber \\
    &= (1-\pi_{\theta}^{(t)}(a^*(s')|s')) Q^{\pi_{\theta}^{(t)}}(s', a^*(s')) - (1-\pi_{\theta}^{(t)}(a^*|s')) \cdot \max_{a \neq {a^*(s')}} Q^{\pi_{\theta}^{(t)}}(s', a) \nonumber \\
    &= (1-\pi_{\theta}^{(t)}(a^*(s')|s')) \underbrace{(Q^{\pi_{\theta}^{(t)}}(s', a^*(s')) - \max_{a \neq {a^*(s')}} Q^{\pi_{\theta}^{(t)}}(s', a))}_{\text{bounded away from 0}} \nonumber \\
    &> \frac{K}{C\beta^t} \cdot (Q^{\pi_{\theta}^{(t)}}(s', a^*(s')) - \max_{a \neq {a^*(s')}} Q^{\pi_{\theta}^{(t)}}(s', a)). \label{eq:adaptive_apg_2}
\end{align}

By following (\ref{eq:adaptive_apg_2}), for all $t \in \{ t_i\}_{i=1}^{\infty}$, we have
\begin{align}
    &\theta_{s', a^*(s')}^{(t+1)} - \theta_{s', a^*(s')}^{(t)} \nonumber \\
    &= \underbrace{(\theta_{s', a^*(s')}^{(t+1)} - \omega_{s', a^*(s')}^{(t)})}_{\text{gradient}} + \underbrace{(\omega_{s', a^*(s')}^{(t)} - \theta_{s', a^*(s')}^{(t)})}_{\text{momentum}} \nonumber \\
    &= \min \Big\{K, \eta^{(t)} \nabla_{\theta}{V^{\pi_{\theta}}(\mu)} \Big\rvert_{\theta = \omega^{(t)}} \Big\} + (\omega_{s', a^*(s')}^{(t)} - \theta_{s', a^*(s')}^{(t)}) \nonumber \\
    &= \min \Big\{K, \eta^{(t)} \cdot \frac{1}{1-\gamma} \cdot d^{\pi_\omega^{(t)}}_{\mu}(s') \cdot \pi_\omega^{(t)}({a^*}|s') \cdot A^{\pi_\omega^{(t)}}(s',{a^*}) \Big\} + (\omega_{s', a^*(s')}^{(t)} - \theta_{s', a^*(s')}^{(t)}) \nonumber \\
    &= K + (\omega_{s', a^*(s')}^{(t)} - \theta_{s', a^*(s')}^{(t)}) \label{eq:adaptive_apg_2-1} \\
    &\ge K + \max_{a \neq a^*(s')} \Big\{ \omega_{s', a}^{(t)} - \theta_{s', a}^{(t)} \Big\} \label{eq:adaptive_apg_3},
\end{align}
where \cref{eq:adaptive_apg_2-1} holds by 
\begin{align*}
    &\eta^{(t)} \cdot \frac{1}{1-\gamma} \cdot d^{\pi_\omega^{(t)}}_{\mu}(s') \cdot \pi_\omega^{(t)}({a^*}|s') \cdot A^{\pi_\omega^{(t)}}(s',{a^*}) \\
    &> \beta^t \cdot \frac{(1-\gamma)^2}{8} \cdot d^{\pi_\omega^{(t)}}_{\mu}(s') \cdot \pi_\omega^{(t)}(a^*|s') \cdot \frac{K}{C\beta^t} \cdot (Q^{\pi_{\theta}^{(t)}}(s', a^*(s')) - \max_{a \neq {a^*}} Q^{\pi_{\theta}^{(t)}}(s', a^*(s'))) \\
    &= \frac{\frac{(1-\gamma)^2}{8} \cdot d^{\pi_\omega^{(t)}}_{\mu}(s') \cdot \pi_\omega^{(t)}(a^*(s')|s') \cdot (Q^{\pi_{\theta}^{(t)}}(s', a^*(s')) - \max_{a \neq {a^*(s')}} Q^{\pi_{\theta}^{(t)}}(s', a^*(s')))}{C} \cdot K \\
    &\ge K,
\end{align*}
and \cref{eq:adaptive_apg_3} holds by \Cref{lemma: adaptive apg enter local concavity}.

Now we discuss two possible cases:
\begin{itemize}
    \item The infinite time sequence $\{ t_i\}_{i=1}^{\infty}$ contains all $t > T_M$. (i.e., $\{ t_i\}_{i=1}^{\infty} = \{ T_M+1, T_M+2, T_M+3, \cdots \}$): \\
    In this case, by \cref{eq:adaptive_apg_3} and \Cref{lemma: adaptive apg enter local concavity}, we have that
    \begin{align*}
        \theta_{s', a^*(s')}^{(t+1)} - \theta_{s', a^*(s')}^{(t)} 
        &\ge K + \max_{a \neq a^*(s')} \Big\{ \omega_{s', a}^{(t)} - \theta_{s', a}^{(t)} \Big\} \\
        &> K + \max_{a \neq a^*(s')} \Big\{ \underbrace{(\theta_{s', a}^{(t+1)} - \omega_{s', a}^{(t)})}_{\text{gradient} < 0} \Big\} + \max_{a \neq a^*(s')} \Big\{ \underbrace{(\omega_{s', a}^{(t)} - \theta_{s', a}^{(t)})}_{\text{momentum}} \Big\} \\
        &\ge K + \max_{a \neq a^*(s')} \Big\{ \theta_{s', a}^{(t+1)} - \theta_{s', a}^{(t)} \Big\}
    \end{align*}
    This indicates a constant growth of the policy parameter $\theta_{s' ,a^*(s')}$. Consequently, by softmax parameterization, it follows that $1 - \pi_\theta^{(t)}(a^* | s') = O(e^{-\bar{c}t})$ for some $\bar{c} > 0$, thereby completing our contradiction.

    \item There exists at least one $(t_i-1) > T_M$ such that $(t_i-1) \notin \{ t_i\}_{i=1}^{\infty}$ (i.e., $1 - \pi_{\theta}^{(t_i-1)}(a^*(s) | s') \le \frac{K}{C \beta^{t_i-1}}$):\\
    In this case, we set $K = \frac{1}{4\sqrt{|\cA|}|\cS|} \cdot \frac{1-\gamma}{4 C_\infty - (1-\gamma)} + 2\ln \beta$ to ensure that the one step improvement $\pi_{\theta}^{(t_i+1)}(a^*(s) | s') - \pi_{\theta}^{(t_i)}(a^*(s) | s')$ is sufficiently large once we follow the update $\theta_{s', a^*(s')}^{(t_i+1)} = \omega_{s', a^*(s')}^{(t_i)} + K$ shown in \cref{eq:adaptive_apg_2-1}.
    By \Cref{lemma: adaptive apg enter local concavity}, we have
    \begin{align}
        &\pi_{\theta}^{(t+1)}(a^*(s) | s') \nonumber \\
        &= \frac{\exp(\theta_{s', a^*(s')}^{(t+1)})}{\sum_{a \neq a^*(s')} \exp(\theta_{s', a}^{(t+1)})} \nonumber \\
        &= \frac{\exp(\theta_{s', a^*(s')}^{(t)} + (\theta_{s', a^*(s')}^{(t+1)} - \omega_{s', a^*(s')}^{(t)}) + (\omega_{s', a^*(s')}^{(t)} - \theta_{s', a^*(s')}^{(t)}))}{\sum_{a \in \cA} \exp(\theta_{s', a}^{(t)} + (\theta_{s', a}^{(t+1)} - \omega_{s', a}^{(t)}) + (\omega_{s', a}^{(t)} - \theta_{s', a}^{(t)}))} \nonumber \\
        &= \frac{\exp(\theta_{s', a^*(s')}^{(t)})}{\sum_{a \in \cA} \exp(\theta_{s', a}^{(t)} + \underbrace{(\theta_{s', a}^{(t+1)} - \omega_{s', a}^{(t)}) + (\omega_{s', a}^{(t)} - \theta_{s', a}^{(t)}) - (\theta_{s', a^*(s')}^{(t+1)} - \omega_{s', a^*(s')}^{(t)}) - (\omega_{s', a^*(s')}^{(t)} - \theta_{s', a^*(s')}^{(t)})}_{\le 0})} \nonumber \\
        &\ge \frac{\exp(\theta_{s', a^*(s')}^{(t)})}{\sum_{a \neq a^*(s')} \exp(\theta_{s', a}^{(t)})} \nonumber \\
        &= \pi_{\theta}^{(t)}(a^*(s) | s'). \label{eq:adaptive_apg_3-1}
    \end{align}
    Following \cref{eq:adaptive_apg_3-1}, since $\pi_{\theta}^{(t)}(a^*(s) | s')$ is increasing for all $t > T_M$, we have
    \begin{align}
        \sum_{a \neq a^*(s')} \pi_{\theta}^{(t_i)}(a| s') = 1 - \pi_{\theta}^{(t_i)}(a^*(s) | s') < 1 - \pi_{\theta}^{(t_i-1)}(a^*(s) | s') \le \frac{K}{C \beta^{t_i-1}}.
        \label{eq:adaptive_apg_4}
    \end{align}
    Moreover, by \Cref{lemma: adaptive apg enter local concavity}, we have 
    \begin{align}
        \frac{\sum_{a \neq a^*(s')} \exp (\theta_{s', a}^{(t_i)})}{\exp(\frac{1}{4\sqrt{|\cA|}|\cS|} \cdot \frac{1-\gamma}{4 C_\infty - (1-\gamma)}) \exp (\theta_{s', a^*(s')}^{(t_i)})} \le \sum_{a \neq a^*(s')} \pi_{\theta}^{(t_i)}(a| s') \le \frac{K}{C \beta^{t_i-1}}, \label{eq:adaptive_apg_5}
    \end{align}
    where the first inequality in \cref{eq:adaptive_apg_5} holds by the setting of $M$.
    
    Combining \cref{eq:adaptive_apg_3,eq:adaptive_apg_4,eq:adaptive_apg_5}, we get
    \begin{align}
        1 - \pi_{\theta}^{(t_i+1)}(a^*(s) | s')
        &= \sum_{a \neq a^*(s')} \pi_{\theta}^{(t_i+1)}(a| s') \nonumber\\
        &= \frac{\sum_{a \neq a^*(s')} \exp(\theta_{s', a}^{(t_i)} + (\theta_{s', a}^{(t_i+1)} - \theta_{s', a}^{(t_i)}))}{\sum_{a \in \cA} \exp(\theta_{s', a}^{(t_i)} + (\theta_{s', a}^{(t_i+1)} - \theta_{s', a}^{(t_i)}))} \nonumber\\
        &= \frac{\sum_{a \neq a^*(s')} \exp(\theta_{s', a}^{(t_i)} + \overbrace{(\theta_{s', a}^{(t_i+1)} - \omega_{s', a}^{(t_i)})}^{<0} + \overbrace{(\omega_{s', a}^{(t_i)} - \theta_{s', a}^{(t_i)}) - (\omega_{s', a^*(s')}^{(t_i)} - \theta_{s', a^*(s')}^{(t_i)})}^{\le0})}{\sum_{a \in \cA} \exp(\theta_{s', a}^{(t_i)} + (\theta_{s', a}^{(t_i+1)} - \omega_{s', a}^{(t_i)}) + (\omega_{s', a}^{(t_i)} - \theta_{s', a}^{(t_i)}) - (\omega_{s', a^*(s')}^{(t_i)} - \theta_{s', a^*(s')}^{(t_i)}))} \nonumber\\ 
        &\le \frac{\sum_{a \neq a^*(s')} \exp(\theta_{s', a}^{(t_i)})}{\exp(\theta_{s', a^*(s')}^{(t_i)} + (\theta_{s', a^*(s')}^{(t_i+1)} - \omega_{s', a}^{(t_i)}) + (\omega_{s', a^*(s')}^{(t_i)} - \theta_{a^*(s')}^{(t_i)}) - (\omega_{s', a^*(s')}^{(t_i)} - \theta_{s', a^*(s')}^{(t_i)}))} \nonumber\\ 
        &\le \frac{\sum_{a \neq a^*(s')} \exp(\theta_{s', a}^{(t_i)}))}{ \exp(\theta_{s', a^*(s')}^{(t_i)} + (\theta_{s', a^*(s')}^{(t_i+1)} - \omega_{s', a^*(s')}^{(t_i)}))} \nonumber\\ 
        &= \frac{\sum_{a \neq a^*(s')} \exp(\theta_{s', a}^{(t_i)})}{ \exp(\theta_{s', a^*(s')}^{(t_i)})} \cdot \exp(-K) \nonumber\\
        &\le \frac{\exp(\frac{1}{2\sqrt{|\cA|}|\cS|} \cdot \frac{1-\gamma}{4 C_\infty - (1-\gamma)})K}{C \beta^{t_i-1}} \cdot \exp(-K) \nonumber\\
        &= \frac{K}{C \beta^{t_i+1}}, \label{eq:adaptive_apg_6}
    \end{align}
    where \cref{eq:adaptive_apg_6} holds by setting $K = \frac{1}{4\sqrt{|\cA|}|\cS|} \cdot \frac{1-\gamma}{4 C_\infty - (1-\gamma)} + 2\ln \beta$.


    Hence we get that once $1 - \pi_{\theta}^{(t)}(a^*(s) | s') > \frac{K}{C \beta^{t}}$ occurs for any $t \in \{ t_i\}_{i=1}^{\infty}$, we could ensure that $1 - \pi_{\theta}^{(t+1)}(a^*(s) | s') \le \frac{K}{C \beta^{t+1}}$ (i.e., $(t+1) \notin \{ t_i\}_{i=1}^{\infty}$ for all $t \in  \{ t_i\}_{i=1}^{\infty}$ with $t > t_i$),
    In other words, for all $t \in  \{ t_i\}_{i=1}^{\infty}$, we have $t-1 \notin  \{ t_i\}_{i=1}^{\infty}$ and hence $1 - \pi_{\theta}^{(t-1)}(a^*(s) | s') \le \frac{K}{C \beta^{t-1}}$.
    Combining with the monotonicity of $1 -\pi_{\theta}^{(t-1)}(a^*(s) | s')$ in \cref{eq:adaptive_apg_3-1}, we get $1 -\pi_{\theta}^{(t)}(a^*(s) | s') \le 1 -\pi_{\theta}^{(t-1)}(a^*(s) | s') \le \frac{K}{C \beta^{t-1}}$.
    We thus see that $1 - \pi_{\theta}^{(t)}(a^*(s) |
	s') \le \frac{K}{C \beta^{t-1}}$ for all $t > T_M$.
This means that for any such $s'$ that has a companion infinite sequence $\{t_i\}$ that satisfies \cref{eq:violate} with $\tilde c = K/C$ and $\bar c = \beta$, the condition \cref{eq:lineardecay} actually holds for $s'$ with $\tilde c = \beta K/C$ and $\bar c = \beta$, so for this setting of $\tilde c$ and $c$ we have obtained a contradiction and thus completed the proof.
\qedhere
\end{itemize}
\end{proof}

\subsection{Numerical Validation of the Linear Convergence Rate}
\begin{wrapfigure}{r}{0.4\textwidth}
\centering
  \includegraphics[width=0.35\textwidth]{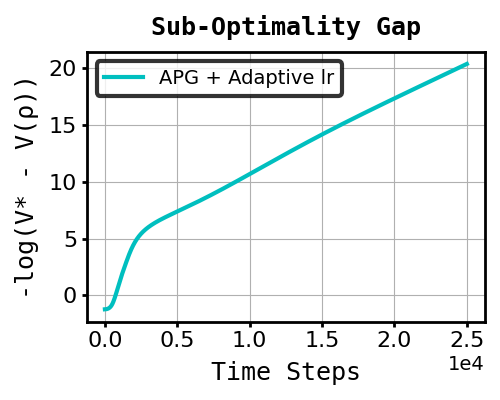}
  \hspace{-5mm}
  \caption{The sub-optimality gap of APG with time-varying
  step sizes under an MDP of five states and five actions with uniform initialization.}
  \label{exp:adaptiveapg}
\end{wrapfigure}

In this subsection, we empirically validate the convergence rate of
APG with an time-varying step sizes scheduling by conducting
experiments on an MDP with five states and five actions. This MDP is identical to the one discussed in \Cref{sec:disc:exp}, and detailed configuration information is provided in \Cref{app:experiment}.
The figure illustrating the logarithmic sub-optimality gap of APG with
an time-varying step sizes scheduling under uniform
initialization is presented in \Cref{exp:adaptiveapg}.
It is evident that after ten iterations, the curve tends to form a
linear pattern.
This behavior aligns with the linear convergence rate demonstrated in \Cref{theorem: adaptive apg convergence rate}.

\endgroup

\newpage
\section{Detailed Explanation of the Motivating Example and the Experimental Configurations}
\label{app:experiment}

\begingroup
\allowdisplaybreaks

\subsection{Motivating Examples of APG (\texorpdfstring{\Cref{sec:alg:behavior}}{})}
\label{subsec: Motivating Examples of APG}

Consider a simple two-action bandit with actions $a^*$ and $a_2$ and a reward function $r(a^*) = 1, r(a_2) = 0$. Accordingly, the objective we aim to optimize is $\mathbb{E}_{a \sim \pi_\theta} [r(a)] = \pi_{\theta}(a^*)$. 
Under softmax parameterization, the Hessian matrix with
respect to the parameters $\theta_{a^*}$ and $\theta_{a_2}$ is:
\begin{align*}
    \mathbf{H}& = \begin{bmatrix}
    \frac{\partial^2 \pi_{\theta}(a^*)}{\partial \theta_{a^*} \partial \theta_{a^*}} & \frac{\partial^2 \pi_{\theta}(a^*)}{\partial \theta_{a^*} \partial \theta_{a_2}}\\
    \frac{\partial^2 \pi_{\theta}(a^*)}{\partial \theta_{a_2} \partial \theta_{a^*}} & \frac{\partial^2 \pi_{\theta}(a^*)}{\partial \theta_{a_2} \partial \theta_{a_2}}
    \end{bmatrix}\\
    &=\begin{bmatrix}
    \pi_\theta(a^*)(1-\pi_\theta(a^*))(1-2\pi_\theta(a^*)) & \pi_\theta(a^*)(1-\pi_\theta(a^*))(2\pi_\theta(a^*)-1)\\
    \pi_\theta(a^*)(1-\pi_\theta(a^*))(2\pi_\theta(a^*)-1) & \pi_\theta(a^*)(1-\pi_\theta(a^*))(1-2\pi_\theta(a^*))
\end{bmatrix}.
\end{align*}
After straightforward calculation, we see that its eigenvalues are
$\lambda_1 = 2\pi_\theta(a^*)(1 - \pi_\theta(a^*))(1 - 2\pi_\theta(a^*)), \lambda_2 = 0$.
Therefore, when $\pi_\theta(a^*) \ge 0.5$, we see that $\lambda_1, \lambda_2 \le
0$ and thus the Hessian is negative semi-definite, which is equivalent
to that the function $\mathbb{E}_{a \sim \pi_\theta} [r(a)]$ is
concave.

%



%
\subsection{Setting of the Experiment for Numerical Validation of the Convergence Rates of APG (\texorpdfstring{\Cref{sec:disc:exp}}{})}
\label{subsec: Experiement for Convergence Rates of APG}

\textbf{(Bandit)}
We conduct a three-action bandit experiment with actions $\cA = [a^*, a_2,
a_3]$ with the corresponding rewards being $r = [r(a^*), r(a_2), r(a_3)] = [1, 0.99, 0]$.
We consider two initializations for the policy parameters. The first
one is the uniform initialization ($\theta^{(0)}=[0,0,0], \pi^{(0)} =
[1/3, 1/3, 1/3]$) and the second one is a hard initialization
such that $\theta^{(0)}=[1,3,5],
\pi^{(0)}=[0.01588, 0.11731, 0.86681]$ and hence the optimal action
has the smallest initial probability.

\textbf{(MDP)}
We conduct an experiment on an MDP with five states and five actions
under the initial state distribution $\rho = [0.3, 0.2, 0.1, 0.15,
0.25]$. The reward function, the initial policy parameters, and the
transition probability can be found in Tables
\ref{exp:5s5a-reward}-\ref{exp:5s5a-tran-5}.

\begin{table}[ht]
\begin{minipage}[t]{0.9\linewidth}\centering
\medskip
\vskip 0.15in
\begin{center}
\begin{small}
\begin{sc}
\begin{tabular}{|c|c|c|c|c|c|}
\hline
\textbf{$r(s,a)$} & \textbf{$a_1$} & \textbf{$a_2$} & \textbf{$a_3$} & \textbf{$a_4$} & \textbf{$a_5$} \\ \hline
\textbf{$s_1$}  & $1.0$          & $0.8$          & $0.6$          & $0.7$          & $0.4$          \\ \hline
\textbf{$s_2$}  & $0.5$          & $0.3$          & $0.1$          & $1.0$          & $0.6$          \\ \hline
\textbf{$s_3$}  & $0.6$          & $0.9$          & $0.8$          & $0.7$          & $1.0$          \\ \hline
\textbf{$s_4$}  & $0.1$          & $0.2$          & $0.6$          & $0.7$          & $0.4$          \\ \hline
\textbf{$s_5$}  & $0.8$          & $0.4$          & $0.6$          & $0.2$          & $0.9$          \\ \hline
\end{tabular}
\end{sc}
\end{small}
\caption{Experimental settings: Reward function.}
\label{exp:5s5a-reward}
\end{center}
\end{minipage}\hfill%
\end{table}

\begin{table}[ht]
\begin{minipage}[t]{0.45\linewidth}\centering
\medskip
\label{exp:5s5a-hard-theta}
\vskip 0.15in
\begin{center}
\begin{small}
\begin{sc}
\begin{tabular}{|c|c|c|c|c|c|}
\hline
\textbf{$\theta_{s,a}^{(0)}$} & \textbf{$a_1$} & \textbf{$a_2$} & \textbf{$a_3$} & \textbf{$a_4$} & \textbf{$a_5$} \\ \hline
\textbf{$s_1$}          & $1$            & $2$            & $3$            & $4$            & $5$            \\ \hline
\textbf{$s_2$}          & $3$            & $4$            & $5$            & $1$            & $2$            \\ \hline
\textbf{$s_3$}          & $5$            & $2$            & $3$            & $4$            & $1$            \\ \hline
\textbf{$s_4$}          & $5$            & $4$            & $2$            & $1$            & $3$            \\ \hline
\textbf{$s_5$}          & $2$            & $4$            & $3$            & $5$            & $1$            \\ \hline
\end{tabular}
\end{sc}
\end{small}
\caption{Experimental settings: Hard initialization.}
\end{center}
\vskip -0.1in
\end{minipage}\hfill
\begin{minipage}[t]{0.45\linewidth}\centering
\medskip
\label{exp:5s5a-uniform-theta}
\vskip 0.15in
\begin{center}
\begin{small}
\begin{sc}
\begin{tabular}{|c|c|c|c|c|c|}
\hline
\textbf{$\theta_{s,a}^{(0)}$} & \textbf{$a_1$} & \textbf{$a_2$} & \textbf{$a_3$} & \textbf{$a_4$} & \textbf{$a_5$} \\ \hline
\textbf{$s_1$}          & $0$            & $0$            & $0$            & $0$            & $0$            \\ \hline
\textbf{$s_2$}          & $0$            & $0$            & $0$            & $0$            & $0$            \\ \hline
\textbf{$s_3$}          & $0$            & $0$            & $0$            & $0$            & $0$            \\ \hline
\textbf{$s_4$}          & $0$            & $0$            & $0$            & $0$            & $0$            \\ \hline
\textbf{$s_5$}          & $0$            & $0$            & $0$            & $0$            & $0$            \\ \hline
\end{tabular}
\end{sc}
\end{small}
\caption{Experimental settings: Uniform initialization.}
\end{center}
\vskip -0.1in
\end{minipage}\hfill
\end{table}

\newpage
\begin{table}[ht]
\begin{minipage}[t]{0.5\linewidth}\centering
\medskip
\label{exp:5s5a-tran-1}
\vskip 0.15in
\begin{center}
\begin{small}
\begin{sc}
\begin{tabular}{|c|c|c|c|c|c|}
\hline
\textbf{$P(s | s_1, a)$} & \textbf{$a_1$} & \textbf{$a_2$} & \textbf{$a_3$} & \textbf{$a_4$} & \textbf{$a_5$} \\ \hline
\textbf{$s_1$}           & $0.1$          & $0.6$          & $0.5$          & $0.4$          & $0.2$          \\ \hline
\textbf{$s_2$}           & $0.5$          & $0.1$          & $0.1$          & $0.3$          & $0.1$          \\ \hline
\textbf{$s_3$}           & $0.1$          & $0.1$          & $0.1$          & $0.1$          & $0.1$          \\ \hline
\textbf{$s_4$}           & $0.2$          & $0.1$          & $0.2$          & $0.1$          & $0.1$          \\ \hline
\textbf{$s_5$}           & $0.1$          & $0.1$          & $0.1$          & $0.1$          & $0.5$          \\ \hline
\end{tabular}
\end{sc}
\end{small}
\caption{Experimental settings: Transition probability $P(\cdot | s_1,
\cdot)$.}
\end{center}
\vskip -0.1in
\end{minipage}\hfill%
\begin{minipage}[t]{0.5\linewidth}\centering
\medskip
\label{exp:5s5a-tran-2}
\vskip 0.15in
\begin{center}
\begin{small}
\begin{sc}
\begin{tabular}{|c|c|c|c|c|c|}
\hline
\textbf{$P(s | s_2, a)$} & \textbf{$a_1$} & \textbf{$a_2$} & \textbf{$a_3$} & \textbf{$a_4$} & \textbf{$a_5$} \\ \hline
\textbf{$s_1$}           & $0.1$          & $0.4$          & $0.1$          & $0.4$          & $0.2$          \\ \hline
\textbf{$s_2$}           & $0.5$          & $0.1$          & $0.4$          & $0.1$          & $0.2$          \\ \hline
\textbf{$s_3$}           & $0.2$          & $0.2$          & $0.3$          & $0.1$          & $0.2$          \\ \hline
\textbf{$s_4$}           & $0.1$          & $0.2$          & $0.1$          & $0.1$          & $0.2$          \\ \hline
\textbf{$s_5$}           & $0.1$          & $0.1$          & $0.1$          & $0.3$          & $0.2$          \\ \hline
\end{tabular}
\end{sc}
\end{small}
\caption{Experimental settings: Transition probability $P(\cdot | s_2,
\cdot)$.}
\end{center}
\vskip -0.1in
\end{minipage}\hfill
\end{table}

\begin{table}[ht]
\begin{minipage}[t]{0.5\linewidth}\centering
\medskip
\label{exp:5s5a-tran-3}
\vskip 0.15in
\begin{center}
\begin{small}
\begin{sc}
\begin{tabular}{|c|c|c|c|c|c|}
\hline
\textbf{$P(s | s_3, a)$} & \textbf{$a_1$} & \textbf{$a_2$} & \textbf{$a_3$} & \textbf{$a_4$} & \textbf{$a_5$} \\ \hline
\textbf{$s_1$}           & $0.6$          & $0.2$          & $0.3$          & $0.1$          & $0.2$          \\ \hline
\textbf{$s_2$}           & $0.1$          & $0.4$          & $0.3$          & $0.4$          & $0.1$          \\ \hline
\textbf{$s_3$}           & $0.1$          & $0.1$          & $0.2$          & $0.3$          & $0.1$          \\ \hline
\textbf{$s_4$}           & $0.1$          & $0.2$          & $0.1$          & $0.1$          & $0.1$          \\ \hline
\textbf{$s_5$}           & $0.1$          & $0.1$          & $0.1$          & $0.1$          & $0.5$          \\ \hline
\end{tabular}
\end{sc}
\end{small}
\caption{Experimental settings: Transition probability $P(\cdot | s_3,
\cdot)$.}
\end{center}
\vskip -0.1in
\end{minipage}\hfill%
\begin{minipage}[t]{0.5\linewidth}\centering
\medskip
\label{exp:5s5a-tran-4}
\vskip 0.15in
\begin{center}
\begin{small}
\begin{sc}
\begin{tabular}{|c|c|c|c|c|c|}
\hline
\textbf{$P(s | s_4, a)$} & \textbf{$a_1$} & \textbf{$a_2$} & \textbf{$a_3$} & \textbf{$a_4$} & \textbf{$a_5$} \\ \hline
\textbf{$s_1$}            & $0.6$          & $0.1$          & $0.2$          & $0.4$          & $0.5$          \\ \hline
\textbf{$s_2$}            & $0.1$          & $0.5$          & $0.1$          & $0.3$          & $0.1$          \\ \hline
\textbf{$s_3$}            & $0.1$          & $0.1$          & $0.1$          & $0.1$          & $0.1$          \\ \hline
\textbf{$s_4$}            & $0.1$          & $0.2$          & $0.1$          & $0.1$          & $0.2$          \\ \hline
\textbf{$s_5$}            & $0.1$          & $0.1$          & $0.5$          & $0.1$          & $0.1$          \\ \hline
\end{tabular}
\end{sc}
\end{small}
\caption{Experimental settings: Transition probability $P(\cdot | s_4,
\cdot)$.}
\end{center}
\vskip -0.1in
\end{minipage}\hfill
\end{table}

\begin{table}[!ht]
\begin{minipage}[t]{1.0\linewidth}\centering
\medskip
\centering
\vskip 0.15in
\begin{center}
\begin{small}
\begin{sc}
\begin{tabular}{|c|c|c|c|c|c|}
\hline
\textbf{$P(s | s_5, a)$} & \textbf{$a_1$} & \textbf{$a_2$} & \textbf{$a_3$} & \textbf{$a_4$} & \textbf{$a_5$} \\ \hline
\textbf{$s_1$}           & $0.2$          & $0.4$          & $0.4$          & $0.1$          & $0.2$          \\ \hline
\textbf{$s_2$}           & $0.2$          & $0.1$          & $0.1$          & $0.4$          & $0.5$          \\ \hline
\textbf{$s_3$}           & $0.2$          & $0.2$          & $0.1$          & $0.2$          & $0.1$          \\ \hline
\textbf{$s_4$}           & $0.2$          & $0.2$          & $0.3$          & $0.1$          & $0.1$          \\ \hline
\textbf{$s_5$}           & $0.2$          & $0.1$          & $0.1$          & $0.2$          & $0.1$          \\ \hline
\end{tabular}
\end{sc}
\end{small}
\caption{Experimental settings: Transition probability $P(\cdot | s_5,
\cdot)$.}
\label{exp:5s5a-tran-5}
\end{center}
\vskip -0.1in
\end{minipage}
\end{table}

\subsection{APG on Atari 2600 Games (\texorpdfstring{\Cref{sec:disc:atari}}{})} 
We empirically evaluate the performance of APG on four Atari 2600 games from the Arcade Learning Environment (ALE) \citep{bellemare2013arcade}.


Our implementation for APG, HBPG, and PG are based on the code base of
Stable Baselines3 and RL Baselines3 Zoo \citep{rl-zoo3,
stable-baselines3}. The hyperparameters, exactly set to their default values, are detailed in \Cref{table:hyper}.
Pseudo code of HBPG is given in \Cref{algorithm:HB}.

In Atari, one major challenge of implementing the exact APG is how to get a good estimate of the true gradient $\nabla_\theta V^{\pi_\theta}(\mu)$ given the large state and action spaces of the Atari games. While it is theoretically possible to get a nearly true gradient by Monte-Carlo estimation, it is not feasible in practice due to the resulting prohibitively large computation time. To address this, we made a few design choices in the practical implementation:
\begin{itemize}
    \item To ensure a sufficiently accurate gradient estimate with a reasonable amount of training time, we use a moderately large batch size of 80 for APG (also for PG and HBPG for a fair comparison) and leverage RMSProp, which is designed to better average the stochastic gradients over successive mini-batches than the vanilla stochastic gradient descent.
    \item To estimate the advantage function (required in the policy gradient), we use the generalized advantage estimator (GAE) \citep{schulman2015high}, which is a popular technique for advantage estimation in deep RL.
\end{itemize}


\begin{table}[!tb]
\vskip 0.15in
\begin{center}
\begin{tabular}{lccc}
\toprule
Hyperparameters & APG (Ours) & HBPG & PG \\
\midrule
Batch Size & 80 & 80 & 80 \\
Entropy coefficient & 0.01 & 0.01 & 0.01 \\
Gamma (discount factor) & 0.99 & 0.99 & 0.99 \\
Lambda for GAE & 1.0 & 1.0 & 1.0 \\
Learning rate & $0.0007$ & $0.0007$ & $0.0007$ \\
Number of environments & 16 & 16 & 16 \\
Time steps & $10^7$ & $10^7$ & $10^7$ \\
Value function coefficient & 0.25 & 0.25 & 0.25 \\
\bottomrule
\end{tabular}
\end{center}
\vskip -0.1in
\caption{Hyperparameters of APG, HBPG and PG in Carnival, Pong,
Riverraid, and Seaquest.}
\label{table:hyper}
\end{table}

\begin{algorithm} [!tb]
	\caption{Heavy-Ball Policy Gradient (HBPG)}
	\label{algorithm:HB}
	\begin{algorithmic}

\STATE \textbf{Input}:
Step size $\eta = \frac{1}{L}$, where $L$ is the Lipschitz constant of
the gradient of the objective function $f$, momentum factor $\beta$.
\STATE \textbf{Initialize}:
$\theta^{(0)}$ and let $\theta^{(-1)} = \theta^{(0)}$.

\FOR{$t = 1$ to $T$} 
    \STATE
    $\theta^{(t)} = \theta^{(t-1)} + \eta \nabla_\theta {V^{\pi_{\theta}}(\mu)} \Big\rvert_{\theta = \theta^{(t-1)}} +
	\beta(\theta^{(t-1)} - \theta^{(t-2)})$
\ENDFOR
	\end{algorithmic}
\end{algorithm}

\endgroup

\clearpage
\begingroup
\allowdisplaybreaks

\section{Convergence Rate of Policy Gradient Under Softmax Parameterization}
\label{app:pg}

The Gradient Descent (GD) method within the realm of optimization exhibits a convergence rate of $O(1/t)$ when applied to convex objectives.
Remarkably, \citep{mei2020global} have demonstrated that the Policy Gradient (PG) method, even when dealing with nonconvex RL objectives, can achieve the same convergence rate, leveraging the non-uniform Polyak–Łojasiewicz (PL) condition under softmax parameterization.
In this paper, we employ a primary proof technique involving the characterization of local $C$-nearly concavity and subsequently establish that the Accelerated Policy Gradient (APG) method attains this convergence regime. Our investigation also reveals that the PG method with softmax parameterization shares this intriguing property.

In this section, we will begin by providing a proof demonstrating the $\tilde{O}(1/t)$ convergence rate for \hyperref[algorithm:PG]{PG} when optimizing the value objectives that are $C$-nearly concave, while employing the softmax parameterized policy. Additionally, we will offer a proof establishing the capability of PG to reach the local $C$-nearly concavity regime within a finite number of time steps, considering the general MDP setting.

\subsection{Convergence Rate Under Nearly Concave Objectives for Policy Gradient}






\begin{theorem}
\label{theorem:GD_convergence_rate}
Let $\left \{ \theta^{(t)} \right \}_{t \ge 1}$ be computed by \hyperref[algorithm:PG]{PG} and that $V^{\pi_\theta}(\mu)$ is $C$-nearly
concave along the gradient update at time $t$ with a constant $C > 1$ for every $t \ge 1$.
Then given any $T \ge 1$, and for any $\theta^{**}$ such that for all $t \le T$ the direction $\theta^{**} - \theta^{(t)}$ satisfies the $C$-nearly concavity and $V^{\pi_{\theta^{**}}}(\mu) \ge V^{\pi_{\theta}^{(t)}}(\mu)$, we have
\begin{align*}
    V^{\pi_{\theta^{**}}}(\mu) - V^{\pi_{\theta}^{(T)}}(\mu) \le \frac{2 L C^2 \left( \left \| \theta^{**} \right \| + \kappa \ln T \right)^2}{T-1}, 
\end{align*}
where $L$ is the Lipschitz constant of the objective and $\kappa \in \mathbb{R}$ is a finite constant.
\end{theorem}

\begin{proof}[Proof of \Cref{theorem:GD_convergence_rate}]  
\phantom{}
Initially, given any $T$, by \cref{eq:PG} and \Cref{lemma:mdp_smoothness}, for any $t < T$, we have
\begin{align}
    -V^{\pi_{\theta}^{(t+1)}}(\mu)
    &\le -V^{\pi_{\theta}^{(t)}}(\mu) - \left \langle \nabla_{\theta}{V^{\pi_{\theta}}(\mu)} \Big\rvert_{\theta = \theta^{(t)}}, \theta^{(t+1)} - \theta^{(t)} \right \rangle + \frac{L}{2}\left \| \theta^{(t+1)} - \theta^{(t)} \right \|^2 \nonumber \\ 
    &= -V^{\pi_{\theta}^{(t)}}(\mu) - \eta \left \| \nabla_{\theta}{V^{\pi_{\theta}}(\mu)} \Big\rvert_{\theta = \theta^{(t)}} \right \|^2 + \frac{L\eta^2}{2}\left \| \nabla_{\theta}{V^{\pi_{\theta}}(\mu)} \Big\rvert_{\theta = \theta^{(t)}} \right \|^2 \nonumber \\
    &= -V^{\pi_{\theta}^{(t)}}(\mu) - \frac{1}{2L}\left \| \nabla_{\theta}{V^{\pi_{\theta}}(\mu)} \Big\rvert_{\theta = \theta^{(t)}} \right \|^2 \label{eq:GD_convergence_rate_1}. 
\end{align}
Additionally, by the hypothesis of the nearly concavity of objective along the direction to $\theta^{**}$, we have
\begin{align}
    -V^{\pi_{\theta}^{(t)}}(\mu) \le -V^{\pi_{\theta}^{**}}(\mu) + C \cdot \left \langle \nabla_{\theta}{V^{\pi_{\theta}}(\mu)} \Big\rvert_{\theta = \theta^{(t)}}, \theta^{**}-\theta^{(t)} \Big\rangle \right.. \label{eq:GD_convergence_rate_2}
\end{align}
Define
\begin{align*}
    \delta^{(t)} := V^{\pi_{\theta}^{**}}(\mu)-V^{\pi_{\theta}^{(t)}}(\mu).
\end{align*}
Then \cref{eq:GD_convergence_rate_1} leads to
\begin{align}
    \delta^{(t+1)} - \delta^{(t)} \le - \frac{1}{2L}\left \| \nabla_{\theta}{V^{\pi_{\theta}}(\mu)} \Big\rvert_{\theta = \theta^{(t)}} \right \|^2. \label{eq:GD_convergence_rate_3}
\end{align}
Moreover, \cref{eq:GD_convergence_rate_2} leads to
\begin{align}
    \delta^{(t)} 
    &\le C \cdot \left \langle \nabla_{\theta}{V^{\pi_{\theta}}(\mu)} \Big\rvert_{\theta = \theta^{(t)}}, \theta^{**}-\theta^{(t)} \right \rangle \nonumber \\
    &\le C \left \| \nabla_{\theta}{V^{\pi_{\theta}}(\mu)} \Big\rvert_{\theta = \theta^{(t)}} \right \| \left \| \theta^{**}-\theta^{(t)} \right \| \nonumber \\
    &\le C \left \| \nabla_{\theta}{V^{\pi_{\theta}}(\mu)} \Big\rvert_{\theta = \theta^{(t)}} \right \| \left( \left \| \theta^{**} \right \| + \left \| \theta^{(t)}  \right \|\right) \nonumber \\
    &\le C \left \| \nabla_{\theta}{V^{\pi_{\theta}}(\mu)} \Big\rvert_{\theta = \theta^{(t)}} \right \| \left( \left \| \theta^{**} \right \| + \max_{t' \le t} \left \| \theta^{(t')}  \right \|\right). \label{eq:GD_convergence_rate_4}
\end{align}
By substituting and rearranging \cref{eq:GD_convergence_rate_3} and \cref{eq:GD_convergence_rate_4}, we have
\begin{align}
    \delta^{(t+1)} &\le \delta^{(t)} - \frac{1}{2 L C^2 \left( \left \| \theta^{**} \right \| + \max_{t'\le t} \left \| \theta^{(t')} \right \| \right)^2} \cdot {\delta^{(t)}}^2. \label{eq:GD_convergence_rate_15}
\end{align}
Consequently, by rearranging and dividing $\delta^{(t)}\delta^{(t+1)}$ on each side of \cref{eq:GD_convergence_rate_15}, we have
\begin{align}
    \frac{1}{\delta^{(t+1)}} &\ge \frac{1}{\delta^{(t)}} + \frac{\delta^{(t)}}{\left( 2 L C^2 \left( \left \| \theta^{**} \right \| + \max_{t' \le t} \left \| \theta^{(t')} \right \| \right)^2 \right) \delta^{(t+1)}}. \nonumber
\end{align}
Furthermore, since PG enjoys monotonic improvement, we have
\begin{align}
    \frac{1}{\delta^{(t+1)}} - \frac{1}{\delta^{(t)}} &\ge \frac{\delta^{(t)}}{\delta^{(t+1)}} \frac{1}{2 L C^2 \left( \left \| \theta^{**} \right \| + \max_{t' \le t} \left \| \theta^{(t')} \right \| \right)^2} 
    \ge \frac{1}{2 L C^2 \left( \left \| \theta^{**} \right \| + \max_{t' \le T} \left \| \theta^{(t')} \right \| \right)^2}. \label{eq:GD_convergence_rate_5}
\end{align}
By implementing telescoping on \cref{eq:GD_convergence_rate_5} from $t = 1$ to $T-1$, we have
\begin{align}
    \delta^{(T)} = V^{\pi_{\theta}^{**}}(\mu)-V^{\pi_{\theta}^{(T)}}(\mu) \le \frac{2 L C^2 \left( \left \| \theta^{**} \right \| + \max_{t' \le T} \left \| \theta^{(t')} \right \| \right)^2}{T-1}. \label{eq:GD_convergence_rate_5_1}
\end{align}

Finally, we claim that the norm of $\theta^{(t)}$ grows in a rate of $O(\ln t)$. Without loss of generality, we assume that $\sum_{a \in \cA} \theta^{(t)}_{s,a} = 0$ for all $s \in \cS$ and $t\ge 1$ (the gradient updates sum up to $0$). Now we prove it by contradiction. Assume that the norm of $\theta^{(t)}$ does not grow in a rate of $O(\ln t)$, then for any $M \in \mathbb{R}$, there must exist a finite time $T_1$ such that ${\left \| \theta^{(t)} \right \|} / {\ln t} > M$ for all $t > T_1$. Moreover, by the asymptotic convergence results (Theorem 5.1) of \citep{agarwal2021theory} and \Cref{assump:unique_optimal}, we have that there exist an upper bound $\bar{M}$ and a finite time $T_2$ such that
\begin{align}
    \theta^{(t)}_{s,a} &< \bar{M}, \text{ for all } a \neq a^*(s), s \in \cS, t \in \mathbb{N}, \label{eq:GD_convergence_rate_6} \\
    \theta^{(t)}_{s,a^*(s)} &> \theta^{(t)}_{s,a}, \text{ for all } a \neq a^*(s), s \in \cS, t > T_2. \label{eq:GD_convergence_rate_7}
\end{align}

Hence, by choosing $M = 2|\cA|$, then we have
\begin{align}
    \frac{\left \| \theta^{(t)} \right \|}{\ln t} > 2|\cA|, \text{ for all } t > \max \{ T_1, T_2 \}. \label{eq:GD_convergence_rate_8}
\end{align}
Note that by \cref{eq:GD_convergence_rate_6,eq:GD_convergence_rate_7} and the fact that ${\left \| \theta^{(t)} \right \|} > 2|\cA|\cdot {\ln t}$, we have that $\theta_{s, a^*(s)} > 2\ln t$. Hence, for all $t > \max \{ T_1, T_2 \}, s \in \cS$, we have that
\begin{align*}
    1 - \pi_\theta^{(t)}(a^*(s)|s) &\le \frac{\sum_{a \neq a^*(s)} \exp(\theta_{s, a})}{\exp(\theta_{s, a^*(s)}) + \sum_{a\neq a^*(s)} \exp(\theta_{s, a})} \le \frac{(|\cA| - 1) \exp(\bar{M})}{\exp(2 \ln(t))} = \frac{(|\cA| - 1) \exp(\bar{M})}{t^2},
\end{align*}
where the last inequality holds by the positivity of the exponential function $\exp$.
Therefore, we can obtain the bound of the gradient with respect to $\theta_{s, a^*(s)}$
\begin{align}
    \left | \frac{\partial{V^{\pi_{\theta}^{(t)}}(\mu)}}{\partial \theta_{s,a^*(s)}} \right | 
    &= \left | \frac{1}{1-\gamma} \cdot d^{\pi_\theta^{(t)}}_{\mu}(s) \cdot \pi_\theta^{(t)}(a^*(s)|s) \cdot A^{\pi_\theta^{(t)}}(s,a^*(s)) \right | \nonumber \\
    &\le \left | \frac{1}{1-\gamma} \cdot 1 \cdot 1 \cdot (Q^{\pi_\theta^{(t)}}(s,a^*(s)) - V^{\pi_\theta^{(t)}}(s)) \right | \nonumber \\
    &= \frac{1}{1-\gamma} \cdot \Bigg|  Q^{\pi_\theta^{(t)}}(s,a^*(s)) - \pi_\theta^{(t)}(a^*(s)|s) Q^{\pi_\theta^{(t)}}(s,a^*(s)) \nonumber \\
    & \quad\quad\quad\quad -\sum_{a \neq a^*(s)} \pi_\theta^{(t)}(a|s) Q^{\pi_\theta^{(t)}}(s,a)) \Bigg| \nonumber \\
    &\le \frac{1-\pi_\theta^{(t)}(a^*(s)|s)}{1-\gamma} \Bigg( \Bigg| Q^{\pi_\theta^{(t)}}(s,a^*(s)) \Bigg| + \max_a \Bigg| Q^{\pi_\theta^{(t)}}(s,a) \Bigg| \Bigg) \nonumber \\
    &\le \frac{2(|\cA|-1)\exp (\bar{M})}{t^2  (1-\gamma)^2}, \label{eq:GD_convergence_rate_9} \\
    \left | \frac{\partial{V^{\pi_{\theta}^{(t)}}(\mu)}}{\partial \theta_{s,a}} \right | 
    &= \left | \frac{1}{1-\gamma} \cdot d^{\pi_\theta^{(t)}}_{\mu}(s) \cdot \pi_\theta^{(t)}(a|s) \cdot A^{\pi_\theta^{(t)}}(s,a) \right | \nonumber \\
    &\le \left | \frac{1}{1-\gamma} \cdot 1 \cdot (1-\pi_\theta^{(t)}(a^*(s)|s)) \cdot \frac{1}{1-\gamma} \right | \nonumber \\
    &= \frac{(|\cA|-1)\exp (\bar{M})}{t^2 (1-\gamma)^2}, \text{ for all } a \neq a^*(s). \label{eq:GD_convergence_rate_10}
\end{align}
Furthermore, by the update rule of PG, for all $t > \max \{ T_1, T_2 \}, s \in \cS$, we have
\begin{align*}
    \left| \theta^{(t+k)}_{s,a^*(s)} \right|
    &= \Bigg| \theta^{(t)}_{s,a^*(s)} + \sum_{i=t}^{t+k-1} \eta \frac{\partial{V^{\pi_{\theta}^{(i)}}(\mu)}}{\partial \theta_{s,a^*(s)}} \Bigg| \le \left| \theta^{(t)}_{s,a^*(s)} \right| + \underbrace{\sum_{i=t}^{\infty} \eta \Bigg| \frac{\partial{V^{\pi_{\theta}^{(i)}}(\mu)}}{\partial \theta_{s,a^*(s)}} \Bigg|}_{< \infty, \text{ by \cref{eq:GD_convergence_rate_9}}}, \\
    \left| \theta^{(t+k)}_{s,a} \right|
    &= \Bigg| \theta^{(t)}_{s,a} + \sum_{i=t}^{t+k-1} \eta \frac{\partial{V^{\pi_{\theta}^{(i)}}(\mu)}}{\partial \theta_{s,a}} \Bigg| \le \left | \theta^{(t)}_{s,a} \right | + \underbrace{\sum_{i=t}^{\infty} \eta \Bigg| \frac{\partial{V^{\pi_{\theta}^{(i)}}(\mu)}}{\partial \theta_{s,a}} \Bigg|}_{< \infty, \text{ by \cref{eq:GD_convergence_rate_10}}}.
\end{align*}
which leads to the contradiction that $\left \| \theta^{(t)} \right \| \rightarrow \infty$ in \cref{eq:GD_convergence_rate_8}.

Therefore, for any $T$, we have $\max_{t'\le T} \lVert \theta^{(t')}\rVert \in O(\ln T)$. Hence, follow from \cref{eq:GD_convergence_rate_5_1}, 
\begin{align*}
    V^{\pi_{\theta}^{**}}(\mu)-V^{\pi_{\theta}^{(T)}}(\mu) 
    &\le \frac{2 L C^2 \left( \left \| \theta^{**} \right \| + \max_{t' \le T} \left \| \theta^{(t')} \right \| \right)^2}{T-1} \\
    &\le \frac{2 L C^2 \left( \left \| \theta^{**} \right \| + \kappa \ln T \right)^2}{T-1},
\end{align*}
where $\kappa \in \mathbb{R}$ is a finite constant.

\end{proof}
\begin{remark}
    \normalfont{It is crucial to emphasize that when using PG with softmax parameterization, the parameters cannot exhibit excessive growth.
    More specifically, we can demonstrate that the growth rate of the parameters remains bounded by $O(\ln t)$. This remarkable bounding property allows us to relax the objective structure to a $C$-nearly concave form while still achieving a convergence rate of $\tilde{O}(1/t)$.
    However, it is worth noting that due to the unbounded nature of softmax parameterization, it must also compensate for a logarithmic factor in its convergence rate, which may not be tightly aligned with results in the optimization regime.}
\end{remark}

\subsection{Convergence Rate of Policy Gradient}

\begin{theorem}[\textbf{Asymptotic Convergence for Softmax Parameterization in \citep{agarwal2021theory}}]
\label{theorem:PG_global_conv}
Assume we follow the gradient ascent update rule as specified in \cref{eq:PG} and that $\mu(s) > 0$ for all states $s$. Suppose $\eta^{(t)} \le \frac{(1-\gamma)^3}{8}$, then we have that for all states $s$, $V^{(t)}(s) \rightarrow V^*(s)$ as $t \rightarrow \infty$.
\end{theorem}

\begin{lemma}
\label{lemma: PG enters local concave}
    Consider a tabular softmax parameterized policy $\pi_{\theta}$. Under \hyperref[algorithm:PG]{PG} with $\eta^{(t)} = \frac{(1 - \gamma)^3}{8}$ and \Cref{assump:unique_optimal}, given any $M > 0$, there exists a finite time $T$ such that for all $t \ge T, s\in\cS$, and $a \neq a^*(s)$, we have (i) $\theta_{s, a^*(s)} - \theta_{s, a} > M$; (ii) $V^{\pi_{\theta}^{(t)}}(s) > Q^*(s, a_2(s))$; (iii) $\left.\frac{\partial V^{\pi_{\theta}}(\mu)}{\partial\theta_{s, a^*(s)}}\right\rvert_{\theta = \theta^{(t)}} > 0 > \left.\frac{\partial V^{\pi_{\theta}}(\mu)}{\partial\theta_{s, a}}\right\rvert_{\theta = \theta^{(t)}}$.
\end{lemma}

\begin{proof}[Proof of \Cref{lemma: PG enters local concave}]
    By \Cref{assump:unique_optimal}, there is a unique optimal action, so \Cref{theorem:PG_global_conv} implies that the policy weight probability of the optimal action approaches 1. For $\theta_{s, a^*(s)} - \theta_{s, a} > M$, since we also have the asymptotic convergence property, we can see that the same strategy works here as the proof shown in \Cref{lemma: enter local concavity}.

    Moreover, by \Cref{lemma: Q^* a_2}, we can show that (ii) can imply the third point. Regarding the second point, it is the direct result of \Cref{theorem:PG_global_conv} because $V^{\pi_{\theta}^{(t)}}$ can always enter a small enough neighborhood such that the values always are greater than $Q^*(s, a_2(s))$. Hence, by \Cref{lemma:softmax_pg,lemma: Q^* a_2}, we obtain the results.
\end{proof}

\begin{theorem}[\textbf{Convergence Rate of PG}]
\label{theorem: PG convergence rate}
    Consider a tabular softmax parameterized policy $\pi_{\theta}$. Under \hyperref[algorithm:PG]{PG} with $\eta^{(t)} = \frac{(1 - \gamma)^3}{8}$, there exists a finite time $T$ such that for all $t \ge T$, we have
    \begin{align*}
        V^{*}(\rho) - V^{{\pi_\theta^{(t)}}}(\rho) \le \frac{1}{1 - \gamma} \left \| \frac{{d^{\pi^*}_{\rho}}}{\mu} \right \|_{\infty} \Bigg( &\frac{16 C^2 (\kappa + 2)^2  [\ln (t-T)]^2}{(1-\gamma)^3(t - T - 1)} + \frac{2 |\cS|}{(1 - \gamma)^2}\left( \frac{|\cA| - 1}{(t-T)^2 + |\cA| - 1} \right)\Bigg),
    \end{align*}
    where $\kappa \in \mathbb{R}$ is a finite constant.
\end{theorem}

\begin{proof}[Proof of \Cref{theorem: PG convergence rate}]
    Let
    \[
        M' \coloneqq \ln \Big[ \frac{2|\cS||\cA|^2}{(1-\gamma)^2 \min_{s \in S} \mu(s)} \Big].
    \]
    By \Cref{lemma: PG enters local concave}, there exists a finite time $T'$ such that for all $t \ge T'$, $s\in\cS$, and $a\neq a^*(s)$, we have (i) $\theta_{s, a^*(s)} - \theta_{s, a} > M'$, (ii) $V^{\pi_{\theta}^{(t)}}(s) > Q^*(s, a_2(s))$, (iii) $\left.\frac{\partial V^{\pi_{\theta}}(\mu)}{\partial\theta_{s, a^*(s)}}\right\rvert_{\theta = \theta^{(t)}} > 0 > \left.\frac{\partial V^{\pi_{\theta}}(\mu)}{\partial\theta_{s, a}}\right\rvert_{\theta = \theta^{(t)}}$. By (iii), we know that the PG updates \cref{eq:PG} at time $t \ge T'$ lie in the feasible update domain $\cU$. Furthermore, by the first point in \Cref{lemma: di_lower_bdd} and (iii), we obtain that all the PG updates after time $T'$ under state rescaling, i.e., $\lVert \bm{d_{s, \cdot}} \rVert_2^2 = 1$, we have $\min_{i \ge 2} \{ d_{s, a^*(s)} -  d_{s, a_i(s)} \} > |\cA|^{-\frac{3}{2}}$. Thus, we further define
    \[
        M \coloneqq \max\left\{2\ln\Big[2(|\mathcal{A}| - 1)\Big] - 2\ln\Big[|\cA|^{-\frac{3}{2}}\Big], M' \right\},
    \]
    by \Cref{lemma: enter local concavity} again, we know that there is a finite time $T \ge T'$ such that for all $t \ge T'$, $s\in\cS$, and $a\neq a^*(s)$, we have (i) $\theta_{s, a^*(s)} - \theta_{s, a} > M'$, (ii) $V^{\pi_{\theta}^{(t)}}(s) > Q^*(s, a_2(s))$, (iii) $\left.\frac{\partial V^{\pi_{\theta}}(\mu)}{\partial\theta_{s, a^*(s)}}\right\rvert_{\theta = \theta^{(t)}} > 0 > \left.\frac{\partial V^{\pi_{\theta}}(\mu)}{\partial\theta_{s, a}}\right\rvert_{\theta = \theta^{(t)}}$, (iv) the state rescaled updates satisfy $\min_{i \ge 2} \{ d_{s, a^*(s)} -  d_{s, a_i(s)} \} > |\cA|^{-\frac{3}{2}}$. Therefore, by \Cref{lemma:local nearly concavity}, we know that the objective function $\theta \to V^{\pi_{\theta}}(\mu)$ is $\frac{3}{2}$-nearly concave along the PG update directions for all $t \ge T$.

    By \Cref{theorem:GD_convergence_rate},
    \begin{align*}
        V^{\pi^{(t)}_{\theta^{**}}}(\mu) - V^{\pi_{\theta}^{(t)}}(\mu) &\le \frac{2 L C^2 \left( \left \| \theta^{**(t)} \right \| + \kappa \ln (t-T) \right)^2}{(t - T) - 1} \\
        &\le \frac{16 C^2 \left( 2 \ln(t-T) + \kappa \ln (t-T) \right)^2}{(1-\gamma)^3(t - T - 1)} \\
        &= \frac{16 C^2 (\kappa + 2)^2  [\ln (t-T)]^2}{(1-\gamma)^3(t - T - 1)},
    \end{align*}
    where $\theta^{(t)**} \coloneqq [2\ln(t-T), 0, 0, \cdots, 0]$ is a chosen surrogate optimal solution at time $t$.
    Additionally, we have the sub-optimality gap between the original optimal solution and the surrogate optimal solution can be bounded as
    \begin{align}
        V^*(\mu) - V^{\pi^{(t)}_{\theta^{**}}}(\mu) &= \frac{1}{1 - \gamma} \sum_{s} d^{\pi^*}_{\mu}(s) \sum_{a} (\pi^*(a | s) - \pi^{(t)}_{\theta^{**}}(a | s)) \cdot A^{\pi^{(t)}_{\theta^{**}}}(s ,a) \nonumber \\
        &\le \frac{2 |\cS|}{(1 - \gamma)^2} \left( 1 - \frac{\exp(2\ln(t-T))}{\exp(2\ln(t-T)) + \exp(0) \cdot (|\cA| - 1)} \right) \label{eq: theorem convergence rate 1 PG}\\
        &= \frac{2 |\cS|}{(1 - \gamma)^2}\left( \frac{|\cA| - 1}{(t-T)^2 + |\cA| - 1} \right),\nonumber 
    \end{align}
    where \cref{eq: theorem convergence rate 1 PG} consider the upper bound of $A^{\pi^{(t)}_{\theta^{**}}}(s ,a) \le \frac{1}{1 - \gamma}$ and the sum of differences between sub-optimal actions can be at most the difference between the optimal actions, which causes a preconstant 2. By \Cref{lemma:performance_difference_in_rho}, we can change the convergence rate from $V^{\pi}(\mu)$ to $V^{\pi}(\rho)$,
    \begin{align*}
        V^{*}(\rho) - V^{{\pi_\theta^{(t)}}}(\rho) \le \frac{1}{1 - \gamma} \left \| \frac{{d^{\pi^*}_{\rho}}}{\mu} \right \|_{\infty} \Bigg( &\frac{16 C^2 (\kappa + 2)^2  [\ln (t-T)]^2}{(1-\gamma)^3(t - T - 1)} + \frac{2 |\cS|}{(1 - \gamma)^2}\left( \frac{|\cA| - 1}{(t-T)^2 + |\cA| - 1} \right)\Bigg),
    \end{align*}
    and we complete the proof.
\end{proof}

\endgroup
\newpage
\section{Challenges of APG Without Momentum Restart Mechanisms}
\label{app:non monotone APG}

\begingroup
\allowdisplaybreaks

In this section, we provide insights into the challenges encountered when applying APG without momentum restart mechanisms. Furthermore, we've established the convergence rate of $\Tilde{O}(1/t^2)$ under APG without restart mechanisms in the bandit setting.
More specifically, In \Cref{app:subsec:napg_algo}, we introduce the algorithm, NAPG, which is the original APG without restart mechanisms. Subsequently, in \Cref{app:subsec:napg_exp}, we provide a numerical validation in the bandit setting to illustrate the non-monotonic improvement of NAPG. Additionally, in \Cref{app:subsec:napg_hard}, we illustrate the challenges of establishing the limiting value functions in the MDPs setting.

\subsection{APG Without Restart Mechanisms}
\label{app:subsec:napg_algo}
For ease of exposition, this subsection presents the algorithm for Nesterov Accelerated Policy Gradient (NAPG) without restart mechanisms as in \Cref{algorithm:nonrestart_APG}. As depicted in \Cref{algorithm:nonrestart_APG}. It is worth noting that without these restart mechanisms, \Cref{algorithm:nonrestart_APG} aligns precisely with the Nesterov acceleration method outlined in \citep{su2014differential}.

\begin{algorithm} [!ht]
	\caption{Nesterov Accelerated Policy Gradient (NAPG) \textit{Without Restart Mechanisms}}
	\label{algorithm:nonrestart_APG}
	\begin{algorithmic}

\STATE \textbf{Input}:
Step size $\eta^{(t)} > 0$.
\STATE \textbf{Initialize}:
$\theta^{(0)} \in \mathbb{R}^{|\mathcal{S}||\mathcal{A}|}$, $\omega^{(0)} = \theta^{(0)}$.

\FOR{$t = 1$ to $T$} 
    \STATE
    \begin{align*} 
    \theta^{(t)} &\leftarrow \omega^{(t-1)} + \eta^{(t)} \nabla_{\theta}{V^{\pi_{\theta}}(\mu)} \Big\rvert_{\theta = \omega^{(t-1)}} \\
    \omega^{(t)} &\leftarrow \theta^{(t)} + \frac{t-1}{t+2}(\theta^{(t)}-\theta^{(t-1)}) 
    \end{align*}
\ENDFOR
	\end{algorithmic}
\end{algorithm}

\subsection{Numerical Validation of the Non-Monotonic Improvement Under NAPG}
\label{app:subsec:napg_exp}

\begin{wrapfigure}{r}{0.45\textwidth}
\centering
  \includegraphics[width=0.4\textwidth]{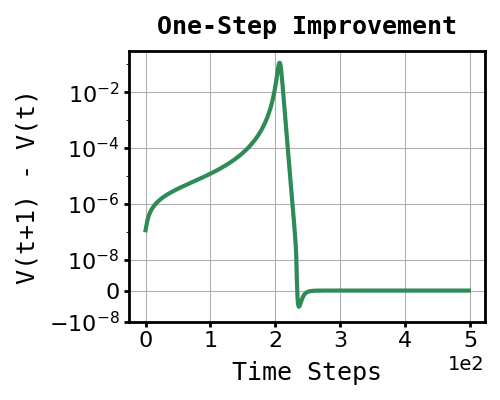}
  \hspace{-5mm}
  \caption{The one-step improvement of APG on a three-action bandit problem.}
  \label{exp:NonMonotone_}
\end{wrapfigure}

In this subsection, we illustrate the difficulties involved in analyzing the convergence of NAPG, compared to the standard policy gradient methods through a numerical experiment.
In contrast to the standard policy gradient (PG) method, which exhibits monotonic improvement, NAPG could experience non-monotonic progress as a result of the momentum term, which could lead to negative performance changes.
To further demonstrate this phenomenon, we conduct a 3-action bandit experiment with a highly sub-optimal initialization, where the weight of the optimal action of the initial policy is extremely small.
More specifically, we conduct a 3-action bandit experiment with actions $\cA = [a^*, a_2, a_3]$, where the corresponding rewards are $r = [r(a^*), r(a_2), r(a_3)] = [1, 0.8, 0]$. We initialize the policy parameters as $\theta^{(0)} = [0, 3, 10]$, which represents a highly sub-optimal initialization. Remarkably, the weight of the optimal action in the initial policy $\pi^{(0)} \approx [0.00005, 0.00091, 0.99904]$ is exceedingly small. As shown in \Cref{exp:NonMonotone_}, the one-step improvement becomes negative around epoch $180$ and provides nearly zero improvement after that point.
Notably, the asymptotic convergence of the standard PG is largely built on the monotonic improvement property, as shown in \citep{agarwal2021theory}. With that said, the absence of monotonic improvement in NAPG poses a fundamental challenge when analyzing the convergence to an optimal policy.


\subsection{Challenges of Establishing the Limiting Value Functions in the MDPs Setting Under NAPG}
\label{app:subsec:napg_hard}
In this subsection, we illustrate the challenges involved in establishing the limiting value functions under NAPG in the MDPs setting.
Recall from \Cref{app:subsec:napg_exp} that NAPG is not guaranteed to achieve monotonic improvement in each iteration due to the momentum. This is one salient difference from the standard PG, which inherently enjoys strict improvement and hence the existence of the limiting value functions (i.e., $\lim_{t\rightarrow \infty}V^{\pi_{\theta^{(t)}}}(s)$) by Monotone Convergence Theorem \citep{agarwal2021theory}. Without monotonicity, it remains unknown if the limiting value functions even exist.

Furthermore, while the gradient step ensures a guaranteed monotonic improvement, which can help counteract the non-monotonicity resulting from the momentum step, it is still a challenging task to quantify this non-monotonic behavior due to the cumulative effect of the momentum. To be more specific, we demonstrate that the influence exerted by the momentum term can lead to a substantial and challenging-to-control effect.

Let $\delta^{(T_0)}:=\theta_{s,a}^{(T_0)}-\theta_{s,a}^{(T_0-1)}$, by the update rule of NAPG, we have
\begin{align}
    \theta^{(T_0+1)}_{s,a} &= \theta^{(T_0)}_{s,a} + \frac{T_0 - 1}{T_0 + 2} \delta^{(T_0)} + \eta^{(T_0+1)}\cdot \frac{\partial V^{\pi_\theta}(\mu)}{\partial \theta_{s,a}}\Big\rvert_{\theta=\omega^{(T_0)}} \label{eq:napg_hard_1} \\
    \theta^{(T_0+2)}_{s,a} &= \theta^{(T_0+1)}_{s,a} + \frac{T_0}{T_0 + 3} \delta^{(T_0 + 1)} + \eta^{(T_0+2)}\cdot \frac{\partial V^{\pi_\theta}(\mu)}{\partial \theta_{s,a}}\Big\rvert_{\theta=\omega^{(T_0 + 1)}} \\
    &= \theta^{(T_0)}_{s,a} + \frac{T_0 - 1}{T_0 + 2} \delta^{(T_0)} + \eta^{(T_0+1)}\cdot \frac{\partial V^{\pi_\theta}(\mu)}{\partial \theta_{s,a}}\Big\rvert_{\theta=\omega^{(T_0)}} \nonumber \\
    &\quad\quad\quad\enspace + \frac{T_0}{T_0 + 3} \Big( \frac{T_0 - 1}{T_0 + 2} \delta^{(T_0)} + \eta^{(T_0+1)}\cdot \frac{\partial V^{\pi_\theta}(\mu)}{\partial \theta_{s,a}}\Big\rvert_{\theta=\omega^{(T_0)}} \Big) \nonumber \\
    &\quad\quad\quad\enspace + \eta^{(T_0+2)}\cdot \frac{\partial V^{\pi_\theta}(\mu)}{\partial \theta_{s,a}}\Big\rvert_{\theta=\omega^{(T_0 + 1)}} \\
    &=  \theta^{(T_0)}_{s,a} + \Big( \frac{T_0 - 1}{T_0 + 2} + \frac{T_0}{T_0 + 3}\frac{T_0 - 1}{T_0 + 2} \Big)\cdot\delta^{(T_0)} \nonumber \\
    &\quad\quad\quad\enspace + \Big( 1 + \frac{T_0}{T_0 + 3} \Big) \cdot  \eta^{(T_0+1)}\cdot \frac{\partial V^{\pi_\theta}(\mu)}{\partial \theta_{s,a}}\Big\rvert_{\theta=\omega^{(T_0)}} \nonumber \\
    &\quad\quad\quad\enspace + \eta^{(T_0+2)}\cdot \frac{\partial V^{\pi_\theta}(\mu)}{\partial \theta_{s,a}}\Big\rvert_{\theta=\omega^{(T_0 + 1)}} \label{eq:napg_hard_2} \\
    \theta^{(T_0+M)}_{s,a}
    &= \theta^{(T_0)}_{s,a} + \Big( \frac{T_0 - 1}{T_0 + 2} + \frac{T_0}{T_0 + 3}\frac{T_0 - 1}{T_0 + 2} + \sum_{\tau=3}^{M} \frac{T_0+1}{T_0+\tau+1}\frac{T_0}{T_0+\tau}\frac{T_0-1}{T_0+\tau-1} \Big)\cdot\delta^{(T_0)} \nonumber \\
    &\quad\quad\quad\enspace + \cdots \label{eq:napg_hard_3} 
\end{align}
where \cref{eq:napg_hard_3} holds by expanding \cref{eq:napg_hard_1}-\cref{eq:napg_hard_2} iteratively.

Note that for large enough $M\in \bbN$,
\begin{align}
    &\sum_{\tau=3}^{M}\frac{(T_0+1)T_0(T_0-1)}{(T_0+\tau+2)(T_0+\tau+1)(T_0+\tau)} \nonumber \\
    &=(T_0+1)T_0(T_0-1)\sum_{\tau=3}^{M}\frac{1}{2}\Big(\frac{1}{(T_0+\tau)(T_0+\tau+1)}-\frac{1}{(T_0+\tau+1)(T_0+\tau+2)} \Big) \nonumber \\
    &=(T_0+1)T_0(T_0-1)\cdot \frac{1}{2}\Big(\frac{1}{(T_0+3)(T_0+4)}-\frac{1}{(T_0+M+1)(T_0+M+2)}\Big) \nonumber \\
    &=\Theta(T_0). \label{eq:napg_hard_4}
\end{align}

Follow from \cref{eq:napg_hard_3} and \cref{eq:napg_hard_4}, we could deduce that the momentum term $\delta^{(T_0)}$ in time $T_0$ results in an effect that is directly proportional to $T_0$. Therefore, it becomes crucial to establish the upper bound of the gradient norm to confine the adverse effects stemming from the momentum term. However, this endeavor involves addressing the intertwined challenge of estimating both the gradient norm and the momentum norm simultaneously.

\endgroup
\newpage
\section{Additional Experiments}
\label{app:add-exp}

\begingroup
\allowdisplaybreaks

\subsection{Stochastic APG (SAPG)}
In this subsection, we conduct an empirical evaluation of the performance of Stochastic Accelerated Policy Gradient (SAPG) on an MDP with 5 states and 5 actions, utilizing the true gradient. In the subsequent experimental results, we set the batch size to $B = 1$ and perform the experiments with 50 different seeds. It's important to highlight that the MDP used is identical to the one discussed in \Cref{sec:disc:exp}, and detailed configuration information is provided in \Cref{app:experiment}.

\subsubsection{Implementation Detail}
For ease of exposition, we provide the pseudo code for Stochastic Accelerated Policy Gradient (SAPG).

\begin{algorithm} [!ht]
\setstretch{1.35}
	\caption{Stochastic Accelerated Policy Gradient (SAPG)}
	\label{algorithm:SAPG}
	\begin{algorithmic}

\STATE \textbf{Input}:
Step size $\eta^{(t)} > 0$, batch size $B \in \mathbb{N}$.
\STATE \textbf{Initialize}:
$\theta^{(0)} \in \mathbb{R}^{|\mathcal{S}||\mathcal{A}|}$, $\omega^{(0)} = \theta^{(0)}$.

\FOR{$t = 1$ to $T$} 
    \STATE
    Sample a batch of $\cS^{(t)} \times \cA^{(t)} \coloneqq \{(s^{(t)}_{1}, a^{(t)}_{1}), (s^{(t)}_{2}, a^{(t)}_{2}), \dots, (s^{(t)}_{B}, a^{(t)}_{B})\}$ where state $s^{(t)}$ is sampled with probability $d^{\pi^{(t)}}_{\mu}(\cdot)$ and action $a^{(t)}$ is sampled with probability $\pi^{(t)}(\cdot|s^{(t)})$. \\
    Calculate stochastic gradient $v^{(t)} = \sum_{s, a \in \cS^{(t)} \times \cA^{(t)}} \frac{\partial V^{\pi_{\theta}}(\mu)}{\partial \theta_{s, a}}\Big\rvert_{\theta = \omega^{(t-1)}}$.
    \begin{align*} 
    \theta^{(t)} &\leftarrow \omega^{(t-1)} + \eta^{(t)} v^{(t)} \\
    \varphi^{(t)} &\leftarrow \theta^{(t)} + \frac{t-1}{t+2}(\theta^{(t)}-\theta^{(t-1)}) \\
    \omega^{(t)} &\leftarrow
    \begin{cases}
        \varphi^{(t)}, & \text{if } V^{\pi_{\varphi}^{(t)}}(\mu) \ge V^{\pi_{\theta}^{(t)}}(\mu), \\
        \theta^{(t)}, & \text{otherwise.}
    \end{cases}
    \end{align*}
\ENDFOR
	\end{algorithmic}
\end{algorithm}

\begin{figure*}[!th]
    \centering
    \hspace{-5mm}
    \subfigure[]{
    \label{exp:SAPG-uniform-value-function}
    \includegraphics[width=0.3\textwidth]{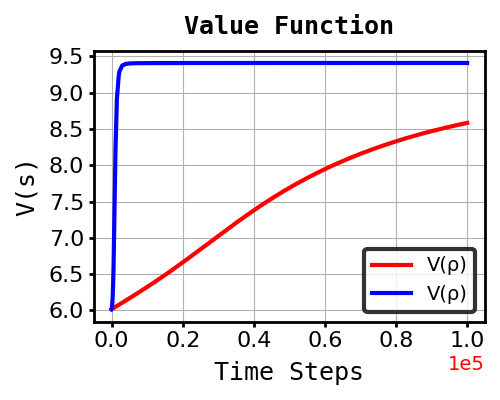}
    }
    \hspace{-0mm}
    \subfigure[]{
    \label{exp:SAPG-hard-value-function}
    \includegraphics[width=0.3\textwidth]{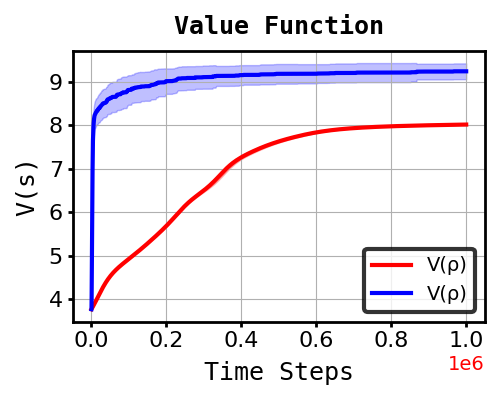}
    }
    \hspace{-0mm}
    \subfigure[]{
    \label{exp:SAPG-uniform-loglog}
    \includegraphics[width=0.3\textwidth]{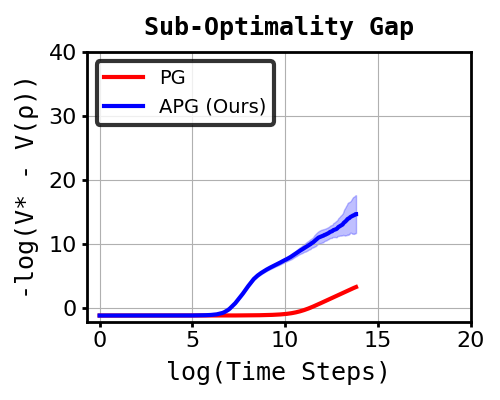}
    }
    \caption{A comparison between the performance of SAPG and SPG under an MDP with 5 states, 5 actions, with the uniform and hard policy initialization: (a)-(b) show the value function under the uniform and the hard initialization, respectively. The optimal objective value $V^{*}(\rho) \approx 9.41$; (c) show the sub-optimality gaps under uniform initialization.}
    \label{exp:SAPG-value}
\end{figure*}

\subsubsection{SAPG under an MDP with 5 states and 5 actions}

In \Cref{exp:SAPG-value}, it is evident that SAPG outperforms Stochastic PG (SPG) with a remarkable speed of convergence. Furthermore, the results presented in \Cref{exp:SAPG-uniform-loglog} suggest that SAPG may exhibit a convergence rate of $O(\frac{1}{t^{\alpha}})$, where $2 > \alpha > 1$. This observation opens up the possibility of extending and exploring this potential in future research. Additionally, it's noteworthy that under hard policy initialization (where the optimal action has the smallest initial probability), SPG tends to get stuck at a local optimum for an extended period, whereas SAPG does not exhibit this issue.



\subsection{Empirical comparison between APG and NPG}

To further highlight the empirical effectiveness of APG compared to NPG, we provide an additional experimental comparison in both BipedalWalker and the Atari games. It is worth noting that our implementation on Atari 2600 games for APG, NPG, and PG are based on the Stable Baselines3 and RL Baselines3 Zoo \citep{rl-zoo3,
stable-baselines3}. The hyperparameters, exactly set to their default values, are detailed in \Cref{app:experiment}.
Additionally, for BipedalWalker, we utilize the codebase of Spinning Up \citep{SpinningUp2018}.

In \Cref{exp:NPG}, we observe that in BipedalWalker and Carnival, although NPG improves fast during the initial training phase, it tends to get stuck thereafter. Additionally, APG achieves acceleration compared to PG and reaches higher performance than both NPG and PG at the $1e7$ timesteps. Furthermore, in Riverraid, we observe that APG demonstrates notable acceleration compared to NPG.

\begin{figure*}[!th]
    \centering
    \hspace{-5mm}
    \subfigure[]{
    \label{exp:NPG-BipedalWalker}
    \includegraphics[width=0.33\textwidth]{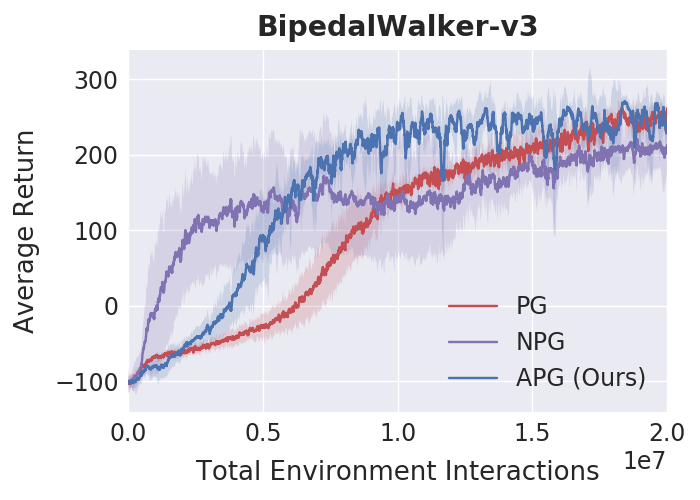}
    }
    \hspace{-0mm}
    \subfigure[]{
    \label{exp:NPG-Carnival}
    \includegraphics[width=0.3\textwidth]{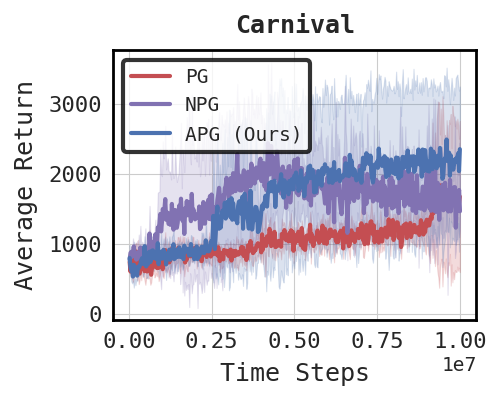}
    }
    \hspace{-0mm}
    \subfigure[]{
    \label{exp:NPG-Riverraid}
    \includegraphics[width=0.3\textwidth]{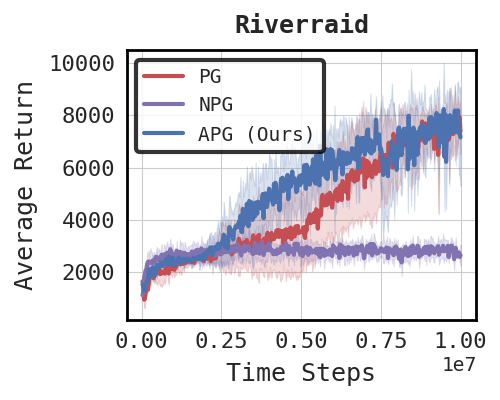}
    }
    \caption{A comparison of the performance of APG and the benchmark algorithms in BipedalWalker and two Atari 2600 games. All the results are averaged over 5 random seeds (with the shaded area showing the range of $\text{mean} \pm \text{std}$).}
    \label{exp:NPG}
\end{figure*}

\endgroup
\color{black}
\newpage
\section{A Detailed Comparison of the Acceleration Methods for RL}
\label{app:add-related}
\begingroup
\allowdisplaybreaks

{In the optimization literature, there are several major categories of acceleration methods, including momentum, regularization, natural gradient, and normalization. Interestingly, these approaches all play an important role in the context of RL.}

\subsection{{Momentum}}

{\textbf{Empirical Aspect}: In practice, momentum is one of the most commonly used approaches for acceleration in policy optimization for RL, mainly due to its simplicity. Specifically, a variety of classic momentum methods (e.g., Nesterov momentum, heavy-ball method, and AdaGrad) have been already implemented in optimization solvers and deep learning frameworks (e.g., PyTorch) as a basic machinery and shown to improve the empirical convergence results than the standard policy gradient methods like A2C and PPO (e.g., see \citep{henderson2018did}).}

{\textbf{Theoretical Aspect}: However, despite the popularity of these approaches in empirical RL, it has remained largely unknown whether momentum could indeed improve the theoretical convergence rates of RL algorithms. To address this fundamental question, we establish the very first convergence rate of Nesterov momentum in RL (termed APG in our paper) and show that APG could improve the rate of PG from $O(1/t)$ to $\tilde{O}(1/t^2)$ (in the constant step-size regime). Given the wide application of momentum, our analytical framework and the insights from our convergence result (e.g., local near-concavity) serve as an important first step towards better understanding the momentum approach in RL.}

{\textbf{Theoretical Advantage}: In addition, we would like to highlight that in the constant step-size regime (which is one of the most widely adopted setting in practice), our proposed APG indeed achieves a superior convergence rate of $\tilde{O}(1/t^2)$, which is better than the $O(1/t)$ rate of vanilla PG and the $O(1/t)$ rate of NPG with constant step sizes.}

\subsection{{Regularization}}

{\textbf{Theoretical Aspect}: In the convex optimization literature, it is known that adding a strongly convex regularizer achieves acceleration by changing the optimization landscape. In the context of RL, despite the non-concave objective, regularization has also been shown to achieve faster convergence, such as the log-barrier regularization with $O(1/\sqrt{t})$ rate \citep{agarwal2021theory} and the entropy regularization (i.e., adding a policy entropy bonus to the reward) with linear convergence rate $O(e^{-ct})$ \citep{mei2020global}, both in the exact gradient setting with constant step sizes.
Additionally, \citet{lan2023policy} presents linear convergence results when employing strongly convex regularizers.}

{Despite the above convergence results, one common attack on regularization approaches is that they essentially change the objective function to that of regularized MDPs and hence do not directly address the original objective in unregularized RL, as pointed out by \citep{mei2021leveraging, xiao2022convergence}.}

\subsection{{Natural Gradients (A Special Case of Policy Mirror Descent)}}

{Natural policy gradient (NPG), another RL acceleration technique, borrows the idea from the natural gradient, which uses the inverse of Fisher information matrix as the preconditioner of the gradient and can be viewed as achieving approximate steepest descent in the distribution space \citep{amari1998natural, kakade2001natural}. Moreover, under direct policy parameterization, NPG is also known as a special case of policy mirror descent (PMD) with KL divergence as the proximal term \citep{shani2020adaptive, xiao2022convergence}. Regarding the convergence rates, NPG has been shown to achieve:
(i) $O(1/t)$ rate with constant step sizes \citep{agarwal2021theory, xiao2022convergence};
(ii) linear convergence either with adaptively increasing step sizes \citep{khodadadian2021linear} or non-adaptive exponentially growing step sizes \citep{xiao2022convergence}.
Notably, under direct policy parameterization and geometrically increasing step sizes, similar linear convergence results can also be established for the more general PMD method \citep{xiao2022convergence}.}

{\textbf{Empirical Issues}: 
Despite the above convergence rates, as also pointed out by \citep{mei2021leveraging}, it is known that NPG in general requires solving a costly optimization problem in each iteration (cf. \citep{kakade2001natural, agarwal2021theory}) even with the help of compatible approximation theorem \citep{kakade2001natural} (unless the policy parametrization adopts some special forms, e.g., log-linear policies \citep{yuan2022linear}). Notably, the TRPO algorithm, a variant of NPG, is known to suffer from this additional computational overhead. Similarly, the PMD in general also requires solving a constrained optimization problem in each policy update \citep{xiao2022convergence}. This computational complexity could be a critical factor for RL in practice.}

\subsection{{Normalization}}

{\textbf{Theoretical Aspect}: In the exact gradient setting, the normalization approach, which exploits the non-uniform smoothness by normalizing the true policy gradient by its L2 gradient norm, has also been shown to achieve acceleration in RL. Specifically, the geometry-aware normalized PG (GNPG) has been shown to exhibit linear convergence rate under softmax policies \citep{mei2021leveraging}. This rate could be largely attributed to the effective increasing step size induced by the normalization technique. }

{\textbf{Empirical Issues}: Despite the linear convergence property, it has been already shown that GNPG could suffer from divergence in the standard on-policy stochastic setting, even in simple 1-state MDPs. This phenomenon is mainly due to the committal behavior induced by the increasing effective step size. This behavior could substantially hinder the use of GNPG in practice.}

\endgroup
\end{document}